\documentclass[12pt]{memoir}  % TODO 12pt, 

\setstocksize{11in}{8.5in}
\settrimmedsize{11in}{8.5in}{*}
\settrims{0in}{0in}
\setlrmarginsandblock{1.3in}{1.3in}{*}
\setulmarginsandblock{1.3in}{1.3in}{*}
\setheadfoot{13pt}{26pt}
\setheaderspaces{*}{13pt}{*}
\checkandfixthelayout
\OnehalfSpacing
\setsecnumdepth{subsubsection}
\headstyles{default}
\usepackage[colorlinks,bookmarksnumbered,bookmarksdepth=subsubsection,unicode=true]{hyperref}
\hypersetup{
pdfauthor = {Pascale Gourdeau},
pdftitle = {Sample Complexity of Robust Learning against Evasion Attacks},
pdfsubject = {Subject},
pdfkeywords = {adversarial machine learning, learning theory, robustness, classification},
pdfcreator = {LaTeX with the hyperref package},
pdfproducer = {},
linkcolor = [HTML]{000000},
citecolor = [HTML]{0000FF},
}
\usepackage[utf8]{inputenc}
\usepackage{csquotes}
\usepackage{showidx}

\usepackage{natbib}
\makeindex

\usepackage{graphicx}
\usepackage{xcolor}
\usepackage{amsmath,amssymb,amsfonts,amsthm}
\usepackage{enumerate}
\usepackage{caption}
\usepackage{subcaption}
\usepackage{mathtools}
\usepackage{pifont}

\usepackage{algorithm,algpseudocode}
\usepackage{tikz,pgf}
\usetikzlibrary{fit,positioning,decorations.pathmorphing,shapes}
\definecolor{myGreen}{rgb}{0.2,0.8,0.3}
\usepackage{graphicx,transparent,eso-pic}
\usetikzlibrary{arrows,automata}
\usetikzlibrary{positioning, fadings, backgrounds,calc}
\tikzstyle{arrow} = [thick,->,>=stealth]

\newif\ifcomments
\newcommand{\todo}[1]{\ifcomments \red{\textbf{TODO:}} \emph{#1} \fi}

\newcommand{\red}[1]{\textcolor{red}{#1}}
\newcommand{\blue}[1]{\textcolor{blue}{#1}}

\newcommand{\green}[1]{\textcolor{myGreen}{#1}}
\newcommand{\R}{\mathbb{R}}
\newcommand{\N}{\mathbb{N}}

\newcommand{\abs}[1]{ \left\vert #1\right\vert }
\newcommand{\norm}[1]{ \left\Vert #1\right\Vert }
\newcommand{\set}[1]{\left\{#1\right\}}
\newcommand{\eval}[2]{\underset{{#1}}{\mathbb{E}}\left[#2\right]}
\newcommand{\Prob}[2]{\underset{{#1}}{\Pr}\left(#2\right)}
\newcommand{\st}{\;.\;}
\newcommand{\err}{\text{err}}
\newcommand{\roblossc}{\mathsf{R}^C_\rho}
\newcommand{\roblosse}{\mathsf{R}^E_\rho}
\newcommand{\robloss}{\mathsf{R}_\rho}
\newcommand{\boolhc}{\set{0,1}^n}

\newcommand{\boolhcfa}{\set{-1,1}^n}
\newcommand{\classes}{\set{0,1}}
\newcommand{\classesfa}{\set{-1,1}}
\newcommand{\Inf}{\mathbf{Inf}}

\newcommand{\poly}{\text{poly}}
\newcommand{\supp}{\text{supp}}
\newcommand{\sgn}{\text{sgn}}
\newcommand{\given}{\;|\;}

\newcommand{\risk}{\mathsf{R}}

\newcommand{\LEQ}{\mathsf{LEQ}}
\newcommand{\EQ}{\mathsf{EQ}}
\newcommand{\EX}{\mathsf{EX}}
\newcommand{\LMQ}{\mathsf{LMQ}}
\newcommand{\MQ}{\mathsf{MQ}}
\newcommand{\adv}{\mathbb{A}}
\newcommand{\PAC}{\mathsf{PAC}}
\newcommand{\VC}{\mathsf{VC}}
\newcommand{\VCrho}{\VC\vert_\rho}
\newcommand{\VCH}{\VC(\mathcal{H})}

\newcommand{\DVCH}{\VC^*(\mathcal{H})}
\newcommand{\RLossLong}{\mathcal{L}_\rho(\C,\H)}
\newcommand{\RLoss}{\mathcal{L}_\rho(\C)}
\newcommand{\Rloss}{\ell_\rho(c,h)}
\newcommand{\RVClong}{\VC(\RLossLong)}
\newcommand{\RVC}{\VC(\RLoss)}
\newcommand{\Lit}{\mathsf{Lit}}
\newcommand{\sat}{\mathsf{SAT}}
\newcommand{\matching}{\mathcal{M}}
\newcommand{\asst}{\mathcal{A}_{\matching}}
\newcommand{\varM}{I_{\matching}}
\newcommand{\satlog}{\sat_{\log(n)}}
\newcommand{\satrho}{\sat_{\rho}}
\newcommand{\satnot}{\sat_{0}}
\newcommand{\prephi}[1]{\Phi^{-1}(#1)}
\newcommand{\A}{\mathcal{A}}
\newcommand{\C}{\mathcal{C}}

\newcommand{\D}{\mathcal{D}}

\newcommand{\E}{\mathcal{E}}
\newcommand{\F}{\mathcal{F}}

\renewcommand{\H}{\mathcal{H}}

\newcommand{\U}{\mathcal{U}}
\newcommand{\X}{\mathcal{X}}
\newcommand{\x}{\mathbf{x}}

\newcommand{\vv}{\mathbf{v}}

\newcommand{\Y}{\mathcal{Y}}
\newcommand{\MonConj}{\textsf{MON-CONJ}}
\newcommand{\Conj}{\textsf{CONJUNCTIONS}}
\newcommand{\parity}{\mathsf{PARITIES}}
\newcommand{\majority}{\mathsf{MAJORITIES}}
\newcommand{\threshold}{\mathsf{THRESHOLDS}}
\newcommand{\thresholds}{\mathsf{THRESHOLDS}}
\newcommand{\dl}{$\mathsf{DL}$}
\newcommand{\dlm}{\mathsf{DL}}
\newcommand{\ltf}{\mathsf{LTF}}
\newcommand{\Halfspaces}{\ltf}
\newcommand{\ltfbool}{\ltf_{\boolhc}}
\newcommand{\maj}{\text{maj}}
\newcommand{\ltfreal}{\ltf_{\mathbb{R}^n}}

\newtheorem{theorem}{Theorem}
\numberwithin{theorem}{chapter}
\newtheorem{proposition}[theorem]{Proposition}
\newtheorem{lemma}[theorem]{Lemma}
\newtheorem{corollary}[theorem]{Corollary}
\newtheorem{example}[theorem]{Example}

\newtheorem{definition}[theorem]{Definition}
\newtheorem{claim}[theorem]{Claim}

\theoremstyle{remark}
\newtheorem{remark}[theorem]{Remark}

\begin{document}
%%%%%%%%%%%%%%%%%%%%%%%%%%%%%%%%%%%%%%%%%%%%%%%%%%%%%
% Title page
%%%%%%%%%%%%%%%%%%%%%%%%%%%%%%%%%%%%%%%%%%%%%%%%%%%%%

% Title (fold)
\pretitle{\begin{center}\cftchapterfont\LARGE}
\posttitle{\end{center}}
\preauthor{\begin{center}\huge}
\postauthor{\end{center}}
\predate{\begin{center}\large}
\postdate{\end{center}}

\title{Sample Complexity of Robust Learning against Evasion Attacks}
\date{ }
\renewcommand\maketitlehookc{
\begin{center}
\vspace{5mm}
\includegraphics[scale=0.5]{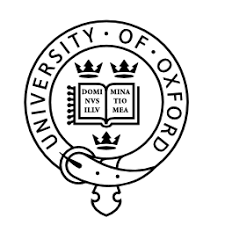}\\
\vspace{15mm}
{
\Large
Pascale Gourdeau\\
\large
Trinity College\\
University of Oxford
}
\end{center}
\vspace{5mm}
}
\renewcommand\maketitlehookd{
\begin{center}
\vspace{5mm}
\large
A thesis submitted for the Degree of \\ \emph{Doctor of Philosophy}\\
%\vspace{5mm}
Trinity 2023
\end{center}
}
% Title (end)

\begin{titlingpage}
\maketitle
\end{titlingpage}

\clearpage
\pagenumbering{roman}

\begin{abstract}

It is becoming increasingly important to understand the vulnerability
of machine learning models to adversarial attacks.  
One of the fundamental problems in adversarial machine learning is to quantify
how much training data is needed in the presence of so-called evasion attacks, where data is corrupted at test time. 
In this thesis, we work with the exact-in-the-ball notion of robustness and study the feasibility of adversarially robust learning from the perspective of learning theory, considering sample complexity.  
%We work within the probabilistically approximately correct (PAC) learning framework of \cite{valiant1984theory}.

We start with two negative results.  
We show that no non-trivial concept class can be
robustly learned in the distribution-free setting against an adversary who can perturb just a single input bit.  
We then exhibit a sample-complexity lower bound: the class of monotone conjunctions %--essentially one of the simplest non-trivial hypothesis classes on the Boolean hypercube--
and any superclass on the boolean hypercube has sample complexity at least exponential in the adversary's budget (that is, the maximum number of bits it can perturb on each input). 
This implies, in particular, that these classes cannot be robustly learned under the uniform distribution against an adversary who can perturb $\omega(\log n)$ bits of the input.

As a first route to obtaining robust learning guarantees, we consider restricting the class of distributions over which training and testing data are drawn.
We focus on learning problems  with probability distributions on the input data that satisfy a
Lipschitz condition: nearby points have similar probability.  
%Our key results, together with the lower bound above, illustrate that the adversary's budget is a fundamental quantity in determining the sample complexity of robust learning.
We  show that, if the adversary is
restricted to perturbing $O(\log n)$ bits, then  one can robustly
learn the class of monotone conjunctions 
with respect to the class of log-Lipschitz
distributions.  
We then extend this result to show the learnability of
   $1$-decision lists, 2-decision lists and monotone $k$-decision lists in the same
  distributional and adversarial setting.
We finish by showing that for every fixed $k$ the class of $k$-decision lists has polynomial sample complexity against a $\log(n)$-bounded adversary. 
The advantage of considering intermediate subclasses of $k$-decision lists is that we are able to obtain improved sample complexity bounds for these cases.
%This sheds
%further light on the question
%of whether
%an efficient PAC learning algorithm can always be used as an efficient
%$\log(n)$-robust learning algorithm under the uniform distribution.

As a second route, we  study learning models where the learner is given more power  through the use of \emph{local} queries.
%We give the first \emph{distribution-free} algorithms that perform robust empirical risk minimization (ERM) for this notion of robustness. 
 %
  The first learning model we consider uses local membership queries (LMQ), where the 
  learner can query the label of points near the training sample.
  We show that, under the uniform distribution, the exponential dependence on the adversary's budget to robustly learn conjunctions and any superclass remains inevitable even when the learner is given access to LMQs in addition to random examples.
  Faced with this negative result, we introduce a local \emph{equivalence} query oracle, which returns whether the hypothesis and target concept agree in a given region around a point in the training sample, as well as a counterexample if it exists.
  We show a separation result: on the one hand, if the query radius $\lambda$ is strictly smaller than the adversary's  perturbation budget $\rho$, then distribution-free robust learning is impossible for a wide variety of concept classes; on the other hand, the setting $\lambda=\rho$ allows us to develop robust empirical risk minimization algorithms in the distribution-free setting.
  We then bound the query complexity of these algorithms based on online learning guarantees and further improve these bounds for the special case of conjunctions. 
  We follow by giving a robust learning algorithm for halfspaces on $\boolhc$. 
  Finally, since the query complexity for halfspaces on $\mathbb{R}^n$ is unbounded, we instead consider adversaries with  \emph{bounded precision} and give  query complexity upper bounds in this setting as well.
%\cite{allen2018learning}
\end{abstract}

%%%%%%%%%%%%%%%%%%%%%%
% ACKNOWLEDGEMENTS
%%%%%%%%%%%%%%%%%%%%%%

\renewcommand*\abstractname{Acknowledgements}
\newpage
%\thispagestyle{empty}
%\begin{center}
%{\bfseries Acknowledgements}
%\end{center}
%\vspace{0.5cm}
\begin{abstract}

I would first like to express my most sincere gratitude to my supervisors James (Ben) Worrell, Varun Kanade, and Marta Kwiatkowska.

Marta, I am extremely grateful for your guidance, support and generosity. 
Your help has been invaluable in setting a research agenda and navigating my DPhil. 
I very much value the time you make for your students, your involvement and reliability.

Varun, thank you for taking a chance working with me in my second year, and for introducing me to learning theory and interesting problems in the field.
Your expertise and knowledge have been beyond helpful. 
I am tremendously grateful for our discussions, and greatly appreciate your insightful comments and approach to research. 
I value your mentorship immensely.

Ben, your enthusiasm for research is inspiring. 
Working with you, I have learned so much on how to approach and solve research problems.
You have helped me look at research as a ludic and collaborative endeavour -- a perspective that I hope will last throughout my career.
I cannot thank you enough for your generosity with your time, energy and ideas.

I would also like to thank my masters supervisors, Prakash Panangaden and Doina Precup, for their help and support which has lasted to this day and has greatly contributed to my academic path.

To my amazing friends, I am forever grateful for your support, care and kindness.
Friends from Montréal, Pearson UWC, Oxford, and beyond: you know who you are, and I love and cherish every one of you.

I would like to thank Gabrielle, Rick and Joanie from the Institut des Commotions Cérébrales, without whom I would most likely never have finished my degree.

I would also like to acknowledge the financial support provided to me during my DPhil: the Clarendon Fund (Oxford University Press) for the Clarendon Scholarship, the Natural Sciences and Engineering Research Council of Canada (NSERC) for the Postgraduate Scholarship, and the European Research Council (ERC) for funding under the European Union's Horizon 2020 research and innovation programme (FUN2MODEL, grant agreement No. 834115).

Finally, I would like to thank my family, especially my parents, Caroline and Richard, who have supported me in ways that words will never do justice to.
A very special thank you to my grandmother, Colette, whose love and wisdom are always with me.
Raymonde and Pierre, I so deeply wish I could share this moment with you.
\end{abstract}
%%%%%%%%%%%%%%%%%%%%%%%%%%%%%%%%%%%%%%%%%%%%%%%%%%%%%
% Table of content
%%%%%%%%%%%%%%%%%%%%%%%%%%%%%%%%%%%%%%%%%%%%%%%%%%%%%
\clearpage
\setcounter{tocdepth}{2}
\tableofcontents
\newpage
\listoffigures            
\newpage  
\listoftables

\clearpage
\pagenumbering{arabic}

\chapter{Introduction}

In the standard theoretical analysis of machine learning, the learning process uses and is evaluated on clean, unperturbed examples. 
Moreover, many machine learning tasks are evaluated according to predictive accuracy alone, e.g., maximizing the accuracy of a classifier with respect to the ground truth which labels the data.
Though there  remain existing knowledge gaps in the literature (e.g., explaining the success of deep neural networks), machine learning theory has generally been  successful at designing algorithms and deriving guarantees to explain generalization in this framework, even in the presence of noise. 

It is  natural to ask whether similar results can be derived when the learning objectives go beyond standard accuracy.
This could be when the learning process allows for the presence of a malicious adversary--which is more powerful than simply adding random noise to the data-- and thus requires \emph{robustness}.
The study of robustness in machine learning falls under the more general umbrella of \emph{trustworthiness} of machine learning models, where other considerations such as privacy, interpretability or fairness come into play, see, e.g., \citep{dwork2008differential,doshi2017towards,kleinberg2017inherent}.
The trustworthiness of machine learning models is  of utmost importance, especially considering the speed at which new technology is currently deployed.
%As such, in practical applications, there is an acute awareness of the need for reliability. %, e.g., interpretability in machine learning for healthcare.
Crucially, learning theory can provide us with valuable tools to explain, evaluate and guarantee the behaviour of safety-critical machine-learning applications.

The focus of this thesis is on the robustness of machine learning algorithms to \emph{evasion attacks}, which happen at test time after a model is trained (without the presence of an adversary).
This is in contrast to \emph{poisoning attacks}, which happen at training time with the goal of reducing the test-time accuracy of a machine learning algorithm.
The distinction between these two settings was proposed by \cite{biggio2013evasion}, who independently observed the phenomenon of adversarial examples presented by \cite{szegedy2013intriguing}, who coined the latter term.

One of the main challenges in the theory of adversarial machine learning is to analyse the intrinsic difficulty of learning in the presence of an adversary that can modify the data.
The present work studies various assumptions in a learning problem, such as properties of the distribution underlying the data, how the learner obtains data, limitations of the adversary, etc., and determines whether robust learning is feasible with a reasonable amount of data. 
Here, reasonable means that the \emph{sample complexity} of a robust learning algorithm, i.e., the amount of data needed to enable guarantees, is \emph{polynomial} in the input space dimension and the learning parameters (e.g., an algorithm's \emph{confidence} and the desired \emph{robust accuracy} of a hypothesis output by the learning algorithm).

\section{Main Contributions}

This thesis focuses on the existence of adversarial examples in classification tasks. 
An adversarial example is obtained from a natural example at test time by adding a perturbation, in the malicious goal of causing a misclassification.
We work under the \emph{exact-in-the-ball} notion of robustness,\footnote{Also known as \emph{error region} risk in \cite{diochnos2018adversarial}.} which relies on the existence of a ground truth function (i.e., there exists a concept that labels the data correctly). 
A misclassification occurs when the hypothesis returned by the learning algorithm and the ground truth \emph{disagree} in the perturbation region. 
This is in contrast to the \emph{constant-in-the-ball} notion of robustness\footnote{Also known as \emph{corrupted input} robustness from the work of \cite{feige2015learning}.} which requires that the unperturbed point be labelled correctly, and that the hypothesis remain \emph{constant} in the perturbation region.
Guarantees derived for the constant-in-the-ball notion of robustness  imply that the hypothesis returned has a certain stability (perhaps at the cost of accuracy in certain cases, as demonstrated in \cite{tsipras2019robustness}), as an optimal algorithm would return a hypothesis that limits the probability of a label change in the perturbation region. On the other hand, guarantees derived for the exact-in-the-ball notion of robustness usually give stronger accuracy, as we want to be \emph{correct} with respect to the ground truth in the perturbation region.
Deciding which notion of robustness to use depends on the learning problem at hand, and what kind of guarantees one wishes to ensure.
We gave in \citep{gourdeau2019hardness,gourdeau2021hardness}  a thorough comparison between these two notions of robustness, and remarked that the exact-in-the-ball notion of robustness is much less studied than the constant-in-the-ball one.

Our motivation in this thesis is to study the intrinsic robustness of learning algorithms from a learning theory perspective in the probably approximately correct (PAC) learning model of \cite{valiant1984theory}. We investigate how different learning settings enable robust learning guarantees, or, to the contrary, give rise to hardness results.
In this sense, our main aim is to delineate the frontier of robust learnability in various learning models. 
We conceptually divide our contributions based on the learning models we have studied.

\paragraph*{Random examples.} In this model, as in the PAC framework, the learner has access to a random-example oracle which samples a point from an underlying distribution, and returns the point along with its label. 
We exhibit an impossibility result \citep{gourdeau2019hardness}, stating that the distribution-free guarantees for (standard) PAC learning cannot be achieved for robust learning under the exact-in-the-ball definition of robustness, highlighting a key obstacle in adversarial machine learning compared to its standard counterpart. 
Here, distribution-free means that the learning guarantees hold for any distribution that generates the data, provided that the training and testing data are both drawn independently from the same distribution.

The above impossibility result is obtained by choosing a badly-behaved, and quite unnatural distribution on the data.
But we show that, even when looking at natural distributions and simple concept classes, robust learning can have high sample complexity.
Indeed, we prove that there is no efficient robust learning algorithm that learns monotone conjunctions under the uniform distribution if the adversary can perturb $\rho=\omega(\log n)$ bits of a test point in $\{0,1\}^n$; the maximum number $\rho$ of bits the adversary is allowed to perturb at test time is called the \emph{perturbation budget}.
This is particularly striking as the class of monotone conjunctions is one the simplest non-trivial concept classes on the boolean hypercube.
We extend this result to establish a general sample complexity lower bound of $\Omega(2^\rho)$ \citep{gourdeau2022sample}, highlighting  an \emph{exponential} dependence on the adversary's budget $\rho$ in the sample complexity of robust learning.
Since linear classifiers and decision lists subsume this class of functions, the lower bound holds for them as well.
To complement these results, we show that, under distributional assumptions and against a \emph{logarithmically-bounded} adversary (i.e., with budget $\rho=O(\log n)$), efficient robust learning is possible for various concept classes. 
We require that the underlying distribution be \emph{log-Lipschitz}; this notion encapsulates the idea that nearby instances should have similar probability masses and includes as particular instances product distributions with bounded means.
We show the above-mentioned result for conjunctions \citep{gourdeau2019hardness}, monotone decision lists \citep{gourdeau2021hardness}, and non-monotone decision lists \citep{gourdeau2022sample}. 
We define the term \emph{robustness threshold} to mean a function $f(n)$ of the input dimension $n$ for which it is possible to efficiently robustly learn against an adversary with budget $f(n)$, but impossible if the adversary's budget is $\omega(f(n))$ (with respect to a given distribution family).
The robustness threshold of these concept classes is thus $\log(n)$ under log-Lipschitz distributions.

In general, the above-mentioned results rely on a proof of independent interest: an upper bound on the $\log(n)$-expansion of subsets of the hypercube defined by $k$-CNF formulas. 
This result relies on concentration bounds for martingales, as well as properties of the resolution proof system.
In all the cases above, as well as for decision trees \citep{gourdeau2021hardness}, the error region between a hypothesis and a target\footnote{That is, for target $c$ and hypothesis $h$ on input space $\X$, the set of points $x\in\X$ such that $c(x)\neq h(x)$.} can be expressed as a union of $k$-CNF formulas. 
By controlling the standard risk, we can bound the robust risk and, as a result, use PAC learning algorithms as black boxes for robust learning.

\paragraph*{Local membership queries.} In this model, introduced by \cite{awasthi2013learning}, the learner has access to the random-example oracle and can query the label of points that are near the randomly-drawn training sample. 
We show that at least $\Omega(2^\rho)$ local membership queries are needed for robustly learning conjunctions under the uniform distribution against an adversary that can perturb $\rho$ bits of the input \citep{gourdeau2022when}. 
We thus have the same exponential dependence in the adversary's budget as with random examples only, implying that adding local membership queries cannot, in general, improve the robustness threshold of this concept class (and any superclass)., e.g. linear classifiers and decision lists.

\paragraph*{Local equivalence queries.} Faced with the lower bound for robust learning with a local membership query oracle, we introduce a learning model where the learner is allowed to query whether the hypothesis is \emph{correct} in a specific region of the space and get a counterexample if not, which we call local equivalence queries in \citep{gourdeau2022when}, following the work of \cite{angluin1987learning}. 

We first establish that, when the query budget is strictly smaller than the perturbation budget (hence the adversary can access regions of the instance space that the learner cannot), distribution-free robust learning with random examples and local equivalence queries is in general impossible for monotone conjunctions and any superclass thereof. 
However, when the query and perturbation budgets coincide, a query to the local equivalence query oracle is equivalent to querying the robust loss and getting a counterexample if it exists.
As a result, the local equivalence query oracle becomes the exact-in-the-ball analogue of the Perfect Attack Oracle of \cite{montasser2021adversarially}.
In this case, efficient distribution-free robust learning becomes possible for a wide variety of concept classes. 
Indeed, we show random-example and local-equivalence-query upper bounds, which we refer  to as sample and query complexity, respectively.
We demonstrate that the query complexity depends on mistake bounds from online learning, and the sample complexity on the VC dimension of the robust loss of a concept class, a  notion of complexity that we have adapted from \cite{cullina2018pac} to the exact-in-the-ball notion of robustness.
We also show that the local equivalence query bound can be improved in the special case of conjunctions.
We moreover establish that the VC dimension of the robust loss between linear classifiers on $\R^n$ is  $O(n^3)$. 

Since the query complexity of linear classifiers is in general unbounded, we study the setting in which we restrict the adversary's \emph{precision} (e.g., the number of bits needed to express an adversarial example).
We use and adapt tools and techniques from \cite{ben2009agnostic}, which pertain to the study of margin-based classifiers in the context of online learning, for our purposes and exhibit finite query complexity bounds.
We then exhibit expected local equivalence query lower bounds that are linear in the \emph{restricted} Littlestone dimension of a concept class (we require that a set of potential counterexamples be in a specific region of the instance space), and show that, for a wide variety of concept classes, they coincide asymptotically with the local equivalence query upper bounds derived in \cite{gourdeau2022when}.
Finally, we offer a more nuanced discussion of the local membership and equivalence query oracles.
In particular, we show that the local equivalence query and its global counterpart, the equivalence query, are in general incomparable.

\section{Thesis Structure}

\subsection*{Chapter~\ref{chap:lit-review}}

This chapter consists of the literature review.
We first review foundational work on classification in the learning theory literature. We then turn our attention to the more recent related work on adversarial robustness in machine learning, particularly in the context of evasion attacks.
We mainly focus on work that is foundational in nature, as it is the lens with which we study adversarial robustness.

\subsection*{Chapter~\ref{chap:background}}

We review necessary technical background to the understanding of the technical contributions of this thesis, which largely focuses on classification in the following models: the PAC framework of \cite{valiant1984theory}, the exact learning framework of \cite{angluin1987learning}, and the online learning setting.
We also review some probability theory and Fourier analysis. 

\subsection*{Chapter~\ref{chap:def-adv-rob}}

We motivate the study of adversarial robustness for classification tasks under the exact-in-the-ball notion of robustness. 
We rigorously discuss the different notions of robust risk and their significance, particularly the impossibility of obtaining distribution-free guarantees in our setting.
We initiate our study of efficient robust learnability (from a sample-complexity point of view) with monotone conjunctions.
We show a sample complexity lower bound that is exponential in the adversary's budget under the uniform distribution, ruling out  the existence of efficient robust learning algorithms against adversaries with a budget super-logarithmic in the input dimension in this setting.
We show, however, that it is possible to robustly learn  monotone conjunctions under log-Lipschitz distributions against a logarithmically-bounded adversary.

The material in this chapter is based on the following papers:

\begin{itemize}
    \item  \textbf{Pascale Gourdeau}, Varun Kanade, Marta Kwiatkowska, and James Worrell, “On the hardness
of robust classification,” in \emph{33rd Conference on Neural Information Processing 
Systems (NeurIPS)}, 2019. 
	\item  \textbf{Pascale Gourdeau}, Varun Kanade, Marta Kwiatkowska, and James Worrell, “Sample complexity bounds for robustly learning decision lists against evasion attacks,” in \emph{International Joint Conference on Artificial Intelligence (IJCAI)}, 2022. 
  \end{itemize}

\subsection*{Chapter~\ref{chap:rob-thresholds}}

In this chapter, we study the \emph{robustness thresholds} of various concept classes under distributional assumptions.
We show the exact learning of parities under log-Lipschitz distributions and of majority functions  under the uniform distribution, giving a robustness threshold of $n$ for these classes.
We then show a robustness threshold of $\log(n)$ for the class of $k$-decision lists, which is parametrized by the size $k$ of a conjunction at each node in the list. 
Since our aim is to bound the sample complexity of robustly learning, we study various restrictions of decision lists: 1-decision lists, 2-decision lists, monotone $k$-decision lists and finally (non-monotone) $k$-decision lists. 
The proofs not only rely on  different technical tools, but they more importantly yield much better sample complexity bounds for the simpler subclasses.
We finish by relating the standard and robust errors of decision trees under log-Lipschitz distributions.

This chapter is based on the following two papers, the first one being the journal version of the NeurIPS 2019 paper presented in the previous chapter:

\begin{itemize}
	\item  \textbf{Pascale Gourdeau}, Varun Kanade, Marta Kwiatkowska, and James Worrell, “On the hardness
of robust classification,” in \emph{Journal of Machine Learning Research (JMLR)}, 2021.  
	\item  \textbf{Pascale Gourdeau}, Varun Kanade, Marta Kwiatkowska, and James Worrell, “Sample complexity bounds for robustly learning decision lists against evasion attacks,” in \emph{International Joint Conference on Artificial Intelligence (IJCAI)}, 2022. 
  \end{itemize}

\subsection*{Chapter~\ref{chap:local-queries}}

We consider learning models in which the learner has access to local queries in addition to random examples.
We first show that local membership queries do not increase the robustness threshold of conjunctions under the uniform distribution.
We then study local equivalence queries, and show that distribution-free robust learning is impossible for a wide variety of concept classes if the query budget is strictly smaller than the adversarial budget.
We demonstrate, however, that when the two coincide, distribution-free robust learning becomes possible.
We exhibit general sample and query complexity upper bounds as well as tighter bounds in the special case of conjunctions.
We also give explicit bounds for linear classifiers on the boolean hypercube. 
We then study linear classifiers in the continuous case and establish a general sample complexity upper bound, as well as a query complexity upper bound when we limit the adversary's precision. 
We complement the upper bounds by showing general lower bounds on the expected number of queries to the local equivalence query oracle and instantiate them for specific concept classes.
We finish by comparing the local membership and equivalence query oracles, as well as how they compare with the membership and equivalence query oracles.

Sections~\ref{sec:lq-models},~\ref{sec:rob-learn-lmq} and~\ref{sec:rob-learn-leq} are based on the following publication:

\begin{itemize}
  \item  \textbf{Pascale Gourdeau}, Varun Kanade, Marta Kwiatkowska, and James Worrell, “When are local queries useful for robust learning?” in \emph{36th Conference on Neural Information Processing 
Systems (NeurIPS)}, 2022. 
  \end{itemize}
  
Sections~\ref{sec:adv-bounded-precision},~\ref{sec:qc-lb-leq} and~\ref{sec:comparing-lq} are based on work that we are currently preparing for submission.

\subsection*{Chapter~\ref{chap:conclusion}}

We conclude by summarizing our contributions and drawing a picture of robust learnability in the learning models we have studied. 
Finally, we outline various avenues for future work. 

\section{Statement of Contribution}

The publications mentioned in the previous section have largely been my own work, with direction from my supervisors James Worrell, Varun Kanade and Marta Kwiatkowska. 
While NeurIPS 2019/JMLR 2021 papers addressed research questions posed by my supervisors, I lead the research -- including the technical aspect by deriving the proofs, and wrote most of the paper.  
For the IJCAI 2022 paper, I continued to lead in the technical development and writing up of the manuscript. In addition to this, I played a major role in formulating the research questions and positioning the work in a wider context. 
For the NeurIPS 2022 paper and subsequent ongoing work, I did most  of the work on my own – from finding and defining the research problem and learning model, providing insights on the problem at hand, deriving the proofs and writing the whole paper. 
I was of course supported by my supervisors: they referred me to a paper and suggested a way to prove a particular bound, they strengthened the paper by providing helpful feedback through nuanced discussions, and reviewed many iterations of the draft. 

\chapter{Literature Review}
\label{chap:lit-review}
This chapter gives an overview of the literature relevant to this thesis.
We start by reviewing classical learning theory results, focusing on classification.
We finish with a review of adversarial machine learning.
While we mention work pertaining to other views on robustness, our focus is the study of robustness to evasion attacks, particularly from a foundational viewpoint.

The results in this chapter are presented at a high level.
However, readers who are not familiar with learning theory may find it beneficial to refer to Chapter~\ref{chap:background}, which gives a thorough technical introduction to various frameworks and complexity measures discussed in this chapter.

\section{The Learning Theory Landscape}

We start with an overview of the established literature in classification in the probably  approximately correct  and online learning frameworks, and then move to learning with access to membership and equivalence queries.

\subsection{Classification}

The probably approximately correct (PAC) learning model of \citet{valiant1984theory} is one of the most well-studied classification models in learning theory.
In this framework, the learner has access to the example oracle, which returns a point $x\sim D$ sampled from an underlying distribution and its label $c(x)$, where $c$ is the target concept (ground truth). The goal is to output a hypothesis $h$ from a hypothesis class $\H$ such that $h$ has low error with high probability.\footnote{In the realizable setting, where it is possible to achieve zero risk, we want the risk to be as close as possible to zero. In the agnostic setting, we compare the risk of the hypothesis output by an algorithm to the risk of the optimal function from the hypothesis class.}
Remarkably, there exists a complexity measure, namely the VC dimension of \citet{vapnik1971uniform}, that characterizes the learnability of a hypothesis class. 
Indeed, it is possible to get both upper and lower bounds for the number of samples needed for learning (i.e., the sample complexity) that are \emph{linear} in the VC dimension.
The upper bound is due to \citet{vapnik1982estimation,blumer1989learnability}, and the lower bounds to 
\citet{blumer1989learnability,ehrenfeucht1989general}. 
These bounds are tight up to a $\log\left(\frac{1}{\epsilon}\right)$ factor, where $\epsilon$ is the parameter controlling the accuracy of the hypothesis output by the learning algorithm.

The one-inclusion graph of \citet{haussler1994predicting}, which also enjoys an upper bound that is linear in the VC dimension, was conjectured to be optimal (in the sense that the upper and lower bounds on sample complexity are tight) by \citet{warmuth2004optimal} until the recent work of  \citet{aden2022one} showing that this is not the case.
However, the breakthrough work of \citet{hanneke2016optimal} showed it is in general possible to get rid of the $\log\left(\frac{1}{\epsilon}\right)$ factor with a majority-vote classifier, following important advances made by \citet{simon2015almost}.

Another popular learning setting is that of online learning, introduced in the seminal work of \citet{littlestone1988learning} and in which a learning algorithm competes against an adversary. At each iteration, the learner is presented with an instance to predict, and afterwards the adversary reveals the true label of the instance.
The goal is to make as few mistakes as possible. 
\citet{littlestone1988learning} studied the \emph{realizable setting}, where there always is a function that makes zero mistakes on the learning sequence, and showed that a notion of complexity (the Littlestone dimension) characterizes online learnability in this framework.
The algorithm achieving this is called the standard optimal algorithm (SOA), which was later adapted by 
\citet{ben2009agnostic} to the \emph{agnostic setting}, where there need not exist a function that makes zero mistakes; the algorithm's performance is instead compared with the best hypothesis \emph{a posteriori}.
There is a vast literature on online learning, and we refer the reader to the book of \citet{cesa2006prediction} for a technical overview and references therein.

\subsection{Learning with Queries}

The works mentioned in the previous section studied classification when the learner has access to random examples. 
Active learning is another learning framework in which the learner is given more power, often through the use of membership and equivalence queries.
Membership queries allow the learner to query the label of any point in the input space $\X$, namely, if the target concept is $c$, the membership query ($\MQ$) oracle returns $c(x)$ when queried with $x\in\X$.
On the other hand, the equivalence query ($\EQ$) oracle takes as input a hypothesis $h$ and returns whether $h=c$, and provides a counterexample $z$ such that $h(z)\neq c(z)$ otherwise.
The goal in the $\MQ+\EQ$ model is usually to learn the target $c$ exactly, which is in contrast to the PAC setting which requires to learn with high confidence a hypothesis with low error.  

The seminal work of \citet{angluin1987learning} showed that deterministic finite automata (DFA) are exactly learnable with a polynomial number of queries to $\MQ$ and $\EQ$ in the size of the DFA. 
Follow-up work generalized these results.
E.g., \citet{bshouty1993exact} showed that poly-size decision trees are efficiently learnable in this setting as well;
\citet{angluin1988queries} later investigated other types of queries and also showed that $k$-CNFs and $k$-DNFs are exactly learnable with access to membership queries;
\cite{jackson1997efficient} showed that, in the $\PAC + \MQ$ setting, the class of $\mathsf{DNF}$ formulas is learnable under the uniform distribution.
But even these powerful learning models have limitations: learning DFAs only with $\EQ$ is hard \citep{angluin1990negative} and, under cryptographic assumptions, DFAs are also hard to learn solely with the $\MQ$ oracle~\citep{angluin1995when}.

On a more applied note, the $\MQ + \EQ$ model has recently been used for recurrent and binarized neural networks \citep{weiss2018extracting,weiss2019learning,okudono2020weighted,shih2019verifying}, and interpretability \citep{camacho2019learning}.
It is also worth noting that the $\MQ$ learning model has been criticized by the applied machine learning community, as labels can be queried in the whole input space, irrespective of the distribution that generates the data.
In particular, \citet{baum1992query} observed that query points generated by a learning algorithm on the handwritten characters often appeared meaningless to human labellers.
\citet{awasthi2013learning} thus offered an alternative learning model to Valiant's original model, the PAC and local membership query ($\EX+\LMQ$) model, where the learning algorithm is only allowed to query the label of points that are close to examples from the training sample.
\citet{bary2020distribution} later showed that many concept classes, including DFAs, remain hard to learn in the $\EX+\LMQ$ model. 

\section{Adversarial Machine Learning}

There has been considerable interest in adversarial machine learning
since the seminal work of~\cite{szegedy2013intriguing}, who coined
the term \emph{adversarial example} to denote the result of applying a
carefully chosen perturbation that causes a classification error to a
previously correctly classified datum. 
This work was largely experimental in nature and presented a striking instability of deep neural networks, where for example a correctly-classified image of a school bus was labelled as an ostrich after a perturbation (imperceptible to the human eye) was applied, as in Figure~\ref{fig:school-bus}.   
\cite{biggio2013evasion}
independently observed this phenomenon with experiments on the MNIST \citep{lecun1998mnist} dataset. However, as pointed out 
by \cite{biggio2017wild}, adversarial machine learning has been 
considered much earlier in the context of spam filtering~(\cite{dalvi2004adversarial,lowd2005adversarial,lowd2005good,barreno2006can}).
Their survey also distinguished two settings: \emph{evasion attacks}, 
where an adversary modifies data at test time, and \emph{poisoning attacks}, 
where the adversary modifies the training data.
For an in-depth review and definitions of different types of attacks, the reader may refer to \citep{biggio2017wild,dreossi2019formalization}.
For an introduction to adversarial defences in practice, see, e.g., \citep{goodfellow2015explaining,zhang2019theoretically}.

\begin{figure}
\begin{center}
\includegraphics[scale=0.35]{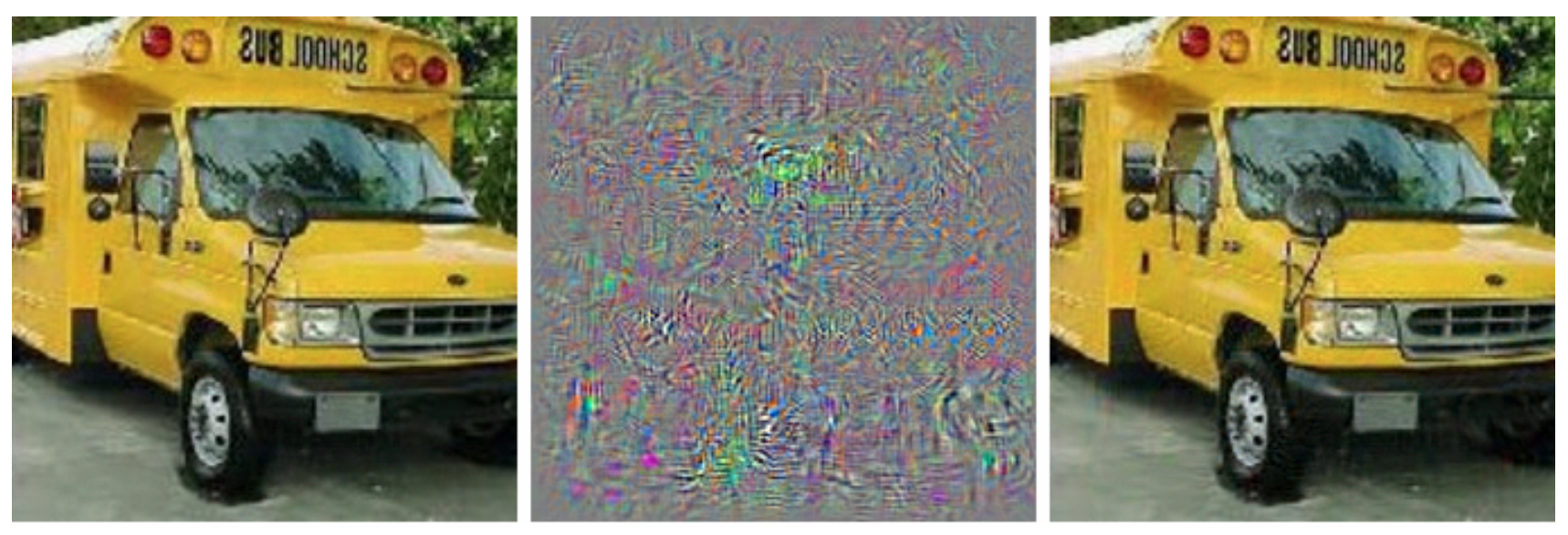}
\end{center}
\caption{A school bus is classified as an ostrich after a small perturbation is applied to the original image \citep{szegedy2013intriguing}.}
\label{fig:school-bus}
\end{figure}

As our work pertains to the robustness of machine learning algorithms to evasion attacks in classification tasks from a learning theory perspective, our review of related work will mainly concern this topic (Section~\ref{sec:evasion-attacks}). 
Before discussing this body of work, we will briefly mention other views on robustness.

Many works have studied the robustness of learning algorithms to poisoning attacks, in which an adversary can modify the training data in order to increase the (standard) error at test time, one of the earliest being that of \citet{kearns1988learning}.
Various types of poisoning attacks have been put forward since then, especially as the study of robustness has garnered interest in recent years. 
Clean-label attacks, proposed by \cite{shafahi2018poison}, are a distinct form of poisoning attacks where the poisoned examples are labelled correctly, i.e., by the target function, and not adversarially.
For a learning-theoretic approach and results on this problem, see \citep{mahloujifar2017blockwise,mahloujifar2019can,mahloujifar2018learning,
mahloujifar2019curse,etesami2020computational,blum2021robust}  (non-exhaustive).
In case there is no restriction on the label of poisoned data, see, e.g., the works of \citep{barreno2006can,biggio2012poisoning,papernot2016towards,steinhardt2017certified} (non-exhaustive).
Finally, for work on defences against poisoning attacks, we refer the reader to \citep{goldblum2022dataset}.

Another view on robustness is out-of-distribution detection, where the goal is to identify outliers at test time. We refer the reader to the textbook \citep{quinonero2008dataset} for an introduction on dataset shifts, and to \citep{fang2022out} for a study on out-of-distribution detection from a PAC-learning perspective, as well as references therein for the empirical work on the matter. 
A more general view on distributional discrepancies at test-time is that of distribution shift. 
See \citep{wiles2022fine} for a taxonomy on various distribution shifts and a review of important work in the area (mostly from an empirical perspective). 

\subsection{Evasion Attacks}
\label{sec:evasion-attacks}

We now turn our attention to the focus of this thesis: robustness to evasion attacks.
For ease of reading, we have  thematically split the related work in this section. 

\paragraph*{Defining Robustness.}
The majority of the guarantees and impossibility results for evasion attacks are based on the existence of adversarial examples. % potentially crafted by an all-powerful adversary. 
However, what is considered to be an adversarial example has been defined in different, and in some respects contradictory, ways in the literature. 
What we refer to as the \emph{exact-in-the-ball} notion of robustness in this work (also known as
\emph{error region} risk in \citep{diochnos2018adversarial}) requires that the hypothesis and the ground truth agree in the perturbation region around each test point; the ground truth must thus be specified on all
input points in the perturbation region.
On the other hand, what we refer to as the
constant-in-the-ball notion of robustness (which is also known as
\emph{corrupted input} robustness from the work of \cite{feige2015learning}) requires that the
unperturbed point be correctly classified and that the points in the perturbation region share its label, meaning that we only need access to the test point labels; the works  \citet{diochnos2018adversarial,dreossi2019formalization,pydi2021many} offer thorough discussions on the subject and also compare robustness definitions.
Moreover, \citet{chowdhury2022robustness} have studied settings where a model's change of label is justified by looking at robust-Bayes classifiers and their standard counterparts.  

We note that \cite{suggala2019revisiting} proposed an alternative definition of robustness, where a perturbation is deemed adversarial if it causes a label change in the hypothesis \emph{while the target classifier's label remains constant}. 
The existence of a ground truth is thus explicitly assumed (which is not in general necessary for constant-in-the-ball robustness).

Rather than studying the existence of a misclassification in the perturbation region, \citet{pang2022robustness} define robustness using the Kullback-Leibler (KL) divergence. The robust loss at a given unperturbed point $x$ is the maximal KL divergence over perturbations $z$ between the underlying labelling function ($\Prob{}{y\given z}$) of $z$ and the hypothesis' label for $z$ (which could also be non-deterministic). The authors proposed this definition of robustness in an effort to avoid the trade-off between accuracy and robustness observed in prior work, e.g., \citep{tsipras2019robustness}.

In the remainder of this section, whenever the robust risk is not explicitly mentioned, the results will hold for the constant-in-the-ball notion of robustness, as it is the most widely used in the literature.

\paragraph*{Existence of Adversarial Examples.} There is a considerable body of work that studies the inevitability of
adversarial examples,
e.g.,~\citep{fawzi2016robustness,fawzi2018adversarial,
fawzi2018analysis,gilmer2018adversarial,shafahi2018adversarial,
tsipras2019robustness}.
These papers characterize robustness in the sense that a classifier's output
on a point should not change if a perturbation of a certain
magnitude is applied to it.  
These works also study
geometrical characteristics of classifiers and statistical
characteristics of classification data that lead to adversarial
vulnerability. 
It has been shown that, in many instances, the vulnerability of learning models to adversarial examples is inevitable due to the nature of the learning problem.
%The majority of the results have been shown for the constant-in-the-ball notion of robustness, see e.g., \citep{fawzi2016robustness,fawzi2018adversarial,fawzi2018analysis,gilmer2018adversarial,shafahi2018adversarial,tsipras2019robustness}.
Notably, \citet{bhagoji2019lower} study robustness to evasion attacks from an optimal transport perspective, obtaining lower bounds on the robust error. %, as well as sample complexity upper and lower bounds in the special case of Gaussian data against an adversary whose perturbation region is convex and origin-centric. 
Moreover, many works exhibit a trade-off between standard accuracy and robustness in this setting, e.g., \citep{tsipras2019robustness,dobriban2020provable}.

As for the exact-in-the-ball definition of robustness, \citet{diochnos2018adversarial} consider the robustness of monotone conjunctions under the uniform distribution. 
Their results  concern the ability of an adversary to magnify the
missclassification error of \emph{any} hypothesis with respect to
\emph{any} target function by perturbing the input.\footnote{We will draw an explicit comparison with the work of \cite{diochnos2018adversarial} in Section~\ref{sec:mon-conj-sc-lb}.} 
\citet{mahloujifar2019curse} generalized the above-mentioned result to Normal Lévy families and a class of well-behaved classification problems (i.e., ones where the error regions are measurable and average distances exist).

\paragraph*{Computational Complexity of Robust Learning.} The computational complexity of robust learning is an active research area.
\citet{bubeck2018cryptographic} and \citet{degwekar2019computational} have shown that there are concept classes that are
hard to robustly learn under cryptographic assumptions, even when robust learning is
information-theoretically feasible. 
\citep{bubeck2019adversarial} established super-polynomial lower
bounds for robust learning in the statistical query framework. 
\cite{diakonikolas2019nearly} study the more specific problem of (standard) proper learning of halfspaces with noise and large $\ell_2$ margins in the agnostic PAC setting, focussing on the computational complexity of this learning problem. They remark that these guarantees can apply to robust learning. 
In follow-up work \citep{diakonikolas2020complexity}, they explicitly study robustness to $\ell_2$ perturbations and generalize their previous results. 
In particular, they obtain computationally-efficient algorithms using an online learning reduction, and building on a hardness result in \citep{diakonikolas2019nearly}, and provide tight running time lower bounds.
Finally, \citet{awasthi2019robustness} draw connections between robustness to evasion attacks and polynomial optimization problems, obtaining  a computational hardness result. 
On the other hand, they exhibit computationally efficient robust learning algorithms for linear and quadratic threshold functions in the realizable case. 

\paragraph*{Sample Complexity of Robust Learning.} Despite being a relatively recent research area, there already exists a vast literature on the sample complexity of robust learning to evasion attacks.
One of the earlier works is that of \citet{cullina2018pac}, who define the notion of adversarial VC dimension to derive sample complexity upper bounds for robust empirical risk minimization (ERM) algorithms, with respect to the constant-in-the-ball robust risk. 
They also study the special case of halfspaces under $\ell_p$ perturbations and show the adversarial VC dimension is in general incomparable with its standard counterpart.
Shortly after, \citet{attias2019improved} adopted a game-theoretic framework to study robust learnability for classification and regression in a setting where the adversary is limited to a fixed number $k$ of perturbations per input. They obtain sample complexity bounds that are linear in both $k$ and the VC dimension of a hypothesis class.
The work of \citet{montasser2019vc} later provided a more complete picture of robust learnability.
The authors show sample complexity upper bounds for robust ERM algorithms that are polynomial in the VC and dual VC dimensions of concept classes, giving general upper bounds that are exponential in the VC dimension. 
They also exhibit sample complexity lower bounds linear in the robust shattering dimension, a notion of complexity introduced therein.
The gap between the upper and lower bounds was closed in their later work \citep{montasser2022adversarially}, where they fully characterize the sample complexity of robust learning with arbitrary perturbation functions. 
%They first show a separation between what they call \emph{local} and \emph{global} learners. Local learners only have access to inflated samples (with respect to the unknown perturbation function), while global learners know the perturbation function $\U:\X\rightarrow2^\X$. 
The robust learning algorithm achieving the upper bound is a generalization of the one-inclusion graph algorithm of \cite{haussler1994predicting}. %, which was defined for the standard risk. 
Their robust variant of the one-inclusion graph is defined for the constant-in-the-ball \emph{realizable} setting,\footnote{\label{note:realizable-c-i-b}I.e., there exists a hypothesis that has zero constant-in-the-ball robust loss.} but the agnostic-to-realizable reduction from previous work \citep{montasser2019vc} can be applied. 
The  (random-example) sample complexity characterizing robust learnability is a notion of \emph{dimension} defined through the edges on the graph structure.

The above bounds consider the supervised setting, where the learner has access to labelled examples.
Since the cost of obtaining data is at times largely due to its labelling,\footnote{Think for example of obtaining images vs needing humans to label them.} studying semi-supervised learning, where the learner has access to both unlabelled as well as labelled examples, is of general interest.
\citet{ashtiani2020black} build on the work of \citet{montasser2019vc} (who showed that proper robust learning, where the learner is required to output a hypothesis from the same class as the potential target concept, is sometimes impossible) and delineate when proper robust learning is possible. 
They moreover draw a more nuanced picture of proper robust learnability with access to unlabelled random examples.
\cite{attias2022characterization} also study the sample complexity of robust learning in the semi-supervised framework.
Notably, in the realizable setting, their labelled sample complexity bounds are linear in a variant of the VC dimension where, for a shattered set, the perturbation region around a given point must share the same label.\footnote{This complexity measure is always upper bounded by the VC dimension, and the gap can be arbitrarily large. }
The unlabelled sample complexity is linear in the sample complexity of \emph{supervised} learning. 
The authors also extend their results to the agnostic setting.

While it is worthwhile to study robust learnability for arbitrary perturbation regions, focussing on specific perturbation functions that are more faithful to real-world problems is of high interest, especially if this can provide better guarantees or a clearer picture of robustness in this setting.
In this vein, \cite{shao2022theory} study the robustness to evasion attacks under \emph{transformation invariances}.
This terminology comes from group theory: the transformations applied to instances form a group, and an invariant hypothesis will give the same label to points in the orbit of every instance in the support of the distribution generating the data.\footnote{E.g., rotating an image of a cat will still result in an image of a cat, while rotating an image of a six can result in an image of nine. Transformation invariances are thus problem specific. }
As a characterization of robust learnability in these settings, they propose two combinatorial measures that are variants of the VC dimension that take into account the orbits of points in the shattered set, and prove nearly-matching upper and lower bounds.

All the works mentioned above study sample complexity through the VC dimension of a concept class, or variants adapted to robust learnability.
On the other hand, \citet{khim2019adversarial,yin2019rademacher,awasthi2020adversarial} instead use the \emph{adversarial} Rademacher complexity to study robust learning.
These works give  results for ERM on linear classifiers and neural networks.

As for the exact-in-the-ball definition of robustness,  \citet{diochnos2020lower} study sample complexity lower bounds. They show that, for a wide family of concept classes, any learning algorithm that is robust against all attacks with budget $\rho=o(n)$ must have a sample complexity that is at least  exponential in the input dimension $n$. 
They also show a superpolynomial lower bound in case $\rho=\Theta( \sqrt n)$.
This, along with the previously-mentioned works of \cite{diochnos2018adversarial,mahloujifar2019curse} are to our knowledge the only other works apart from ours that consider the sample complexity of exact-in-the-ball robust learning from a theoretical perspective.

\paragraph*{Relaxing Robustness Requirements.}
Most adversarial learning guarantees and impossibility results in the literature have focused on all-powerful adversaries.
Recent works have studied learning problems where the adversary's power is curtailed.
One way to do this is to consider  \emph{computationally-bounded} adversaries. 
E.g, \citet{mahloujifar2019can} and \citet{garg2020adversarially} study the robustness of classifiers to polynomial-time attacks. 
They show that, for product distributions, an initial constant error implies the existence of a (black-box) polynomial-time attack for adversarial examples that are $O(\sqrt{n})$ bits away from the test instances. 
However, \cite{garg2020adversarially} show a separation result for a learning problem where a classifier can be successfully attacked by a computationally-unbounded adversary, but not by a  polynomial-time bounded adversary subject to standard cryptographic hardness assumptions.

It is also possible to relax the optimality condition when evaluating a hypothesis.
\citet{ashtiani2023adversarially} and \citet{bhattacharjee2023robust} both study  \emph{tolerant robust learning}, where the learner is evaluated relative to the hypothesis with the best robust risk under a slightly larger perturbation region. \cite{ashtiani2023adversarially} show that  this setting enables better sample complexity bounds that the standard robust setting for metric spaces $(\X,d)$ in case the perturbation region is a ball with respect to the metric $d$. \cite{bhattacharjee2023robust} build on their work and instead consider problems with a geometric niceness property called \emph{regularity} to get more general perturbation regions.
They obtain matching sample complexity bounds to \citep{ashtiani2023adversarially} as well as propose a variant of robust ERM as a simpler robust learning algorithm for this problem.

Another relaxation of the robust learning objective is a probabilistic variant of robust learning. 
\citet{viallard2021pac} derive PAC-Bayesian generalization bounds (where the output is a posterior distribution over hypotheses after seeing the data) for the averaged risk on the perturbations, rather than working in a worst-case scenario.
\citep{robey2022probabilistically} also consider probabilistic robustness, where the aim is to output a hypothesis that is robust to \emph{most} perturbations.

\paragraph*{Increasing the Learner's Power.} 
To improve robustness guarantees, it is also possible to give the learner access to more powerful oracles than the random-example one. 
\cite{montasser2020reducing,montasser2021adversarially} study robust learning with access to a (constant-in-the-ball) robust loss oracle, which they call the Perfect Attack Oracle (PAO).
For a perturbation type $\U:\X\rightarrow 2^\X$, hypothesis $h$ and labelled point $(x,y)$, the PAO returns the constant-in-the-ball robust loss of $h$ in the perturbation region $\U(x)$ and a counterexample $z\in\U(x)$ where $h(z)\neq y$ if it exists.
In the constant-in-the-ball \emph{realizable} setting, the authors use online learning results to show sample and query complexity bounds that are linear and quadratic in the Littlestone dimension of concept classes, respectively \citep{montasser2020reducing}.
\citet{montasser2021adversarially} moreover use the algorithm from \citep{montasser2019vc} to get  sample and query complexity upper bounds  that respectively have a linear and exponential dependence on the VC and dual VC dimensions of the hypothesis class at hand.
Finally, they extend their results to the agnostic setting and derive lower bounds.

\chapter{Background}
\label{chap:background}
In this chapter, we introduce the necessary background and notation for the main contributions of this thesis.
We start by reviewing standard learning theory concepts in Section~\ref{sec:learning-theory}, before moving to probability theory  in Section~\ref{sec:prob-theory}. 
We finish with an overview of Fourier analysis in Section~\ref{sec:fourier-analysis}.

\paragraph*{Notation.} Throughout this text, we will use $[n]$ to denote the set $\set{1,\dots,n}$. 
The symbol $\Delta$ will represent the symmetric difference between two sets: $I\Delta J=\set{x \given x\in I\setminus J \text{ or } x\in J\setminus I}$. 
We will use the asymptotic notation ($o,O,\omega,\Omega,\Theta$), with the convention that the symbol $\;\widetilde{\;}$ (e.g., $\widetilde{O}$) omits the logarithmic factors.
Given a metric space $(\X,d)$ and $\lambda\in\R$, we denote by $B_\lambda(x)$ the ball $\set{z\in\X\given d(x,z)\leq\lambda}$  of radius $\lambda$ centred at $x$.
We will use the symbol $\mathbf{1}[\cdot]$ for the indicator function.
Finally, for a given formula $\varphi$ and instance $x$, we denote by $x\models \varphi$ the event that $x$ satisfies $\varphi$.

\section{Learning Theory: Classification}
\label{sec:learning-theory}

Learning theory offers an elegant abstract framework to analyse the behaviour of machine learning algorithms, as well as to provide performance and correctness guarantees or show impossibility results. 
There exist various learning settings, depending on assumptions on how the data is obtained and on the learning objectives. 
This thesis is primarily concerned with \emph{binary classification}, where, given an input space $\X$, the goal is to output a function $h:\X\rightarrow\{0,1\}$ called a \emph{hypothesis}, which upon being given an instance $x\in\X$ outputs a label $h(x)\in\{0,1\}$.
The more general task of learning a function $\X\rightarrow\Y$ is called \emph{multiclass classification} when $\Y$ is a discrete finite set, and \emph{regression} when $\Y=\R$.

In this section, we give an overview of  three learning settings for binary classification: learning with random examples in the Probably Approximately Correct (PAC) framework, the mistake-bound model of online learning, and learning with membership and equivalence queries. 
For each setting, we discuss various notions of complexity that control the amount of data needed to learn, i.e., the \emph{sample complexity}.
In all cases, we will be using the terms learning algorithm, learner and learning process interchangeably to denote a process of data acquisition and analysis resulting in outputting a hypothesis $h$ as above. 
For a more in-depth introduction to the concepts presented in this section, we refer the reader to \cite{mohri2012foundations} and \cite{shalev2014understanding}, both excellent introductory textbooks on learning theory.

\subsection{The PAC Framework}
\label{sec:pac}

The Probably Approximately Correct (PAC) framework of \cite{valiant1984theory}, depicted in Figure~\ref{fig:pac}, formalises the desired behaviour of a learning algorithm. 
In this learning setting, a learning algorithm has access to \emph{random examples} drawn in an i.i.d. fashion from an underlying distribution $D$, 
and we wish to output a hypothesis that has small \emph{error} with high \emph{confidence}.
The error $\err_D(h,c)$ of a hypothesis with respect to $D$ is measured against a \emph{ground truth function} or \emph{target concept} $c:\X\rightarrow\{0,1\}$ which labels the data, and is defined as
$$\err_D(h,c)=\Prob{x\sim D}{c(x)\neq h(x)}\enspace.$$
The set of points $x\in \X$ such that $c(x)\neq h(x)$ is often referred to as the \emph{error region}.
We sometimes model the sampling process by having access to the random example oracle $\EX(c,D)$.
The ``probably'' part of the PAC learning framework speaks to the confidence of the learning algorithm, and allows for the possibility that a sample $S\sim D^m$ of size $m$ drawn from the underlying distribution $D$ is not representative of $D$. 
The ``approximately'' part of PAC learning refers to the requirement that the hypothesis have sufficiently high accuracy, a relaxation from learning \emph{exactly.}
Both the confidence and accuracy parameters are inputs to the learning algorithm, and are \emph{learning parameters}.

Another important parameter that affects the sample complexity is how \emph{large} the instance size is, e.g., the larger the number of pixels for image classification is, the larger the amount of data needed to learn could be.
This is usually controlled by the \emph{dimension $n$ of the input space},  in reference to $\boolhc$ and $\R^n$.
To this end we consider a collection of pairs of input space and concepts classes $\X_n$ and $\C_n$ for each dimension $n$, where $\C_n$ is a set of functions $c:\X_n\rightarrow\{0,1\}$. 

We are now ready to formally define the PAC learning setting.

\begin{definition}[PAC Learning, Realizable Setting]
\label{def:pac-realizable}
For all $n\in\N$, let $\C_n$ be a concept class over $\X_n$ and let $\C=\bigcup_{n\in\N}\C_n$.
We say that $\C$ is \emph{PAC learnable using hypothesis class $\H$} and sample complexity function $m(\cdot,\cdot,\cdot,\cdot)$ if there exists an algorithm $\A$ that satisfies the following:
for all $n\in\N$, for every $c\in\C_n$, for every $D$ over $\X_n$, for every $0<\epsilon<1/2$ and $0<\delta<1/2$, if whenever $\A$ is given access to $m\geq m(n,1/\epsilon,1/\delta,\text{size}(c))$ examples drawn i.i.d. from $D$ and labeled with $c$, $\A$ outputs an $h\in\H$ such that with probability at least $1-\delta$, 
\begin{equation*}
\err_D(h,c)=\Prob{x\sim D}{c(x)\neq h(x)}\leq \epsilon\enspace.
\end{equation*}
We say that $\C$ is statistically efficiently PAC learnable if $m$ is polynomial in $n,1/\epsilon$, $1/\delta$ and size$(c)$, and computationally efficiently PAC learnable if $\A$ runs in polynomial time in $n,1/\epsilon$, $1/\delta$ and size$(c)$  and $h$ is polynomially evaluatable.
\end{definition}

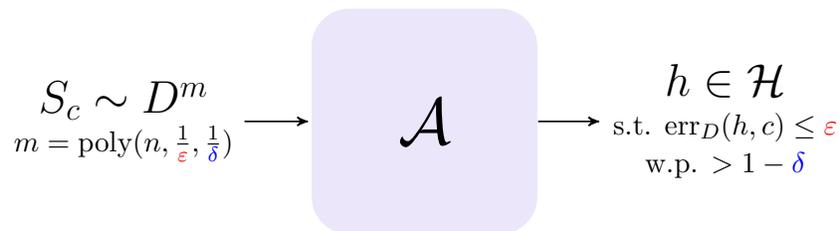
\begin{figure}
\begin{center}
\blue{Probably} \red{Approximately} Correct Learning\\
\medskip
\medskip
\begin{tikzpicture}[->,>=stealth',shorten >=1pt,auto,semithick]
\node[fill=blue!80!red!10, minimum size=3cm, align=center, rounded corners=0.5cm] (algorithm) {\Huge{$\A$}};
\node[left of=algorithm,align=center,xshift=-3cm] (sample) {\Large{$S_c\sim D^m$} \\ \small{$m=\text{poly}(n,\frac{1}{\red{\varepsilon}},\frac{1}{\blue{\delta}})$}};
\node[right of=algorithm,align=center,xshift=3cm] (hypothesis) {\Large{$h\in\H$} \\ \small{s.t. err$_D(h,c)\leq \red{\varepsilon}$}\\ \small{w.p. $>1-\blue{\delta}$}};
\path 	(sample) edge (algorithm)
		(algorithm) edge (hypothesis);
\end{tikzpicture}
\end{center}
\caption{A visual representation of sample-efficient PAC learning. $S_c$ means that the sample $S$ has been labelled with the ground truth $c$.}
\label{fig:pac}
\end{figure}

\paragraph*{Size$(c)$ and polynomial evaluatability.}
Two additional requirements from accuracy and confidence are introduced in the above definition: these are a sample complexity function dependent on the size $size(c)$ of the target concept $c$, and, if one requires computational efficiency, the fact that $h$ is \emph{polynomially evaluatable}.
The size of a concept is defined through a \emph{representation scheme}. 
Essentially, there could exist several representations of a function, e.g., a function can be computed by many different boolean circuits. 
Assuming that there exists a function measuring the size of a representation, the size of a concept $c$ is the minimal size of a representation of $c$.
The second requirement is natural: if the hypothesis is not required to be polynomially evaluatable, then the learner could simply ``offload'' the learning process at test time (there is nothing to do at training, so it would be considered ``efficient''), and overall require arbitrarily high computational complexity. 

\paragraph*{Proper vs improper learning.}
The setting where $\C=\H$ is called \emph{proper learning}, and \emph{improper learning} if $\C\subseteq\H$.
While requiring proper learning does not affect the sample complexity of learning very much,\footnote{It is possible to get rid of the $\log1/\epsilon$ factor of Theorem~\ref{thm:vc-pac-upper-bound} as shown by the recent breakthrough of \cite{hanneke2016optimal} with an improper learner, but, aside from this, the sample complexity bounds in Theorems~\ref{thm:vc-pac-upper-bound} and~\ref{thm:vc-pac-lower-bound} are tight for any consistent learner.} it can affect its computational efficiency.
Indeed, unless $\mathsf{RP}=\mathsf{NP}$, which is widely believed not to be the case, it is impossible to computationally efficiently \emph{properly} learn the class of 3-term formulas in disjunctive normal form (DNF), i.e., formulas of the form $T_1\vee T_2 \vee T_3$ where the $T_i$'s are conjunctions of arbitrary lengths.
However, it is possible to computationally efficiently PAC learn 3-CNF formulas properly (formulas in conjunctive normal form where each term is a disjunction of at most 3 literals), and this class subsumes 3-term DNFs.
Hence, one can use the PAC-learning algorithm for 3-CNF to (improperly) PAC learn 3-term DNFs in a computationally efficient manner. 

\paragraph*{The distribution-free assumption.}
PAC learning is \emph{distribution-free}, in the sense that no assumptions are made about the distribution from which the data is generated.
As long as the training data is sampled i.i.d. from a given distribution $D$, and that the algorithm is tested on independent examples drawn from $D$, the learning guarantees hold.
Of course, this is sometimes not a sensible assumption to make in practice. 
Many lines of work consider learning settings that allow for this and provide a more realistic learning framework, e.g., when noise is added to the data, or when the training and testing distributions differ (i.e., distribution shift), as outlined in Chapter~\ref{chap:lit-review}.

\paragraph*{Realizable vs agnostic learning.}
The \emph{realizability assumption} of Definition~\ref{def:pac-realizable}, where there always exists a concept with zero error, does not always hold.
In the presence of noise, or more generally in the absence of a \emph{deterministic} labelling function $c$ representing the ground truth (e.g., there is a joint distribution on $\X\times\Y$), we instead work in the \emph{agnostic setting}. 
In this setting, the goal is rather to learn a hypothesis that does well compared to the best concept in the concept class:

\begin{definition}[PAC Learning, Agnostic Setting]
\label{def:pac-agnostic}
Let $\C_n$ be a concept class over $\X_n$ and let $\C=\bigcup_{n\in\N}\C_n$.
We say that $\C$ is \emph{agnostically PAC learnable using $\H$} with sample complexity function $m(\cdot,\cdot,\cdot,\cdot)$ if there exists an algorithm $\A$ that satisfies the following:
for all $n\in\N$, for every $D$ over $\X_n\times\{0,1\}$, for every $0<\epsilon<1/2$ and $0<\delta<1/2$, if whenever $\A$ is given access to $m\geq m(n,1/\epsilon,1/\delta,s)$ labelled examples drawn i.i.d. from $D$, where $s=\underset{c\in\C_n}{\sup}\;\text{size(c)}$, $\A$ outputs an $h\in\H$ such that with probability at least $1-\delta$, 
\begin{equation*}
\err_D(h)\leq \underset{c\in\C_n}{\inf} \err_D(c)+\epsilon\enspace,
\end{equation*}
where $\err_D(h)=\Prob{(x,y)\sim D}{ h(x)\neq y}$.
We say that $\H$ is statistically efficiently agnostically learnable if $m$ is polynomial in $n,1/\epsilon$, $1/\delta$ and $s$, and computationally efficiently agnostically learnable if $\A$ runs in polynomial time in $n,1/\epsilon$, $1/\delta$ and $s$, and $h$ is polynomially evaluatable.
\end{definition}

The definition above allows for improper learning ($\C$ is usually called the ``touchstone'' class), but we can recover proper learning by setting $\C=\H$.
In this work, unless otherwise stated, we will assume the realizability of a learning problem, and the sample complexity bounds will be derived for this setting.
Note that there exist PAC guarantees for classes of finite VC dimension in the agnostic setting as well, at the cost of a multiplicative factor of $1/\epsilon$ in the sample complexity.
See \citep{kearns1994toward,haussler1992decision} for original work on the matter and the textbook \citep{mohri2012foundations} for an introduction on the topic.

\subsection{Complexity Measures}

While it is possible to derive sample complexity bounds for specific hypothesis classes, one can take a more general approach with the use of  \emph{complexity measures}.
Indeed, a complexity measure assigns to each hypothesis class $\H$ a function (w.r.t. the size $n$ of the instance space) quantifying its richness.
Intuitively, as the complexity measure increases, more data should be needed to identify a candidate hypothesis that would generalize well on unseen data.
We briefly note that the standard theory outlined in this chapter has failed to explain the recent success of overparametrised deep neural networks in practice which in many ways remains an open problem in the learning theory literature.

The first complexity measure we will study is perhaps the simplest one: the size of $\H$.  
Similarly to $\C$, the class $\H$ is defined as the union $\bigcup_{n\in\N}\H_n$, and the size of $\H$ is a function of $n$.
The theorem below, known as Occam's razor, gives an upper bound on the sample complexity of learning with finite hypothesis classes, given access to a \emph{consistent} learner.
A consistent learner is a learning algorithm that outputs a hypothesis that has zero empirical loss on the training sample, i.e., a hypothesis that correctly classifies all the points in the training sample. 

\begin{theorem}[Occam's Razor \citep{blumer1987occam}]
Let $\C$ and $\H$ be a concept and hypothesis classes, respectively. 
Let $\A$ be a consistent learner for $\C$ using $\H$.
Then, for all $n\in\N$, for every $c\in\C_n$, for every $D$ over $\X_n$, for every $0<\epsilon<1/2$ and $0<\delta<1/2$, if whenever $\A$ is given access to $m\geq \frac{1}{\epsilon}\left(\log(|\H_n|)+\log(1/\delta)\right)$ examples drawn i.i.d. from $D$ and labeled with $c$, then $\A$ is guaranteed to output an $h\in\H_n$ such that $\err_D(h,c)<\epsilon$ with probability at least $1-\delta$.
Furthermore, if $\log(|\H_n|)$ is polynomial in $n$ and size$(c)$, and $h$ is polynomially evaluatable, then $\C$ is statistically efficiently PAC-learnable using $\H$.
\end{theorem}

While the theorem above can be useful if $\H_n$ is finite for all $n$, it does not tell us much when $\H$ is infinite.
To this end, one would want to consider complexity measures that are meaningful for infinite concept classes as well. 
In the PAC setting, a useful complexity measure is the Vapnik Chervonenkis (VC) dimension of a hypothesis class, from the work of \cite{vapnik1971uniform}. 
It turns out that this measure fully \emph{characterizes} the learnability of a concept class, in the sense that one can obtain upper \emph{and} lower bounds on the sample complexity that are both \emph{linear} in the VC dimension of $\H$.

In order to define the VC dimension of a concept class, we must first define the notion of \emph{shattering} of a set. 
In Figure~\ref{fig:shatter-eg}, we give an example of a set being shattered by linear classifiers in $\R^2$.

\begin{definition}[Shattering]
Given a class of functions $\F$ from input space $\X$ to $\set{0,1}$, we say that a set $S\subseteq\X$ is \emph{shattered by $\F$} if all the possible dichotomies of $S$ (i.e., all the possible ways of labelling the points in $S$) can be realized by some $f\in\F$. 
\end{definition}

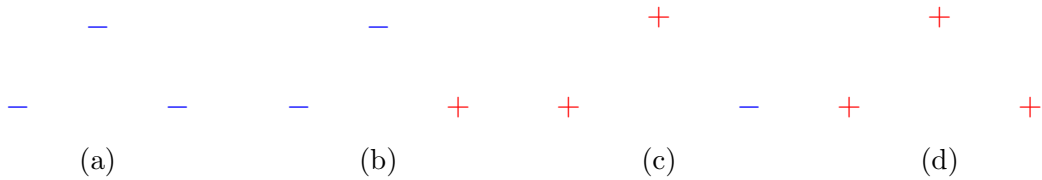
\begin{figure}
  \begin{subfigure}[b]{0.24\textwidth}
    \centering
      \begin{tikzpicture}[> = stealth,  shorten > = 1pt,   auto,   node distance = 1.5cm]
        \node (v) {$\blue{-}$};
        \node (w) [below left of=v] {$\blue{-}$};
        \node (t) [below right of=v] {$\blue{-}$};
      \end{tikzpicture}
      \subcaption{}
  \end{subfigure}
  \begin{subfigure}[b]{.24\textwidth}
    \centering
      \begin{tikzpicture}[> = stealth,  shorten > = 1pt,   auto,   node distance = 1.5cm]
        \node (v) {$\blue{-}$};
        \node (w) [below left of=v] {$\blue{-}$};
        \node (t) [below right of=v] {$\red{+}$};
      \end{tikzpicture}
      \subcaption{}
  \end{subfigure}
  \begin{subfigure}[b]{.24\textwidth}
    \centering
      \begin{tikzpicture}[> = stealth,  shorten > = 1pt,   auto,   node distance = 1.7cm]
        \node (v) {$\red{+}$};
        \node (w) [below left of=v] {$\red{+}$};
        \node (t) [below right of=v] {$\blue{-}$};
      \end{tikzpicture}
      \subcaption{}
  \end{subfigure}
  \begin{subfigure}[b]{.24\textwidth}
    \centering
      \begin{tikzpicture}[> = stealth,  shorten > = 1pt,   auto,   node distance = 1.7cm]
        \node (v) {$\red{+}$};
        \node (w) [below left of=v] {$\red{+}$};
        \node (t) [below right of=v] {$\red{+}$};
      \end{tikzpicture}
      \subcaption{}
  \end{subfigure}
\caption{A set $X$ of three points in $\R^2$ that is shattered by linear classifiers. Subfigures (a)-(d) represent different dichotomies on $X$; note that (b) and (c) are not the only labellings with one and two positively labelled points, respectively, but the other cases are symmetric.}
\label{fig:shatter-eg}
\end{figure}

We are now ready to define the VC dimension of a class.

\begin{definition}[VC Dimension]
The VC dimension of a hypothesis class $\mathcal{H}$, denoted $\VC(\mathcal{H})$, is the size $d$ of the largest set that can be shattered by $\mathcal{H}$.
If no such $d$ exists then $\VC(\mathcal{H})=\infty$.
\end{definition}

Figure~\ref{fig:vc-eg} illustrates the argument that no set in $\R^2$ of size 4 can be shattered by linear classifiers.

An important property of the VC dimension is that it is upper bounded by $\log\abs{\H}$.
Indeed, a shattered set $S$ of size $m$ needs $2^m$ distinct functions to achieve all its possible labellings.

It also is possible to define the VC dimension  through the \emph{growth function} of a concept class.
For some finite set of instances $S$, we denote by $\Pi_\C(S)=\{c\vert_S\given c\in\C\}$ the set of distinct restrictions of concepts in $\C$ on the set $S$, which is referred to as the set of all possible dichotomies on $S$ induced by $\C$.
Then a shattered set $S$ satisfies $\abs{\Pi_\C(S)}=2^{\abs{S}}$, and the VC dimension is thus the largest set satisfying this relationship.

\begin{figure}
  \begin{subfigure}[b]{0.45\textwidth}
    \centering
      \begin{tikzpicture}[> = stealth,  shorten > = 1pt,   auto,   node distance = 1.5cm]
        \node (v) {$\blue{-}$};
        \node (t) [below of=v] {$\red{+}$};
        \node (u) [below left of=t] {$\blue{-}$};
        \node (s) [below right of=t] {$\blue{-}$};
      \end{tikzpicture}
      \subcaption{}
  \end{subfigure}
  \begin{subfigure}[b]{.45\textwidth}
    \centering
      \begin{tikzpicture}[> = stealth,  shorten > = 1pt,   auto,   node distance = 1.5cm]
        \node (v) {$\blue{-}$};
        \node (t) [right of=v] {$\red{+}$};
        \node (u) [below of=v] {$\red{+}$};
        \node (s) [right of=u] {$\blue{-}$};
      \end{tikzpicture}
      \subcaption{}
  \end{subfigure}
\caption{Any set of four points cannot be shattered by linear classifiers. Indeed, we distinguish two cases: either (a) one point is strictly in the convex hull of the three other points, and is the only point of its label (or all points are on the same line, which gives a similar argument) or (b) all points are on the boundary of the convex hull, in which case labelling opposite points with the same label gives an unachievable labelling. This argument is a special case for $\R^2$ which can be generalized to $\R^n$ using Radon's theorem. }
\label{fig:vc-eg}
\end{figure}
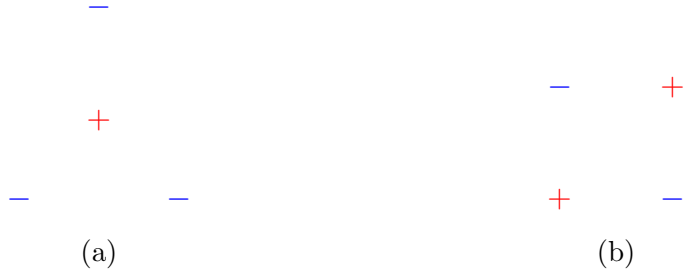

\begin{definition}[Growth Function] For any natural number $m\in\N$, the growth function is defined as $\Pi_\C(m)=\max\set{\abs{\Pi_\C(S)}\given \abs{S}=m}$.
\end{definition}

Denote by $\Phi_d(m)$ the summation $\sum_{i=0}^d{m\choose i}$.
The growth function of a concept class $\C$ can be bounded as follows, as a function of $m$ and the VC dimension $d$.

\begin{lemma}[Sauer-Shelah]
Let $\C$ be a concept class of VC dimension $d$. Then
$$\Pi_\C(m)\leq \Phi_d(m) \leq \left(\frac{em}{d}\right)^d\enspace.$$
\end{lemma}

As previously mentioned, the VC dimension characterizes PAC learnability.
We start with a sample complexity upper bound that is linear in the VC dimension, due to \cite{vapnik1982estimation} and \cite{blumer1989learnability}.

\begin{theorem}[VC Dimension Sample Complexity Upper Bound]
\label{thm:vc-pac-upper-bound}
Let $\C$ be a concept class. 
Let $\A$ be a consistent learner for $\C$ using a hypothesis class $\H$ of VC dimension $\VC(\H)=d$.
Then $\A$ is a PAC-learning algorithm for $\C$ using $\H$ provided it is given an i.i.d. sample $S\sim D^m$ drawn from some $D$ and labelled with some $c\in\C$, where 
\begin{equation*}
m \geq \kappa_0\cdot\frac{1}{\epsilon}\left(d\log\frac{1}{\epsilon}+\log\frac{1}{\delta}\right)\enspace,
\end{equation*} 
for some universal constant $\kappa_0$.
\end{theorem}

We now have a sample complexity lower bound that is also linear in the VC dimension, due to \cite{blumer1989learnability} and \cite{ehrenfeucht1989general}. The proofs of both Theorems~\ref{thm:vc-pac-upper-bound} and~\ref{thm:vc-pac-lower-bound} appear in reference textbooks such as \citep{mohri2012foundations} and \citep{shalev2014understanding}.

\begin{theorem}[VC Dimension Sample Complexity Lower Bound]
\label{thm:vc-pac-lower-bound}
Let $\C$ be a concept class with VC dimension $d$. Then any PAC-learning algorithm for $\C$ requires $\Omega\left(\frac{d}{\epsilon}+\frac{1}{\epsilon}\log\frac{1}{\delta}\right)$ examples. 
\end{theorem}

While the bounds of Theorems~\ref{thm:vc-pac-upper-bound} and~\ref{thm:vc-pac-lower-bound} are tight up to a $\log\frac{1}{\epsilon}$, the breakthrough work of \cite{hanneke2016optimal} recently showed the existence of a specific learning algorithm that is optimal in the sense that its sample complexity matches that of Theorem~\ref{thm:vc-pac-lower-bound} up to constant factors, and thus avoids the $\log\frac{1}{\epsilon}$ dependence.

\subsection{Some Concept Classes and PAC Learning Algorithms}
\label{sec:pac-algos}

In this section, we introduce various concept classes that have been studied in the learning theory literature, along with PAC learning algorithms.
All the algorithms outlined below are consistent on a given training sample, given we are working in the realizable setting. 
A bound on the VC dimension of these concept classes directly gives sample complexity upper bounds as per Theorem~\ref{thm:vc-pac-upper-bound}.
We start with concept classes defined on the boolean hypercube $\X=\boolhc$. 

\paragraph*{Singletons.} For an input space $\X$, the class of singletons is the class of functions $\set{x\mapsto\mathbf{1}[x=x^*]\given x^*\in\X}$.

\paragraph*{Dictators.} The class of dictators on $\boolhc$ is the class of functions determined by a single bit, i.e., functions of the form $h(x)=x_i$ or $h(x)=\bar{x_i}$ for $i\in[n]$. 
Dictators are subsumed by conjunctions. 
Monotone dictators are dictators where negations are not allowed, i.e., functions of the form $h(x)=x_i$.

\paragraph*{Conjunctions.}
Conjunctions, which we denote $\Conj$, are perhaps one of the simplest non-trivial concept classes one can study on the boolean hypercube.
A conjunction $c$ over $\{0,1\}^n$ is a logical formula over a set of literals $l_1,\dots,l_k$ from $\set{x_1,\bar{x_1},\dots,x_n,\bar{x_n}}$, where, for $x\in\X_n$, $c(x)=\bigwedge_{i=1}^k l_i$. 
The \emph{length} of a conjunction $c$ is the number of literals in $c$.\footnote{We use the term \emph{length} for conjunctions that are not equivalent to the constant function 0.} 
For example, $c(x)=x_1\wedge\bar{x_2}\wedge{x_5}$ is a conjunction of length 3.
Monotone conjunctions are the subclass of conjunctions where negations are not allowed, i.e., all literals are of the form $l_i=x_j$ for some $j\in[n]$.
Note that this implies that monotone conjunctions do not include the constant function 0.

\begin{algorithm}
\caption{PAC-learning algorithm for conjunctions}
\label{alg:conj-pac}
\begin{algorithmic}
\Require $S_c \sim D^m$
\State $L \gets \set{x_1,\bar{x_1},\dots,x_n,\bar{x_n}}$
\State $h(x)=\bigwedge_{l\in L} l$ \Comment{$h=0$}
\For {$(x,c(x))\in S$}
	\If {$c(x)\neq h(x)$} \Comment{Only happens if $c(x)=1$}
		\State $L\gets L\setminus \set{l\in L\given l(x)=0}$
	\EndIf
\EndFor
\end{algorithmic}
\end{algorithm}
The standard PAC learning algorithm to learn conjunctions is as outlined in Algorithm~\ref{alg:conj-pac}.
We start with the constant hypothesis $h(x)=\bigwedge_{i\in I_h} (x_i \wedge  \bar{x_i})\equiv 0$, where $I_h=[n]$. 
To ensure consistency, for each example $x$ in the training sample, we remove a literal $l$ from $h$ if $c(x)=1$ and $l(x)=0$, as if $l$ is in the conjunction, $h$ must evaluate to $0$ on $x$.
After seeing all the examples in the training set $S$, the resulting hypothesis will thus be consistent on $S$.
Note that $\VC(\Conj_n)=n $ \citep{natschlager1996exact}. 
Finally, Algorithm~\ref{alg:conj-pac} can also be used for monotone conjunctions, but where the initial hypothesis is $h(x)=\bigwedge_{i\in[n]}x_i$. 

\paragraph*{CNF and DNF formulas.}
A formula $\varphi$ in the conjunctive normal form (CNF) is a conjunction of clauses, where each clause is itself a disjunction of literals. 
A $k$-CNF formula is a CNF formula where each clause contains at most $k$ literals.
For example, $\varphi= (x_1 \vee x_2) \wedge (\bar{x_3} \vee x_4) \wedge \bar{x_5}$ is a 2-CNF.
On the other hand, a DNF formula is a disjunction of clauses, where each clause is itself a conjunction of literals. 
A $k$-DNF is defined analogously to a $k$-CNF.

\paragraph*{Decision lists.}
Given a positive integer $k$, a $k$-decision list $f\in k$-{\dl} is a list $(K_1,v_1),\dots,(K_r,v_r)$ of pairs
where $K_j$ is a term in the set of all conjunctions of size at most $k$ with literals drawn from $\set{x_1,\bar{x_1},\dots,x_n,\bar{x_n}}$, $v_j$ is a value in $\set{0,1}$, and $K_r$ is $\mathtt{true}$.
The output $f(x)$ of $f$ on $x\in\boolhc$ is $v_j$, where $j$ is the least index such that the conjunction $K_j$ evaluates to $\mathtt{true}$ on $x$.
Decision lists subsume conjunctions. 
Indeed, a conjunction $c(x)=\bigwedge_{i=1}^k l_i$ can be expressed as the following 1-decision list: $(\neg l_1,0),\dots,(\neg l_k,0),(\mathtt{true},1)$.

The PAC-learning algorithm for decision lists, introduced by \cite{rivest1987learning}, is outlined in Algorithm~\ref{alg:dl-pac}.
The sample size $m$ is given by Theorem~\ref{thm:vc-pac-upper-bound} and an observation that the size of the class is $O\left(3^{\abs{C_{n,k}}}\abs{C_{n,k}}!\right)$, where  $C_{n,k}$ is the set of conjunctions of length at most $k$ on $n$ variables, giving a VC dimension bound of $O(n^{k}\log n)$.
Note that, as we consider $k$ to be a fixed constant, the sample complexity bound is polynomial in $n$ and the learning parameters.

\begin{algorithm}
\caption{PAC-learning algorithm for 1-decision lists from \cite{rivest1987learning}}
\label{alg:dl-pac}
\begin{algorithmic}
\Require $S\sim D^m$
\State $L:= \set{x_i,\bar{x_i}}_{i=1}^n$ \Comment{Set of all literals}
\State $h=\emptyset$ \Comment{Empty decision list}
\While {$S\neq\emptyset$}
	\If {$\exists b\in\set{0,1}$ s.t. $\forall (x,y)\in S,\;y=b$}
		\State $S\gets\emptyset$
		\State append $(\mathtt{true},b)$ to $h$
	\Else
		\For {$l\in L$ s.t. $ \exists (x,y)\in S $ s.t. $ l(x)=1$} \Comment{$l$ is true for some $x$}
			\If {$\exists b\in\set{0,1}$ s.t. $\forall (x,y)\in S\; \left(l(x)=1\Rightarrow y=b\right)$}
				\State append $(l,b)$ to $h$
				\State $S\gets S\setminus \set{(x,y)\in S\given l(x)=1}$
			\EndIf
		\EndFor
	\EndIf
\EndWhile
\end{algorithmic}
\end{algorithm}

Note that, while the algorithm above is for 1-decision lists, it is sufficient to only consider this case.
Indeed, if we are dealing with $k$-decision lists, we can draw our attention to the set $C_{n,k}$ of conjunctions of length at most $k$ on $n$ variables by defining the following injective map:
\begin{equation}
\label{eqn:conj-embedding}
\Phi:\boolhc\rightarrow\set{0,1}^{C_{n,k}}\enspace,
\end{equation}
where $\Phi(x)_{c_i}=\mathbf{1}[x\models c_i]$ for $c_i\in C_{n,k}$, i.e. whether $x$ satisfies clause $c_i$.
Now, any distribution $D$ on $\boolhc$ induces a well-defined distribution $D'$ on $\set{0,1}^{C_{n,k}}$.
Moreover, since $\abs{{C_{n,k}}}=O(n^k)$, an input $x\in\boolhc$ and a 1-decision $h$ on $\boolhc$ can respectively be transformed into $\Phi(x)\in\set{0,1}^{C_{n,k}}$ and a $k$-decision list $h'$ on $\set{0,1}^{C_{n,k}}$ in polynomial time, for a fixed $k$, and vice-versa in the case of going from $h'$ to $h$. 
It also follows that $\err_D(h,c)=\err_{D'}(h',c')$, where $c'$ is the $k$-decision list on $\set{0,1}^{C_{n,k}}$ induced by $c$.
Hence, an efficient learning algorithm for 1-decision lists can be used as a black box to efficiently learn $k$-decision lists.

Finally, the class of $k$-decision lists subsume $k$-CNF and $k$-DNF \citep{rivest1987learning}.

\paragraph*{Decision trees.}
A decision tree $T$ is a binary tree whose nodes are positive literals in $\set{x_1,\dots,x_n}$. For a given node with variable $x_i$, the edge to its left child node is labelled with 0 and the  edge to its right child node is labelled as 1, representing the value of the $x_i$ for a given instance $x\in\boolhc$. 
The leaves take label in $\set{0,1}$; a given $x\in\boolhc$ induces a path from the root to a leaf in $T$, which will give the label $T(x)$. 
Decision trees generalize 1-decision lists: a 1-decision list is a decision tree with each node having at most one child. 
Note that it is currently unknown whether polynomial-sized decision trees are PAC learnable.

\paragraph*{Parities.}
Parities are defined with respect to a subset $I\subseteq[n]$ of indices as $f_I(x)=\left(\sum_{i\in I}x_i\right) \bmod 2$, i.e. the output is whether adding the bits at indices in $S$ results in an odd or even sum. 
Learning parities amounts to learning the set $S$.  
Given a set of examples $(X,Y)\subseteq\boolhc\times\set{0,1}$, where each $(x,y)\in (X,Y)$ is a labelled example, finding this set is equivalent to finding a solution $a\in\set{0,1}^n$ to the system of linear equations $Xa=Y$ in the finite field $\mathbb{F}_2$.
The set $J:=\set{j\in[n]\given a_j=1}$ gives a hypothesis $f_J(x)=\left(\sum_{j\in J}x_j\right) \bmod 2$ consistent with the data. 
This can be done using Gaussian elimination, provided a solution exists (this is guaranteed by the realizability assumption). 
See \citep{helmbold1992learning,goldberg2006some} for details.

Note that, when working in $\boolhcfa$ instead of $\boolhc$, we can define the parity function as $f_I(x)=\prod_{i\in I}x_i$ instead.
This representation will be especially relevant in Section~\ref{sec:fourier-analysis} when we introduce Fourier analysis concepts.
%Working in $\boolhcfa$ rather than $\boolhc$ can be a useful trick to ease the analysis of a learning algorithm, or properties of a hypothesis class.

\paragraph*{Majorities.}
Similarly to parities, majorities are defined with respect to a set $I$ of indices, as follows: $\maj_I(x)=\mathbf{1}\left[\sum_{i\in I} x_i \geq \abs{I}/2\right]$. 
Again, when working in $\boolhcfa$ instead of $\boolhc$, majority functions are defined as $\maj_I(x)=\sgn\left(\sum_{i\in I}x_i\right)$.
Clearly, from the representations above, majorities are subsumed by linear classifiers, which are defined further below.

\paragraph*{Linear classifiers.}
The class of linear classifiers (also known as halfspaces and linear threshold functions)  on input spaces $\X=\boolhc$ or $\X=\R^n$ are defined as $\set{x\mapsto \sgn\left(w\cdot x+b\right)\given w\in\R^n, b\in\R}$, where the $w_i\in w$ are the \emph{weights} and $b$ is the \emph{bias}. 
When the instance space is the reals, we will denote the class as $\ltfreal$.
Moreover, we will denote by $\ltfbool^W$  the class of linear threshold functions on $\boolhc$ with integer weights  such that the sum of the absolute values of the weights and the bias is bounded above by $W$, and $W^+$ when the weights are positive. 
Finally, when the weights and the bias are binary, i.e., $w_i,b\in\{0,1\}$ for all $ i$, the class is called \emph{boolean threshold functions}.

The VC dimension of halfspaces is $n+1$. 
The upper bound of  $n+1$ can be shown by using Radon's theorem (any set of size $n+2$ in $\R^n$ can be partitioned into two subsets whose convex hulls intersect), and the lower bound can be obtained by showing that the set $\set{\mathbf{e}_i}_{i=1}^n\cup \mathbf{0}$ can be shattered.
The support vector machine (SVM) algorithm, or solving a system of linear inequalities with linear programming, can be used as a consistent learner for this concept class. 
Finally, the class of conjunctions is subsumed by linear classifiers: a conjunction $f(x)=\bigwedge_{i=1}^k l_i$ can be represented as the linear classifier $g(x)=\sgn\left(\sum_{i\in I^+}x_i-\sum_{i\in I^-}x_i-\abs{I}+1\right)$, where $I^+=\set{j\in [n]\given \exists i \st l_i=x_j}$ and $I^-=\set{j\in [n]\given  \exists i \st l_i=\bar{x_i}}$.

\subsection{Online Learning: The Mistake-Bound Model}
\label{sec:online}

In online learning, the learner is given access to examples \emph{sequentially}.
At each time step $t$, the learner receives an example $x_t$, predicts its label $\hat{y}_t$ using a given hypothesis class $\H$, receives the true label $y_t$ and can update its hypothesis, typically when $\hat{y}_t\neq y_t$. 
A fundamental distinction between the PAC-  and online-learning models is that, in the latter, there are usually no distributional assumptions on the data.\footnote{Some lines of work in online learning look at mild distributional assumptions in the learning problem in order to get better guarantees, but the basic mistake-bound online learning set-up assumes that examples (or more generally losses in the regret framework) can be given in an adversarial and adaptive manner.}
Thus, we need to evaluate the learner's performance with different benchmark than the error $\err_D(h)$ from the (offline) PAC setting.

%\subsubsection{The Mistake-Bound Model}

In the mistake-bound model, examples and their labels can be given in an adversarial fashion. 
The performance of the learner is evaluated with respect to the number of mistakes it makes compared to the ground truth; we again assume the \emph{realizability} of the learning problem, meaning that there is a target concept $c\in\C$ such that $c(x_t)=y_t$ for all $t$. 
Crucially, the target concept need not be chosen a priori: the only requirement is that, at every time $t$, there exists a concept $c\in\C$ that is consistent on the past sequence of points $(x_1,y_1), \dots, (x_t,y_t)$.
The goal of the learner is to learn the target exactly.

We now formally define the mistake-bound model of online learning.

\begin{definition}[Mistake Bound] For a given hypothesis class $\C$ and instance space $\X = \bigcup_n \X_n$, we say that an
algorithm $\A$ learns $\C$ with mistake bound $M$ if $\A$ makes at most
$M$ mistakes on any sequence of samples consistent with a concept $c \in \C$. 
\end{definition}

In the mistake bound model, we usually require that $M$ be polynomial in $n$ and size$(c)$.
A good example where this holds is the online learning algorithm for conjunctions, outlined in Algorithm~\ref{alg:conj-online}, which is immediately adapted from its PAC-learning counterpart.
Indeed, Algorithm~\ref{alg:conj-pac} only changes its hypothesis whenever it sees a positive example $(x,1)$ such that $h(x)=0$, and works through the sample sequentially.

\begin{algorithm}
\caption{PAC-learning algorithm for conjunctions, online version}
\label{alg:conj-online}
\begin{algorithmic}
\State $L \gets \set{x_1,\bar{x_1},\dots,x_n,\bar{x_n}}$
\State $h(x)=\bigwedge_{l\in L} l$ \Comment{$h=0$}
\For {$t=1,2,\dots$}
	\State Receive $x_t$
	\State Predict $h(x_t)$
	\State Receive true label $y_t$
	\If {$h(x_t)\neq y_t$} \Comment{Only happens if $y=1$}
		\State $L\gets L\setminus \set{l\in L\given l(x)=0}$
	\EndIf
\EndFor
\end{algorithmic}
\end{algorithm}

Unlike with conjunctions, the vast majority of PAC-learning algorithms cannot be so straightforwardly tailored to online learning, resulting in a rich literature on algorithms, benchmarks and guarantees specific to this setting.

One of the simplest general-purpose algorithms for online learning in the realizable mistake-bound model is the halving algorithm, outlined in Algorithm~\ref{alg:halving}.

\begin{algorithm}
\caption{Halving algorithm}
\begin{algorithmic}
\Require A hypothesis class $\mathcal{H}$
\For {$t=1,2,\dots$}
\State Receive example $x_t$
\State $V^{(b)}_t \gets \set{h\in V_t\given h(x_t)=b}$
\State $\hat{y_t}=\arg\max_b \abs{V_t^{(b)}}$ \Comment{Predict label acc. to a majority vote}
\State Receive true label $y_t$
\State $V_{t+1} \gets V^{(y_t)}_t$
\EndFor
\end{algorithmic}
\label{alg:halving}
\end{algorithm}

At each time step, the learner predicts the label of a new point according to the majority vote of the hypotheses consistent with the sequence of data seen so far, which is denoted as $V_t$.
It is easy to see that the halving algorithm will make at most $\log\abs{\H}$ mistakes: every time the learner makes a mistake on $(x_t,y_t)$, at least half of the hypotheses are not consistent with $(x_t,y_t)$, and are thus eliminated.
There are two significant disadvantages to this learning algorithm: (i) its computational complexity, with a runtime $\Omega(\abs{\H})$, as it requires iterating through the whole hypothesis class to get a majority vote and (ii) it can only be used on \emph{finite} concept classes.
Note that these drawbacks can be addressed by instead drawing a hypothesis at random from the version space, as argued in \citep{maass1991line}.
We will now address the second drawback and turn our attention to potentially \emph{infinite} concept classes.

We have seen that, in PAC learning, the VC dimension of a concept class characterizes its learnability, enabling learning guarantees for infinite concept classes that have finite VC dimension. 
One may wonder whether there exists an analogous complexity measure to the VC dimension when working in the mistake-bound model.
It turns out that such a measure exists in this setting: the Littlestone dimension, defined and proved to characterize online learnability in \citep{littlestone1988learning}.  
In order to define the Littlestone dimension, we must first define Littlestone trees.

\begin{definition}[Littlestone Tree]
A Littlestone tree for a hypothesis class $\mathcal{H}$ on $\X$ is a complete binary tree $T$ of depth $d$ whose internal nodes are instances $x\in\X$.
Each edge is labeled with $0$ or $1$ and corresponds to the potential labels of the parent node.
Each path from the root to a leaf must be consistent with some $h\in\mathcal{H}$, i.e. if $x_1,\dots,x_d$ with labelings $y_1,\dots,y_d$ is a path in $T$, there must exist $h\in\mathcal{H}$ such that $h(x_i)=y_i$ for all $i$. 
\end{definition}

We are now ready to define the Littlestone dimension.

\begin{definition}[Littlestone Dimension]
The Littlestone dimension of a hypothesis class $\mathcal{H}$, denoted $\Lit(\mathcal{H})$, is the largest depth $d$ of a Littlestone tree for $\mathcal{H}$. If no such $d$ exists then $\Lit(\mathcal{H})=\infty$.
\end{definition}

\paragraph*{Relationship to other complexity measures.}
Before showing that the Littlestone dimension characterizes online learnability in this setting, we will study some of its properties.
First, the Littlestone dimension is an upper bound on the VC dimension. 
Indeed, it is possible to convert any shattered set $X=\set{x_1,\dots,x_d}$ of size $d$ into a Littlestone tree of depth $d$, where the nodes at depth $i$ are all $x_i$ and every path from the root to a leaf corresponds to a dichotomy on $X$.

Moreover, from the definition of Littlestone trees, since each path from the root to a leaf of a tree is achievable by a distinct function $h\in\H$, the Littlestone dimension is bounded above by the logarithm of the size of $\H$.
We then have the following inequality for all $\H$
\begin{equation}
\label{eqn:vc-lit-log}
\VC(\H)\leq \Lit(\H) \leq \log(\abs{\H})\enspace.
\end{equation}

It can be shown that the gaps between the terms in Equation~\ref{eqn:vc-lit-log} can be arbitrarily large.
To show the gap between $\VC(\H)$ and $\Lit(\H)$, consider the set $\threshold=\bigcup_{a\in\R}\mathbf{1}[x\geq a]$  of threshold functions on $\R$.
The VC dimension of $\threshold$ is 1, as a set of one point can be shattered, but a set of two points $x_1<x_2\in\R$ cannot achieve the labelling $(1,0)$. 
However, its Littlestone dimension is infinite: consider the interval $[0,1]$.
At each depth $i$ of the Littlestone tree, the set of nodes from left to right is $\set{\frac{j+1}{2^i}}_{j=0}^{2^{i-1}}$, and the labelling of all the left edges is $1$ and $0$ for right edges.
For a given depth $i$, a path $p$ from the root to node $x_{i,j}:=\frac{j+1}{2^i}$ for some $j\in \set{0,1,\dots,2^{i-1}}$ (including $x_{i,j}$'s label) is thus consistent with the threshold function $\mathbf{1}[x\geq x^*]$ where $x^*$ is the deepest node in $p$ (inclusive of $x_{i,j}$) that is positively labelled.
This infinite gap between the VC and Littlestone dimensions clearly illustrates that online and offline (PAC) learnability are fundamentally different from each other, as some concept classes are PAC learnable but not online learnable.
To show the other arbitrary large gap between $\Lit(\H) $ and $\log(\abs{\H})$, consider the singletons on $\R$, i.e. the class of functions $\bigcup_{a\in\R}\mathbf{1}[x = a]$.
While the class is infinite, any Littlestone tree, which must be complete, has depth 1, as each hypothesis in the class labels a unique point (the target $a$) positively. Thus $\Lit(\H)=1$.

We now show that the Littlestone dimension lower bounds the number of mistakes any online learning makes.

\begin{theorem}{\citep{littlestone1988learning}}
Any online learning algorithm for $\C$ has mistake bound $M\geq \Lit(\C)$.
\end{theorem}
\begin{proof}
Let $\A$ be any online learning algorithm for $\C$.
Let $T$ be a Littlestone tree of depth $\Lit(\C)$ for $\C$. 
Clearly, an adversary can force $\A$ to make $\Lit(\C)$ mistakes by sequentially and adaptively choosing a path in $T$ in function of $\A$'s predictions.
\end{proof}

As previously suggested, the Littlestone dimension can also upper bound the number of mistakes made by an online learning algorithm. 
This bound is achieved for arbitrary concept classes with finite Littlestone dimension by the Standard Optimal Algorithm from \citet{littlestone1988learning}, outlined in Algorithm~\ref{alg:soa}.

\begin{algorithm}
\caption{Standard Optimal Algorithm from \cite{littlestone1988learning}}
\begin{algorithmic}
\Require A hypothesis class $\mathcal{\C}$
\For {$t=1,2,\dots$}
\State Receive example $x_t$
\State $V^{(b)}_t \gets \set{h\in V_t\given h(x_t)=b}$
\State $\hat{y_t}=\arg\max_b \Lit(V^{(b)}_t)$ 
\State Receive true label $y_t$
\State $V_{t+1} \gets V^{(y_t)}_t$
\EndFor
\end{algorithmic}
\label{alg:soa}
\end{algorithm}

The SOA works in a similar fashion as the halving algorithm, only considering at time $t$ the version space $V_t$ of hypotheses that are consistent with the sequence of examples so far.
However, instead of taking the majority vote, the algorithm predicts the label $\hat{y_t}$ of a new point according to the subclass (w.r.t. a label prediction $b\in\set{0,1}$) with larger Littlestone dimension.
The theorem below completes the proof that the Littlestone dimension characterizes online learnability. 

\begin{theorem}
The Standard Optimal Algorithm from \citet{littlestone1988learning} makes at most $\Lit(\C)$ mistakes in the mistake-bound model.
\end{theorem}

\begin{proof}
We will show that, at every mistake, the Littlestone dimension of the subclass $V_t$ decreases by at least 1 after receiving the true label $y_t$.

Suppose that, at time $t$, $y_t= \arg\min_b \Lit(V^{(b)}_t)$. 
Note that $V_{t+1} = V^{(y_t)}_t$.
Now, consider any two Littlestone trees  $T_{y_t}$ and  $T_{\hat{y}_t}$ of maximal depths for $V^{(y_t)}_t$ and $V^{(\hat{y})}_t$, respectively. 
By definition, neither tree can contain $x_t$, so it is possible to construct a Littlestone tree $T$ for $V_t$ of depth $\min_b \Lit(V^{(b)}_t)+1$ (recall that $T$ must be complete).
Then $\Lit(V_t)\geq\Lit(V^{({y_t})}_t)+1=\Lit(V_{t+1})+1 $, as required.\footnote{Note that the Littlestone dimension does not necessarily decrease when $y_t=\hat{y}_t$, as we could have $V_t=V_t^{y_y}$.}
\end{proof}

While the SOA has an optimal mistake bound and is defined for arbitrary concept classes, it remains highly inefficient in general, as computing the Littlestone dimension of the concept subclasses could be very costly. 
For the remainder of this section, we will consider online learning algorithms for specific concept classes in order to circumvent some of these issues.\footnote{We will discuss the algorithms and their mistake bounds here, but we refer the reader to the references for the algorithms themselves and their analysis.}

The first algorithm we will look at is Winnow, which is for linear threshold functions with bounded weights in the boolean hypercube. % and is outlined in Algorithm~\ref{alg:winnow}.
This algorithm and its analysis are due to \cite{littlestone1988learning}.

%\begin{algorithm}
%\caption{Winnow}
%\begin{algorithmic}
%%\Require Learning rate $\eta>0$
%\State $\w_1\gets \mathbf{1}$ \Comment{Initialize all $n$ weights to 1}
%\For {$t=1,2,\dots$}
%	\State Receive example $x_t\in\boolhc$
%	\State $\hat{y_t}\gets\sgn(\w_t\cdot x_t\geq b)$
%	\State Receive true label $y_t$
%	\If {$\hat{y_t}=1$ and $y_t=0$}
%		\For $j=1,\dots,n$
%			\If {$x_{t,j}=1$}
%				\State $w_{t+1,j}=0$
%			\Else
%				\State $w_{t+1,j}=w_{t,j}$
%			\EndIf
%		\EndFor
%	\ElsIf {$\hat{y_t}=0$ and $y_t=1$}
%		\For $j=1,\dots,n$
%			\If {$x_{t,j}=1$}
%				\State $w_{t+1,j}=2w_{t,j}$
%			\Else
%				\State $w_{t+1,j}=w_{t,j}$
%			\EndIf
%		\EndFor
%	\Else
%		\State $\w_{t+1}=\w_t$
%	\EndIf
%\EndFor
%\end{algorithmic}
%\label{alg:winnow}
%\end{algorithm}
%
We now recall the mistake upper bound for Winnow in the special case of $\ltfbool^{W+}$, where the weights are positive integers.\footnote{See \url{https://www.cs.utexas.edu/~klivans/05f7.pdf} 
for a full derivation.}

\begin{theorem}[Winnow Mistake Bound]
\label{thm:winnow-mistake}
The Winnow algorithm for learning the class $\ltfbool^{W+}$  makes at most $O(W^2\log(n))$ mistakes.
\end{theorem}

We now look at the perceptron algorithm, which first appeared in \cite{rosenblatt1958perceptron}, and whose first proofs of convergence were shown in \cite{block1962perceptron} and \cite{novikoff1963convergence}.

%\begin{algorithm}
%\label{alg:perceptron}
%\caption{Perceptron Algorithm}
%\begin{algorithmic}
%\Require Learning rate $\eta>0$
%\State $\w_1\gets \w_0$
%\For {$t=1,\dots,T$}
%	\State Receive example $x_t$
%	\State $\hat{y_t}\gets\sgn(\w_t\cdot x_t)$
%	\State Receive true label $y_t$
%	\If {$\hat{y_t}\neq y_t$}
%		\State $\w_{t+1}\gets\w_t+\eta y_t x_t$
%	\Else
%		\State $\w_{t+1}\gets\w_t$
%	\EndIf
%\EndFor
%\end{algorithmic}
%\end{algorithm}

While it is not possible to have a mistake bound for linear classifiers in $\R^n$, as the Littlestone dimension is infinite, requiring a \emph{margin} on the data ensures a finite mistake bound with the perceptron algorithm, as stated below.

\begin{theorem}[Mistake Bound for Perceptron, Margin Condition; Theorem~7.8 in \cite{mohri2012foundations}]
\label{thm:mistake-bound-perceptron}
Let $\x_1, \dots, \x_T \in \R^n$ be a sequence of $T$ points with $\norm{\x_t}\leq r $ for all $1\leq t \leq T$ for some $r>0$.
Assume that there exists $\gamma>0$ and $\vv\in\R^n$ such that for all $1\leq t \leq T$, $\gamma \leq \frac{y_t(\vv\cdot \x_t)}{\norm{\vv}}$.
Then, the number of updates made by the Perceptron algorithm when
processing $\x_1, \dots, \x_T$  is bounded by $r^2/\gamma^2$.
\end{theorem}

%\todo{Review this section in light with potential changes to our paper \emph{When are Local Queries Useful for Robust Learning?}. If algorithms are only used in a black-box manner, remove the code.}
%
%\question{Should I explain a bit more how Winnow and Perceptron work?}
%
%\textbf{Ben:} Either give the code + explanation or don't give the code.
%
%\textbf{Pascale:} Could I simply put the code (and explanation) in the appendix? 

%\subsubsection{The Agnostic Setting: Minimizing the Regret}
%
%\question{There are more things I could say about online learning, e.g., the regret/expert advice framework, how the Littlestone dimension shows up there, other learning algorithms like (randomized) weighted majority, and the role randomization can play in regret bounds. However, we're not using any of that. Is it sufficient to add a remark to the previous section, or should I go into more depth?}

\todo{Add some background on variant of online learning where a label is set before a learner's prediction.}

\subsection{Learning with Membership and Equivalence Queries}
\label{sec:active-learning}

So far, we have studied models where the learner does not have any control over the data it gets: in the PAC setting, labelled instances are received i.i.d. from the random example oracle, and in the online setting, the new points can be given adversarially.
In this sense the learner is quite passive during the learning process.
We will now turn our attention towards learning models where the learner is more active, and, in addition to receiving random examples, can make queries to an oracle, also sometimes referred to as teacher.

For simplicity, we will for now assume that there is no distribution underlying the data.
Hence, similarly to the mistake-bound model of online learning, the goal is to learn the target concept \emph{exactly} on the instance space.
We will start by defining two different types of queries: membership and equivalence queries.

\begin{definition}
A \emph{membership oracle} $\MQ(c)$ defined for a concept $c\in\C$ returns the value $c(x)$ when queried with an instance $x\in\X$.
\end{definition}

The terminology refers to the fact that $\C$ is a class of boolean functions, which can be interpreted as a subsets of $\X$. 
Then, a membership query returns whether an instance $x$ is in the target subset of $\X$. 
In the case of real-valued functions, a \emph{value oracle} might be a more appropriate term.

\begin{definition}
An \emph{equivalence query oracle} $\EQ(c)$ defined for a target concept $c\in\C$ takes as input a representation of a hypothesis $h$ and returns whether or not $h$ agrees with $c$ on the input space $\X$.
If $h\neq c$ on $\X$, $\EQ(c)$ also returns an instance $x\in\X$, called a \emph{counterexample}, such that $h(x)\neq c(x)$.
\end{definition}

With these two types of queries, we will now present the exact learning model for concept classes in this setting, where the goal is to learn a hypothesis $h$ such that for all $x\in\X$, $h(x)=c(x)$.
We formally define this model below, where we will assume that the learning algorithm is deterministic.

\begin{definition}
A concept class $\C$ is \emph{efficiently exactly learnable} using membership and equivalence queries if there exists a polynomially-evaluatable hypothesis class $\H$, a learning algorithm $\A$ and a polynomial $p(\cdot,\cdot)$ such that for all $n\geq1$, $c\in\C$, whenever $\A$ is given access to the $\MQ(c)$ and $\EQ(c)$ oracles, it halts in time $p(n,\text{size}(c))$ and outputs some $h\in\H_n$ such that $h(x)=c(x)$ for all instances $x\in\X$. Furthermore, every query made to $\EQ(c)$ by $\A$ must made with some $h\in\H_n$.
\end{definition}

The exact learning model with access to $\MQ$ and $\EQ$ has a long history, particularly in automata theory, where the seminal work of \citet{angluin1987learning} presented an exact learning algorithm, called $L^*$, to exactly learn deterministic finite automata.

Before going further, a few remarks are in order. 
First, the efficiency in this definition is with respect to the computational complexity of the problem. 
This entails requiring statistical efficiency as well, in the sense that the number of queries to the $\MQ$ and $\EQ$  oracles be also polynomial in $n$ and size$(c)$. 

Second, it may seem that having access to an equivalence oracle is an impractical requirement. 
After all, while it makes sense to consider membership oracles, as they can often be simulated by human ``experts'' (e.g., captioning done by internet users), it could perhaps be unrealistic to expect humans or automated systems to simulate the equivalence oracle in practice.
However, the following result shows that, if the exact learning requirement can be relaxed to PAC learning, i.e., allowing for accuracy and confidence parameters, then one can simply work in the $\EX+\MQ$  learning model, and forgo equivalence queries.

\begin{theorem}
\label{thm:exact-pac-mq}
Let $\C$ be exactly efficiently learnable using membership and equivalence queries.
Then $\C$ is efficiently PAC-learnable using random examples and membership queries.
\end{theorem}

The proof, omitted for brevity, relies on the fact that it is possible to simulate (with sufficient accuracy) the $\EQ$ oracle with access to random examples.

Third, we have assumed that the learning algorithm is deterministic.
It would be possible to accommodate randomized learning algorithms with the addition of a confidence parameter $\delta$ as in PAC learning. 
In this case, the probability of failure would not come from the randomness in sampling the data, but rather from the fact that we are working with an algorithm with internal randomization, which could result in computational gains.

Now, note that it is possible to efficiently exactly learn conjunctions in the $\MQ+\EQ$ model (just by using the $\EQ$ oracle). 
We simply need to use the online learning version of the algorithm  (Algorithm~\ref{alg:conj-online}) and, instead of receiving an instance and predicting its label, the learner gives the hypothesis $h$ to $\EQ(c)$ and receives a counterexample if $h\neq c$. 
The number of calls to $\EQ$ is upper bounded by the mistake bound (the reasoning is the same as in the online setting).

A more interesting class of functions to study is the class $\mathsf{MONOTONE\text{-}DNF}$, i.e., functions of the form $T_1\vee\dots\vee T_r$ where each $T_i$ is a monotone conjunction $\bigwedge_{j\in S_i}x_j$.
It is not known whether $\mathsf{MONOTONE\text{-}DNF}$ is PAC learnable. 
However, it can be shown that this class can be exactly learned in the $\MQ+\EQ$ model (and thus is PAC learnable when the learner has additional access to $\MQ$ by Theorem~\ref{thm:exact-pac-mq}).

We finish this section by formally introducing \emph{local} membership queries (LMQ), which were mentioned in Chapter~\ref{chap:lit-review}. They were introduced by \cite{awasthi2013learning} and shown to circumvent some impossibility results in the standard PAC setting (or impossibility conjectures).
Here, given a sample $S$ drawn from the example oracle $\EX(c,D)$, a membership query for a point $x$ is $\lambda$-\emph{local} if there exists $x'\in S$ such that $x\in B_\lambda(x')$, i.e., an algorithm can only query the label of points within distance $\lambda$ of the training sample.

\begin{definition}[PAC Learning with $\lambda$-$\LMQ$ ]
\label{def:lmq}
Let $\X$ be the instance space equipped with a metric $d$, $\C$ a concept class over $\X$, and $\D$ a class of distributions over $\X$. We say that $\C$ is $\rho$-robustly learnable using $\lambda$-local membership queries with respect to $\D$ if there exists a learning algorithm $\A$ such that for every $\epsilon > 0$, $\delta > 0$, for every distribution $D\in\D$ and every target concept $c\in\C$, the following hold:
\begin{enumerate}
\item $\A$ draws a sample $S$ of size $m = \poly(n, 1/\delta, 1/\epsilon,\text{size}(c))$ using the example oracle $\EX (c, D)$
\item Each query $x'$ made by $\A$ to the $\LMQ$ oracle is $\lambda$-local with respect to some example $x \in S$
\item  $\A$ outputs a hypothesis $h$ that satisfies $\err_D(h,c)\leq \epsilon$ with probability at least $1-\delta$ 
\item The running time of $\A$ (hence also the number of oracle accesses) is polynomial in $n$, $1/\epsilon$, $1/\delta$, $\text{size}(c)$ and the output hypothesis $h$ is polynomially evaluable.
\end{enumerate}
\end{definition}

We conclude this section by remarking that learnability in the above setting is with respect to a family $\D$ of distributions, rather than the distribution-free setting of PAC learning. 
This is because LMQs have mostly been used in the literature for learning problems which require distributional assumptions. 

\section{Probability Theory}
\label{sec:prob-theory}

In this section, we first present log-Lipschitz distributions, a family of distributions that will be studied throughout the text. 
We then introduce martingales, which are sequences of random variables satisfying certain properties. 
They can be used to give concentration bounds for random variables which are not necessarily independent, such as bits in instances from $\boolhc$ sampled from log-Lipschitz distributions.

\subsection{Log-Lipschitz Distributions}

While it is natural to consider product distributions on the input space $\boolhc$, such as the uniform distribution,  independence among the values of the bits of an input is seldom a reasonable assumption to make in practice (e.g., two features may be correlated).
By working with log-Lipschitz distributions, we can still operate in a regime where some distributional assumptions hold, but where the requirements are less stringent than for product distributions.
A distribution is log-Lipschitz if the logarithm of the density function is $\log(\alpha)$-Lipschitz with respect to the Hamming distance:

\begin{definition}
\label{def:log-lipschitz}
A distribution  $D$ on $\boolhc$ is said to be $\alpha$-$\log$-Lipschitz if 
for all input points $x,x'\in \boolhc$, if $d_H(x,x')=1$, then $|\log(D(x))-\log(D(x'))|\leq\log(\alpha)$.
\end{definition}

The intuition behind $\log$-Lipschitz distributions is that points 
that are close to each other must not have frequencies that greatly 
differ from each other. 
From the definition, it is straightforward to see that if two points $x,x'$ differ only by one bit, then $D(x)/D(x')\leq \alpha$.
Thus, neighbouring points in $\boolhc$ have probability masses that differ
by at most a multiplicative factor of $\alpha$.  
This implies that the decay of probability mass along a chain of neighbouring points is at most exponential. 
Not having sharp changes to the underlying distribution is a very natural assumption, and weaker than many other distributional assumptions in the literature.
Again note that features are allowed a small dependency between each other and, by construction,
log-Lipschitz distributions are supported on the whole input space. 
Log-Lipschitz distributions have been studied in \cite{awasthi2013learning}, 
and their variants in \cite{feldman2012data,koltun2007approximately}.

\paragraph*{Examples of log-Lipschitz distributions.} The uniform distribution is $\log$-Lipschitz
with parameter $\alpha=1$. Another example of $\log$-Lipschitz 
distributions is the class of product distributions where the probability
of drawing a $0$ (or equivalently a $1$) at index $i$ is in the interval 
$\left[\frac{1}{1+\alpha},\frac{\alpha}{1+\alpha}\right]$. 
For an example where some of the bits are not independent, let $\eta\in (1/2,1)$ and let the input space be  $\boolhc$ again. 
We first draw $x_1$ uniformly at random (u.a.r.), and then let $x_2$ be $x_1$ with probability $\eta$ and $\bar{x_1}$ with probability $1-\eta$.
The remaining bits are drawn u.a.r.
Then, this distributions is $\frac{\eta}{1-\eta}$-log-Lipschitz.

\paragraph*{Properties.}
Log-Lipschitz distributions have the following useful properties, which 
we will often refer to in our proofs.
\begin{lemma}
\label{lemma:log-lips-facts}
Let $D$ be an $\alpha$-$\log$-Lipschitz distribution over $\boolhc$. 
Then the following hold:
\begin{enumerate}
\item\label{test} For $b\in\{0,1\}$, $\frac{1}{1+\alpha}\leq \Prob{x\sim D}{x_i=b}\leq\frac{\alpha}{1+\alpha}$.
\item For any $S\subseteq[n]$, the marginal distribution $D_{\bar{S}}$ is $\alpha$-$\log$-Lipschitz, where $D_{\bar{S}}(y)=\sum_{y'\in\{0,1\}^S} D(yy')$.
\item For any $S\subseteq[n]$ and for any property $\pi_S$ that only depends on variables $x_S$, the marginal with respect to $\bar{S}$ of the conditional distribution $(D|\pi_S)_{\bar{S}}$ is $\alpha$-$\log$-Lipschitz.
\item For any $S\subseteq[n]$ and $b_S\in\{0,1\}^S$,  we have that $\left(\frac{1}{1+\alpha}\right)^{|S|}\leq \Prob{x\sim D}{x_i=b}\leq\left(\frac{\alpha}{1+\alpha}\right)^{|S|}$.
\end{enumerate}
\end{lemma}

\begin{proof}

To prove (1), fix $i\in[n]$ and $b\in\{0,1\}$ and denote by $x^{\oplus i}$ the result of flipping the $i$-th bit of $x$. Note that

\begin{align*}
\Prob{x\sim D}{x_i=b}
&=\sum_{\substack{z\in\boolhc :\\ z_i=b}} D(z) \\
&=\sum_{\substack{z\in\boolhc :\\ z_i=b}} \frac{D(z)}{D(z^{\oplus i})}D(z^{\oplus i})\\
&\leq\alpha\sum_{\substack{z\in\boolhc :\\ z_i=b}} D(z^{\oplus i})\\
&=\alpha\Prob{x\sim D}{x_i\neq b}
\enspace.
\end{align*}
The result follows from solving for $\Prob{x\sim D}{x_i=b}$.

Without loss of generality, let $\bar{S}=\{1,\dots,k\}$ for some $k\leq n$. 
Let $x,x'\in \{0,1\}^{\bar{S}}$ with $d_H(x,x')=~1$.

To prove (2), let $D_{\bar{S}}$ be the marginal distribution. Then,
\begin{equation*}
D_{\bar{S}}(x)
=\sum_{y\in\{0,1\}^S}D(xy)
=\sum_{y\in\{0,1\}^S}\frac{D(xy)}{D(x'y)}D(x'y)
\leq \alpha \sum_{y\in\{0,1\}^S} D(x'y)
= \alpha D_{\bar{S}}(x')
\enspace.
\end{equation*}

To prove (3), denote by $X_{\pi_S}$ the set of points in $\{0,1\}^S$ satisfying property $\pi_S$, and by $xX_{\pi_S}$ the set of inputs of the form $xy$, where $y\in X_{\pi_S}$.  
By a slight abuse of notation, let $D(X_{\pi_S})$ be the probability of drawing a point in $\boolhc$ that satisfies $\pi_S$.
Then,
\begin{equation*}
D(xX_{\pi_S})
= \sum_{y\in X_{\pi_S}} D(xy)
= \sum_{y\in X_{\pi_S}} \frac{D(xy)}{D(x'y)} D(x'y)
\leq  \alpha \sum_{y\in X_{\pi_S}} D(x'y )
= \alpha D(x'X_{\pi_S})
\enspace.
\end{equation*}
We can use the above to show that
\begin{equation*}
(D|\pi_S)_{\bar{S}} (x)
= \frac{D(xX_{\pi_S}) }{D(x'X_{\pi_S})} \frac{D(x'X_{\pi_S})}{D(X_{\pi_S})}
\leq \alpha (D|\pi_S)_{\bar{S}} (x')\enspace.
\end{equation*}

Finally, (4) is a corollary of (1)--(3).

\end{proof}

\subsection{Concentration Bounds and Martingales}

Let us start with some notation and probability theory basics.
A random variable $X$ on a sample space $\Omega$, which represents the set of all possible outcomes, is a real-valued measurable function $X:\Omega\rightarrow\R$.
Turning our attention to discrete random variables, the conditional probability of $X$ given a random variable $Y$ is defined as $$\Prob{}{X=x\given Y=y}=\frac{\Prob{}{X=x \wedge Y=y}}{\Prob{}{Y=y}}\enspace.$$
We can now use this to define the conditional expectation as $\eval{}{X\given Y=y}=\sum_x \Prob{}{X=x\given Y=y}\enspace$, where $\Prob{}{ Y=y}$ is assumed to be non-zero.
While these are defined for discrete random variables, they can  be extended to continuous random variables.
Moreover, note that the conditional expectation $\eval{}{X\given Y}$ is itself a random variable.

\paragraph*{Useful facts.} The law of total expectation, which in full generality states that $\eval{}{X}=\eval{}{\eval{}{X\given Y}}$, can also be formulated as $$\eval{}{X}=\sum_y \Prob{}{Y=y}\eval{}{X\given Y=y}\enspace.$$
Moreover, the linearity of expectation also holds under conditioning, i.e., 
$$\eval{}{X+Z\given Y}=\eval{}{X\given Y}+\eval{}{Z\given Y}\enspace.$$

Concentration inequalities and tail bounds are key tools to provide guarantees in machine learning. 
Among the most commonly used and well-known bounds are the Hoeffding inequality and the Chernoff bound, stated below.

\begin{theorem}[\cite{hoeffding1963probability}]
Let $X_1,\dots,X_n$ be $n$ independent random variables such that $X_i:\Omega\rightarrow[0,1]$. Denote by $\bar{X}=\frac{1}{n}\sum_{i=1}^n X_i$ their arithmetic mean and let $\mu=\eval{}{\bar{X}}$.
Then, for every $t\geq 0$,
\begin{equation}
\Prob{}{\abs{\bar{X}-\mu}\geq t}\leq 2\exp\left(-2mt^2\right)
\enspace.
\end{equation}
\end{theorem}

The Chernoff bound is the multiplicative form of Hoeffding's inequality.

\begin{theorem}[\cite{chernoff1952measure}]
Let $X_1,\dots,X_n$ be $n$ independent random variables such that $X_i:\Omega\rightarrow\classes$. Denote their sum by $\bar{X}=\sum_{i=1}^n X_i$  and let $\mu=\eval{}{\bar{X}}$. 
Then, for every $0\leq\delta\leq 1$,
\begin{align*}
\Prob{}{\bar{X}\leq(1-\delta)\mu}&\leq \exp\left(-\delta^2\mu/2\right)
\enspace,\\
\Prob{}{\bar{X}\geq(1+\delta)\mu}&\leq \exp\left(-\delta^2\mu/3\right)
\enspace.
\end{align*}
\end{theorem}

Both results rely on the \emph{independence} of the random variables, which is not always a reasonable assumption to make.
Which tools are available to us when independence cannot be guaranteed?

Martingales offer us the opportunity to weaken assumptions on the random variables, which are allowed to depend on each other.
To this end, we consider a \emph{sequence} of random variables, where the value of a given random variable is a function of the preceding ones.
Additional requirements on their expectation and conditional expectation are given in order to get meaningful mathematical objects to study.

\begin{definition}
A \emph{martingale} is a sequence of random variables $X_0,X_1,\dots$ of bounded expectation, i.e., $\eval{}{\abs{X_i}}<\infty$, for all $i$, such that, for every $i\geq 0$, $\eval{}{X_{i+1} \given X_0,\dots, X_i}= X_i$. 
More generally, a sequence of random variables $Z_0,Z_1,\dots$ is a martingale with respect to the sequence $X_0,X_1,\dots$  if for all $n\geq0$
\begin{enumerate}
\item[(i)] $Z_n$ is a function of $X_0,\dots,X_n$,
\item[(ii)] $\eval{}{\abs{Z_{n}}}<\infty$,
\item[(iii)] $\eval{}{Z_{n+1} \given X_0,\dots, X_n}= Z_n$. 
\end{enumerate}
When $\eval{}{Z_{n+1} \given X_0,\dots, X_n}\leq Z_n$ the sequence is a supermartingale, and when $\eval{}{Z_{n+1} \given X_0,\dots, X_n}\geq Z_n$, the sequence is a submartingale. 
\end{definition}

\begin{example}[Gambler's fortune.]
Suppose a gambler plays a sequence of fair games, meaning that $\eval{}{X_i \given X_0,\dots,X_{i-1}}=0$, where $X_i$ is the gains (or losses) incurred at every game $i$. 
We are interested in the cumulative gains $Z_n=\sum_{i=0}^n X_i$, the gambler's total gains at the end of the $n$-th game. 
If $\eval{}{\abs{X_{i}}}<\infty$ for all games $i$, then $\eval{}{\abs{Z_{n}}}<\infty$ as well. 
Moreover,
$$\eval{}{Z_{n+1} \given X_0,\dots,X_{n}}=\eval{}{X_{n+1} \given X_0,\dots,X_{n}}+\eval{}{Z_{n} \given X_0,\dots,X_{n}}=Z_n\enspace,$$
together implying that the sequence $Z_0,Z_1,\dots$. is a martingale.
Note that the assumptions are quite permissive: the gambler's strategy can fully depend on the history of the previous games.
\end{example}

Now, when bounding the difference between two consecutive random variables, one can obtain a  powerful concentration bound, known as the Azuma-Hoeffding inequality.

\begin{theorem}[Azuma-Hoeffding Inequality] 
Let $X_0,\dots,X_n$ be (super)martingales such that $\abs{X_i-X_{i+1}}\leq c_i$.
Then for any $\lambda>0$:
\begin{align*}
\Prob{}{X_n-X_0\geq \lambda}\leq \exp\left(-\frac{\lambda^2}{2\sum_{i=1}^n c_i^2}\right)\enspace,\\
\Prob{}{X_n-X_0\leq -\lambda}\leq \exp\left(-\frac{\lambda^2}{2\sum_{i=1}^n c_i^2}\right)\enspace.
\end{align*}
\end{theorem}

Note that this inequality is similar in form to the Chernoff bounds, though the gain in generality results in a weaker bound. 

%\paragraph*{Application to Gambler's ruin.} \question{Should I continue this example? We can show that the gains are concentrated $O(\sqrt{n\log n})$ around 0.}
%
%\textbf{Ben:} Not crucial, but could be useful.
%
As previously mentioned, martingales and the Azuma-Hoeffding inequality will be valuable when considering log-Lipschitz distributions, where the values of the bits in an instance are not assumed to be independent.

\section{Fourier Analysis}
\label{sec:fourier-analysis}

In this section, we  introduce basic Fourier analysis concepts for boolean functions, i.e., functions of the form $f:\boolhc\rightarrow\set{0,1}$, which comprise a large part of the functions studied in this thesis.
As previously mentioned, it is also possible to look at functions of the form $f:\set{-1,1}^n\rightarrow\set{-1,1}$. 
In fact, this is what we will do in this section as it eases analyses and notation.
For various reasons, the encoding $\varphi:\classes\rightarrow\classesfa$ satisfying $\varphi(0)=1$ and $\varphi(1)=-1$ for both the input and output spaces is usually preferred.
In general, one can also consider real-valued functions $f:\boolhcfa\rightarrow\R$. 
The type of functions for a given theorem will be featured in the theorem statements, unless it is clear from the context.
A thorough introduction to the Fourier analysis of boolean functions, as well as the proofs omitted in this section, can be found in the textbook by \cite{odonnell2014analysis}.

Fourier analysis relies on considering functions' \emph{Fourier expansion}: their representation as  real multilinear polynomials. 
We start with some notation. 
For a subset $S\subseteq[n]$, we denote by $\chi_S(x)$ the monomial $\prod_{i\in S}x_i$ (with $\chi_S({\emptyset})=1$ by convention) corresponding to the set $S$.
As stated below, the Fourier expansion of a given function is unique.

\begin{theorem}[Fourier Expansion Theorem]
Every function $f:\boolhcfa\rightarrow\R$ can be uniquely expressed as the following multilinear polynomial
\begin{equation}
\label{eqn:fourier-exp}
f(x)=\sum_{S\subseteq[n]}\widehat{f}(S)\chi_S(x)
\enspace,
\end{equation}
which is called the \emph{Fourier expansion} of $f$.
A term $\widehat{f}(S)\in \R$ is called the \emph{Fourier coefficient} of $f$ on $S$, and the collection of the Fourier coefficients is called the \emph{Fourier spectrum} of $f$.
\end{theorem}

Note that each $\chi_S(x)$ is a parity function defined for a subset $S$ of indices, which, together with the theorem above, imply that any function can be represented as a linear combination of parity functions.
In fact, the set of all such parity functions forms an orthonormal basis for the set of all functions $f:\boolhcfa\rightarrow\R$, and the Fourier coefficients of $f$ are given by 
\begin{equation}
\widehat{f}(S)=\left\langle f,\chi_S\right\rangle:=\eval{x\sim\boolhcfa}{f(x) \chi_S(x)} = \frac{1}{2^n}\sum_{S\subseteq [n]}f(x)\chi_S(x)
\enspace,
\end{equation}
where $x\sim\boolhcfa$ means that $x$ is chosen u.a.r. from $\boolhcfa$.

We now give a few more properties of boolean functions and results that will be used later in Chapter~\ref{chap:rob-thresholds}.
We start by defining the influence of a coordinate.

\begin{definition}
The \emph{influence} of coordinate $i\in[n]$  on $f:\boolhcfa\rightarrow\classesfa$ is defined as 
\begin{equation}
\Inf_i[f]=\Prob{x\sim\boolhcfa}{f(x)\neq f(x^{\oplus i})}
\enspace,
\end{equation}
where $x^{\oplus i}$ denotes the result of flipping the $i$-th bit of $x$.
For $x\in\boolhcfa$, we say that $i$ is \emph{pivotal} on $x$ if $f(x)\neq (x^{\oplus i})$.
\end{definition}

We have the following statement giving an explicit formula for the influence of a bit as a function of the Fourier spectrum. 

\begin{theorem}
For $f:\boolhcfa\rightarrow\R$ and $i\in[n]$, $\Inf_i[f]=\sum_{S\ni i}\widehat{f}(S)^2$. 
\end{theorem}

We will later study majority functions. 
These functions have an important property: monotonicity, which is defined below.

\begin{definition}
We say that $f:\boolhcfa\rightarrow\classesfa$ is \emph{monotone} if $f(x)\leq f(x')$ for all $x\leq x'$ (coordinate-wise).
\end{definition}

Finally, the following proposition states that the influence of a bit $i$ on a monotone function is equal to the Fourier coefficient of the singleton $i$.

\begin{proposition}
\label{prop:monotone-influence}
If $f:\boolhcfa\rightarrow\classesfa$ is monotone, then $\Inf_i[f]=\widehat{f}(i)$.
\end{proposition}

%\comment{ Could we extend the majority results to product distributions? In this case, we would need to talk about this here as well.}

\chapter[Robustness \& Monotone Conjunctions]{Robustness: a Monotone Conjunction Case Study}
\label{chap:def-adv-rob}

In this chapter, we first review and study the implications of two different notions of robustness to evasion attacks from a learning-theory point of view.
We then settle on a particular notion of robustness, which speaks to the fidelity of the hypothesis to the target concept, and show a separation between the standard and robust learning settings. 
We finally study monotone conjunctions under distributional assumptions, and show that the sample complexity of robust learning in this setting is controlled by the adversary's perturbation budget at test time.

\section{Defining Robust Learnability}

In this thesis, we study the problem of \emph{robust classification} with respect to evasion attacks, where an adversary can perturb data at \emph{test time}.
This is a generalization of  standard classification tasks, outlined in Section~\ref{sec:learning-theory}, which are defined on an input space $\X_n$ of dimension $n$ and a finite output space $\Y$. 
Common examples of input spaces are $\{0,1\}^n$, $[0,1]^n$, and $\R^n$.
We focus on \emph{binary classification}, namely where $\Y=\{0,1\}$, and on the \emph{realizable setting}. 
Recall that, in the standard (non-robust) setting, this means that there exists a \emph{target concept}, also sometimes referred to as a \emph{ground truth} function.
Thus whenever we get access to a randomly drawn labelled sample $S$ from an unknown underlying distribution $D$, there exists a target concept $c:\X\rightarrow\Y$ such that $y=c(x)$ for all the labelled points $(x,y)\in S$.
In the PAC-learning framework of \cite{valiant1984theory}, which will form the basis of our study of robust classification, the goal is to find a function $h$ that approximates $c$ with high probability over the training sample.
We point the reader towards Section~\ref{sec:pac} for a PAC-learning overview.

Note that PAC learning is \emph{distribution-free}, in the sense that no assumptions are made about the distribution from which the data comes from.

\subsection{Two Notions of Robustness}

The notion of robustness can be accommodated within the
basic set-up of PAC learning by adapting the definition of the risk
function.  In this section we review two of the main definitions of
\emph{robust risk} to \emph{evasion attacks} that have been used in the literature.  For
concreteness and simplicity we consider the boolean hypercube $\boolhc$ as the input space, with
metric $d:\mathcal{X}\times\mathcal{X}\rightarrow\mathbb{N}$, where
$d(x,y)$ is the Hamming distance of $x,y\in\mathcal{X}$.  
%Note that this is the most natural notion of distance in this case: we can recover the Hamming distance from any metric induced by a norm on $\X$. (\todo{Check})
Given $x\in\mathcal{X}$, we write $B_\rho(x)$ for the ball $\{y\in
\mathcal{X} : d(x,y)\leq \rho\}$ with center $x$ and radius $\rho\geq
0$.
We recall the works of \citep{diochnos2018adversarial,dreossi2019formalization,pydi2021many,chowdhury2022robustness}, mentioned in Chapter~\ref{sec:evasion-attacks}, which also offer thorough discussions on the choice of robust risk.

The first definition of robust risk we will consider asks that the hypothesis be
exactly equal to the target concept in the ball $B_\rho(x)$ of
radius $\rho$ around a test point $x\in\mathcal{X}$.
We also note that it is possible to consider arbitrary perturbation functions $\U:\X\rightarrow 2^\X$, but that the guarantees and impossibility results obtained in this chapter are derived for the specific case $\U(x)=B_\rho(x)$.
The robustness parameter $\rho$, which is referred to as the \emph{adversary's budget}, features explicitly in many of the bounds.
The exact-in-the-ball notion of robustness is the one will work with in this thesis:

\begin{definition}[Exact-in-the-ball Robustness]
\label{def:loss-correct}
Given respective hypothesis and target functions
$h,c:\mathcal{X}\rightarrow\{0,1\}$, distribution $D$ on
$\mathcal{X}$, and robustness parameter $\rho\geq 0$, we define the
\emph{exact-in-the-ball} robust risk of $h$ with respect to $c$ to be
\begin{equation}
\label{eqn:e-i-b-rob}
\roblosse(h,c)=\Prob{x\sim D}{\exists z\in B_\rho(x):h(z)\neq c(z)}
\, .
\end{equation}
\end{definition}

While this definition captures a natural notion of robustness, an
obvious disadvantage is that evaluating the empirical loss requires
 the learner to have knowledge of the target function outside of the
training set, e.g., through membership queries.  Nonetheless, 
by considering a learner who has
oracle access to the predicate ${\exists z\in B_\rho(x):h(z)\neq
  c(z)}$, we can use the exact-in-the-ball framework to analyze 
sample complexity of robust learning, which will be addressed in Chapter~\ref{chap:local-queries}.
Moreover, even if one cannot evaluate the empirical loss on a training sample, the guarantees obtained in this chapter and in Chapter~\ref{chap:rob-thresholds} do not rely on an algorithm's capacity to compute or estimate the robust risk.

A popular alternative to the exact-in-the-ball risk function in
Definition~\ref{def:loss-correct} is the following
\emph{constant-in-the-ball risk} function:
\begin{definition}[Constant-in-the-ball Robustness]
\label{def:loss-constant}
Given respective hypothesis and target functions
$h,c:\mathcal{X}\rightarrow\{0,1\}$, distribution $D$ on
$\mathcal{X}$, and robustness parameter $\rho\geq 0$, we define the
constant-in-the-ball robust risk of $h$ with respect to $c$ as
\begin{equation}
\label{eqn:c-i-b-rob}
\roblossc(h,c)=\Prob{x\sim D}{\exists z\in B_\rho(x):h(z)\neq c(x)}
\enspace.
\end{equation}
\end{definition}

Figure~\ref{fig:rob-loss-comp} highlights an example where the two notions of robustness differ.

\begin{figure}
\begin{center}
\includegraphics[scale=0.15]{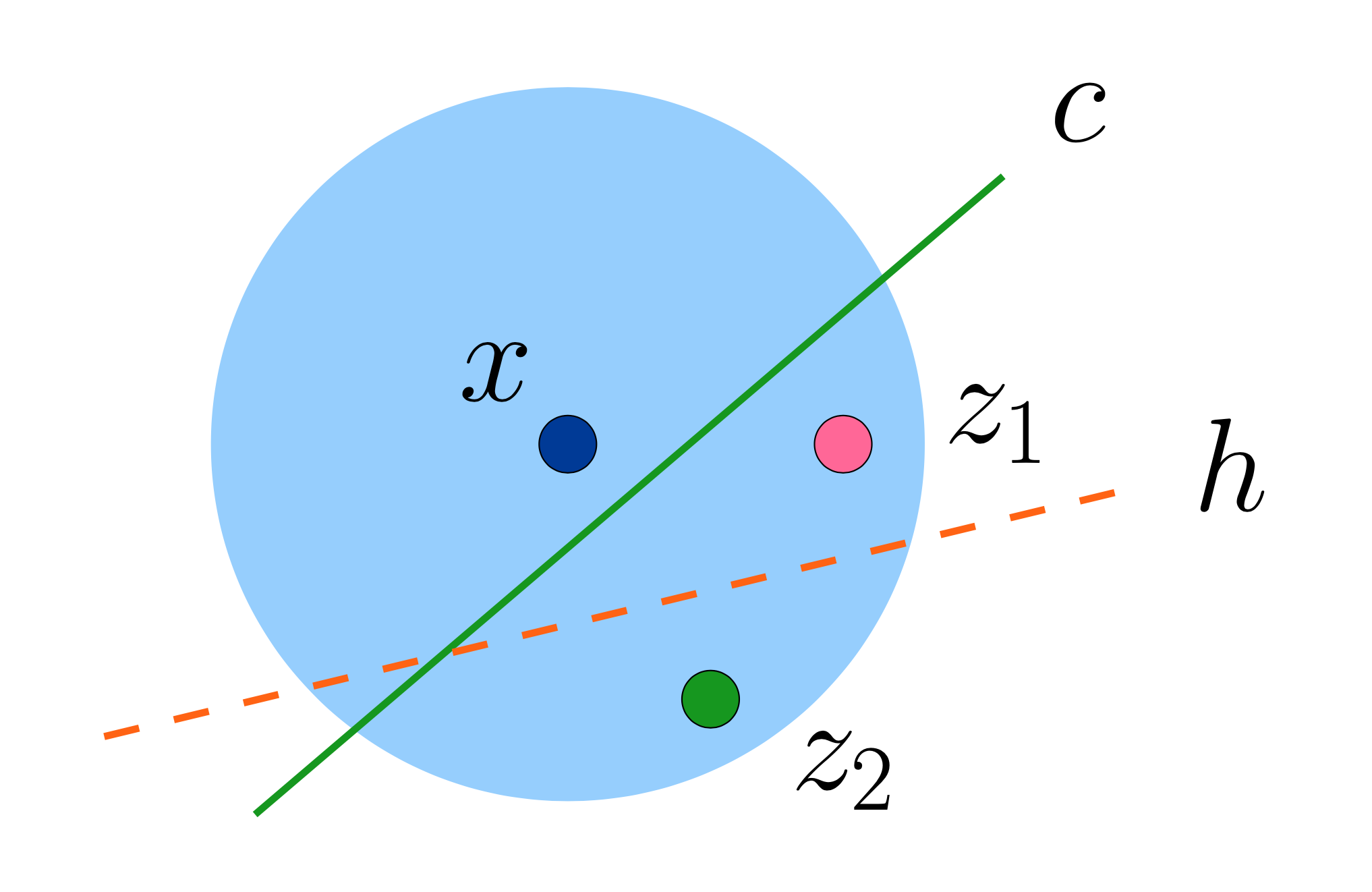}
\end{center}
\caption{The natural point $x$ has robust loss of 1 with respect to both notions of robustness: $z_1$ is a counterexample for exact-in-the-ball robustness (as $c(z_1)\neq h(z_1)$), and $z_2$ for constant-in-the-ball robustness (as $c(x)\neq h(z_2)$).}
\label{fig:rob-loss-comp}
\end{figure}

An obvious advantage of the constant-in-the-ball risk over
the exact-in-the-ball version is that, in the former,
evaluating the loss at point $x\in\mathcal{X}$ requires only knowledge
of the correct label of $x$ and the hypothesis $h$.  In particular,
this definition can also be carried over to the non-realizable
setting,\footnote{A note on terminology: realizability in this thesis refers to the existence of a ground truth $c$ and the requirement $\C\subseteq\H$. Then there will always be a $h\in\H$ such that $\err_D(c,h)=\roblosse(c,h)=0$. As explained later, it can be that $\robloss^C(c,c)>0$. In the literature, realizability with respect to the \emph{constant-in-the-ball} notion of robustness is in reference to a family of distributions on $\X\times\Y$ for which there exists $h\in\H$ such that $\Prob{(x,y)\sim D}{\exists z\in B_\rho(x):h(z)\neq y}=0$. We will make it explicit whenever we work with the latter type of realizability.} in which there is no target, but rather a joint distribution on $\X\times\Y$. 
Then Equation~\ref{eqn:c-i-b-rob} becomes $\Prob{(x,y)\sim D}{\exists z\in B_\rho(x):h(z)\neq y}$.

Despite the advantages of the constant-in-the-ball risk, from a foundational
point of view this notion of risk has some drawbacks: under this definition, it is possible to have strictly
positive, and even sometimes constant, robust risk  in the case that $h=c$.
In fact, this view of robustness can in some circumstances be in conflict with accuracy in the
traditional sense as pointed out by~(\cite{tsipras2019robustness}).
More in line with our work, \citet{chowdhury2022robustness} argue for robustness to be considered as a locally adaptive measure, where sometimes a label change is justified.

\begin{example}
Under the uniform distribution, for $c\in\MonConj$ 
of constant length $k$, $\mathsf{R}_1^C(c,c)\geq\Prob{x\sim D}{\exists !i\in[n]\st c(x)\neq c(x^{\oplus i})}=\frac{k+1}{2^k}$, and in the case 
of decision lists, any list $c$ of the form $((x_i,0),(x_j,1),\dots)$ 
satisfies $\mathsf{R}_1^C(c,c)\geq\Prob{x\sim D}{ x_j =1}=1/2$.
In the case of parity functions, 
it suffices to flip one bit of the index set to switch the label, so under any distribution 
$\roblossc(c,c)=1$ for any $\rho\geq1$.
\end{example}

\begin{figure}
	\begin{subfigure}{0.27\textwidth}
  		\centering
  			\includegraphics[width=\textwidth]{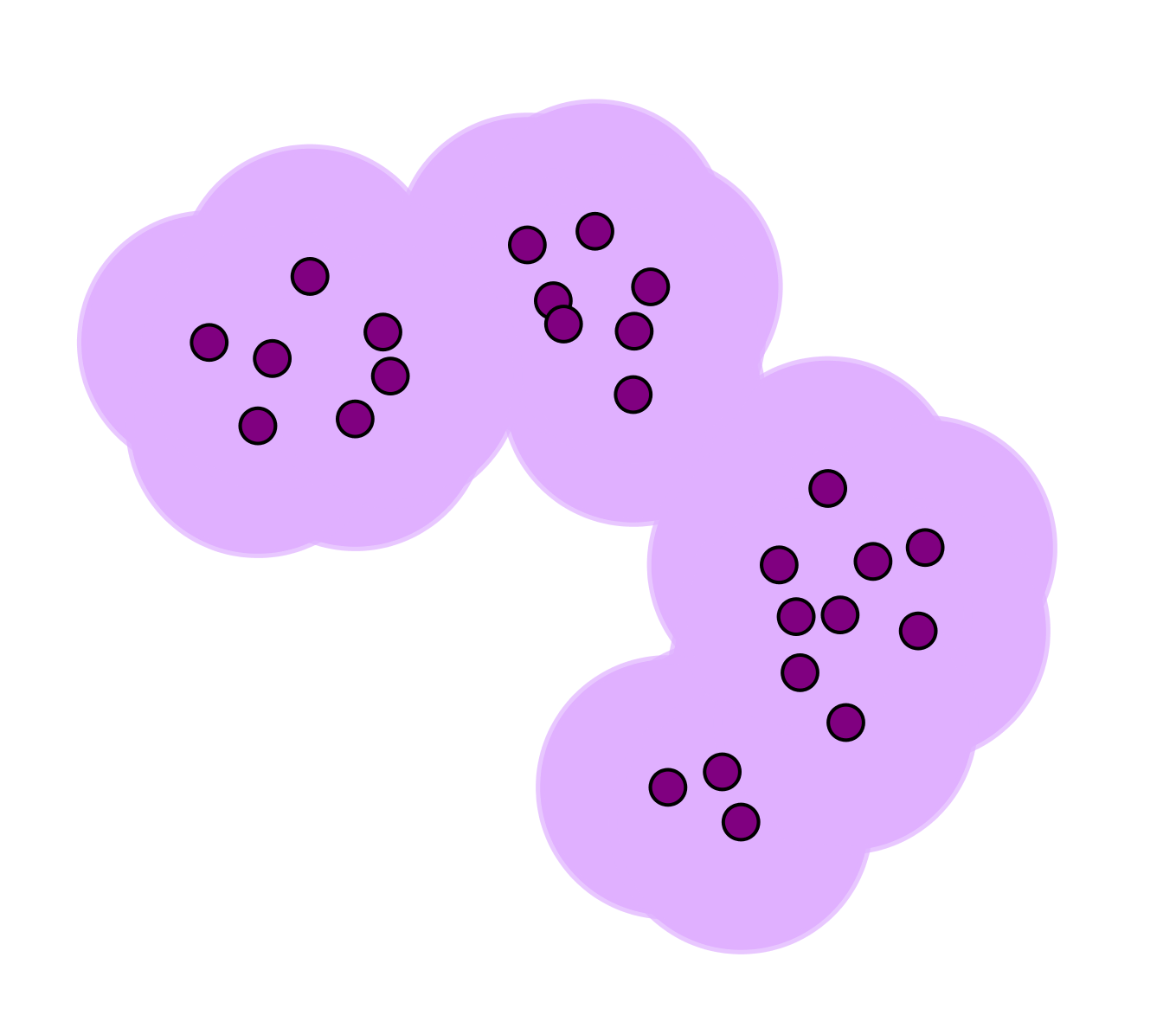}
  			\caption{}%Robust learning requires the target to be constant in the support of the distribution.}
  			\label{fig:constant}
	\end{subfigure}
	\hfill
	\begin{subfigure}{0.27\textwidth}
		\includegraphics[width=\textwidth]{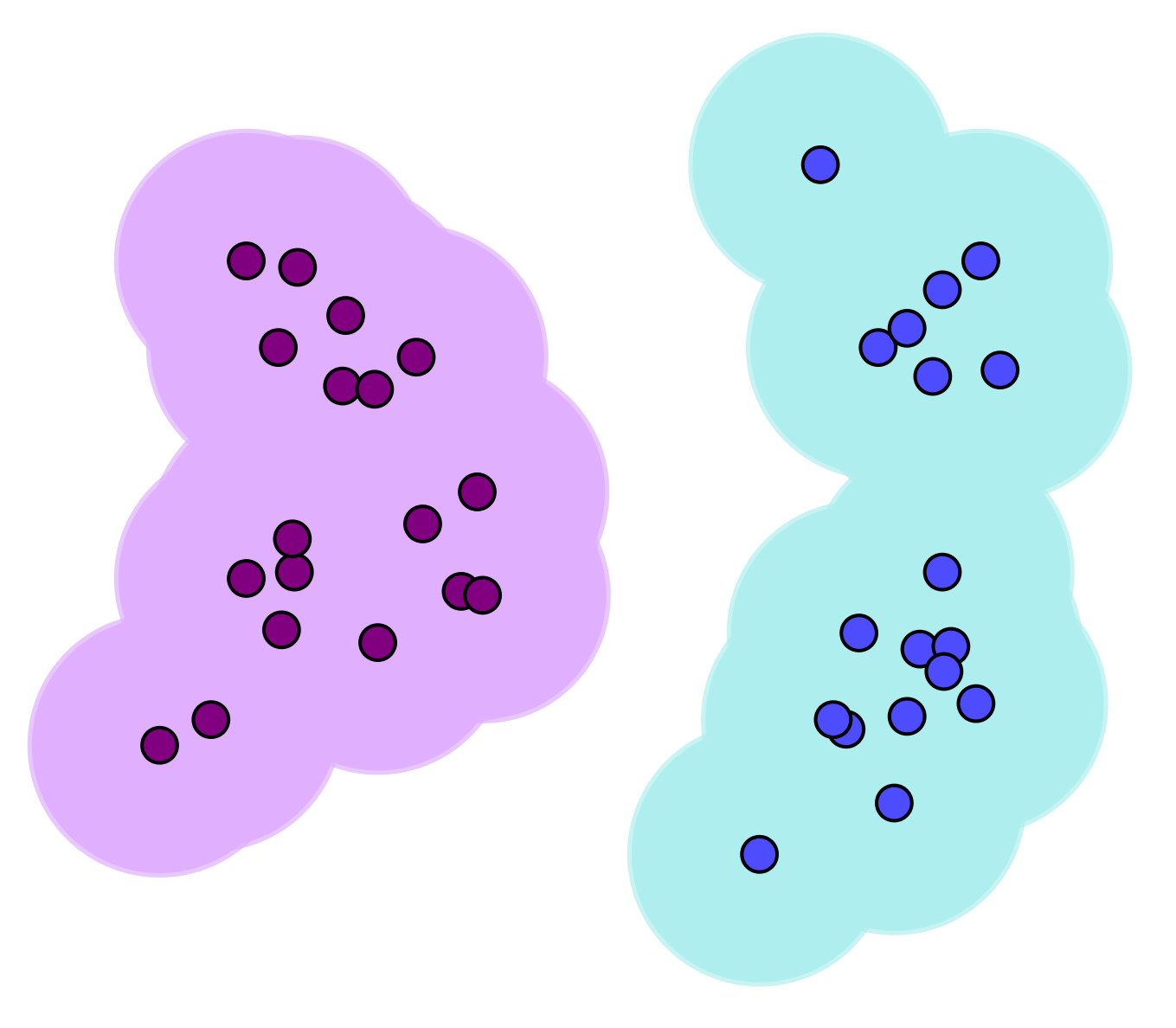}
  		\caption{}%The $\rho$-expansion of the support of the distribution allows non-trivial concepts to be robustly learned. }
		\label{fig:sep-supp}
	\end{subfigure}
	\hfill
	\begin{subfigure}{0.27\textwidth}
 	 \centering
		\includegraphics[width=\textwidth]{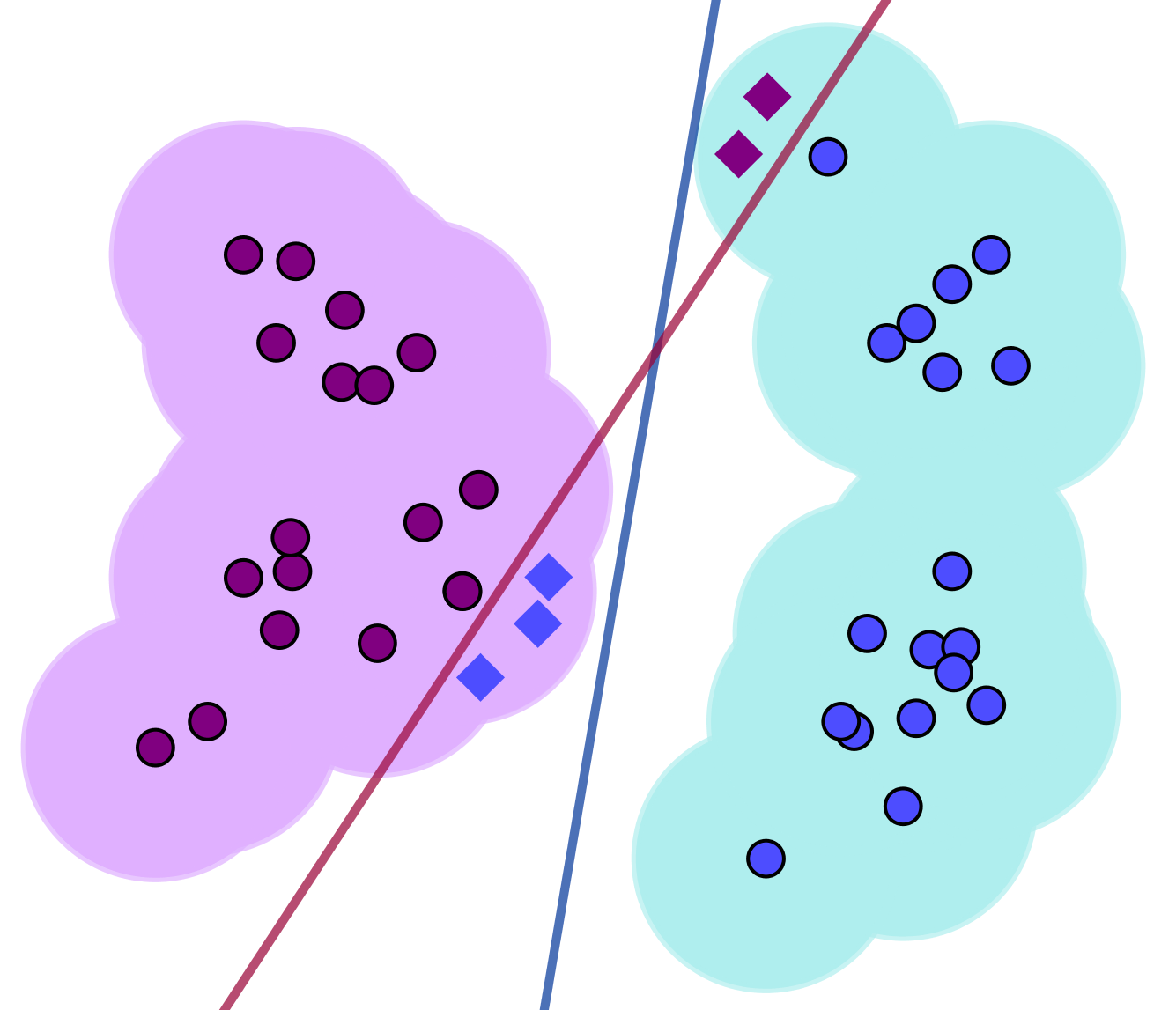}
		\caption{}
		\label{fig:diff-losses}%An example where the robust losses differ. The red concept is the target, while the blue one is the hypothesis. The blue points are the support of the distribution and the shaded region represents their $\rho$-expansion. The red points represent perturbed inputs.}
	\end{subfigure}
\caption{In all the examples above, the circles represent the support of the distribution, and the shaded region, its $\rho$-expansion (i.e., the points at a distance at most $\rho$ from points in the support of the distribution). (a) The support of the distribution is such that $\roblossc(h,c)=0$ can only be achieved if $c$ is constant.
(b) The $\rho$-expansion of the support of the distribution and target $c$ admit hypotheses $h$ such that $\roblossc(h,c)=0$ (i.e., any $h$ that does not cross the shaded regions).
(c) An example where $\roblossc$ and $\roblosse$ differ. The red concept, which crosses the shaded regions, is the target; the blue one is the hypothesis. The diamonds represent perturbed inputs which cause $\roblosse(c,h)>0$, while $\roblossc(h,c)=0$.
}
\label{fig:rob-losses}
\end{figure}

Let us note in passing that the
risk functions $\roblossc$ and $\roblosse$ are in general
incomparable.  Figure~\ref{fig:diff-losses} gives an example in which
$\roblossc=0$ and $\roblosse>0$.  Additionally, when we work in the
hypercube, or a bounded input space, as $\rho$ becomes larger, we
eventually require the function to be constant in the whole space.
Essentially, to $\rho$-robustly learn in the constant-in-the-ball realizable setting, we
require concept and distribution pairs to be represented as two sets
$D_+$ and $D_-$ whose $\rho$-expansions don't intersect, as
illustrated in Figures~\ref{fig:constant} and~\ref{fig:sep-supp}.  

We finish by pointing out that, in some cases in the (standard) realizable setting,
 the target $c$ is not the robust risk minimizer for $\rho = 1$: the constant concept is! 
This is easy to see for parity functions, as $\mathsf{R}_1^C(c,0)=\mathsf{R}_1^C(c,1)=1/2$ 
under the uniform distribution while ${R}_1^C(c,c)=1$. A similar result holds
for monotone conjunctions:

\begin{proposition}
Under the uniform distribution, for any non-constant concept $c\in\MonConj$, we have that $\mathsf{R}_1^C(c,c)>\mathsf{R}_1^C(c,0)$.
\end{proposition}

\begin{proof}
Let $\X=\boolhc$ and $D$ be the uniform distribution on $\X$.
Let $c(x)=x_1 \wedge \dots \wedge x_k$ for some $k\in[n]$.
Then, 
\begin{align*}
\mathsf{R}_1^C(c,c)
&=\Prob{x\sim D}{\exists z\in B_\rho (x) \st c(z)\neq c(x)}\\
&=\Prob{x\sim D}{c(x)=1}+\Prob{x\sim D}{\exists !i\in[k] \st x_i=0}\\
&=\mathsf{R}_1^C(c,0)+\Prob{x\sim D}{\exists !i\in[k] \st x_i=0}\\
&>\mathsf{R}_1^C(c,0)\enspace.
\end{align*}
\end{proof}
\newcommand{\cmark}{\ding{51}}%
\newcommand{\xmark}{\ding{55}}%

\begin{table}[]
\begin{tabular}{l|c|c}
\textbf{Property}  & \textbf{$\roblossc$: constant-in-the-ball} & \multicolumn{1}{c}{\textbf{$\roblosse$: exact-in-the-ball}} \\ \hline
$\robloss(c,c)=0$?            & \red{\xmark}  & \green{\cmark} \\
$c=\arg\min_h\robloss(c,h)$?  & \red{\xmark}  & \green{\cmark} \\
$S$ enough to evaluate $\widehat{\robloss}$? & \green{\cmark} & { \red{\xmark}} \\
Behaviour of $h$ as $\rho\rightarrow n$ & $h=$ constant & { $h=c$ (exact)}        
\end{tabular}
\caption{The pros and cons of the two robust risk functions. 
The last line refers to the behaviour of hypotheses minimizing the robust risk as the perturbation region increases. At the extreme case, when the perturbation region is the whole space, the robust risk minimizer for the constant-in-the-ball risk is a constant function, while it is the target for the exact-in-the-ball risk (as we require exact learning). }
\label{tab:rob-risk-comp}
\end{table}

The discussion above, which pertains to the boolean hypercube, 
makes apparent the fact that the exact-in-the-ball 
and constant-in-the-ball definitions of robust risk both rely on different
distributional and concept class assumptions. The constant-in-the-ball  
notion of robust risk relies on a strong distributional assumption (for e.g., a margin 
condition) and/or on the stability of functions in the concept class. The exact-in-the-ball 
is more relevant in cases where we cannot assume that the probability mass near the 
boundary is small, and wish to be correct with respect to the target function.
Table~\ref{tab:rob-risk-comp} summarizes the advantages and disadvantages of both robust risks.
Figure~\ref{fig:cifar-mnist} shows real-life examples where such assumptions can come into play.

\begin{figure}
\begin{center}
\includegraphics[scale=0.5]{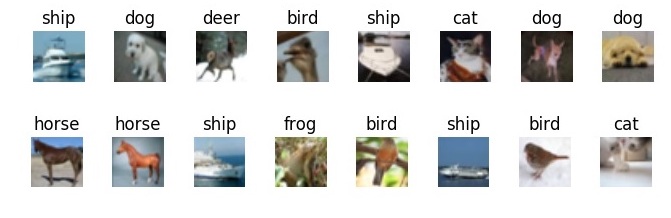}\\
\vspace{3mm}
\includegraphics[scale=0.5]{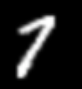}\hspace{5mm}
\includegraphics[scale=0.57]{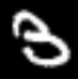}
\end{center}
\caption{Images from the CIFAR-10 (above) and MNIST (below) datasets, respectively from \citep{krizhevsky2009learning} and \citep{lecun1998mnist}. 
While the margin assumption generally holds for CIFAR (e.g., the ``boat'' and ``dog'' classes are well-separated), this is not necessarily the case for MNIST (the three above could easily be transformed into an eight, and the left-hand side picture could be a one or a seven).}
\label{fig:cifar-mnist}
\end{figure}

Overall, choosing a robust risk function should depend on the learning problem at hand, and it is possible that other robustness frameworks could bring more nuance and faithfulness to practical robustness considerations.
For the moment, to lay the foundations of robust learnability, we will work with the exact-in-the-ball notion of robustness in the PAC framework. 
Our choice of robust risk comes from the fact that the constant-in-ball risk is much better understood than for the
exact-in-the-ball one (most papers we have mentioned in Chapter~\ref{chap:lit-review} have used the former). 
 
Having settled on a risk function, we now formulate the definition of
robust learning.  For our purposes a \emph{concept class} is a family
$\mathcal{C} = \{\mathcal{C}_n\}_{n\in \mathbb{N}}$, with
$\mathcal{C}_n$ a class of functions from $\{0,1\}^n$ to $\{0,1\}$.
Likewise, a \emph{distribution class} is a family $\mathcal{D} = \{
\mathcal{D}_n\}_{n\in\mathbb{N}}$, with $\mathcal{D}_n$ a set of
distributions on $\{0,1\}^n$.  Finally, a \emph{robustness function} is
a function $\rho:\mathbb{N}\rightarrow \mathbb{N}$, which is fixed a priori.

\begin{definition}
\label{def:robust-learning}
Fix a function $\rho:\N\rightarrow\N$. We say that an algorithm $\A$
\emph{efficiently} $\rho$-\emph{robustly learns} a concept class $\C$
with respect to distribution class $\mathcal{D}$ if there exists a
polynomial $\poly(\cdot,\cdot,\cdot,\cdot)$ such that for all
$n\in\mathbb{N}$, all target concepts $c\in \C_n$, all distributions
$D \in \mathcal{D}_n$, and all accuracy and confidence parameters
$\epsilon,\delta>0$, if $m \geq
\poly(n,1/\epsilon,1/\delta,\text{size}(c))$, whenever $\A$ is given access to
a sample $S\sim D^m$ labelled according to $c$, it outputs a polynomially evaluatable function
$h:\{0,1\}^n\rightarrow\{0,1\}$ such
that $\Prob{S\sim D^m}{\mathsf{R}^E_{\rho(n)}(h,c)<\epsilon}>1-\delta$.
\end{definition}

Note that our definition of robust learnability requires polynomial sample
complexity and allows improper learning (the hypothesis $h$ need not
belong to the concept class $\mathcal{C}_n$).

\subsection{A Separation between PAC and Robust Learning}

In the standard PAC framework, a hypothesis $h$ is considered to
have zero risk with respect to a target concept $c$ when $\Prob{x\sim
D}{h(x)\neq c(x)}=0$.  We have remarked that exact learnability (in the sense that $c=h$ on all of $\X$, not just the support of the distribution)
implies robust learnability; next we give an example of a
concept class $\C$ and distribution $D$ such that $\C$ is PAC learnable under $D$ with zero risk and yet cannot be
robustly learned under $D$ (regardless of the sample complexity).

\begin{lemma}
\label{lemma:dictators}
The class of  dictators is not 1-robustly learnable (and thus not robustly learnable for any $\rho\geq1$) with respect to the robust risk of Definition~\ref{def:loss-correct} in the distribution-free setting. 
\end{lemma}
\begin{proof}
Let $c_1$ and $c_2$ be the dictators on variables $x_1$ and $x_2$, respectively.
Let $D$ be such that $\Prob{x\sim D}{x_1=x_2}=1$ and $\Prob{x\sim D}{x_k=1}=\frac{1}{2}$
for $k\geq3$. 
Draw a sample $S\sim D^m$ and label it according to $c\sim U(c_1,c_2)$. 
By the choice of $D$, the elements of $S$ will have the same label regardless of whether $c_1$ or $c_2$ was picked.
However, for $x\sim D$, it suffices to flip any of the first two bits to cause $c_1$ and $c_2$ to disagree on the perturbed input.
We can easily show that, for any $h\in\set{0,1}^\X$, 
 $
 \mathsf{R}^E_1(h,c_1)+ \mathsf{R}^E_1(h,c_2)
\geq  \mathsf{R}^E_1(c_1,c_2) = 1.
$
Then 
\begin{equation*}
\underset{c\sim U(c_1,c_2)}{\mathbb{E}}\eval{S\sim D^m}{ \mathsf{R}^E_1(h,c)} \geq 1/2 \enspace.
\end{equation*}
We conclude that one of $c_1$ or $c_2$ has robust risk at least 1/2.
\end{proof}
Note that a PAC learning algorithm with error probability threshold $\varepsilon=1/3$ will either output  $c_1$ or $c_2$ and will hence have standard risk  zero.

The result above highlights an important distinction between standard and robust learning.
We will further study this separation in the next section.

\section{The Distribution-Free Assumption}

In this section, we show that no \emph{non-trivial} concept class is efficiently 1-robustly learnable in the boolean hypercube, implying that such a class is also not efficiently $\rho$-robustly learnable for any $\rho\geq1$. 
As a consequence, there exists a fundamental separation between the standard PAC learning setting and its robust counterpart.
Indeed, (efficient) robust learnability in the \emph{distribution-free} setting would require access to a more powerful learning model or distributional assumptions when considering a learner who only has access to the random example oracle $\EX(c,D)$.

We start by defining trivial concept classes.

\begin{definition}
Let $\C_n$ be a concept class on $\boolhc$, and define $\C=\bigcup_{n\geq 1} \C_n$.
We say that a class of functions is trivial if $\C_n$ has at most two functions, which moreover differ on every point.
\end{definition}

A simple example of a trivial concept class is the set of constant functions $\classes$. 
More generally, once a function in $\C$ is fixed, there is only one (uniquely defined) function that can be added to $\C$ and preserve its triviality. 
We show below that these are the only classes that are \emph{distribution-free} robustly learnable. 

\begin{theorem}
\label{thm:no-df-rl}
For any concept class $\C$, $\C$ is efficiently distribution-free robustly learnable iff it is trivial.
\end{theorem}

Note that this is in stark contrast with the work of \cite{montasser2019vc}, which gives \emph{distribution-free} robust learning guarantees for the \emph{constant-in-the-ball} notion of robustness.
This approach relies on an improper learner and a sample inflation\footnote{For a given perturbation function $\U:\X\rightarrow 2^\X$ and training sample $S=\set{(x_i,y_i)}_{i=1}^m$, the inflated sample is $S_\U=\set{(\U(x_i),y_i)}_{i=1}^m$.} made possible by the \emph{realizability} of the learning problem under the constant-in-the-ball robust risk.
The agnostic case follows from a non-trivial reduction from the agnostic to the realizable setting. 
This highlights a fundamental difference between the constant-in-the-ball and exact-in-the-ball robustness guarantees.
Indeed, to perform such a sample inflation technique in our setting, one would need to have knowledge outside the training sample; this is addressed in Chapter~\ref{chap:local-queries}.

The idea behind the proof of Theorem~\ref{thm:no-df-rl}  is a generalization of the proof of Lemma~\ref{lemma:dictators} that dictators are not robustly learnable. 
However, note that we construct a distribution whose support is all of $\X$.
It is possible to find two hypotheses $c_1$ and $c_2$ and create a distribution such that $c_1$ and $c_2$ will with high probability look identical on samples of size polynomial in $n$ but have robust risk $\Omega(1)$ with respect to one another. 
Since any hypothesis $h$ in $\set{0,1}^\X$ will disagree either with $c_1$ or $c_2$ on a given point $x$ if $c_1(x)\neq c_2(x)$, by choosing the target hypothesis $c$ at random from $c_1$ and $c_2$, we can guarantee that $h$ won't be robust against $c$ with positive probability.
This shows that \emph{efficient} robust learnability is in general impossible. However, the same argument as in Lemma~\ref{lemma:dictators} can be made to show that, even with infinite sample complexity, non-trivial classes are not robustly learnable.
Finally, note that an analogous argument can be made for a more general setting (e.g., for the input space $\R^n$).

The proof of Theorem~\ref{thm:no-df-rl} relies on the following lemma, which states that the robust risk satisfies the triangle inequality:
\begin{lemma}
\label{lemma:robloss-triangle}
Let $c_1,c_2\in\{0,1\}^\X$ and fix a distribution on $\X$. 
Then for all $h:\boolhc\rightarrow\set{0,1}$
\begin{equation*}
\roblosse(c_1,c_2)
\leq \roblosse(h,c_1)+\roblosse(h,c_2)
\enspace.
\end{equation*}
\end{lemma}
\begin{proof}
Let $x\in\boolhc$ be arbitrary, and suppose that $c_1$ and $c_2$ differ on some $z\in B_\rho(x)$. 
Then either $h(z)\neq c_1(z)$ or $h(z)\neq c_2(z)$. The result follows.
\end{proof}

We are now ready to prove Theorem~\ref{thm:no-df-rl}.

\begin{proof}[Proof of Theorem~\ref{thm:no-df-rl}]
First, if $\C$ is trivial, we need at most one example to identify the target function.

For the other direction, suppose that $\C$ is non-trivial, and for a given $c\in\C$, denote by $I_c\subseteq[n]$ the index set of relevant variables in the function $c$.\footnote{This means that if $i\in I_c$ there exists $x\in\boolhc$ such that $c(x^{\oplus i})$, the output of $c$ on flipping the $i$-th bit of $x$, differs from $c(x)$. }
We first start by fixing any learning algorithm and polynomial sample complexity function $m$. 
Let $\eta=\frac{1}{2^{\omega(\log n)}}$, $0<\delta<\frac{1}{2}$, and note that for any constant $a>0$,
\begin{equation*}
\underset{n\rightarrow\infty}{\lim}\;
n^a\log(1-\eta)^{-1}=0
\enspace,
\end{equation*}
and so any polynomial in $n$ is $o\left(\left(\log(1/(1-\eta))\right)^{-1}\right)$.
Then it is possible to choose $n_0$ such that for all $n\geq n_0$, 
\begin{equation}
\label{eqn:sample-size-ub}
m\leq \frac{\log(1/\delta)}{2n\log(1-\eta)^{-1}}
\enspace.
\end{equation}

Since $\C$ is non-trivial, we can choose concepts $c_1,c_2\in \C_{n}$ and points $x,x'\in\boolhc$ such that $c_1$ and $c_2$ agree on $x$ but disagree on $x'$.
This implies that there exists a point $z\in\boolhc$ such that (i)~$c_1(z)=c_2(z)$ and (ii)~it suffices to change \emph{only one bit} in $I:=I_{c_1}\cup I_{c_2}$ to cause $c_1$ to disagree on $z$ and its perturbation.
Let $D$ be a product distribution such that
\begin{align*}
&\Prob{x\sim D}{x_i=z_i}=
\begin{cases}
	1-\eta &\quad\text{if }i\in I\\
	\frac{1}{2}&\quad\text{otherwise}
\end{cases}
\enspace.
\end{align*}

Draw a sample $S\sim D^m$ and label it according to $c\sim U(c_1,c_2)$. 
Then,
\begin{equation}
\label{eqn:same-label-sample-2}
\Prob{S\sim D^m}{\forall x\in S\quad c_1(x)=c_2(x)}
\geq\left(1-\eta\right)^{m| I |}
\enspace.
\end{equation}
Bounding the RHS below by $\delta>0$, we get that, as long as 
\begin{equation*}
%\label{eqn:sample-size-2}
m\leq \frac{\log(1/\delta)}{| I |\log(1-\eta)^{-1}}
\enspace,
\end{equation*}
Equation~\ref{eqn:same-label-sample-2} holds with probability at least $\delta$. This is enabled by the requirement from Equation~\ref{eqn:sample-size-ub}.

However, if $x=z$, then it suffices to flip one bit of $x$ to get $x'$ such that $c_1(x')\neq c_2(x')$. 
Then,
\begin{equation}
\roblosse(c_1,c_2)\geq\Prob{x\sim D}{x_{I}=z_{I}}=\left(1-\eta\right)^{| I |}
\enspace.
\end{equation}
The constraints on $\eta$ and the fact that $| I |\leq n$ are sufficient to guarantee that the RHS is $\Omega(1)$.
Let $\alpha>0$ be a constant such that  $\roblosse(c_1,c_2)\geq\alpha$.

We can use the same reasoning as in Lemma~\ref{lemma:robloss-triangle} to argue that, for any $h\in\set{0,1}^\X$, 
\begin{equation*}
 \mathsf{R}^E_1(c_1,h)+ \mathsf{R}^E_1(c_2,h)
\geq  \mathsf{R}^E_1(c_1,c_2) \enspace.
\end{equation*}
Finally, we can show that
\begin{equation*}
\underset{c\sim U(c_1,c_2)}{\mathbb{E}}\eval{S\sim D^m}{ \mathsf{R}^R_1(h,c)} \geq \alpha\delta/2 ,
\end{equation*}
hence there exists a target $c$ with expected robust risk bounded below by a constant.%\footnote{For a more detailed reasoning, we refer the reader to the proof of Theorem~\ref{thm:mon-conj}, where we bound the expected value $\eval{c,S}{\roblosse(\A(S),c)}$ of the robust risk of a target chosen at uniformly random and the hypothesis outputted by a learning algorithm $\A$ on a sample $S$.}
\end{proof}

In the next section, we will show that, even when looking at problems with distributional assumptions, robust learning can be hard from an information-theoretic point of view.

\section{An Adversarial Sample Complexity Lower Bound}
\label{sec:mon-conj-sc-lb}

In this section, we will show that any robust learning algorithm for
monotone conjunctions under the uniform distribution must have an
exponential sample-complexity dependence on the adversary's budget
$\rho$.  This result extends to any superclass of monotone
conjunctions, such as CNF formulas, decision lists and linear classifiers.  It
is a generalization of Theorem~13 in \cite{gourdeau2021hardness}, an earlier version of the work presented in this thesis,
which shows that no sample-efficient robust learning algorithm exists
for monotone conjunctions against adversaries that can perturb
$\omega(\log(n))$ bits of the input under the uniform distribution.

Monotone conjunctions are perhaps the simplest class of functions to study in learning theory.
Recall that a conjunction $c$ over $\{0,1\}^n$ can be represented by a set of literals $l_1,\dots,l_k$, where, for $x\in\X_n$, $c(x)=\bigwedge_{i=1}^k l_i$. 
Monotone conjunctions are the subclass of conjunctions where negations are not allowed, i.e. all literals are of the form $l_i=x_j$ for some $j\in[n]$.
The standard PAC learning algorithm to learn conjunctions is outlined in Algorithm~\ref{alg:conj-pac} in Section~\ref{sec:pac-algos}, and can straightforwardly be adapted for monotone conjunctions, with the slight distinction that the initial hypothesis is $\bigwedge_{i\in[n]} x_i$.

\begin{theorem}
\label{thm:mon-conj}
Fix a positive increasing robustness function $\rho:\N\rightarrow\N$. 
If $\rho$ is a function of the input dimension $n$, then for $\kappa<2$ and sufficiently large $n$, any $\rho(n)$-robust learning algorithm for {\MonConj} has a sample complexity lower bound of $2^{\kappa\rho(n)}$ under the uniform distribution.
Otherwise, the same lower bound holds whenever $\rho$ is a sufficient large constant (with respect to $\kappa$). 
\end{theorem}

The idea behind the proof is to show that, for any $\kappa<2$, there exists a sufficiently large input dimension (that depends on $\kappa$ and the function $\rho$) such that a sample of size $2^{\kappa\rho}$ from the uniform distribution will not be able to distinguish between two disjoint conjunctions of length $2\rho$. 
However, the robust risk between these two conjunctions can be lower bounded by a constant.
Hence, there does not exist a robust learning algorithm with sample complexity 
$2^{\kappa\rho}$ that works for the uniform distribution, and arbitrary input dimension and confidence and accuracy parameters.

Recall that the sample complexity of PAC learning conjunctions is
$\Theta(n)$ in the non-adversarial setting.  On the other hand, our
adversarial lower bound in terms of the robust parameter is superlinear in $n$ as soon as the adversary can perturb more than
$\log(\sqrt{n})$ bits of the input.

The proof of Theorem~\ref{thm:mon-conj} relies on the lemmas below, as well as Lemma~\ref{lemma:robloss-triangle}. 
Lemma~\ref{lemma:bound-loss} lower bounds the robust risk between two disjoint monotone conjunctions as a function of the adversarial budget $\rho$, while Lemma~\ref{lemma:concepts-agree} lower bounds the probability that these two concepts are indistinguishable on a polynomially-sized sample.

\begin{lemma}
\label{lemma:bound-loss}
Under the uniform distribution, for any $n\in\N$, disjoint $c_1,c_2\in{\MonConj}$ of even length $3\leq l\leq n/2$  on $\boolhc$ and robustness parameter $\rho= l/2$, we have that $\robloss(c_1,c_2)$ is bounded below by a constant that can be made arbitrarily close to $\frac{1}{2}$ as $l$ (and thus $\rho$) increases. 
\end{lemma}
\begin{proof}
For a  hypothesis $c\in{\MonConj}$, let $I_c$ be the set of variables in $c$.
Let $c_1,c_2\in\C$ be as in the statement of the lemma.
Then the robust risk $\robloss(c_1,c_2)$ is bounded below by 
\begin{equation*}
\Prob{x\sim D}{c_1(x)=0\wedge x\text{ has at least $ \rho $ 1's in }I_{c_2}}\geq (1-2^{-2\rho})/2\enspace.
\end{equation*}
\end{proof}

Now, the following lemma shows that, for sufficiently
large input dimensions, a sample of size $2^{\kappa\rho}$ from the
uniform distribution will look constant with probability $1/2$ if
labelled by two disjoint monotone conjunctions of length $2\rho$.  

\begin{lemma}
\label{lemma:concepts-agree}
For any constant $\kappa<2$, for any robustness parameter $\rho\leq n/4$, for any disjoint monotone conjunctions $c_1,c_2$ of length $2\rho$, there exists $n_0$ such that for all $n\geq n_0$, a sample $S$ of size $2^{\kappa\rho}$ sampled i.i.d. from $D$ will have that $c_1(x)=c_2(x)=0$ for all $x\in S$ with probability at least $1/2$.
\end{lemma}

\begin{proof}
We begin by bounding the probability that $c_1$ and $c_2$ agree on an
i.i.d. sample of size $m$.  We have
\begin{equation}
\label{eqn:zero-label-sample}
\Prob{S\sim D^m}{\forall x\in S  \cdot c_1(x)=c_2(x)=0}
%&=\left(\Prob{x\sim D}{x\in X_{00}}\right)^m\\
=\left(1-\frac{1}{2^{2\rho}}\right)^{2m}
\enspace.
\end{equation}
In particular, if 
%Bounding the RHS below by $1/2$, we get that, as long as 
\begin{equation}
\label{eqn:sample-size}
m\leq \frac{\log(2)}{2\log(2^{2\rho}/(2^{2\rho}-1))}
\enspace,
\end{equation}
then the RHS of  Equation~\ref{eqn:zero-label-sample} is at least $1/2$.

Now, let us consider the following limit, where $\rho$ is a function of the input parameter $n$:
\begin{align*}
\underset{n\rightarrow\infty}{\lim}\;
2^{\kappa\rho}\log\left(\frac{2^{2\rho}}{2^{2\rho}-1}\right)
&=\frac{-\log(4)}{\kappa\log(2)}\;\underset{n\rightarrow\infty}{\lim}\;\frac{2^{\kappa\rho}}{1-2^{2\rho}}   \\
&=\frac{-\log(4)}{\kappa\log(2)}\;\frac{\kappa\log(2)}{-2\log(2)}\underset{n\rightarrow\infty}{\lim}\;\frac{2^{\kappa\rho}}{2^{2\rho}}   \\
&= \underset{n\rightarrow\infty}{\lim}\;2^{(\kappa-2)\rho}\\
&=
\begin{cases} 
0 & \text{if $\kappa<2$} \\
1 & \text{if $\kappa=2$} \\
\infty & \text{if $\kappa>2$}
\end{cases}
\enspace,
\end{align*}
where the first two equalities follow from l'H\^opital's rule.

Thus if $\kappa<2$ then $2^{\kappa\rho}$ is $o\left(\left(\log\left(\frac{2^{2\rho}}{2^{2\rho}-1}\right)\right)^{-1}\right)$.

\end{proof}

\begin{remark}
  Note that for a given $\kappa<2$, the lower bound $2^{\kappa\rho}$
  holds only for sufficiently large $\rho(n)$.  By looking at
  Equation~\ref{eqn:zero-label-sample}, and letting $m=2^\rho$,
  we get that $\rho(n)\geq2$ is a sufficient condition for it to hold.  
  If we want a lower bound for robust learning that
  is larger than that of standard learning (where the dependence is
  $\Theta(n)$) for a $\log(n)$ adversary, setting $m=2^{1.7\rho}$
  and requiring $\rho(n)\geq6$, for e.g., would be sufficient.
\end{remark}

We are now ready to prove Theorem~\ref{thm:mon-conj}.

\begin{proof}[Proof of Theorem~\ref{thm:mon-conj}]
Fix any algorithm $\A$ for learning {\MonConj}.
We will show that the expected robust risk between a randomly chosen target function and any hypothesis returned by $\A$ is bounded below by a constant.
%Fix a function $\poly(\cdot,\cdot,\cdot,\cdot,\cdot)$, and note that, since $\text{size}(c)$ and $\rho$ are both at most $n$, we can simply consider a function $\poly(\cdot,\cdot,\cdot)$ in the variables $1/\epsilon,$ and $1/\delta,n$ instead.

Let $\delta=1/2$, and fix a positive increasing adversarial-budget function $\rho(n)\leq n/4$ ($n$ is not yet fixed).
Let $m(n)=2^{\kappa\rho(n)}$ for an arbitrary $\kappa<0$.
Let $n_0$ be as in Lemma~\ref{lemma:concepts-agree}, where $m(n)$ is the fixed sample complexity function. 
Then Equation~(\ref{eqn:sample-size}) in the proof of Lemma~\ref{lemma:concepts-agree} holds for all $n\geq n_0$.
%Set such a $n$ that also satisfies $l\leq n/2$ and $l\geq 3$. 

Now, let $D$ be the uniform distribution on $\boolhc$ for $n\geq \max(n_0,3)$, and choose $c_1$, $c_2$ as in Lemma~\ref{lemma:bound-loss}.
Note that $\robloss(c_1,c_2)>\frac{5}{12}$ by the choice of $n$.
Pick the target function $c$ uniformly at random between $c_1$ and $c_2$, and label $S\sim D^{m(n)}$ with $c$.
By Lemma~\ref{lemma:concepts-agree}, $c_1$ and $c_2$ agree with the labeling of $S$ (which implies that all the points have label $0$) with probability at least~$\frac{1}{2}$ over the choice of $S$. 

Define the following three events for $S\sim D^m$: 
\begin{align*}
&\E:\;{c_1}_{|S}={c_2}_{|S}\;,\enspace
\E_{c_1}:\;c=c_1\;,\enspace
\E_{c_2}:\;c=c_2\enspace.
\end{align*}

Then, 

\begin{align*}
\eval{c,S}{\robloss(\A(S),c)}
&\geq\Prob{c,S}{\E}\eval{c,S}{\robloss(\A(S),c)\;|\;\E} \\
&>\frac{1}{2}(
\Prob{c,S}{\E_{c_1}}\eval{S}{\robloss(\A(S),c)\;|\;\E\cap\E_{c_1}}+\Prob{c,S}{\E_{c_2}}\eval {S}{\robloss(\A(S),c)\;|\;\E\cap \E_{c_2}}
)\\
&=\frac{1}{4}\;\eval{S}{\robloss(\A(S),c_1)+\robloss(\A(S),c_2)\;|\;\E}\\
&\geq \frac{1}{4}\;\eval{S}{\robloss(c_2,c_1)}\\
&=\frac{5}{48}
\enspace, 
\end{align*}
where the first inequality is due to the Law of Total Expectation.  The strict inequality comes from Lemma~\ref{lemma:concepts-agree}, the last inequality from Lemma~\ref{lemma:robloss-triangle}, and the last equality from Lemma~\ref{lemma:bound-loss}.
\end{proof}

\paragraph*{Comparison with \cite{diochnos2018adversarial,diochnos2020lower}}
First, \cite{diochnos2018adversarial} considers the robustness
of monotone conjunctions under the uniform distribution on the boolean
hypercube for the exact-in-the-ball notion of risk.  However,~\cite{diochnos2018adversarial} does not
address the sample and computational complexity of learning: their
results rather concern the ability of an adversary to magnify the
missclassification error of \emph{any} hypothesis with respect to
\emph{any} target function by perturbing the input.  For
example, they show that an adversary who can perturb $\Theta(\sqrt{n})$
bits can increase the missclassification probability from $0.01$ to
$1/2$. The main tool used
in~\cite{diochnos2018adversarial} is the isoperimetric inequality for
the boolean hypercube, which gives lower bounds on the volume of the
expansions of arbitrary subsets.  On the other hand, we use the
probabilistic method to establish the existence of a single hard-to-robustly-learn 
target concept for any given algorithm with sample
complexity exponential in $\rho$.

 The work of \cite{diochnos2020lower} shows
an exponential lower bound on the sample complexity of robust PAC
learning of a wide family of concept classes, which are called 
\emph{$\alpha$-close}, meaning that there must exist two concepts in the 
class that have (standard) error $\alpha$.
These bounds hold under Normal Lévy distributions (which include
product distributions under the Hamming distance in $\boolhc$) against
all adversaries that can perturb up to $o(n)$ bits. 
Closer to our results of this section, they also show a superpolynomial lower bound in sample complexity against  adversaries that can perturb $\widetilde{\Theta}(\sqrt{n})$
bits. 
This thesis obtains the same result against a weaker adversary in the special case of the uniform distribution:
we show that a weaker adversary, who can perturb
only $\omega(\log n)$ bits, renders it impossible to robustly learn monotone
conjunctions (and any superclass) with polynomial sample complexity.  In fact, we will show in Section~\ref{sec:mon-conj-rob-ub} that 
$\Theta(\log n)$ is indeed the threshold for the efficient robust PAC learning of this class under log-Lipschitz distributions, which include the uniform distribution.

\section[Logarithmically-Bounded Adversary]{Robust Learnability Against a Logarithmically-Bounded Adversary}
\label{sec:mon-conj-rob-ub}

In the previous section, we  exhibited an exponential dependence on the adversary's budget to robustly learn monotone conjunctions under the uniform distribution. 
We now turn our attention to a wider family of distributions, log-Lipschitz distributions, and show that robust learnability can be guaranteed in this setting whenever the adversary is logarithmically bounded.
These results show that the sample complexity of robust learning in our setting is controlled by the adversary's budget.

\subsection{Log-Lipschitz Distributions}

A thorough introduction to log-Lipschitz distributions can be found in Section~\ref{sec:prob-theory}, but we will recall the formal definition here.

\textbf{Definition~\ref{def:log-lipschitz}.}
\emph{A distribution  $D$ on $\boolhc$ is said to be $\alpha$-$\log$-Lipschitz if 
for all input points $x,x'\in \boolhc$, if $d_H(x,x')=1$, then $|\log(D(x))-\log(D(x'))|\leq\log(\alpha)$.}

While it may be tempting to work under product distributions over the instance space $\boolhc$, as many concentration bounds and Fourier analysis tools are readily available for this setting, independence between features is usually not a reasonable assumption to make in practice.
Indeed, it often happens that some features are correlated, e.g., a person's height and weight. 
By loosening the product distribution requirement to a log-Lipschitz one, we allow for some dependence between the features.
However, from a robustness point of view, it is sensible to ensure that features are not too dependent on each other (which is also encapsulated by log-Lipschitzness).
Taking this to the extreme, suppose a feature has been duplicated, i.e., there exist indices $i,j$ in $[n]$ such that $x_i=x_j$ for all points in the support of the distribution. 
Then, an instance such that $x_i=\bar{x_j}$ does not represent a meaningful instance of the problem to be learned, and perhaps it would be unfair to require an algorithm to perform well in such cases.
Moreover, it is unclear how one would measure robustness performance in this scenario.

In a sense, log-Lipschitz distributions encapsulate a natural desideratum when considering both robustness guarantees and realistic assumptions on the data, and furthermore provide a sound abstract framework to study robust learnability.

\subsection{A Robustness Guarantee}

We now look at robustly learning monotone conjunctions on $\boolhc$ under log-Lipschitz distributions when the adversary can flip $\log(n)$ bits of the input at test time.
We remark that, when one has access to membership queries, one can easily exactly learn monotone conjunctions over the whole input space: we start with the instance where all bits are 1 (which is always a positive example, as the constant function 0 does not belong to this class), and we can test whether each variable is in the target conjunction by setting the corresponding bit to 0 and requesting the label.
However, robustly learning monotone conjunctions with access to random examples only is not so straightforward, as positive examples, which are more informative than negative ones, could be difficult to come by under the underlying distribution.

We now formally state the main result of this section.

\begin{theorem}
\label{thm:mon-conj-rob}
  Let $\mathcal{D}=\{\mathcal{D}_n\}_{n\in\mathbb{N}}$, where $\mathcal{D}_n$ is a set of
     $\alpha$-$\log$-Lipschitz distributions on $\{0,1\}^n$ for all $n\in\mathbb{N}$.   Then the class of
    monotone conjunctions is $\rho$-robustly learnable with respect to
    $\mathcal{D}$ for robustness function $\rho(n)=O(\log n)$.
\end{theorem}

This theorem combined with Theorem~\ref{thm:mon-conj} shows that $\rho(n)=\log (n)$ is essentially the threshold for the efficient robust learnability of the class \MonConj.
Note that here, and in all future results, we consider the constant $\alpha$ that parametrises the log-Lipschitz distribution to be fixed, and it appears explicitly in the sample complexity.

The main idea to prove Theorem~\ref{thm:mon-conj-rob} is that, on the one hand, it is possible to efficiently exactly learn the target conjunction if its length is logarithmic in the input dimension.
On the other hand, we can otherwise efficiently $\rho$-robustly learn (but not necessarily exactly learn) longer conjunctions as the robustness parameter is logarithmic in the input dimension, and thus the adversary's budget is insufficiently large to cause a label change (with high probability). 
This is  a simple example that shows that robust learning does not necessarily imply exact learning. 

\begin{proof}[Proof of Theorem~\ref{thm:mon-conj-rob}]

  We show that the algorithm $\mathcal{A}$ for PAC-learning monotone
  conjunctions (see Algorithm~\ref{alg:conj-pac} in Chapter~\ref{chap:background}) is a robust learner for an appropriate choice
  of sample size.  We start with the hypothesis $h(x)=\bigwedge_{i\in I_h} x_i$, 
  where $I_h=[n]$. For each example $x$ in $S$, we remove $i$ 
  from $I_h$ if $c(x)=1$ and $x_i=0$.

Let $\D$ be a class of $\alpha$-$\log$-Lipschitz distributions. 
 Let $n\in\mathbb{N}$ and $D\in \mathcal{D}_n$.  Suppose moreover
  that the target concept $c$ is a conjunction of $l$ variables.
  Fix $\varepsilon,\delta>0$. Let $\eta=\frac{1}{1+\alpha}$, 
  and note that by Lemma~\ref{lemma:log-lips-facts},
  for any $S\subseteq[n]$ and $b_S\in\{0,1\}^S$,  we have that 
  $\eta^{|S|}\leq \Prob{x\sim D}{x_i=b}\leq(1-\eta)^{|S|}$.

\paragraph{Claim 1.} 
If
$m \geq \left\lceil \frac{\log n-\log \delta}{\eta^{l+1}}
\right\rceil$ then given a sample $S \sim D^m$, algorithm
$\mathcal{A}$ outputs $c$ with probability at least $1-\delta$.

  \emph{Proof of Claim 1.} Fix $i \in \{1,\ldots,n\}$.  Algorithm
  $\mathcal{A}$ eliminates $i$ from the output hypothesis just in
  case there exists $x\in S$ with $x_{i}=0$ and $c(x)=1$.  Now we
  have $\Prob{x\sim D}{x_{i}=0 \wedge c(x)=1}\geq \eta^{l+1}$
  and hence
\[
\Prob{S\sim D}{\forall  x \in S \cdot i\text{ remains in }I_h}\leq 
(1-\eta^{l+1})^m 
 \leq  e^{-m\eta^{l+1}}
 =  \frac{\delta}{n} \, .
\]
The claim now follows from union bound over $i \in \{1,\ldots,n\}$.

\paragraph{Claim 2.}  If $l \geq \frac{8}{\eta^2}\log(\frac{1}{\varepsilon})$
and $\rho \leq \frac{\eta l}{2}$ then
$\Prob{x\sim D} {\exists z \in B_\rho(x) \cdot c(z)=1}\leq \varepsilon$.

\emph{Proof of Claim 2.}  
Define a random variable $Y=\sum_{i\in I_c} \mathbb{I}(x_i=1)$.  
We simulate $Y$ by the following process. 
Let $X_1,\dots,X_l$ be random variables taking value in $\{0,1\}$, and which may be dependent. 
Let $D_i$ be the marginal distribution on $X_i$ conditioned on $X_1,\dots,X_{i-1}$. 
This distribution is also $\alpha$-$\log$-Lipschitz by Lemma~\ref{lemma:log-lips-facts}, and hence,
\begin{equation}
\label{eqn:marg-bound}
\Prob{X_i\sim D_i}{X_i=1}\leq 1-\eta
\enspace.
\end{equation}

Since we are interested in the random variable $Y$ representing the number of 1's in $X_1,\dots,X_l$, 
we define the random variables $Z_1,\dots,Z_l$ as follows:
\begin{equation*}
Z_k = \left(\sum_{i=1}^k X_i\right)-k(1-\eta)\enspace,
\end{equation*} 
with the convention that $Z_0=0$.
The sequence $Z_0,Z_1, \dots, Z_l$ is a supermartingale with respect to $X_1,\dots,X_l$:
\begin{align*}
\eval{}{Z_{k+1}\given X_1,\dots,X_k}
&=\eval{}{Z_{k}+X_{k+1}'-(1-\eta)\given X'_1,\dots,X'_k}\\
%&=Z_k+\eval{}{\mathbf{1}[X_{k+1}'=1]\given X'_1,\dots,X'_k}-(1-\eta)\\
&=Z_k+\Prob{}{X_{k+1}'=1\given X'_1,\dots,X'_k}-(1-\eta)\\
&\leq Z_k
\enspace. \tag{by (\ref{eqn:marg-bound})}
\end{align*}
Now, note that all $Z_k$'s satisfy $|Z_{k+1}-Z_k|\leq 1$, and that $Z_l=Y-l(1-\eta)$. 
We can thus apply the Azuma-Hoeffding (A.H.) inequality (see Section~\ref{sec:prob-theory} for details) to get 
\begin{align*}
\Prob{}{Y\geq l-\rho}
&\leq \Prob{}{Y\geq l(1-\eta)+\sqrt{2\ln(1/\varepsilon)l}}	\\
&=\Prob{}{Z_l-Z_0\geq \sqrt{2\ln(1/\varepsilon)l}}	\\
&\leq \exp\left(-\frac{\sqrt{2\ln(1/\varepsilon)l}^2}{2l}\right)				\tag{A.H.}\\
&=\varepsilon
\enspace,
\end{align*}
where the first inequality holds from the given bounds on $l$ and $\rho$:
\begin{align*}
l-\rho &=(1-\eta)l + \frac{\eta l}{2} 
               + \frac{\eta l}{2} - \rho \\
          & \geq (1-\eta) l + \frac{\eta l}{2} 
           \tag{since $\rho \leq \frac{\eta l}{2}$}\\
          & \geq  (1-\eta) l + \sqrt{2\log(1/\varepsilon) l} \enspace.
            \tag{since $l \geq \frac{8}{\eta^2}\log(\frac{1}{\varepsilon})$}
\end{align*}
This completes the proof of Claim 2.

We now combine Claims 1 and 2 to prove the theorem.  Define
$l_0 := \max(\frac{2}{\eta}\log n,
\frac{8}{\eta^2}\log(\frac{1}{\varepsilon}))$.  Define
$m:=\left\lceil \frac{\log n-\log \delta}{\eta^{l_0+1}}
\right\rceil$.  Note that $m$ is polynomial in $n$, $\delta$,
$\varepsilon$.  

Let $h$ denote the output of algorithm $\mathcal{A}$ given a sample
$S\sim D^m$.  We consider two cases.  If $l \leq l_0$ then, by
Claim 1, $h=c$ (and hence the robust risk is $0$) with probability at
least $1-\delta$.  If $l_0 \leq l$ then, since $\rho=\log n$, we
have $l \geq \frac{8}{\eta^2}\log(\frac{1}{\varepsilon})$ and
$\rho \leq \frac{\eta l}{2}$ and so we can apply Claim 2.  By Claim
2 we have
\[ \roblosse(h,c) \leq \Prob{x\sim D}{\exists z \in B_\rho(x) \cdot c(z)=1} \leq \varepsilon \]

\end{proof}

Now that we have shown robust learnability against logarithmically-bounded adversaries, we define robustness thresholds, a term that will be used throughout this thesis.

\begin{definition}[Robustness Threshold]
\label{def:rob-threshold}
A \emph{robustness threshold} for concept class $\C$ and distribution family $\D$ is an adversarial budget function $\rho:\N\rightarrow\R$ of the input dimension $n$ such that, if the adversary is allowed perturbations of magnitude $\rho(n)$, then there exists a sample-efficient $\rho(n)$-robust learning algorithm for $\C$ under $\D$, and if the adversary's budget is $\omega(\rho(n))$, then such an algorithm does not exist.
\end{definition}

As a consequence of Theorems~\ref{thm:mon-conj} and~\ref{thm:mon-conj-rob}, we get the following result.

\begin{theorem}
\label{thm:rob-threshold-mon-conj}
The robustness threshold for {\MonConj} under log-Lipschitz distributions is $\rho(n)=\log(n)$.
\end{theorem}

As we will discuss further in the next chapter,  we finish by remarking that the  method employed to show efficient robust learnability in this section is to use a (proper) PAC-learning algorithm as a \emph{black box} and control the accuracy parameter to ensure robustness to evasion attacks.

\section{Summary}

This chapter offered a thorough discussion on the choice of robust risk for evasion attacks. 
Settling on the exact-in-the-ball robust risk, we then showed that the standard PAC and robust learning settings are fundamentally different in that, contrary to the former, the latter requires distributional assumptions to ensure (efficient) learnability.
But even when considering the natural uniform distribution, we showed that the efficient robust learning of monotone conjunctions, a very simple concept class, cannot be guaranteed against an adversary that has a superlogarithmic perturbation budget.
However, we showed that this result is tight: guarantees can be obtained for logarithmically-bounded adversaries under log-Lipschitz distributions.
Overall, these results show that the adversarial budget is a fundamental quantity in determining the sample complexity of robust learning under these distributional assumptions.
This motivates the term \emph{robustness threshold}, adversarial budget functions characterizing efficient robust learnability for a given distribution family.
The next chapter will study the robustness thresholds of various concept classes under smoothness assumptions.

\chapter[Robustness Thresholds: Random Examples]{Robustness Thresholds with Random Examples}
\label{chap:rob-thresholds}

In this chapter, we further explore the \emph{robustness thresholds} (Definition~\ref{def:rob-threshold}) of various concept classes under distributional assumptions, again with respect to the \emph{exact-in-the-ball} notion of robustness.
In Section~\ref{sec:exact}, we start by showing that exact (and thus robust) learning is possible for parities under log-Lipschitz distributions and majority functions under the uniform distribution.
We then turn our attention to decision lists, where we show a robustness threshold of $\log n$ under log-Lipschitz distributions in Section~\ref{sec:dl}.
Section~\ref{sec:dt} concludes the technical contributions of this chapter by relating the standard and robust errors of decision trees.

In Section~\ref{sec:dl}, we demonstrate the robust learnability of $k$-decision lists by first looking at the case where $k=1$, which forms the foundation of the generalization to $k$-\dl. 
This simpler case provides a substantial intuition behind the reasoning for the more complex case of $k>1$, while also giving better sample complexity bounds in the specific case $k=1$.
We then distinguish two set-ups for the case  $k>1$: (i) $2$-{\dl} and monotone $k$-{\dl}, and (ii) non-monotone $k$-\dl.
While the second case is more general, the first one results in better sample complexity upper bounds. Indeed, the dependence on $k$, which we consider to be a fixed constant, in the degree of the former is $\poly(k)$, while it is $2^{\poly(k)}$ for the latter.

This chapter concludes with Section~\ref{sec:rt-summary}, which summarizes the results of this chapter.
As explained in more detail in that section, the methods in this chapter can be viewed as relating the mass of the error region between the target and hypothesis and the $\rho$-expansion of the error region, where $\rho$ is the adversarial budget. Indeed, the general, unifying approach in proving robustness thresholds in this work is to express the discrepancy between two functions (i.e., instances where they disagree) as a logical formula $\varphi$.
We then relate the size of the set of satisfying assignments of $\varphi$ to the size of its expansion.
This means that we can control the robust risk by controlling the standard risk, thus allowing the use of standard PAC-learning algorithms as black boxes for robust learning, provided the adversary is logarithmically-bounded.

\section{Exact Learning}
\label{sec:exact}

In the previous chapter, monotone conjunctions of sufficiently large length satisfied a certain form of stability, ensuring that one could obtain robust learning guarantees without the need to exactly learn them.
In this section, we turn our attention to unstable concept classes, where one in general cannot ensure robustness without having learned the target exactly.

\subsection{Parity Functions}

%\todo{parity definition}

In this section, we show that the concept class $\parity$ of parity functions are efficiently exactly learnable under log-Lipschitz distributions.
As these distributions have support on the whole input space, it follows that this implies efficient robust learning of parities.

Recall that parity functions are of the form $f_I(x)=\sum_{i\in I} x_i \bmod 2$, where $I\subseteq [n]$. 
Note that exact learning is necessary under our notion of robustness: if $I$ is non-empty, then every instance is on the decision boundary, as it suffices to flip a single bit in $I$ to cause $f_I$ to change label.
The idea to show robust learnability of parity functions is to show that, for a class of $\alpha$-log-Lipschitz distributions, a proper PAC-learning algorithm can be used as a black box for exact learning.

\begin{theorem}
\label{thm:parity-uniform}
$\parity$ is exactly learnable under $\alpha$-log-Lipschitz distributions.
\end{theorem}

\begin{proof}
Consider a proper PAC-learning algorithm $\A$ with sample complexity $\poly(\cdot)$ for $\parity$ (see e.g., \cite{goldberg2006some}).
Let $\D$ be a family of $\alpha$-$\log$-Lipschitz distributions and let $D\in\D$ be arbitrary.
Let $\epsilon, \delta>0$ be the accuracy and confidence parameters, $n$ be the input dimension, and $c(x)=\sum_{i\in I_c} x_i \bmod 2$ be the target concept.
For any $h(x)=\sum_{i\in I_h}x_i \mod 2$, letting $I_\Delta=\{i\in[n]\given i\in I_c\Delta I_h\}$ be the symmetric difference between the sets $I_c$ and $I_h$, we have that if $I_\Delta$ is non-empty,
\begin{equation*}
	\Prob{x\sim D}{h(x)\neq c(x)}
	= \Prob{x\sim D}{\sum_{i\in I_\Delta} x_i \bmod 2= 1}
	\geq \frac{1}{1+\alpha}
	\enspace.
\end{equation*}
This follows from Lemma~\ref{lemma:log-lips-facts}(ii): for some $i\in I$, the marginal of $x_i$ conditioned on the points $\set{x_j\given j\in I\setminus \{i\}}$ is also $\alpha$-log-Lipschitz.
Then no matter what value the points in $\set{x_j\given j\in I\setminus {i}}$ take, we know that the probability that $x_i$ causes a mismatch in parity is bounded below by $1/(1+\alpha)$ by Lemma~\ref{lemma:log-lips-facts}(i).
%In the case $I$ is empty, but $b\neq b'$, $\Prob{x\sim D}{h(x)\neq c(x)}=1$. 
Then, any proper PAC-learning algorithm\footnote{E.g., performing Gaussian elimination on the matrix $\mathbf{X}$ of examples and label vector $\mathbf{y}$ and returning a possible solution vector $\mathbf{z}\in\{0,1\}^n$ (i.e., $\mathbf{Xz}=\mathbf{y}$), where $a_i=1$ if and only if $\mathbf{z}_i=1$, would be a proper learning algorithm.} 
with accuracy parameter $\epsilon<1/(1+\alpha)$ will return $c$ with probability at least $1-\delta$.
\end{proof}

We then have the following corollary.

\begin{corollary}
$\parity$ is $\rho$-robustly learnable under $\alpha$-log-Lipschitz distributions for any $\rho$.
\end{corollary}

\subsection{Majority Functions}
\label{sec:maj}

We will now work in the input space $\X=\{-1,1\}^n$ and with majority functions.
For $I\subseteq [n]$, define $\maj_I:\X\rightarrow\{-1,1\}$ as $\maj_I(x)=\sgn\left(\sum_{i\in I} x_i\right)$.
For simplicity, we will suppose that $|I|$ is odd.
We will show that we can exactly learn majority functions, and thus robustly learn them, with the use of Fourier analysis.
We give an overview of Fourier analysis in the boolean hypercube in Section~\ref{sec:fourier-analysis}.
For a function $f$, denote by $\widehat{f}(S)$ its Fourier coefficient on subset $S\subseteq[n]$. 
If $S$ is a singleton $\set{i}$, we simply write $\widehat{f}(i)$.

The idea is to show that we can exactly learn the Fourier coefficients $\widehat{\maj_I}(i)$ of singleton sets  $\{i\}$ for $1\leq i \leq n$ of any majority function with arbitrarily high confidence. 
We note that some of the results below are already known, but have not been applied to robustness. We have included proofs for completeness.

\begin{theorem}
\label{thm:maj-fc}
Let $\maj_I:\X\rightarrow\{-1,1\}$ be a majority function. 
Then for $i\in[n]$, we have that $\widehat{\maj}_I(i)\geq \sqrt{2/\pi n}$ if $i\in I$ and 0 otherwise.
\end{theorem}

This result, which is part of the Fourier analysis folklore and whose proof is included below for completeness, gives us a simple algorithm to learn majority functions.
Indeed, Theorem~\ref{thm:maj-fc} states that the Fourier coefficient of a bit in the majority is sufficiently large (bounded away from 0) to distinguish it from bits that are not in the majority function. 

\begin{proof}[Proof of Theorem~\ref{thm:maj-fc}]
Since majorities are monotone functions, for $i\in[n]$, we have that
\begin{align*}
\widehat{\maj_I}(\{i\})
&=\textbf{Inf}_i[\maj_I]
\enspace,
\end{align*} 
where 
$\textbf{Inf}_i[f]$ is the influence of the $i$-th bit on the function $f$, defined as $$\Prob{x\sim\{-1,1\}^n}{f(x)\neq f(x^{\oplus i})}\enspace,$$ and $x^{\oplus i}$ is the vector resulting in flipping the $i$-th bit of $x$. 
This result follows from that fact that, for a monotone function, the Fourier coefficient of a singleton $\{i\}$ is simply the influence of bit $i$ (see Proposition~\ref{prop:monotone-influence}).
Clearly, if $i\not\in I$, then $\widehat{\maj_I}(\{i\})=0$.
Otherwise, we need to compute the probability that exactly half of the bits in $I\setminus\{i\}$ are 1.
Letting $X=\sum_{j\in I\setminus\{i\}} \textbf{1}[x_j=1]$ and $k=|I|-1$,
\begin{equation}
\label{eqn:prob-half-bits-one}
\Prob{x\sim\{-1,1\}^n}{X=k/2}
={k \choose k/2}\left(\frac{1}{2}\right)^k
\enspace.
\end{equation}
Using the inequality $\sqrt{2\pi n}\left(\frac{n}{e}\right)^n\leq n!\leq \sqrt{2\pi n}\left(\frac{n}{e}\right)^n e^{\frac{1}{12n}}$ \citep{robbins1955remark}, we can derive a lower bound for Equation~(\ref{eqn:prob-half-bits-one}) and show that
\begin{equation}
\Prob{x\sim\{-1,1\}^n}{X=k/2}\geq \sqrt{\frac{2}{\pi k}}\geq \sqrt{\frac{2}{\pi n}}\enspace,
\end{equation}
whenever $i\in I$.
\end{proof}

It is possible to estimate the Fourier coefficients of a function to a high accuracy, provided one has enough data (see Chapter~3 in \cite{odonnell2014analysis} for more details).
The theorem below shows that we only need to look at the Fourier coefficients of singletons in order to identify the target majority under the uniform distribution.

\begin{theorem}
$\majority$ is exactly learnable under the uniform distribution.
\end{theorem}

\begin{proof}
Suppose we have a sample $\{(x^{(j)},y^{(j)})\}_{j=1}^m$ where the $x^{(j)}$'s are taken i.i.d. from the uniform distribution. 
We can use the Fourier coefficient's empirical estimates $\widetilde{\maj}_S(i)=\frac{1}{m}\sum_{j=1}^m y^{(j)}x_i^{(j)}$ to construct an estimate $\widetilde{S}$ of $S$ as follows.
For accuracy parameter $\epsilon=\frac{1}{2\sqrt{\pi n}}$, if $\widetilde{\maj}_S(\{i\})\geq \frac{1}{\sqrt{\pi n}} - \epsilon$, then $i\in \widetilde{S}$, and otherwise $i\not\in \widetilde{S}$.
We output the function $\maj_{\widetilde{S}}$. 

What is the probability that $\widetilde{S}\neq S$?
By a  standard application of the Chernoff bound, we can get an estimate of $\widetilde{\maj}_S(i)$  with accuracy $\pm\epsilon$ and confidence $1-\delta$ with $O(\frac{1}{\epsilon^2}\log(\frac{1}{\delta}))$ samples.
For a given index $i$, we choose accuracy $\epsilon=1/(2\sqrt{\pi n})$ and confidence $\delta/n$, and by a union bound over all indices, we know that $O(n\log(\frac{n}{\delta}))$ samples are sufficient to guarantee that $\widetilde{S}\triangle S\neq \emptyset$ with probability at most $\delta$, and we are done.
\end{proof}

\begin{remark}
We can easily extend the reasoning used for majority functions to linear threshold functions with weights in $\{-1,0,1\}$.
It suffices to notice that linear functions are \emph{unate} in all directions, meaning that for all $i$, either $f(x^{(i\rightarrow-1)})\leq f(x^{(i\rightarrow 1)})$ for all $x$ (i.e., $f$ is monotone in the $i$-th direction) or if  $f(x^{(i\rightarrow-1)})\geq f(x^{(i\rightarrow 1)})$ for all $x$ (i.e., $f$ is antimonotone in the $i$-th direction), and use the following theorem:

\begin{proposition}
\label{prop:fc-inf}
For $i\in[n]$ and $f:\{-1,1\}^n\rightarrow \{-1,1\}$,
\begin{equation*}
|\hat{f}(i)| \leq \mathbf{Inf}_i[f]\enspace,
\end{equation*}
with equality if and only if $f$ is unate in the $i$-th direction.
\end{proposition}

\begin{proof}[Proof of Proposition~\ref{prop:fc-inf}]
Let $f$ be unate in all directions, and for a fixed $i$ let $A_i=\{x\;|\; f(x) \neq f(x^{\oplus i}) \}$.
\begin{align*}
\hat{f}(i)
&= \frac{1}{2^n}\sum_x f(x)x_i\\
&= \frac{1}{2^n}\sum_{x\in A_i : x_i=1}  f(x) - \sum_{x\in A_i : x_i=-1}  f(x)\\
&=\frac{\pm |A_i|}{2^n}\\
&= \pm \mathbf{Inf}_i[f]\enspace.
\end{align*}
\end{proof}

Then we can use the same reasoning as in the majority case to argue that for $f(x)=\sgn\left(\sum_i a_i x_i\right)$, the Fourier coefficient of $i$ for $a_i\neq0$ will be at least $\Theta(1/\sqrt{n})$ away from $0$ and that $\sgn(a_i)=\sgn(\hat{f}(i))$, implying that we can exactly learn this class of functions.
\end{remark}
%
%\question{Can we extend these results to product distributions?}
%
%\question{Should I talk about the relationship between the Fourier spectrum and standard error, and where the reasoning breaks for robustness?}

\section{Decision Lists}
\label{sec:dl}

From a robust learnability point of view, the  concept classes from the previous section are not very interesting, since we simply learn them exactly, and thus robustly, for any robustness parameter.
In this section, we study the class of \emph{decision lists}, which is much more expressive than (monotone) conjunctions.
Decision lists were introduced in \cite{rivest1987learning}, where they were shown to be efficiently PAC learnable.
We denote by $k$-{\dl}  the class of decision lists with conjunctive clauses of size at most $k$ at each decision node.  
Decision lists generalize formulas in disjunctive normal form (DNF) and conjunctive normal form (CNF): $k\text{-}\mathsf{DNF}\cup k\text{-}\mathsf{CNF} \subset k\text{-\dl}$, where $k$ refers to the number of literals in each clause.
Formally, a decision list is a list $L$ of pairs
\begin{equation*}
(K_1,v_1),\dots,(K_r,v_r)\enspace,
\end{equation*}
where $K_j$ is a term in the set of all conjunctions of size at most $k$ with literals drawn from $\set{x_1,\bar{x_1},\dots,x_n,\bar{x_n}}$, $v_j$ is a value in $\set{0,1}$, and $K_r$ is $\mathtt{true}$.
The output $f(x)$ of $f$ on $x\in\boolhc$ is $v_j$, where $j$ is the least index such that the conjunction $K_j$ evaluates to $\mathtt{true}$ on $x$.
For more details on decision lists and their PAC-learning algorithm, see Section~\ref{sec:pac-algos}.

Showing the efficient robust learnability of decision lists against logarithmically-bounded adversaries relies on first getting guarantees for simpler cases: 1-decision lists, 2-decision lists and \emph{monotone} $k$-decision lists.
As we will discuss later, reducing the robust learnability of $k$-{\dl} to the robust learnability of 1-{\dl} apparently cannot be done in the same way as in the standard PAC setting.
Robustly learning $k$-decision lists for $k \geq 2$ requires a totally new argument based on the hypergraph structure of $k$-$\mathsf{CNF}$ formulas.
Our first approach to show robust learnability (which yields better sample complexity bounds) can only be applied to 2-decision lists and \emph{monotone} $k$-{\dl}. 
Proving the robust learnability of non-monotone $k$-{\dl} requires a different approach %from the monotone case, this time 
relying on a combinatorial argument and induction. 

\subsection{1-Decision Lists}
 
In this section, we show that 1-decision lists are robustly (but not necessarily exactly) efficiently learnable for robustness parameter $\rho=O(\log n)$ under log-Lipschitz distributions. 
At the heart of the result lies a similar argument to the one from the previous chapter showing the robust learnability of monotone conjunctions against a logarithmically-bounded adversary.
Indeed, the discrepancy between two 1-decision lists can be represented as a conjunction, and the argument from Chapter~\ref{chap:def-adv-rob} can easily be extended to this setting.
Note that, as in Chapter~\ref{chap:def-adv-rob}, the log-Lipschitz parameter $\alpha$ is considered as a constant and appears explicitly in the sample complexity upper bounds.

This section will be dedicated to proving the following theorem.

\begin{theorem}
\label{thm:1dl-rob-learn}
The class 1-{\dl} is  efficiently $\rho$-robustly learnable, i.e. with polynomial sample complexity, under the class of $\alpha$-$\log$-Lipschitz distributions with robustness threshold $\rho=\Theta(\log n)$.
\end{theorem}

Recall that we have already shown in Chapter~\ref{chap:def-adv-rob} that an adversary with a perturbation budget $\omega(\log n)$ renders efficient robust learning impossible for monotone conjunctions under the uniform distribution.
Since monotone conjunctions are subsumed by 1-decision lists, the lower bound of Theorem~\ref{thm:mon-conj} extends to 1-{\dl}.

We now state the main result of this section.

\begin{theorem}
\label{thm:dl-rob}
  Let $\mathcal{D}=\{\mathcal{D}_n\}_{n\in\mathbb{N}}$, where $\mathcal{D}_n$ is a set of
     $\alpha$-$\log$-Lipschitz distributions on $\{0,1\}^n$ for all $n\in\mathbb{N}$.   Then the class of
    1-decision lists is $\rho$-robustly learnable with respect to
    $\mathcal{D}$ for robustness function $\rho(n)=\log n$.
\end{theorem}

As previously mentioned, Theorem~\ref{thm:1dl-rob-learn} follows from Theorem~\ref{thm:dl-rob} combined with Theorem~\ref{thm:mon-conj}. 
Note that, while the result is stated for $\rho=\log n$, it can straightforwardly be extended to $\rho=C\log n$ for some constant $C$, at the cost of a larger polynomial degree for the sample complexity upper bound.
To prove the above result, we first need the following definitions and lemmas.

 \begin{definition}
  Given a 1-decision list $c=\left((l_1,v_1),\dots,(l_r,v_r)\right)$ and
  $x\in\X$, we say that \emph{$x$ activates node $i\in\{1,\ldots,r\}$}
  in $c$
  if $x \models l_i$ and $x \not\models l_j$ for all $j$ such that
  $1\leq j < i$.
\end{definition}

The following definition will play a role in our analysis of
1-decision lists.

\begin{definition}
  Let $c$ and $h$ be decision lists.  Given $d\in \mathbb{N}$, we say
  that $h$ is \emph{consistent with $c$ up to depth $d$}, denoted
  $c=_d h$, if $c(x)=h(x)$ for all $x\in \X$ such that the nodes in
  $c$ and $h$ respectively activated by $x$ have level at most $d$.
\label{def:DL-consistent}
\end{definition}

Note that, given a 1-decision list
  $f=\left((l_1,v_1),\dots,(l_r,v_r)\right)$, we can assume without
  loss of generality that $f$ is in a minimal representation, namely
  that
\begin{itemize}
\item[(i)] A literal $l$ only appears once in the list (otherwise we
  can remove all occurrences of $l$ except the first
  one without changing the output of the list),
\item[(ii)] There does not exist $1\leq i<j\leq d$ such that $l_i=\bar{l_j}$, as otherwise it is impossible to go past $l_j$ in the list (note that if there exists $1\leq i<d$ such that $l_d=\bar{l_i}$, we can simply set $l_d$ to true).
\end{itemize}
We will henceforth assume that all decision lists are in their minimal representation.

Now, under log-Lipschitz distributions, if two 1-decision lists
  have an error below a certain threshold, they must be consistent up
  to a certain depth.

\begin{lemma}
\label{lemma:consistent-dl}
Let $h,c\in 1$-{\dl} and let $D$ be an $\alpha$-log-Lipschitz distribution. 
If $\Prob{x\sim D}{h(x)\neq c(x)}<\left(1+\alpha\right)^{-2d}$, then $c=_d h$.
\end{lemma}

\begin{proof}
We will show the contrapositive.
Let $c=\left((l_1,v_1),\dots,(l_r,v_r)\right)$ and
$h=\left((l_1',v_1'),\dots,(l_s',v_{s}')\right)$ be 1-decision lists.
Let $c\neq_d h$, meaning that there exists $x\in \X$ such that $x$
activates node $i_0$ in $c$ and node $i_1$ in $h$ such that $i_0,i_1\leq d$ and
$v_{i_0}\neq v'_{i_1}$.  In particular, the following must hold
\begin{align*}
& x \models \neg l_i \qquad 1\leq i<i_0\enspace,\\
& x\models \neg l_i'\qquad 1\leq i<i_1\enspace,\\
& x\models  l_{i_0} \wedge  l_{i_1} \enspace.
\end{align*}
By Lemma~\ref{lemma:log-lips-facts},
the probability of drawing such an $x$ is at least
$\left(1+\alpha\right)^{-i_0-i_1}\geq \left(1+\alpha\right)^{-2d}$.
\end{proof}

The next step in the argument is to derive an upper bound on the
robust loss $\roblosse(c,h)$ under the condition that $c=_d h$.  To
this end, the key technical lemma, which pertains to the $\rho$-expansion of the set of satisfying assignments of a conjunction on $\boolhc$, is as follows:

\begin{lemma}
\label{lemma:rob-risk-dl}
Let $D$ be an $\alpha$-$\log$-Lipschitz distribution on the
$n$-dimensional boolean hypercube and let $\varphi$ be a 
conjunction of $d$ literals.
Set $\eta=\frac{1}{1+\alpha}$.
Then for all $0<\varepsilon<1/2$,
if $d\geq \max\left\{
  \frac{4}{\eta^2}\log\left(\frac{1}{\varepsilon}\right) ,
  \frac{2\rho}{\eta} \right\}$, then 
$\Prob{x\sim D}{\left(\exists y \in B_\rho(x) \cdot y \models
    \varphi\right)} \leq \varepsilon$.
\end{lemma}

The proof of the above lemma, which is in essence nearly identical to the proof of Claim~2 in Theorem~\ref{thm:mon-conj-rob}, is included in Appendix~\ref{app:rob-risk-dl} for completeness.

We are now ready to prove that $1$-DL is efficiently $\rho$-robustly learnable for $\rho=\log n$.\\

\begin{proof}[Proof of Theorem~\ref{thm:dl-rob}]
Let $\A$ be the  (proper) PAC-learning algorithm for 1-DL as in \cite{rivest1987learning}, with sample complexity $\poly(\cdot)$.
Fix the input dimension $n$, target concept $c$  and distribution $D\in \D_n$, and let $\rho=\log n$.
Fix the accuracy parameter $0<\varepsilon<1/2$ and confidence parameter $0<\delta<1/2$ and let $\eta=1/(1+\alpha)$.
Let $d_0=\max\left\{\frac{2}{\eta}\log n,\frac{4}{\eta^2}\log\frac{2}{\varepsilon}\right\}$ and let $m=\lceil\poly(n,1/\delta,\eta^{-2d_0})\rceil$, and note that this is polynomial in $n$, $1/\delta$ and $1/\varepsilon$.

Let $S\sim D^m$ and $h=\A(S)$.  Then
$ \Prob{x\sim D}{h(x)\neq c(x)}<\eta^{2d_0}$ with probability at least
$1-\delta$.  But, by Lemma~\ref{lemma:consistent-dl},
$ \Prob{x\sim D}{h(x)\neq c(x)}<\eta^{2d_0}$ implies that then
$c=_{d_0} h$.  Hence $c=_{d_0} h$ with probability at least
$1-\delta$.  Then, to cause an error, an adversary must activate a
node at depth greater than $d_0$ in either $h$ or $c$.

We now apply
Lemma~\ref{lemma:rob-risk-dl} to show that the probability to
activate a node at depth greater than $d_0$ in $c$ is at most
$\varepsilon/2$ (and symmetrically for $h$), which suffices to conclude
that $\roblosse(c,h)<\varepsilon$ with probability at least $1-\delta$.
Indeed, writing $c=((l_1,v_1),\ldots,(l_r,v_r))$ and 
$\varphi:=\neg l_1\wedge \cdots \wedge \neg l_{d_0}$, observe that 
\begin{gather}
  \Prob{x\sim D}{\left(\exists z\in B_\rho(x) \cdot z \models
    \varphi\right)}
\label{eq:activate}
\end{gather}
is precisely the probability for the adversary to be able to activate
a node at depth $>d_0$ in $c$.  Now to apply
Lemma~\ref{lemma:rob-risk-dl} we note that
by definition of $d_0$ we have
$d_0\geq \frac{4}{\eta^2}\log\frac{2}{\varepsilon}$, and,
since $\rho=\log n$, we furthermore have $d_0 \geq \frac{2\rho}{\eta}$; thus the
lemma implies that Equation~\ref{eq:activate} is at most $\varepsilon/2$, as we
require. 
\end{proof}

\subsection{Generalizing from 1-DL to $2$-DL and Monotone $k$-DL}
\label{sec:mon-k-dl}
This section is concerned with robust learning for $k$-{\dl}.
In the non-adversarial setting, learnability of $k$-{\dl} can be
reduced to learnability of 1-{\dl} (see Section~\ref{sec:pac-algos} for details).  We start by observing that 
it is not straightforward to apply this reduction in the presence of
an adversary.

The classical reduction of learning $k$-{\dl} to 1-{\dl} involves an
embedding $\Phi: \X_n \rightarrow \X_{n'}$, for $n':=O(n^k)$,
that maps valuations of a collection of $n$ propositional variables to
valuations of the collection of $k$-clauses over these variables.
Then, for any function $c:\X_n\rightarrow\{0,1\}$ computed by a
$k$-decision list, there is a function $c':\X_{n'}\rightarrow \{0,1\}$
computed by a 1-decision list such that $c'\circ \Phi = c$.  The image under $\Phi$ of an $\alpha$-log-Lipschitz
distribution $D$ on $\X_n$ remains log-Lipschitz on $\X_{n'}$, albeit
with a slightly larger constant (see Lemma~\ref{lemma:matching-embedding}).  The problem is that the map $\Phi$
is not Lipschitz with respect to the Hamming metric -- indeed the image
under $\Phi$ of two points with Hamming distance $\log n$ in $\X_n$
can have distance $\Omega(n)$ in $\X_{n'}$, which is not logarithmic
in the dimension $n'=O(n^k)$.

We therefore take a direct approach to establishing robust
learnability of $k$-{\dl} in this section.  The argument follows a
similar pattern to the previous section, in particular involving a
suitable generalization of Lemma~\ref{lemma:rob-risk-dl}.  There are
new ingredients relating to the hypergraph structure of propositional
formulas in conjunctive normal form.
The argument for the consistency over a given \emph{depth} from 1-{\dl} will be generalized to consistency over \emph{covers} of a certain size.
Establishing this result can be done in the case of 2-{\dl} and monotone
$k$-{\dl}, through a resolution closure argument that cannot be generalized to non-monotone $k$-{\dl}, which will be explained in more details below.
However, we later show that we can use similar tools, together with an induction argument to extend the result to non-monotone $k$-{\dl}, at the cost of a larger (but still polynomial) sample complexity.

We start with some background on propositional logic.  We regard a
formula $\varphi$ in conjunctive normal form (CNF) as being a set of
clauses, with each clause being a set of literals.  A $k$-CNF is a CNF 
formula where all clauses contain at most $k$ literals. 
For two disjunctive clauses $K_1:=a_1\vee\dots\vee a_m\vee c$ and $K_2:=b_1\vee\dots\vee b_n\vee \bar{c}$, the \emph{resolution} rule implies the disjunctive clause $K:=a_1\vee\dots\vee a_m\vee b_1\vee\dots\vee b_n$. 
$K$ is called the \emph{resolvent} of clauses $K_1$ and $K_2$.

\begin{definition}[Resolution Closure]
We say that
$\varphi$ is \emph{closed under resolution} if, for any two clauses in
$\varphi$, their resolvent also belongs to $\varphi$. The \emph{resolution
closure} of CNF formula $\varphi$, denoted $\mathrm{Res}^*(\varphi)$,
is the smallest resolution-closed set of clauses that contains
$\varphi$.
\end{definition}

We can consider a CNF formula as a hypergraph whose vertices are
literals and whose hyperedges are clauses.  
Recall that a hypergraph $G$ is a set $V(G)$ of vertices and a set $E(G)$ of hyperedges, where a hyperedge is a set $\set{v_1,\dots,v_l}$ of vertices in $V(G)$.
With this identification
in mind, define a \emph{cover} of a CNF formula $\varphi$ as a set of
literals $C$ such that every clause in $\varphi$ contains a literal
from $C$ (i.e., a set of vertices in $V(G)$ such that every edge contains a vertex in $C$).  
Note that if all the literals in a given cover are true (which in general may not be possible), this represents a satisfying assignment of $\varphi$.
Define also a \emph{matching} of $\varphi$ to be a set $M$
of clauses such that no two clauses in $M$ contain the same literal (i.e., a set of edges from $E(G)$ such that no two edges share a vertex).
By a well known result for hypergraphs, for a minimal cover $C$ and
maximal matching $M$ we have that $|C|\leq k |M|$, where $k$ is the
maximum number of literals in any clause of $\varphi$ \citep{furedi1988matchings}.
Assume now that $\varphi$ is closed under resolution.  We claim that a
minimal cover is satisfiable as a set of literals. 

\begin{claim}
Let $\varphi$ be a CNF formula that is closed under resolution.
Then a minimal cover $C$ in $\varphi$ is satisfiable as a set of literals.
\end{claim}

\begin{proof}
Suppose for a
contradiction that $C$ is a minimal cover that is not satisfiable,
i.e., such that $p,\neg p \in C$ for some variable $p$.  By minimality
of $C$, $\varphi$ contains clauses $\{p\}\cup f$ and
$\{\neg p\} \cup f'$ such that $C$ intersects neither $f$ nor $f'$. But
then the resolvent $f\cup f' $ is also a clause of $\varphi,$ and
since $C$ is a cover we must have that $C$ meets $f\cup f'$---a
contradiction.  The claim is established.
\end{proof}

Now, for given depths $i,j$ in the target and hypothesis decision lists, we define a formula expressing exits at depths $i$ and $j$, respectively.

\begin{definition}
Fix $c,h \in k$-{\dl}, where  $c=((K_1,v_1),\ldots,(K_r,v_r))$ and
$h=((K'_1,v'_1),\ldots,(K'_s,v'_s))$ and the clauses $K_i,K_i'$ are 
conjunctions of $k$ literals. Given $i\in\{1,\ldots,r\}$ and 
$j \in \{1,\ldots,s\}$, define a CNF formula $\varphi^{(c,h)}_{i,j}$ 
by writing
\[ \varphi^{(c,h)}_{i,j} := \mathrm{Res}^*((\neg K_1 \wedge \cdots \wedge \neg
  K_{i-1} \wedge K_i) \wedge ( \neg K'_1 \wedge \cdots \wedge \neg
  K'_{j-1}\wedge K'_j)) \, . \]
\end{definition}
Notice that the formula $\varphi^{(c,h)}_{i,j}$ represents the set of
inputs $x\in \X$ that respectively activate vertex $i$ in $c$ and
vertex $j$ in $h$.

Our reliance on the following proposition is the reason that the results in this
section apply only to the classes 2-{\dl} and monotone $k$-{\dl},

\begin{proposition}
  Let $c,h \in k$-{\dl}.  Then $\varphi^{(c,h)}_{i,j}$ is a $k$-CNF
  formula for all $i$ and $j$ in case either $k=2$ or $c$ and $h$ are
  both monotone.
  \label{prop:stay-k}
\end{proposition}
\begin{proof}
  If $k=2$ then $\varphi^{(c,h)}_{i,j}$ is the resolution closure of a
  2-CNF formula, which remains a 2-CNF formula.  Similarly,  if $c$ and $h$ are
  monotone then $\varphi^{(c,h)}_{i,j}$ is the resolution closure of a
  $k$-CNF in which positive literals only appear in singleton
  clauses.  It is clear that the latter is again a $k$-CNF
  formula.
\end{proof}

\begin{remark}
\label{rmk:res-closure}
It is easy to construct an example of a non-monotone $k$-CNF whose resolution closure is not a $k$-CNF:  the  3-CNF 
$\varphi:=(x_1 \vee x_2 \vee x_3) \wedge (\neg x_1 \vee x_4 \vee x_5)$
has resolvent $( x_2 \vee x_3 \vee x_4 \vee x_5)$, so $\mathrm{Res}^*(\varphi)$ is a 4-CNF.
\end{remark}

We now have the following definition,
in the spirit of consistency over a given depth for 1-{\dl}
(Definition~\ref{def:DL-consistent}).

\begin{definition}
  Given $s \in \mathbb{N}$, we say that $c,h \in k$-{\dl} are
  \emph{equivalent to cover-size $s$}, denoted $c \equiv_s h$, if
  $c(x)=h(x)$ for all $x\in \X$ and for all nodes $i,j$ such that
  $\varphi^{(c,h)}_{i,j}$ has a cover of size at most $s$ and
  $x\models \varphi^{(c,h)}_{i,j}$.
    \label{def:cover-depth}
\end{definition}

Next we argue that if the discrepancy between $c$ and $h$ is sufficiently
small then they are equivalent to a suitably large cover size.

\begin{lemma}
\label{lemma:consistent-dl-2}
Let $D$ be an $\alpha$-log-Lipschitz distribution and let $c$ and $h$
be decision lists.  If
$\Prob{x\sim D}{h(x)\neq c(x)}<\left(1+\alpha\right)^{-s}$ then
$c\equiv_s h$.
\end{lemma}

\begin{proof}
  We prove the contrapositive.  Suppose $c\not\equiv_s h$.  By
  definition, there exist $i,j$ such that $\varphi^{(c,h)}_{i,j}$ has a
  minimum satisfiable cover $C$ of size at most $s$ and $v_i\neq v'_j$.  In
  particular, we have that $c(x)\neq h(x)$ for all $x \in \X$ that
  satisfy $\varphi^{(c,h)}_{i,j}$.  But the probability that $x\sim D$
  satisfies $\varphi^{(c,h)}_{i,j}$ is at least the probability that $x$
  satisfies $C$.  Since $C$ is minimal it does not contain
  complementary literals.  Hence, the probability that $x\sim D$
  satisfies $C$ is at least $\left(1+\alpha\right)^{-s}$ by 
  Lemma~\ref{lemma:log-lips-facts}.
\end{proof}

The following is a generalization of Lemma~\ref{lemma:rob-risk-dl}.
\begin{lemma}
  \label{lemma:rob-risk-mon-dl-2}
  Let $\varphi$ be a $k$-CNF formula that has no cover of size $s$.
  Let $D$ be an $\alpha$-log-Lipschitz distribution on
  valuations for $\varphi$.  Let $0<\varepsilon<1/2$ be arbitrary
  and set $\eta:=\left(\frac{1}{1+\alpha}\right)^k$.  If
%  $s \geq \max\left\{\frac{8k^2}{\eta^2} \left( 2\log \left(
%        \frac{8k}{\eta^2}\right) + \log\left(
%        \frac{1}{\varepsilon}\right)\right),\frac{2k\rho}{\eta}\right\}$
%  then
$\frac{s}{k(k+1)} \geq \max \left\{
  \frac{4}{\eta^2}\log\left(\frac{1}{\varepsilon}\right),
  \frac{2\rho}{\eta}\right \}$ then
  $\Prob{x\sim D}{\exists z \in B_\rho(x) \cdot z \models \varphi}
    \leq \varepsilon$.

\end{lemma}

To prove Lemma~\ref{lemma:rob-risk-mon-dl-2}, we will need the following result, which states that log-Lipschitzness is preserved (albeit with a different constant) when encoding the truth values of a given variable-disjoint matching of $\varphi$. 

\begin{lemma}
\label{lemma:matching-embedding}
Let $\Phi:\X_n\rightarrow\X_d$ be the embedding encoding the truth values of (disjunctive) clauses in a variable-disjoint matching $M$ of size $d$ under an assignment $x\in\X_n$.
Let $D$ be an $\alpha$-log-Lipschitz distribution on $\X_n$ and define $D'$ on $\X_d$ as follows:
\begin{equation*}
D'(y) := \sum_{x\in\prephi{y}} D(x)
\enspace,
\end{equation*}
where $y\in\X_d$.
Then $D'$ is $\alpha'$-log-Lipschitz for $\alpha'=(\alpha+1)^k-1$.
\end{lemma}

\begin{proof}
Let $y,y'\in \X_d$ be such that $d_H(y,y')=1$, i.e. $y$ and $y'$ disagree on exactly one clause in $M$. 
We want to upper bound the quantity $D(y)/D(y')$ by $\alpha'=(\alpha+1)^k-1$.
To this end, and without loss of generality, let $y_1\neq y_1'$ and let the clause $K_1$ in $M$ where $y$ and $y'$ disagree be a function of the first $k$ bits in $\X_n$.
Because $M$ is variable disjoint, and since $K_1$ is a disjunction of literals, if we fix the bits $x_{k+1}, \dots, x_n$, then there exists a unique assignment of $x_1, \dots, x_k$ such that $\Phi(x)_1=0$ (where $x=x_1\dots x_n$), and thus the remaining $2^k-1$ are such that $K_1$ evaluates to $1$. 
Hence, to upper bound $D(y)/D(y')$, we will assume that $y_1=1$ and $y_1'=0$.

Now, we can partition the preimage $\prephi{y}$ into $\{P_{x'}\}_{x'\in\prephi{y'}}$, where each $x\in P_{x'}$ disagrees with $x'$ on at least one of the first $k$ bits and is the same on the remaining $n-k$ bits. 
Thus
\begin{align*}
\frac{D'(y)}{D'(y')}
&= \frac{\sum_{x'\in\prephi{y'}} \sum_{x\in P_{x'}} D(x)} {\sum_{x'\in\prephi{y'}} D(x')}\\
&\leq \frac{\sum_{x'\in\prephi{y'}} D(x') \sum_{x\in P_{x'}} \alpha^{d_H(x,x')} } {\sum_{x'\in\prephi{y'}} D(x')} \tag{by log-Lipschitzness of $D$}\\
&= \frac{\left( (\alpha+1)^k -1 \right)\sum_{x'\in\prephi{y'}} D(x') } {\sum_{x'\in\prephi{y'}} D(x')}\\
&= (\alpha+1)^k -1 
\enspace,
\end{align*}
where we used the fact $(\alpha+1)^k=\sum_{i=0}^k {k \choose i} \alpha^i$ for the third step.
\end{proof}

We are now ready to prove Lemma~\ref{lemma:rob-risk-mon-dl-2}.
\begin{proof}{Proof of Lemma~\ref{lemma:rob-risk-mon-dl-2}}
  Since $\varphi$ has no cover of size $s$, it has a matching $M$ such
  that $|M|\geq\frac{s}{k}$.  By definition, each literal appears in at
  most one clause in $M$, hence, by removing at most a fraction
  $\frac{k}{k+1}$ of the
  clauses in $M$, we can assume without loss of generality that each
  variable occurs in at most one clause of $M$ and $M$ has cardinality 
  $ d:=\frac{s}{k(k+1)}$.

  Consider the map $\Phi:\X_n \rightarrow \X_d$, where $\Phi(x)$
  encodes the truth values of the clauses in $M$ under the assignment
  $x$.  Since the clauses in $M$ are variable-disjoint, $\Phi$ is
  non-expansive under the respective Hamming metrics on $\X_n$ and
  $\X_d$, meaning that $d_H(\Phi(x),\Phi(y))\leq d_H(x,y)$ for all 
  $x,y\in\X_n$.  Thus for all $x\in \X_n$,
  \[ \exists y \in B_\rho(x) \cdot y \models \varphi \implies
    \boldsymbol{1}\in B_\rho(\Phi(x)) \, . \] It will suffice to show
  that the probability over $x\sim D$ that the right-hand side condition of
  the above implication holds true is at most $\varepsilon$.
  
  Define a distribution $D'$ on $\X_d$ by
  $D'(y):=\sum_{x \in \Phi^{-1}(y)} D(x)$.  By Lemma~\ref{lemma:matching-embedding}, we have that
  $D'$ is $\alpha'$-log-Lipschitz for $\alpha':= (\alpha+1)^k-1$.
  We wish to upper-bound the probability over 
  $x' \sim D'$ that $\boldsymbol{1}\in B_\rho(x')$.
  For this, we will apply Lemma~\ref{lemma:rob-risk-dl} over the space
  $\X_d$ with distribution $D'$.  Indeed,
  our assumptions on $\eta$ and $s$ entail that
   $\eta = \frac{1}{1+\alpha'}$ and $d \geq \max \left\{
  \frac{4}{\eta^2}\log\left(\frac{1}{\varepsilon}\right),
  \frac{2\rho}{\eta}\right \}$.  Thus Lemma~\ref{lemma:rob-risk-dl}
   gives that $\Prob{x' \sim D'}{\boldsymbol{1}\in B_\rho(\Phi(x'))}
   \leq \varepsilon$.  This concludes the proof.
 \end{proof}
    
We are now ready to prove the main result of the section.

\begin{theorem}
\label{thm:dl-rob-2}
  Let $\mathcal{D}=\{\mathcal{D}_n\}_{n\in\mathbb{N}}$, where $\mathcal{D}_n$ is a set of
     $\alpha$-$\log$-Lipschitz distributions on $\{0,1\}^n$ for all $n\in\mathbb{N}$.   Then the classes of
    2-decision lists and monotone $k$-decision lists (for every fixed
    $k$) are $\rho$-robustly learnable with respect to
    $\mathcal{D}$ for robustness function $\rho(n)=\log n$.
\end{theorem}
\begin{proof}
  Let $\A$ be the (proper) PAC-learning algorithm for k-DL as
  in~\cite{rivest1987learning}, with sample complexity $\poly(\cdot)$.
  Fix the input dimension $n$, target concept $c$ and distribution
  $D\in \D_n$, and let $\rho=\log n$.  Fix the accuracy parameter
  $0<\varepsilon<1/2$ and confidence parameter $0<\delta<1/2$ and let
  $\eta=1/(1+\alpha)$.  Let
  $s_0 = k(k+1) \max \left\{
  \frac{4}{\eta^2}\log\left(\frac{e^4n^{2k+2}}{16\varepsilon}\right),
  \frac{2\rho}{\eta}\right \}$,
%  \max\left\{\frac{8k^2}{\eta^2} \left( 2\log \left(
 %       \frac{8k}{\eta^2}\right) + \log\left(
  %      \frac{n^2}{\varepsilon}\right)\right),\frac{2k\rho}{\eta}\right\}$
  write $m=\lceil\poly(n,1/\delta,\eta^{-s_0})\rceil$, and note that
  $m$ is polynomial in $n$, $1/\delta$ and $1/\varepsilon$.

  Let $S\sim D^m$ and $h=\A(S)$.  Then
$ \Prob{x\sim D}{h(x)\neq c(x)}<\eta^{-s_0}$ with probability at least
$1-\delta$.  But, by Lemma~\ref{lemma:consistent-dl-2},
$ \Prob{x\sim D}{h(x)\neq c(x)}<\eta^{s_0}$ implies that then
$c\equiv_{s_0} h$.  Hence $c\equiv_{s_0} h$ with probability at least
$1-\delta$.  

In case $c\equiv_{s_0} h$, an input $x \in \X$ only leads to a
classification error if it activates nodes $i$ and $j$ in $c$ and $h$
respectively such that the formula $\varphi^{(c,h)}_{i,j}$ has no 
cover of cardinality $s_0$.
Fix $i$ and $j$ such that $\varphi^{(c,h)}_{i,j}$ has no 
cover of cardinality $s_0$.
Now $\varphi^{(c,d)}_{i,j}$ is a $k$-CNF
formula by Proposition~\ref{prop:stay-k}.  Hence 
the probability that a $\rho$-bounded adversary can
make $\varphi^{(c,d)}_{i,j}$ true  is at most $\frac{16\varepsilon}{e^4n^{2k+2}}$ by
Lemma~\ref{lemma:rob-risk-mon-dl-2}.  Taking a union bound over all
possible choices of $i$ and $j$ (there are 
$\sum_{i=1}^k{n\choose k}\leq k\left(\frac{en}{k}\right)^k$ possible clauses
in $k$-decision lists, which gives us a crude estimate of 
$k^2\left(\frac{en}{k}\right)^{2k}\leq \frac{e^4n^{2k+2}}{16} $ choices of 
$i$ and $j$) we conclude that
$\roblosse(h,c) < \varepsilon$.

\end{proof}

\subsection{Non-Monotone Decision Lists}
\label{sec:k-dl}

In this section, we extend the reasoning from the previous section to non-monotone $k$-\dl, thus showing the efficient robust learnability of this concept class under log-Lipschitz distributions. 
This is done by the following result of independent interest: under log-Lipschitz distributions, the 
probability mass of the $\log(n)$-expansion of the set of satisfying
assignments of a $k$-CNF formula can be bounded above by an arbitrary
constant $\varepsilon>0$, given an upper bound on the probability of a
satisfying assignment.  The latter bound is polynomial in
$\varepsilon$ and $1/n$.  
Given two decision lists $c,h\in k$-DL, the set of inputs in
which $c$ and $h$ differ can be written as a disjunction of
polynomially many (in the combined length of $c$ and $h$) $k$-CNF
formulas.  The $\log(n)$-expansion of this set is then the set of
inputs where a $\log(n)$-bounded adversary can force an error at test
time. The combinatorial approach, below, 
differs from the approach of Section~\ref{sec:mon-k-dl} in the
special case of monotone $k$-DL, which relied on facts about
propositional logic.

Before going further, let us outline where the reasoning from Section~\ref{sec:mon-k-dl} fails when the monotonicity assumption does not hold.
The idea behind the proof of Lemma~\ref{lemma:rob-risk-mon-dl-2} was ultimately to show the existence of a sufficiently large matching in the hypergraph structure of a $k$-CNF formula to guarantee that an adversary could not cause a misclassification.
We obtained a maximal matching and transformed it into a \emph{variable-disjoint} one (crucial for the adversarial argument) through a minimal cover, which we can guarantee is satisfiable by the resolution closure property.
It is crucial that the latter be satisfiable in order to show that bounding the error results in consistency over covers (Lemma~\ref{lemma:consistent-dl-2}).
When considering non-monotone $k$-CNF formulas, the resolution closure could result in a $k'$-CNF formula where $k'$ depends on the number of variables $n$. 
As the value $k'$ would appear in the degree of the polynomial upper bounding the sample complexity, we would not be able to guarantee efficient robust learnability.
%Moreover, without taking the resolution closure, it is possible that a minimal cover is \emph{not} satisfiable, i.e., a variable appears as both a negative and positive literal, thus breaking the maximal matching argument. 
Our reasoning below still makes use of the maximal matching idea, but we directly relate the standard and robust risks.
%if we fail to find a sufficiently-large variable-disjoint one, we instead rely on induction to show that the $\rho$-expansion of a $k$-CNF is not too large.

%exiting at
%certain depths by a $k$-CNF, we can bound the probability that an
%adversary can cause an exit at nodes where $c$ and $h$ disagree.
%Previously, \cite{gourdeau2021hardness} showed that, for monotone
%$k$-DL, the expansion of these $k$-CNF formulas could be dealt with
%using facts about propositional logic.
%Instead, we use their result to prove a
%similar statement for any $k$-DL (and thus any $k$-CNF) by induction.

Now, for a given formula $\varphi$ on variables in $\boolhc$, we will denote by $\satrho$ the set $\set{x\in\boolhc \given \exists z \in B_\rho(x) \st  z \models \varphi}$ of instances in $\boolhc$ that are at most $\rho$ bits away from a satisfying assignment of $\varphi$.
Setting $\rho=0$, we recover the set of satisfying assignments of $\varphi$.
Note that if $\varphi$ is a formula expressing the discrepancy between two functions $c$ and $h$, then $\satnot(\varphi)$ represents the instances in $\boolhc$ contributing to the standard loss (hence the probability measure of the set $\satnot(\varphi)$ is the \emph{error} between the two functions under a given distribution).
Similarly, $\satrho(\varphi)$ represents the set of instances contributing to the $\rho$-\emph{robust} loss, and its probability measure is the \emph{robust risk} between the two functions $c$ and $h$.

\begin{theorem}
\label{thm:k-cnf}
Suppose that $\varphi\in k$-CNF and let $D$ be an
$\alpha$-log-Lipschitz distribution on the valuations of $\varphi$.
Then there exist constants $C_1,C_2,C_3,C_4\geq 0$ that depend on
$\alpha$ and $k$ such that if the probability  of a satisfying
assignment $\satnot(\varphi)$ satisfies
$\Prob{x\sim D}{x\in\satnot(\varphi)} <
C_1\varepsilon^{C_2}\min\set{\varepsilon^{C_3},n^{-C_4} } $, then the
$\log(n)$-expansion of the set of satisfying assignments has
probability mass bounded above by $\varepsilon$.
\end{theorem}

\begin{corollary}
\label{cor:k-dl}
The class of $k$-decision lists is efficiently $\log(n)$-robustly learnable under log-Lipschitz distributions.
\end{corollary}

Given Theorem~\ref{thm:k-cnf}, the proof of Corollary~\ref{cor:k-dl} is similar to Theorem~\ref{thm:dl-rob-2}, and is included in Appendix~\ref{app:cor-k-dl} for completeness.
We note that it is imperative that the constants $C_i$ do not depend on the learning parameters or the input dimension, as the quantity $C_1\varepsilon^{C_2}\min\set{\varepsilon^{C_3},n^{-C_4} }$ is directly used as the accuracy parameter in the (proper) PAC learning algorithm for decision lists, which is used as a black box.

%\subsection{Proof of Theorem~\ref{thm:k-cnf}}

To prove Theorem~\ref{thm:k-cnf}, we will need several lemmas from the previous section.
Some of these have been adapted to this setting, and are outlined below for ease of reading.
The proofs of these lemmas are (nearly) identical to those in the previous section, and hence are omitted avoid redundancy.
The first is an adaptation of Lemma~\ref{lemma:consistent-dl} for conjunctions, which was originally stated for decision lists:

\begin{lemma}%[Lemma 17 from \cite{gourdeau2021hardness}]
\label{lemma:conj-err-length}
Let $\varphi$ be a conjunction and let $D$ be an $\alpha$-log-Lipschitz distribution. 
If $\Prob{x\sim D}{x\models \varphi}<\left(1+\alpha\right)^{-d}$, then $\varphi$ is a conjunction on at least $d$ variables.
\end{lemma}

%The second result, which states an upper bound on the expansion of satisfying assignments for conjunctions, will be used for the base case of the induction proof. 
%
%\begin{lemma}%[Lemma 18 from \cite{gourdeau2021hardness} ]
%\label{lemma:rob-risk-conj-dl}
%Let $D$ be an $\alpha$-$\log$-Lipschitz distribution on the
%$n$-dimensional Boolean hypercube and let $\varphi$ be a 
%conjunction of $d$ literals.
%Set $\eta=\frac{1}{1+\alpha}$.
%Then for all $0<\varepsilon<1/2$,
%if $d\geq \max\left\{
%  \frac{4}{\eta^2}\log\left(\frac{1}{\varepsilon}\right) ,
%  \frac{2\rho}{\eta} \right\}$, then 
%$\Prob{x\sim D}{\left(\exists y \in B_\rho(x) \cdot y \models
%    \varphi\right)} \leq \varepsilon$.
%\end{lemma}

Finally, we will use the following lemma, which will be used in the inductive step of the induction proof.
It is nearly identical to Lemma~\ref{lemma:rob-risk-mon-dl-2}, which was stated for covers instead.

\begin{lemma}%[Lemma 23 from \cite{gourdeau2021hardness}]
  \label{lemma:rob-risk-dl-2}
  Let $\varphi$ be a $k$-CNF formula that has a set of variable-disjoint clauses of size $M$.
  Let $D$ be an $\alpha$-log-Lipschitz distribution on
  valuations for $\varphi$.  Let $0<\varepsilon<1/2$ be arbitrary
  and set $\eta:=\left({1+\alpha}\right)^{-k}$.  If
 $M \geq \max \left\{
  \frac{4}{\eta^2}\log\left(\frac{1}{\varepsilon}\right),
  \frac{2\rho}{\eta}\right \}$ then
  $\Prob{x\sim D}{\exists z \in B_\rho(x) \cdot z \models \varphi}
    \leq \varepsilon$.

\end{lemma}

We are now ready to prove Theorem~\ref{thm:k-cnf}.  The main idea
behind the proof is to consider a given $k$-CNF formula $\varphi$ and
distinguish two cases: (i) either $\varphi$ contains a
sufficiently-large set of variable-disjoint clauses, in which case the
adversary is not powerful enough to make $\varphi$ satisfied by
Lemma~\ref{lemma:rob-risk-dl-2}; or (ii) we can rewrite $\varphi$ as the
disjunction of a sufficiently small number of $(k-1)$-CNF formulas,
which allows us to use the induction hypothesis to get the desired
result.  The final step of the proof is to derive the constants
mentioned in the statement of Theorem~\ref{thm:k-cnf}.
 
\begin{proof}[Proof of Theorem~\ref{thm:k-cnf}]

We will use the lemmas above and restrictions on $\varphi$ to show the following.\\

\emph{Induction hypothesis:}
Suppose that $\varphi$  is a $(k-1)$-CNF formula and let $D$ be an $\alpha$-log-Lipschitz distribution on the valuations of $\varphi$. 
Then there exist constants $C_1,C_2,C_3,C_4\geq 0$ that depend on $\alpha$ and $k$ and satisfy $C_3\geq\frac{\eta}{2}C_4$ such that if $\Prob{x\sim D}{x\in\satnot(\varphi)} < C_1\varepsilon^{C_2}\min\set{\varepsilon^{C_3},n^{-C_4} }$, then $\Prob{x\sim D}{x\in\satlog(\varphi)}\leq \varepsilon $.\\
%\begin{align}
%\label{eqn:ind-hyp}
%&\satnot(\varphi) < C_1\varepsilon^{C_2}\min\set{\varepsilon^{C_3},n^{-C_4} } \notag \\
%&\implies\satlog(\varphi)\leq \varepsilon 
%\enspace.
%\end{align}

\emph{Base case:} This follows from Lemmas~\ref{lemma:conj-err-length} and \ref{lemma:rob-risk-dl}.
Set $\eta$ to $(1+\alpha)^{-1}$, and $C_1=1$, $C_2=0$, $C_3=\frac{4}{\eta^2}$ and $C_4=\frac{2}{\eta}$. 
Note that $C_3\geq\frac{\eta}{2}C_4$.\\

\emph{Inductive step:} 
Suppose $\varphi\in k$-CNF and let $D$ be an $\alpha$-log-Lipschitz distribution on the valuations of $\varphi$. 
Set $\eta=(1+\alpha)^{-k}$.
Let $C_1',C_2',C_3',C_4'$ be the constants in the induction hypothesis for $\varphi'\in (k-1)$-CNF.
Set the following constants:
\begin{align*}
&C_1= C_1'2^{-k(C_2'+C_3')}\\ & C_2=C_2'+C_3' \\
&C_3= \frac{8}{\eta^2}\max\set{C_2',C_3'}\\ &C_4= \frac{2}{\eta}\max\set{C_2',C_3'}\enspace,
\end{align*}
and note that these are all constants that depend on $k$ and $\alpha$ by the induction hypothesis, and that $C_3\geq\frac{\eta}{2}C_4$.

Let $\Prob{x\sim D}{x\in\satnot (\varphi)}< C_1 \varepsilon^{C_2}\min\set{\varepsilon^{C_3},n^{-C_4} }$.
Let $\mathcal{M}$ be a maximal set of clauses of $\varphi$ such that no two clauses contain the same variable.
Denote  by $I_\matching$ the indices of the variables in $\matching$ and let $M=\max\set{\frac{4}{\eta^2}\log\frac{1}{\varepsilon},\frac{2}{\eta}\log n}$. \\

We distinguish two cases:

(i) $\abs{\matching}\geq M$:  
Then 
$$\Prob{x\sim D}{x\models \varphi}
\leq\Prob{x\sim D}{x\models \bigwedge_i C_i} 
\leq (1-\eta^k)^{\abs{\matching}}
\leq \exp(-\eta^{k\abs{\matching}})
\enspace,$$
We can then invoke Lemma~\ref{lemma:rob-risk-dl-2} to guarantee that $\Prob{x\sim D}{x\in\satlog}\leq \varepsilon$, and we get the required result.
\\

(ii) $\abs{\matching}<M$:

Then let $\asst$ be the set of assignments of variables in $\matching$, i.e. $a\in\asst$ is a function $a:I_{\matching}\rightarrow \set{0,1}$, which represents a partial assignment of variables in $\varphi$.
We can thus rewrite $\varphi$ as follows:
\begin{equation*}
\varphi \equiv \bigvee_{a\in\asst} \left( \varphi_a \wedge \bigwedge_{i\in\varM}  l_i \right)\enspace,
\end{equation*}
where $\varphi_a$ is the restriction of $\varphi$ under assignment $a$
and $l_i$ is $x_i$ in case $a(i)=1$ and $\bar{x_i}$ otherwise.  For
short, denote by $\varphi_a'$ the formula
$\varphi_a \wedge \bigwedge_{i\in\varM} l_i$.  By the
  maximality of $\mathcal{M}$ every clause in $\varphi$ mentions some
  variable in $\mathcal{M}$, and hence $\varphi_a'$ is $(k-1)$-CNF.
Moreover, the formulas $\varphi_a'$ are disjoint, in the sense that if
some assignment $x$ satisfies $\varphi_a'$, it will not satisfy
another $\varphi_b'$ for a distinct index $b$.
Note also that 
$$A_{n,\varepsilon}:=\abs{\asst}\leq
2^k\max\set{\left(\frac{1}{\varepsilon}\right)^{4/\eta^2},
  n^{2/\eta}}\enspace
\, . $$

Thus, 
\begin{equation}
\label{eqn:std-risk-split}
\Prob{x\sim D}{x\in\satnot(\varphi)} %= \Prob{x\sim D}{x\models \varphi}
=\sum_{a\in\asst}\Prob{x\sim D}{x\models \varphi_a' }
= \sum_{a\in\asst}\Prob{x\sim D}{x\in \satnot(\varphi_a')}
\enspace.
\end{equation}
By the induction hypothesis, we can guarantee that if
\begin{align}
\label{eqn:std-risk}
 \Prob{x\sim D}{x\in\satnot(\varphi_a')} < \;&C_1'  \left(\frac{\varepsilon}{A_{n,\varepsilon}}\right)^{C_2'}
 \min \set{\left(\frac{\varepsilon}{A_{n,\varepsilon}}\right)^{C_3'},n^{-C_4'} } 
\end{align} 
 for all $\varphi_a'$ then the $\log(n)$-expansion $\satlog(\varphi)$ can be bounded as follows:
\begin{align*}
\Prob{x\sim D}{x\in\satlog(\varphi)}&=\Prob{x\sim D}{\exists z\in B_{\log n}(x) \st z \models \varphi}\\
&=\sum_{a\in\asst}  \Prob{x\sim D}{\exists z\in B_{\log n}(x) \st z \models \varphi_a'} \\
&\leq \sum_{a\in\asst}   \frac{\varepsilon}{A_{n,\varepsilon}} \tag{I.H.} \\
&=\varepsilon \enspace.
\end{align*}

By Equation~\ref{eqn:std-risk-split}, the upper bound  $\Prob{x\sim D}{x\in\satnot (\varphi)}< C_1 \varepsilon^{C_2}\min\set{\varepsilon^{C_3},n^{-C_4} }$  implies an upper bound $\Prob{x\sim D}{x\in\satnot (\varphi_a')}< C_1 \varepsilon^{C_2}\min\set{\varepsilon^{C_3},n^{-C_4} }$ on the probability of the restrictions $\varphi_a'$.
Thus it only remains to show that the condition on $\satnot (\varphi)$  implies that Equation~\ref{eqn:std-risk} holds.

Let us rewrite the RHS of Equation~\ref{eqn:std-risk} as follows, where each of the equations is a stricter condition on $\satnot (\varphi_a')$ than its predecessor:
\begin{align*}
& C_1'  \left(\frac{\varepsilon}{A_{n,\varepsilon}}\right)^{C_2'}\min \set{\left(\frac{\varepsilon}{A_{n,\varepsilon}}\right)^{C_3'},n^{-C_4'} } \\
& \geq C_1'  \left(\frac{\varepsilon}{2^k}\right)^{C_2'}\min\set{\varepsilon^{4C_2'/\eta^2},n^{-2C_2'/\eta}}\min \set{\left(\frac{\varepsilon^{1+4/\eta^2}}{2^k}\right)^{C_3'},\left(\frac{\varepsilon n^{-2/\eta}}{2^k}\right)^{C_3'},n^{-C_4'} } \\
& = C_1'  \left(\frac{\varepsilon}{2^k}\right)^{C_2'}\min\set{\varepsilon^{4C_2'/\eta^2},n^{-2C_2'/\eta}}\min \set{\left(\frac{\varepsilon^{1+4/\eta^2}}{2^k}\right)^{C_3'},\left(\frac{\varepsilon n^{-2/\eta}}{2^k}\right)^{C_3'} }\\
& = C_1'  2^{-k(C_2'+C_3')}\varepsilon^{C_2'+C_3'}\min\set{\varepsilon^{4C_2'/\eta^2},n^{-2C_2'/\eta}} \min \set{\varepsilon^{4C_3'/\eta^2}, n^{-2C_3'/\eta} }\\
&\geq C_1'  2^{-k(C_2'+C_3')}\varepsilon^{C_2'+C_3'}\min\set{\varepsilon^{8C_2'/\eta^2},n^{-4C_2'/\eta},\varepsilon^{8C_3'/\eta^2}, n^{-4C_3'/\eta} }\\
&= C_1'  2^{-k(C_2'+C_3')}\varepsilon^{C_2'+C_3'}\min\set{\varepsilon^{8\max\set{C_2',C_3'}/\eta^2},n^{-4\max\set{C_2',C_3'}/\eta} }\\
&= C_1 \varepsilon^{C_2}\min\set{\varepsilon^{C_3},n^{-C_4}}
\enspace,
\end{align*}
where the first step is by definition of $A_{n,\varepsilon}$, the second from the induction hypothesis, which guarantees $C_3'\geq\frac{\eta}{2}C_4'$, and the fourth from the property $\min\set{a,b}\cdot\min\set{c,d}\geq\min\set{a^2,b^2,c^2,d^2}$.
Finally, the last equality follows by the definition of the $C_i$'s.

Note that we set $\eta=(1+\alpha)^{-k}$ to be able to apply Lemma~\ref{lemma:rob-risk-dl-2} in the first part of the inductive step.
Then, $A_{n,\epsilon}$ is a function of $\eta=(1+\alpha)^{-k}$.
When we consider the distribution on the valuations of the restriction $\varphi_a'$, we still operate with an $\alpha$-log-Lipschitz distribution on its valuations, by Lemma~\ref{lemma:log-lips-facts}.

\emph{Constants.}
We want to get explicit constants $C_1,C_2,C_3$ and $C_4$ as a function of $k$ and $\eta
$.
Note that $\eta=(1+\alpha)^{-k}$ is dependent on $k$.  
Let us recall the recurrence system from the inductive step:
\begin{align*}
&C_1^{(k)}= C_1^{(k-1)}2^{-k(C_2^{(k-1)}+C_3^{(k-1)})} \\
& C_2^{(k)}=C_2^{(k-1)}+C_3^{(k-1)} \\
&C_3^{(k)}= \frac{8}{\eta^2}\max\set{C_2^{(k-1)},C_3^{(k-1)}} \\
&C_4^{(k)}= \frac{2}{\eta}\max\set{C_2^{(k-1)},C_3^{(k-1)}}\enspace.
\end{align*}
It is easy to see that $C_3^{(k)}\geq C_2^{(k)}$ for all $k\in\N$.
If we fix $\eta=(1+\alpha)^{-k}$ at each level of the recurrence,  we can now consider the following recurrence system, which dominates the previous one:
\begin{align*}
&C_1^{(k)}= C_1^{(k-1)}2^{-2kC_3^{(k-1)}}\\
 & C_2^{(k)}=2C_3^{(k-1)} \\
&C_3^{(k)}= \frac{8}{\eta^2}C_3^{(k-1)}\\
&C_4^{(k)}= \frac{2}{\eta}C_3^{(k-1)}\enspace.
\end{align*}
We can now see that 
\begin{align*}
&C_2^{(k)}= 2\left(\frac{8}{\eta^2}\right)^{k-1} = 2(8(1+\alpha)^{2k})^{k-1}\\
&C_3^{(k)}=\left(\frac{8}{\eta^2}\right)^k=(8(1+\alpha)^{2k})^{k}  \\
&C_4^{(k)}= \frac{2}{\eta}\left(\frac{8}{\eta^2}\right)^{k-1}= 2(1+\alpha)^k(8(1+\alpha)^{2k})^{k-1}
\enspace.
\end{align*}
Finally, we can get a lower bound on the value of $C_1^{(k)}$ as follows:
\begin{align*}
C_1^{(k)}
&=\prod_{i=2}^k 2^{-2iC_3^{(i-1)}}\\
%&=2^{-2kC_3^{(k-1)}}2^{-2(k-1)C_3^{(k-2)}}\dots2^{-2\cdot 2 C_3^{(1)}}\\
%&=2^{-2\sum_{i=2}^k i\cdot  C_3^{(i-1)}}\\
&=2^{-2\sum_{i=2}^k i\cdot  \left(\frac{8}{\eta^2}\right)^{(i-1)}}\\
&\geq 2^{-2k^2\left(\frac{8}{\eta^2}\right)^{(k-1)}} \\
&= 2^{-2k^2(8(1+\alpha)^{2k})^{k-1}}
\enspace,
\end{align*}
which concludes the proof.
\end{proof}

\paragraph*{Comparing the sample complexity of monotone and non-monotone $k$-{\dl}.}  Earlier in this chapter, we stated that directly using the non-monotone analysis could result in higher sample complexity in case we are working with monotone  $k$-{\dl}.
Indeed, observe that the maximum degree of the $\frac{1}{\epsilon}$ term in the polynomial for \emph{monotone} decision lists in Theorem~\ref{thm:dl-rob-2} is $O(k^2(1+\alpha)^2)$ and the maximum degree of the $n$ term is $O(k^3(1+\alpha)^2)$, while the maximum degrees of the $\frac{1}{\epsilon}$ and $n$ terms are $O(8^k(1+\alpha)^{2k^2})$ and  $O(k\cdot8^k(1+\alpha)^{2k^2})$ for \emph{non-monotone} decision lists in Corollary~\ref{cor:k-dl}, respectively. 
Thus, for both the $\frac{1}{\epsilon}$ and $n$ terms, using the general $k$-decision list bound comes at the cost of a polynomial degree that has an exponential dependence in $k$.

\section{Decision Trees}
\label{sec:dt}
In this section, we show that, under $\alpha$-log-Lipschitz
distributions, for any two decision trees and perturbation budget
$\rho(n)=O(\log n)$, the $\rho$-robust risk is bounded above by a
polynomial in the number $n$ of propositional variables, the 
combined size $m$ of the trees, and their standard
risk.  This result makes explicit the relationship between both
notions of risk.

Despite the fact that it is not known whether the class of decision
trees is PAC-learnable, relating the standard and robust risks for
this class is still of interest if we can show that a small enough
standard risk only incurs a polynomial blowup in the robust risk.
This could be particularly compelling in the local membership query
model of \cite{awasthi2013learning}, where an algorithm can request
labels for points that are $O(\log(n))$ bits away from a point in the
training sample.  The authors showed that, in this framework, the
class of polynomial-sized decision trees is learnable (in polynomial
time) under product distributions using $O(\log(n))$-local membership
queries.  Moreover, \cite{odonnell2007learning} show that monotone
decision trees are PAC learnable under the uniform distribution, so
our result holds in this setting as well.

\paragraph*{Terminology.}
A decision tree $c$ over $n$ propositional variables is a finite
binary tree whose internal nodes are labeled by elements of the set
$\{1,\ldots,n\}$ and whose leaves are labeled either $0$ or $1$.  The
depth of a leaf is the number of internal nodes of the tree in the
(unique) path from the root to the given leaf.  An input
$x\in \X=\{0,1\}^n$ determines a path through such a tree, starting at
the root, as follows: at an internal node with label $i$ descend to
the left child if $x_i=0$ and descend to the right child if $x_i=1$.
We say that $x\in \X$ \emph{activates a given leaf node} if the path
determined by $x$ leads to the given leaf.  In this way a decision
tree $c$ determines a function $c: \X\rightarrow \{0,1\}$, where
$c(x)$ is the label of the leaf activated by $x$.

Given two decision trees $c,h$, both over $n$ propositional variables,
and given $d\in \mathbb{N}$, we say that $c$ and $h$ are \emph{consistent up to depth $d$}, denoted $c=_d h$, if for all $x \in \X$ such that $x$ activates leaves of depth at most $d$ in both $c$ and $h$, we
have $c(x)=h(x)$.  In the same vein as
Lemma~\ref{lemma:consistent-dl}, given $d \in \mathbb{N}$ we have that
$c=_d h$ provided that $\Prob{x\sim D}{h(x)\neq c(x)}$ is sufficiently
small:

\begin{lemma}
\label{lemma:risk-depth}
Let $D$ be a $\alpha$-log-Lipschitz distribution. If
$\Prob{x\sim D}{h(x)\neq c(x)}<(1+\alpha)^{-2d}$ then $c =_d h$.
\end{lemma}
We omit the proof of Lemma~\ref{lemma:risk-depth}, which
follows that of Lemma~\ref{lemma:consistent-dl} \emph{mutatis mutandis}.

We can now bound the robust risk between decision trees as a
polynomial in the of the number of propositional variables, the
log-Lipschitz constant, their combined size, and their standard risk.

\begin{theorem}
  Let $c$ and $h$ be two decision trees on $n$ propositional variables
  with at most $m$ nodes in total for both trees.  Let $D$ be an
  $\alpha$-log-Lipschitz distribution on $\X_n$ and $\rho=\log n$.
  There is a fixed polynomial $\mathrm{poly}(\cdot,\cdot,\cdot)$ such that
  for all $0<\varepsilon<\frac{1}{2}$, if
  $\Prob{x\sim D}{h(x)\neq c(x)} < \mathrm{poly}(\frac{1}{m},\frac{1}{n},\varepsilon)$,
then   $\roblosse(c,h)<\varepsilon$.
  \label{thm:tree-loss}
\end{theorem}
\begin{proof}
  Write
  $d:=\max\left\{
    \frac{4}{\eta^2}\log\left(\frac{m}{\varepsilon}\right) ,
    \frac{2\rho}{\eta} \right\}$ and define
  $\mathrm{poly}(\frac{1}{m},\frac{1}{n},\varepsilon) :=
  (1+\alpha)^{-2d}$.

  The assumption that
  $\Prob{x\sim D}{h(x)\neq c(x)} < (1+\alpha)^{-2d}$ implies that $c$
  and $h$ are consistent to depth $d$ by Lemma~\ref{lemma:risk-depth}.  This means that
  $c(x) \neq h(x)$ only on those inputs $x \in \X$ that activate some
  leaf node of depth strictly greater than $d$, either in $c$ or $h$.
  By Lemma~\ref{lemma:rob-risk-dl}, for each such node the probability
  that a $\rho$-bounded adversary can activate the node by perturbing
  the bits of a randomly generated input $x\sim D$ is at most
  $\frac{\varepsilon}{m}$.  Taking a union bound over the nodes
  of depth $>d$ (there are at most $m$ of them), we conclude that $\roblosse(h,c) \leq \varepsilon$.
\end{proof}

\section{Summary of Results and Open Problems}
\label{sec:rt-summary}
In this chapter, we showed the efficient robust learnability of various concept classes under distributional assumptions, as outlined in Table~\ref{tab:rob-thresholds-summary}.
We finish this chapter by commenting on the general techniques used throughout this text and discussing avenues for future work.

\begin{table}[]
\begin{tabular}{l|c|c}
\textbf{Concept Class}  & \textbf{Distributional Assumption} & \multicolumn{1}{l}{\textbf{Robustness Threshold}} \\ \hline
Non-trivial             & None (distribution-free)         & {0}             \\
Mon. Conjunctions   & log-Lipschitz              & { $\Theta(\log(n))$}          \\
Parities                & log-Lipschitz              & {$n$ (exact)}                \\
Majorities              & Uniform                    & {$n$ (exact)}                \\
Mon. Decision Lists & log-Lipschitz              & {$\Theta(\log(n))$}        \\
Non-Mon. DL         & log-Lipschitz              & {$\Theta(\log(n))$}          \\
\noindent Halfspaces              & log-Lipschitz              & {$\Theta(\log(n))$?}         \\
PAC classes             & Uniform                    & {?}                         
\end{tabular}
\caption{The robustness thresholds of concept classes from Chapters~\ref{chap:def-adv-rob} and~\ref{chap:rob-thresholds}, and open problems.}
\label{tab:rob-thresholds-summary}
\end{table}

The techniques from this chapter can be viewed as bounding the expansion of sets in the boolean hypercube.
These sets represent supersets of the error region between the target concept and hypothesis, and their expansions, the instances an adversary could perturb to cause a misclassification.
In general, we consider the \emph{measure} of these sets, but, when working under the uniform distribution, we can simply consider their \emph{size}.

As the target and hypothesis come from the same concept class in all the results presented in this chapter, the standard \emph{proper} PAC-learning algorithms can be used as black boxes for efficient robust learning.
Indeed, by controlling the measure of the error region (by some polynomial $p(\cdot,\cdot)$ in $\epsilon$ and $1/n$), we can control the measure of its $\rho$-expansion and bound it above by the desired robust accuracy $\epsilon$. 
In some cases (parities, majorities, conjunctions that are logarithmically-bounded in length, shallow decision lists), the standard error is always sufficiently large to allow exact learning (i.e., if the standard error is strictly smaller than $p(\epsilon,1/n)$, it must be zero).
However, in general, exact learning is not a prerequisite for robust learning.

In all the cases where we showed robust, but not exact, robust learning, we expressed the error region as a disjunction $\varphi$ of $k$-CNF formulas.
The set $\mathsf{SAT}(\varphi)$ of satisfiable assignments of $\varphi$ thus represents the indicator set of whether an instance $x\in\boolhc$ belongs to the error region.
Likewise, $\mathsf{SAT}_\rho(\varphi)$, the set of points at distance at most $\rho$ from $\mathsf{SAT}(\varphi)$, represents the set of points incurring a robust loss against a $\rho$-bounded adversary.
This argument is illustrated in Figure~\ref{fig:unifying-result}.

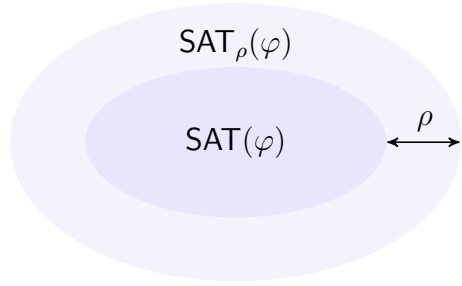
\begin{figure}
\begin{center}
 	\begin{tikzpicture}[<->,>=stealth',shorten >=1pt,auto,semithick]
 		\node[ellipse,fill=blue!80!red!5,minimum width=6cm,minimum height=3.7cm] (sat-log) {};
 		\node[ellipse,fill=blue!90!red!10,minimum width=4cm,minimum height=2cm] (sat) {$\mathsf{SAT}(\varphi)$};
 		\node[above of=sat, yshift=+0.3cm] (sat-log-label) {$\mathsf{SAT}_\rho(\varphi)$}; 
 		\path (sat.east) edge node {$\rho$} (sat-log.east);
  	\end{tikzpicture}
\end{center}
\caption{A unifying result. When $\rho=\log(n)$, $\mathsf{SAT}_\rho(\varphi)$, the $\rho$-expansion of the error region, is not too large compared to the set $\mathsf{SAT}(\varphi)$.}
\label{fig:unifying-result}
\end{figure}

A compelling avenue for future work would be to derive sample complexity lower bounds for $k$-{\dl} which have an explicit dependence on $k$ as well as the adversarial budget, as the lower bound $\Omega(2^\rho)$ on the sample complexity was derived for monotone conjunctions.

Another clear direction forward is to generalize the results obtained in this chapter to a wider variety of concept classes. 
An immediate candidate for this is the class of linear classifiers, which are the building blocks of more expressive concept classes such as neural networks.
Since linear classifiers subsume monotone conjunctions, the exponential dependence on the adversarial budget under the uniform distribution shown in Chapter~\ref{chap:def-adv-rob} extends to this concept class as well.
It must then be that the robustness threshold of linear classifiers under log-Lipschitz distributions is $O(\log n)$. 
Moreover, note that we showed in  Section~\ref{sec:maj} that majority functions can be exactly and thus robustly learned under the uniform distribution. Since majorities are subsumed by linear classifiers and that their robustness threshold is $n$, there is no evidence yet on linear classifiers having a robustness threshold that is $o(\log n)$ under the uniform distribution.

Were the robustness threshold of linear classifiers to also be $\log n$, an interesting open problem would be whether this extends to concept classes with polynomially-bounded VC dimension.

\begin{center}
\textbf{Open Problem:}\\
\emph{Let $\A$ be a sample-efficient PAC-learning algorithm for concept class $\C$ on $\boolhc$.
Is $\A$ also a
sample-efficient $\log(n)$-robust learning algorithm for $\C$ under the uniform distribution?}
\end{center}

A positive result could be based on properties of $\log(n)$ expansions of ``nice'' subsets of $\boolhc$, e.g., through the use of isoperimetric inequalities.
In any case, characterizing the efficient robust learnability of concept classes under the uniform distribution with a complexity measure akin to the VC dimension in PAC learning is a compelling avenue for future work.

To solve this open problem, one may be tempted to extend the following result to robust learning, which relates the error between two functions $f$ and $g$ and their respective Fourier spectra, $\widehat{f}(S)$, $\widehat{g}(S)$ for $S\subseteq[n]$.
To apply this result to learning theory, one of the two functions would be the target $f$, and the other $g$, an approximation of $f$ through Fourier coefficient estimates:

\begin{theorem}[\cite{linial1993constant}]
Let $g:\boolhc\rightarrow\R$ be a real-valued function and $D$ be the uniform distribution on $\boolhc$. 
For any $f:\boolhc\rightarrow\classesfa$, 
\begin{equation}
\Prob{x\sim D}{f(x)\neq\sgn(g(x))}
\leq \sum_{S\subseteq[n]}\left(\widehat{f}(S)-\widehat{g}(S)\right)^2\enspace.
\end{equation}
\end{theorem}
The proof is reproduced below.
\begin{proof}
First note that since the image of $f$ is $\classesfa$, 
\begin{equation}
\label{eqn:indicator-sgn-bound}
\mathbf{1}[f(x)\neq \sgn(g(x))]\leq \abs{f(x)-g(x)}\enspace.
\end{equation}
Squaring both sides, and taking the expectation, we get:
\begin{align}
\Prob{x\sim D}{f(x)\neq\sgn(g(x))}
&\leq \eval{x\sim D}{(f(x)-g(x))^2}\\
&= \sum_{S\subseteq[n]}\left(\widehat{f-g}(S)\right)^2 \label{eqn:parseval}\\
&\leq \sum_{S\subseteq[n]}(\widehat{f}(S)-\widehat{g}(S))^2\label{eqn:fc-plus-minus}\enspace,
\end{align}
where Equation~\ref{eqn:parseval} follows from Parseval's formula, and Equation~\ref{eqn:fc-plus-minus} from the identity $\widehat{f\pm g}\leq\widehat{f}(S)\pm \widehat{g}(S)$.
\end{proof}

Where does the reasoning break when considering the robust risk?

If we look at Equation~\ref{eqn:indicator-sgn-bound}, its robust counterpart would be:
\begin{equation}
\mathbf{1}[\exists z \in B_\rho(x) \st f(x)\neq \sgn(g(x))]
\leq \underset{z\in B_\rho(x)}{\max} \abs{f(z)-g(z)}\enspace.
\end{equation}
Now, squaring both sides and taking the expectation, we get
\begin{align}
\Prob{x\sim D}{\exists z \in B_\rho(x) \st f(x)\neq \sgn(g(x))}
&\leq \eval{x\sim D}{\underset{z\in B_\rho(x)}{\max} \abs{f(z)-g(z)}^2}%\\
%&= \sum_{S\subseteq[n]}\left(\widehat{f-g}(S)\right)^2 \\
%&\leq \sum_{S\subseteq[n]}(\widehat{f}(S)-\widehat{g}(S))^2
\enspace.
\end{align}
Note that we cannot take the $\max$ out of the expectation, as it  is defined with respect to $x$ (and, less importantly, this function is not convex, implying that Jensen's inequality cannot be applied).
It is then apparent that relating the robust risk and the Fourier spectrum, if it is possible, would require a more complex argument than in the standard classification case.

\chapter{Robust Learning with Local Queries}
\label{chap:local-queries}

The previous chapters of this thesis considered a learning model in which the learner only has access to random examples. 
This is rather restrictive for the learner, especially considering the adversary's power: the study of the \emph{existence} of adversarial examples in our setting assumes that the adversary has full knowledge of the target and no computational limitations.
In the face of the impossibility or hardness of robustly learning certain concept classes from the previous chapters, it is natural to study whether these issues can be circumvented and robust learning guarantees obtained by giving more power to the learner -- a line of thinking echoed in practice.
For example, adversarial training \citep{goodfellow2015explaining,madry2018towards} and data augmentation are common procedures in applied machine learning. 
In the latter, data is moderately altered\footnote{E.g., images are slightly rotated or translated, which does not change their label.} and added to the dataset, usually with the goal of improving accuracy. 
In the former, the goal is to improve robust accuracy; the training dataset is augmented with adversarial examples, which are usually found for a specific model after training.

This chapter investigates  the power of \emph{local queries} in robust learning.
Local queries allow the learner to obtain information in the vicinity of the training sample.
This setting sits between the  PAC-learning framework of \cite{valiant1984theory} and the  membership and equivalence query model of \cite{angluin1987learning}, in which there is no distribution, and where the learner can obtain information on the whole instance space (see Section~\ref{sec:active-learning} for more background on the topic).

We now outline our contributions.
Section~\ref{sec:lq-models} recalls the local membership query ($\LMQ$) model of \cite{awasthi2013learning}, and introduces local \emph{equivalence} queries.
In Section~\ref{sec:rob-learn-lmq}, we show that local membership queries do not improve the robustness threshold of conjunctions under the uniform distribution: giving the learner access to both the $\EX$ and $\LMQ$ oracles still results in a joint sample and query complexity that is \emph{exponential} in the adversarial budget.
This justifies studying the more powerful local equivalence query model in our setting.
In Section~\ref{sec:rob-learn-leq}, we first show that distribution-free robust learning remains impossible for a wide variety of concept classes in the case in which the region covered by local equivalence queries is a strict subset of the adversary's perturbation region.
However, when the two regions coincide,\footnote{This is the equivalent of querying the robust loss on a point and obtaining a counterexample, if it exists.} we do get distribution-free robust learning guarantees.
In particular, we give general sample and query complexity upper bounds, as well as bounds for specific concept classes.
However, the query complexity can be unbounded in case  the Littlestone dimension of a concept class is infinite. 
We address this potential issue in Section~\ref{sec:adv-bounded-precision}, where we limit the adversary's \emph{precision} and give upper bounds on the query complexity in this setting with techniques and tools adapted from the online learning of margin-based hypothesis classes \citep{ben2009agnostic}.
In Section~\ref{sec:qc-lb-leq}, we give general local equivalence query lower bounds and instantiate them to particular concept classes.
We finish the technical contributions of this chapter with a more nuanced comparison between the local membership and equivalence query oracles, and between the local and global oracles in Section~\ref{sec:comparing-lq}.
We conclude this chapter with Section~\ref{sec:lq-summary}, which outlines avenues for future work.

\section{Two Local Query Models}
\label{sec:lq-models}

In this section, we present two query models in which the learner can gather information local to the training sample, in the spirit of membership and equivalence queries \citep{angluin1987learning} (Section~\ref{sec:active-learning}).
The main distinction is that, given a sample $S$ drawn from the example oracle, a query for a point $x$ is $\lambda$-\emph{local} if there exists $x'\in S$ such that their distance is at most $\lambda$.
We first present the $\lambda$-local membership query ($\lambda$-$\LMQ$) set-up of \cite{awasthi2013learning}, which allows the learner to query the label of points that are at distance at most $\lambda$ from a sample $S$ drawn randomly from $D$.
In the formal definition of the LMQ model below, we have changed the standard risk to the robust risk for our purposes (the model was initially developed in the context of standard binary classification).

\begin{definition}[$\lambda$-$\LMQ$ Robust Learning]
\label{def:lmq}
Let $\X_n$ be the instance space together with a metric $d$, $\C_n$ a concept class over $\X_n$, and $\D_n$ a class of distributions over $\X_n$. We say that $\C_n$ is $\rho$-robustly learnable using $\lambda$-local membership queries with respect to $\D_n$ if there exists a learning algorithm $\A$ such that for every $\epsilon > 0$, $\delta > 0$, for every distribution $D\in\D_n$ and every target concept $c\in\C_n$, the following hold:
\begin{enumerate}
\item $\A$ draws a sample $S$ of size $m = \poly(n, 1/\delta, 1/\epsilon,\text{size}(c))$ using the example oracle $\EX (c, D)$;
\item Each query $x'$ made by $\A$ to the $\LMQ$ oracle is $\lambda$-local with respect to some example $x \in S$, i.e., $x'\in B_\lambda(x)$;
\item  $\A$ outputs a hypothesis $h$ that satisfies $\risk_\rho^D(h,c)\leq \epsilon$ with probability at least $1-\delta$;
\item The running time of $\A$ (hence also the number of oracle accesses) is polynomial in $n$, $1/\epsilon$, $1/\delta$ and the output hypothesis $h$ is polynomially evaluable.
\end{enumerate}
\end{definition}

Note that, similarly to $\rho$, we implicitly consider $\lambda$ to be a function of the input dimension $n$. 
Moreover, we implicitly assume that a concept $c\in\C_n$ can be represented in size polynomial in $n$, where $n$ is the input dimension; otherwise a parameter $size(c)$ can be introduced in the sample and query complexity requirements.
A similar assumption will apply to the local equivalence query model below.
Finally, note that, in both cases, it is also possible to extend this definition to an arbitrary neighbourhood function $\U:\X\rightarrow 2^\X$ (similarly to how the adversarial perturbation function can be generalized in the same fashion).

Inspired by the $\lambda$-$\LMQ$ learning model, we define  the $\lambda$-local equivalence query ($\lambda$-$\LEQ$) model where, for a point $x$  in a sample $S$ drawn from the underlying distribution $D$ and for a given $h\in\H$, the learner is allowed to query with $(h,x)$ an oracle that returns whether $h$ agrees with the ground truth $c$ in the ball $B_\lambda(x)$ of radius $\lambda$ around $x$.
If they disagree, a counterexample in $B_\lambda(x)$ is returned as well.
Clearly, by setting $\lambda=n$, we recover the equivalence query ($\EQ$) oracle.
Note, moreover, that when $\lambda=\rho$, this is equivalent to querying the (exact-in-the-ball) robust loss around a point.

\begin{definition}[$\lambda$-$\LEQ$ Robust Learning]
\label{def:leq}
Let $\X_n$ be the instance space together with a metric $d$, $\C$ a concept class over $\X_n$, and $\D$ a class of distributions over $\X_n$. We say that $\C$ is $\rho$-robustly learnable using $\lambda$-local equivalence queries with respect to distribution class, $\D$, if there exists a learning algorithm, $\A$, such that for every $\epsilon > 0$, $\delta > 0$, for every distribution $D\in\D$ and every target concept $c\in\C$, the following hold:
\begin{enumerate}
\item $\A$ draws a sample $S$ of size $m = \poly(n, 1/\delta, 1/\epsilon)$ using the example oracle $\EX (c, D)$;
\item Each query made by $\A$ at $x \in S$ and for a candidate hypothesis $h$ to $\lambda$-$\LEQ$ either confirms that $c$ and $h$ coincide on $B_\lambda(x)$ or returns $z\in B_\lambda(x)$ such that $c(z) \neq h(z)$. $\A$ is allowed to update $h$ after seeing a counterexample;
\item  $\A$ outputs a hypothesis $h$ that satisfies $\risk_\rho^D(h,c)\leq \epsilon$ with probability at least $1-\delta$;
\item The running time of $\A$ (hence also the number of oracle accesses) is polynomial in $n$, $1/\epsilon$, $1/\delta$ and the output hypothesis $h$ is polynomially evaluable.
\end{enumerate}
\end{definition}

\paragraph*{Partial queries.}
We remark that both the $\LMQ$ and $\LEQ$ oracles are specific instances of the partial equivalence queries of \cite{maass1992lower}.
In their set-up, the learner can give as input to the $\EQ$ oracle a partial function $h:\X\rightarrow\{0,1,*\}$.
The oracle only evaluates the correctness of $h$ on the restricted domain $\set{x\in\X \given h(x)\neq*}$.

A (local) membership query on $x^*\in\X$ is equivalent to the partial equivalence query for the function
\begin{equation*}
h(x)
=
\begin{cases}
0	&	x=x^* \\
* 		&	\text{otherwise}
\end{cases}
\enspace.
\end{equation*}

Indeed, if $\EQ$ returns ``correct'', we know that $h(x)=0$. 
Alternatively, the only possible counterexample is $x$ with $h(x)=1$.

Likewise, a $\lambda$-local equivalence query $(h,x^*)$ is equivalent to a partial equivalence query of the form
\begin{equation*}
h'(x)
=
\begin{cases}
h(x) 	&	x\in B_\lambda(x^*) \\
* 		&	\text{otherwise}
\end{cases}
\enspace.
\end{equation*}

However, in our set-up, contrary to \citep{maass1992lower}, the learner is restricted to a set of \emph{specific} partial queries rather than having access to any partial query, and is evaluated in the robust PAC-learning framework rather than in the online learning one.

\paragraph*{Comparison with online learning.}
We remark that the $\LEQ$ model evokes the online learning setting, where the learner receives counterexamples after making a prediction, but with a few key differences. 
Contrary to the online setting (and the exact learning framework with $\MQ$ and $\EQ$), there is an underlying distribution with which the performance of the hypothesis is evaluated in both the $\LMQ$ and $\LEQ$ models.
Moreover, in the mistake-bound model of online learning, when receiving a counterexample, the only requirement is that there be a concept that correctly classifies all the data given to the learner up until that point, and so the counterexamples can be given in an \emph{adversarial} fashion, in order to maximize the regret. 
However, both the $\LMQ$ and $\LEQ$ models require that a target concept be chosen a priori, so as to have a well-defined $\EX(c,D)$ oracle.
This is closer to the variant of the online learning setting in which an adversary must fix an instance's label before the learner makes a prediction \citep{littlestone1988learning}.

\section{Robust Learning with Local Membership Queries}
\label{sec:rob-learn-lmq}

In this section, we study the power of local membership queries in robust learning. 
We will focus on whether giving access to a $\lambda$-$\LMQ$ oracle can improve the robustness thresholds from Chapter~\ref{chap:rob-thresholds}.

We show a negative result: the amount of data needed to $\rho$-robustly learn conjunctions under the uniform distribution has an exponential dependence on the adversary's budget $\rho$ even when the learner has access to the $\LMQ$ oracle (in addition to the $\EX$ oracle).
Here, the lower bound on the sample drawn from the example oracle is $2^\rho$, which is the same as the lower bound for \emph{monotone} conjunctions derived in Theorem~\ref{thm:mon-conj}, and the local membership query lower bound is $2^{\rho-1}$. 
%The constants before $\rho$ can be slightly altered with a trade-off between the two.
The result relies on showing that there there exists a family of conjunctions that remain indistinguishable from each other on any sample of size $2^\rho$ and any sequence of $2^{\rho-1}$ LMQs with constant probability.

%The proof of the following theorem can be found in Appendix~\ref{app:conj-lmq-lb}.

\begin{theorem}
\label{thm:conj-lmq-lb}
Fix a monotone increasing robustness function $\rho:\N\rightarrow \N$ satisfying $2 \leq \rho(n) \leq n/4$ for all $n$.
Then, for any query radius $\lambda$, any $\rho(n)$-robust learning algorithm for the class $\Conj$ with access to the $\EX$ and $\lambda$-$\LMQ$ oracles has  joint sample and query complexity lower bounds of $2^\rho$ and $2^{\rho-1}$ under the uniform distribution.
\end{theorem}

%\todo{Space allowing, move proof of theorem to main body. Accepted submissions are allowed 10 content pages (vs 9 pages at submission).}
\begin{proof}
Let $D$ be the uniform distribution and, without loss of generality, let $\rho \geq 2$.
Fix two disjoint sets $I_1$ and $I_2$ of $2\rho$ indices in $[n]$ (i.e., $\vert I_1\vert = \vert I_2\vert=2\rho$), which will be the set of variables appearing in potential target conjunctions $c_1$ and $c_2$, respectively (i.e., their support).
We have $2^{4\rho}$ possible pairs of such conjunctions, as each variable can appear as a positive or negative literal.

Let us consider a randomly drawn sample $S$ of size $2^\rho$.
We will first consider what happens when all the examples in $S$ and the queried inputs $S'$ are negatively labelled.
Each negative example $x\in S$ allows us to remove at most $2^{2\rho+1}$ pairs from the possible set of pairs of conjunctions, as each component $x_{I_1}$ and $x_{I_2}$ removes at most one conjunction from the possible targets. 
By the same reasoning, each LMQ that returns a negative example can remove at most $2^{2\rho+1}$ pairs of conjunctions.
Note that the parameter $\lambda$ is irrelevant in this setting as each LMQ can only test one concept pair.
Thus, after seeing any random sample of size $2^\rho$ and querying any $2^{\rho -1}$ points, there remains 
\begin{equation}
\label{eqn:consistent-lmq}
\frac{2^{4\rho} - 2^{3\rho+1}-2^{3\rho}}{2^{4\rho}}\geq 1/4
\end{equation}
of the initial conjunction pairs that label all points in $S$ and $S'$ negatively. 
Then, choosing a pair $(c_1,c_2)$ of possible target conjunctions uniformly at random and then choosing $c$ uniformly at random gives at least a $1/4$ chance that $S$ and $S'$ only contain negative examples (both conjunctions are consistent with this).

Moreover, note that any  two conjunctions in a pair will have a robust risk lower bounded by $15/32$ against each other under the uniform distribution (see Lemma~\ref{lemma:bound-loss}).
Thus, any learning algorithm $\A$ with LMQ query budget $m'=2^{\rho-1}$ and strategy $\sigma:(\boolhc \times \set{0,1})^m\rightarrow(\boolhc \times \set{0,1})^{m'}$ (note that the queries can be adaptive) can do no better than to guess which of $c_1$ or $c_2$ is the target if they are both consistent on the augmented sample $S\cup\sigma(S)$, giving an expected robust risk lower bounded by a constant.
Letting $\E$ be the event that all points in both $S$ and $\sigma(S)$ are labelled zero, we get
\begin{align*}
\eval{c,S}{\risk_\rho^D(\A(S\cup\sigma(S)),c)}
&\geq \Prob{c,S}{\E}\eval{c,S}{\risk_\rho^D(\A(S\cup\sigma(S)),c)\given \E} \tag{Total Expectation}
 \\
&\geq \frac{1}{4}\;\eval{c,S}{\risk_\rho^D(\A(S\cup\sigma(S)),c)\given \E}
\tag{Equation~\ref{eqn:consistent-lmq}} \\
&= \frac{1}{4}\cdot\frac{1}{2}\;\eval{S}{\risk_\rho^D(\A(S\cup\sigma(S)),c_1)+\risk_\rho^D(\A(S\cup\sigma(S)),c_2)\given \E} 
\tag{Random choice of $c$} \\
&\geq\frac{1}{8}\;\eval{S}{\risk_\rho^D(c_1,c_2)\given \E} \tag{Lemma~\ref{lemma:robloss-triangle}}\\
&>\frac{1}{8}\cdot\frac{15}{32} \tag{Lemma~\ref{lemma:bound-loss}}\\
&=\frac{15}{256}
\enspace,
\end{align*}
which completes the proof.
\end{proof}

Now, since the local membership query lower bound  above has an exponential dependence on $\rho$, any perturbation budget $\omega(\log n)$ will require a sample and query complexity that is superpolynomial in $n$, giving the following corollary.

\begin{corollary}
The robustness threshold of the class $\Conj$ under the uniform distribution with access to $\EX$ and an $\LMQ$ oracle is $\Theta(\log(n))$.
\end{corollary}

Observe that the robustness threshold above is the same as when only using the $\EX$ oracle in Theorem~\ref{thm:mon-conj}, and that, as decision lists and halfspaces both subsume conjunctions, the lower bound of Theorem~\ref{thm:conj-lmq-lb} also holds for these classes.
Since we cannot improve the robustness threshold of conjunctions and superclasses under the uniform distribution with access to the $\LMQ$ oracle, we will turn our attention to a more powerful oracle in the next section.

\section{Robust Learning with Local Equivalence Queries}
\label{sec:rob-learn-leq}

In this section, we investigate the power of a local equivalence query oracle in the \emph{distribution-free} robust learning setting. 
We start with a negative result which shows that for a wide variety of concept classes, if $\lambda<\rho$, then \emph{distribution-free} robust learnability is impossible in the $\EX$+$\lambda$-$\LEQ$ model -- regardless of how many queries are allowed. 
This strengthens the impossibility result presented in Theorem~\ref{thm:no-df-rl}.
However, the regime $\lambda=\rho$, which implies giving similar power to the learner as the adversary, enables robust learnability guarantees.
Indeed, Section~\ref{sec:sc-ub-leq} exhibits upper bounds on sample sizes that will guarantee \emph{robust} generalization. 
These bounds are logarithmic in the size of the hypothesis class (finite case) and linear in the VC dimension of the  \emph{robust} loss of a concept class (infinite case). 
Section~\ref{sec:qc-ub-leq} draws a comparison between our framework and the online learning setting, and exhibits robustly consistent learners. 
It furthermore studies conjunctions and presents a robust learning algorithm that is \emph{both} statistically and computationally efficient.
It concludes by looking at linear classifiers in the discrete and continuous cases. 
We adapt the Winnow algorithm in the former setting. 
In the latter, we exhibit a sample complexity upper bound while outlining key obstacles to derive query complexity upper bounds, which will be addressed in Section~\ref{sec:adv-bounded-precision}. 

\subsection{Impossibility of Distribution-Free Robust Learning for $\lambda<\rho$}
\label{sec:imposs-df-leq}

We start with a negative result, saying that whenever the local query radius is strictly smaller than the adversary's budget, monotone conjunctions are not distribution-free robustly learnable. 
Note that this result goes beyond efficiency: no query can distinguish between two potential targets. 
Choosing the target uniformly at random lower bounds the expected robust risk, and hence renders robust learning impossible in this setting. 

\begin{theorem}
\label{thm:mon-conj-df-leq}
For locality and robustness parameters $\lambda,\rho\in\N$ with $\lambda < \rho $, monotone conjunctions (and any superclass) are not distribution-free $\rho$-robustly learnable with access to a $\lambda$-$\LEQ$ oracle.
\end{theorem}

The proof is similar in spirit to the earlier distribution-free impossibility results from Chapter~\ref{chap:def-adv-rob}.

\begin{proof}
Fix $\lambda,\rho\in\N$ such that $\lambda < \rho $, and consider the following monotone conjunctions: $c_1(x)=\bigwedge_{1 \leq i \leq \rho} x_i$ and $c_2(x)=\bigwedge_{1 \leq i \leq \rho +1} x_i$.
Let $D$ be the distribution on $\boolhc$ which puts all the mass on $\mathbf{0}$.
Then, the target concept is drawn at random between $c_1$ and $c_2$. 
Now, $c_1$ and $c_2$ will both give all points in $B_\lambda(\mathbf{0})$ the label 0, so the learner has to choose a hypothesis that is consistent with both $c_1$ and $c_2$ (otherwise the robust risk is 1 and we are done). 
However, the learner has no way of distinguishing which of $c_1$ or $c_2$ is the target concept, while these two functions have a $\rho$-robust risk of 1 against each other under $D$.
Formally, 
\begin{align}
\risk_\rho^D(c_1,c_2) 
&= \Prob{x\sim D}{\exists z\in B_\rho(x)\st c_1(z)\neq c_2(z)}\notag\\
&= \mathbf{1}[\exists z\in B_\rho(\mathbf{0})\st c_1(z)\neq c_2(z)]\notag\\ 
&=1 \label{eqn:mon-conj-risk-1}
\enspace,
\end{align}
where such $z=\mathbf{1}_\rho\mathbf{0}_{n-\rho}$.
To lower bound the expected robust risk, letting $\A$ be any learning algorithm and $\E$ be the event that all points in a randomly drawn sample $S$ are all labeled 0, we have
\begin{align*}
\eval{c,S}{\risk_\rho^D(\A(S),c)}
&= \eval{c,S}{\risk_\rho^D(\A(S),c)\given \E} \tag{By construction of $D$}
 \\
&= \frac{1}{2}\;\eval{S}{\risk_\rho^D(\A(S),c_1)+\risk_\rho^D(\A(S),c_2)\given \E} 
\tag{Random choice of $c$} \\
&\geq\frac{1}{2}\;\eval{S}{\risk_\rho^D(c_1,c_2)\given \E} \tag{Lemma~\ref{lemma:robloss-triangle}}\\
&=\frac{1}{2} \tag{Equation~\ref{eqn:mon-conj-risk-1}}
\enspace.
\end{align*}
\end{proof}

The result holds for monotone conjunctions and all superclasses (e.g., decision lists and halfspaces), but, in fact, we can generalize this reasoning to any concept class that has a certain form of stability: 
if we can find concepts $c_1$ and $c_2$ in $\C$ and points $x,x'\in \X$ such that $c_1$ and $c_2$ agree on $B_\lambda(x)$ but disagree on $x'$, then if $\lambda < \rho$, the concept class $\C$ is not distribution-free $\rho$-robustly learnable with access to a $\lambda$-$\LEQ$ oracle.
It suffices to ``move'' the center of the ball $x$ until we find a point in the set $B_\rho(x)\setminus B_\lambda(x)$ where $c_1$ and $c_2$ disagree, which is guaranteed to happen by the existence of $x'$.
As hinted earlier, this is not possible for parities, as any two parity functions $f_I$ and $f_J$ with index sets $I$ and $J$, respectively, will disagree on $B_1(x)$ for any $x\in\boolhc$, as it suffices to flip a bit in the symmetric difference $I\Delta J$ to cause them to disagree.

\subsection{Sample Complexity Upper Bounds}
\label{sec:sc-ub-leq}

In this section, we show that we can derive sample complexity upper bounds for \emph{robustly} consistent learners, i.e., learning algorithms that return a hypothesis with a \emph{robust} loss of zero on a training sample.
Note that, crucially,  the exact-in-the-ball notion of robustness and its realizability imply that any robust ERM algorithm will achieve zero empirical robust loss on a given training sample.
As we will see in the next sections, the challenge is to find a \emph{robustly} consistent learning algorithm that uses queries to $\rho$-$\LEQ$.
The first bound is for finite classes, where the dependency is logarithmic in the size of the hypothesis class. 
The proof is a simple application of Occam's razor and is included in Appendix~\ref{app:occam} for completeness.
The argument is similar to \cite{bubeck2019adversarial}.

\begin{lemma}
\label{lemma:occam}
Let $\C$ be a concept class and $\mathcal{H}$ a hypothesis class.
Any $\rho$-robust ERM algorithm using $\mathcal{H}\supseteq\C$ on a sample of size $m\geq \frac{1}{\epsilon}\left(\log |\mathcal{H}_n|+\log\frac{1}{\delta}\right)$ is a $\rho$-robust learner for  $\C$.
\end{lemma}

For the infinite case, we cannot immediately use the VC dimension as a tool for bounding the sample complexity of robust learning.
To this end, we use the VC dimension of the robust loss between two concepts, which is the VC dimension of the class of functions representing the $\rho$-expansion of the error region between any possible target and hypothesis.
This is analogous to the adversarial VC dimension defined by \cite{cullina2018pac} for the constant-in-the-ball definition of robustness. 

\begin{definition}[VC dimension of the exact-in-the-ball robust loss]
\label{def:rob-vc}
Given a target concept class $\C$, a hypothesis class $\mathcal{H}$ and a robustness parameter $\rho$, the VC dimension of the robust loss between $\C$ and $\H$ is defined as $\RVClong$, where $\RLossLong=\set{\Rloss: x\mapsto \mathbf{1}[\exists z\in B_\rho(x)\st c(z) \neq h(z)] \given c\in\C, h\in\mathcal{H}}$.
Whenever $\C=\mathcal{H}$, we simply write $\RVC$.
\end{definition}

We now show that we can use the VC dimension of the robust loss to upper bound the sample complexity of robustly-consistent learning algorithms. 
We will use this result in Section~\ref{sec:qc-halfspaces} when dealing with an infinite concept class: halfspaces on $\R^n$.
%The proof of the lemma below follows a simple modification of the standard PAC learning setting and can be found in Appendix~\ref{app:rob-vc}.

\begin{lemma}
\label{lemma:rob-vc}
Let $\C$ be a concept class and $\mathcal{H}$ a hypothesis class. Any $\rho$-robust ERM algorithm using $\mathcal{H}$ on a sample of size $m\geq \frac{\kappa}{\epsilon}\left(\RVClong\log(1/\epsilon)+\log\frac{1}{\delta}\right)$ for sufficiently large constant $\kappa$ is a $\rho$-robust learner for  $\C$.
\end{lemma}

\begin{proof}[Proof Sketch of Lemma~\ref{lemma:rob-vc}.]
The proof is very similar to the VC dimension upper bound in PAC learning.
The main distinction is that, instead of looking at the error region of the target and any function in $\mathcal{H}$, we look at its $\rho$-expansion.
Namely, let the target $c\in\C$ be fixed and, for $h\in\mathcal{H}$, consider the function $\Rloss: x\mapsto \mathbf{1}[\exists z\in B_\rho(x)\st c(z) \neq h(z)]$ and define a new concept class $\Delta_{c,\rho}(\mathcal{H})=\set{\Rloss \given h\in\mathcal{H}}$.
It is easy to show that $\VC(\Delta_{c,\rho}(\mathcal{H}))\leq \RVClong$, as any sign pattern achieved on the LHS can be achieved on the RHS.
The rest of the proof follows from the definition of an $\epsilon$-net and the bound on the growth function of $\Delta_{c,\rho}(\mathcal{H})$; see Appendix~\ref{app:lemma:rob-vc} for details.
\end{proof}

\begin{remark}
\label{rmk:rho-tradeoff}
Note that, for $\X=\boolhc$ and the Hamming distance, as $\rho(n)/n$ tends to $1$, we move towards the  exact and online learning settings, and the underlying distribution becomes less important. 
In this case, the VC dimension of the robust loss starts to decrease.
Indeed, say if $\rho=n$, then $\RLoss$ only contains the constant functions $0$ and $1$. 
We thus only need a single example to query the $\LEQ$ oracle (which has become the $\EQ$ oracle).
However, this comes at a cost: the \emph{query complexity} upper bounds presented in the next sections could be tight.
\end{remark}

\subsection{General Query Complexity Upper Bounds}
\label{sec:qc-ub-leq}

In the previous section, we derived sample complexity upper bounds for robustly consistent learners.
The challenge is thus to create algorithms that perform robust empirical risk minimization, as we are operating in the realizable setting.
We begin by showing that  online learning results can be used to guarantee robust learnability. 
We recall the online learning setting in Section~\ref{sec:online}.
We denote by $\Lit(\C)$ the Littlestone dimension of a concept class $\C$, which appears in the query complexity bound in the theorem below.

\begin{theorem}
\label{thm:soa}
A concept class $\C$ is $\rho$-robustly learnable with the Standard Optimal Algorithm (SOA)  \citep{littlestone1988learning} using the $\EX$ and $\rho$-$\LEQ$ oracles with sample complexity $ m(n,\epsilon,\delta) = \frac{1}{\epsilon}\left(\RVC\log(1/\epsilon)+\log\frac{1}{\delta}\right)$ and query complexity $r (n,\epsilon,\delta) = m(n,\epsilon,\delta)\cdot \Lit(\C)$.
Furthermore, if $\C$ is a finite concept class on $\boolhc$, then $\C$ is $\rho$-robustly learnable with sample complexity $ m (n,\epsilon,\delta) = \frac{1}{\epsilon}\left(\log(|\C|)+\log\frac{1}{\delta}\right)$ and query complexity $r (n,\epsilon,\delta) = m(n,\epsilon,\delta)\cdot \Lit(\C) $.
\end{theorem}

\begin{proof}
The sample complexity bounds come from Lemmas~\ref{lemma:occam} and~\ref{lemma:rob-vc} and the fact that the Standard Optimal Algorithm (SOA) is a consistent learner, as it will be given counterexamples in the perturbation region until a robust loss of zero is achieved. 

For each query to $\LEQ$, a counterexample is returned, or the robust loss is zero. 
Then, using the mistake upper bound of SOA, which is $\Lit(\C)$, we get the query upper bound.
\end{proof}

Of course, some concept classes, e.g., thresholds, have infinite Littlestone dimension, so Theorem~\ref{thm:soa} is not useful in these settings. 
In Section~\ref{sec:adv-bounded-precision}, we will study assumptions on the adversary's precision that give finite query upper bounds for linear classifiers.
But even if the Littlestone dimension is finite, the SOA can be computationally inefficient, or even untractable.
However, if we have access to an online learning algorithm with a mistake bound, it is possible to obtain robust learning guarantees.
Indeed, the theorem below exhibits a query upper bound for robustly learning with an online algorithm $\A$ with a given mistake upper bound $M$.
This is moreover particularly useful in case $\A$ is \emph{computationally} efficient (which is not the case for the Standard Optimal Algorithm in Theorem~\ref{thm:soa}) and  $M$ is polynomial in the input dimension.

\begin{lemma}
\label{lemma:rob-mistake-bound}
Let $\C$ be a concept class learnable in the online setting with mistake  bound $M(n)$.
Then $\C$ is $\rho$-robustly learnable using the $\EX$ and $\rho$-$\LEQ$ oracles with sample complexity $ m (n,\epsilon,\delta) = \frac{1}{\epsilon}\left(\RVClong+\log\frac{1}{\delta}\right)$ and query complexity $r (n,\epsilon,\delta) = m (n,\epsilon,\delta) \cdot M(n)$.
\end{lemma}

We remark that, implicit in the statement of the above lemma is the assumption that all potential mistakes must be contained in the potential perturbation region ($B_\rho(\supp(D))$, the $\rho$-expansion of the support of the distribution).
We now proceed with the proof of the above lemma.

\begin{proof}
The sample complexity bound is obtained from Lemma~\ref{lemma:rob-vc} and, for each point in the sample, a query to $\LEQ$ can either return a robust loss of 0 or 1 and give a counterexample. 
Since the mistake bound is $M$, we have a query upper bound of $r = m \cdot M$, as required.
\end{proof}

%\begin{remark}
%
%\end{remark}
%
%Now, if we are given a universal lower bound for learning $\C$ in the online setting, can we get a lower bound for $\rho$-robustly learning $\C$ with access to $\rho$-$\LEQ$? 
%If we can construct a distribution on $\X$ such that the robust learning algorithm will make $\Omega(M_L)$ mistakes with constant probability and the expected loss of the hypothesis will be bounded below by a constant, then we get a lower bound for our setting. 
%This is what was done for conjunctions, and note that singletons illustrate that we can get a lower bound strictly greater than $M_L$ in our setting: for any constant $K$, it suffices to let $D$ be the uniform distribution on $K$ points whose $\rho$-balls don't intersect and choose the target singleton u.a.r. from these $K$ points. Then, for any $K/2$ queries to $\LEQ$, the target singleton won't be identified with probability $\geq 1/2$, and the learner can do no better than guess at random if any of the $K/2$ is the target singleton, giving expected robust risk bounded below by $\Omega(1/K)$. 

\begin{remark}
In this section, we have assumed that $\lambda=\rho$. 
In Section~\ref{sec:sep-eq-leq}, where we compare the power of $\LEQ$ and $\EQ$, we will see a robust learning scenario where $\lambda>\rho$ dramatically increases the query complexity.
\end{remark}

\subsection{Improved Query Complexity Bounds for Conjunctions}
\label{sec:qc-conj}

We now show  how to improve the query upper bound from the previous section in the special case of conjunctions. 
Moreover, the algorithm used to robustly learn conjunctions is both statistically and \emph{computationally} efficient, which is not the case for the Standard Optimal Algorithm.

\begin{theorem}
\label{thm:conj-df-leq}
The class $\Conj$ is efficiently $\rho$-robustly learnable in the distribution-free setting using the $\EX$ and $\rho$-$\LEQ$ oracles with at most $O\left( \frac{1}{\epsilon}\left(n+\log\frac{1}{\delta}\right)\right)$ random examples and $O\left( \frac{1}{\epsilon}\left(n+\log\frac{1}{\delta}\right)\right)$ queries to $\rho$-$\LEQ$.
\end{theorem}

The algorithm achieving the above bounds is a straightforward adaptation of the online learning algorithm for conjunctions in Section~\ref{sec:online}.

\begin{proof}
%Let $\mathcal{A}$ be the standard proper algorithm to learn conjunctions.
Let $c$ be the target conjunction and let $D$ be an arbitrary distribution. 
We describe an algorithm $\A$ with polynomial sample and query complexity with access to a $\rho$-$\LEQ$ oracle.
By Lemma~\ref{lemma:occam}, if we can guarantee that $\A$ returns a hypothesis with zero robust loss on a i.i.d. sample of size $m= O\left( \frac{1}{\epsilon}\left(n+\log\frac{1}{\delta}\right)\right)$ with a polynomial number of queries to the $\rho$-$\LEQ$ oracle, we are done.

The algorithm  is similar to the standard PAC learning algorithm, in that it only learns from positive examples.
Indeed, the original hypothesis $h$ is a conjunction of all $2n$ literals. 
After seeing a positive example $x$, $\A$ removes from $h$ the literals $\bar{x_i}$ for $i=1,\dots,n$, as they cannot be in $c$.
Note that, by construction, any hypothesis $h$ returned by $\A$ always satisfies $c \subseteq h$.\footnote{We overload $c,h$ to mean both the functions and the set of literals in the conjunction, as it will be unambiguous to distinguish them from context.}
Thus, any counter example returned by the $\LEQ$ oracle will have that $c(z)=1$ and $h(z)=0$. 
This allows us to remove at least one literal from the hypothesis set for every counterexample.
Now, it is easy to see that, for $c\subseteq h' \subseteq h$, if the robust loss $\mathbf{1}[\exists z \in B_\lambda(x) \st c(z) \neq h(z)]$  on $x$ w.r.t. $h$ is zero, so will be the robust loss on $x$ w.r.t. the updated hypothesis $h'$. 
Hence, $\A$ makes at most $m+2n$ queries to the $\LEQ$ oracle.
\end{proof}

Note that the query upper bound that we get is of the form $m+M$, as opposed to $m\cdot M$ from Lemma~\ref{lemma:rob-vc} (where $m$ is the sample complexity and $M$ the mistake bound).
Indeed, any update to the hypothesis will not affect the consistency of previously queried points with robust loss of zero.
Thus, once zero robust loss is achieved on a point, it does not need to be queried again.  

\subsection{Bounds for Linear Classifiers}
\label{sec:qc-halfspaces}

In this section, we first derive sample and query complexity upper bounds for linear classifiers on $\boolhc$ with bounded weights. 
We then derive sample complexity bounds for linear classifiers on $\R^n$ and outline obstacles for query complexity upper bounds. 
Note that the robustness threshold of linear classifiers on $\boolhc$ \emph{without} access to the $\LEQ$ oracle remains an open problem, as pointed out in Chapter~\ref{chap:rob-thresholds}.

%\subsubsection{Linear Classifiers on $\boolhc$}

Let $\Halfspaces_{\boolhc}^W$  be the class of linear threshold functions on $\boolhc$ with integer weights  such that the sum of the absolute values of the weights and the bias is bounded above by $W$. We have the following theorem, whose proof relies on bounding the size of $\Halfspaces_{\boolhc}^W$  and using the mistake bound for Winnow \citep{littlestone1988learning}. 

\begin{theorem}
\label{thm:ltf-bool-df}
The class $\Halfspaces_{\boolhc}^W$ is $\rho$-robustly learnable with access to the $\EX$ and $\rho$-$\LEQ$ oracles using the Winnow algorithm with sample complexity $m(n,\epsilon,\delta)=O\left(\frac{1}{\epsilon}\left(n+\min\set{n,W}\log (W+n)+\log\frac{1}{\delta}\right)\right)$ and local equivalence query complexity $r(n,\epsilon,\delta)=$ $O(m(n,\epsilon,\delta) \cdot W^2\log(n))$.
\end{theorem}

\begin{proof}
The sample complexity bound uses Lemma~\ref{lemma:occam}.
Note the class $\Halfspaces_{\boolhc}^W$  has size $O(2^n (n+W)^{\min\set{n,W}})$.
This  is a simple application of the \emph{stars and bars} identity, where $W$ is the number of stars and $n+1$ the number of bars (as we are considering the bias term as well): ${n+W\choose W}=O( (n+W)^{min\set{n,W}})$. 
The $2^n$ term comes from the fact that each weight can be positive or negative.
The query complexity uses the fact that the mistake bound for Winnow for $\Halfspaces_{\boolhc}^W$ is $O(W^2\log(n))$ in the case of positive weights (the full statement can be found in Section~\ref{sec:online}).
\citet{littlestone1988learning} outlines how to use the Winnow algorithm when the linear classifier's weights can vary in sign, at the cost of doubling the input dimension and weight bound (see Theorem 10 and Example 6 therein).
\end{proof}

We now turn our attention to linear classifiers $\ltfreal$ on $\R^n$. 
We first show that, when considering an adversary with bounded $\ell_2$-norm perturbations, we can bound the sample complexity of robust learning for this class through a bound on the VC dimension of the robust loss. 
However, the query complexity is infinite in the general case (we will later prove an infinite lower bound in Corollary~\ref{cor:ltf-real-leq-lb}).
This is because the Littlestone dimension of thresholds, and thus halfspaces, is infinite (see Section~\ref{sec:online} for details). 
We will address this issue in Section~\ref{sec:adv-bounded-precision}.

\begin{theorem}
\label{thm:sc-ub-ltf-real}
Let the adversary's budget be measured by the $\ell_2$ norm.
Then any $\rho$-robust ERM learning algorithm for $\ltfreal$ on $\R^n$ has sample complexity $m=O(\frac{1}{\epsilon}( n^3 + \log (1/\delta)))$.
\end{theorem}

The proof of this theorem relies on deriving an upper bound on the VC dimension of the robust loss of halfspaces.
This enables us to bound the sample complexity needed to guarantee robust accuracy. 
We will need the following theorem from \cite{goldberg1995bounding}:

\begin{theorem}[Theorem 2.2 in \citep{goldberg1995bounding}]
\label{thm:goldberg}
Let $\set{\C_{k,n}}_{k,n\in\N}$ be a family of concept classes where concepts in $\C_{k,n}$ and instances are represented by $k$ and $n$ real values, respectively.
Suppose that the membership test for any instance $\alpha$ in any concept $C$ of $\C_{k,n}$ can be expressed as a boolean formula $\Phi_{k,n}$ containing $s = s(k,n)$ distinct atomic predicates, each predicate being a polynomial inequality or equality over $k+n$ variables (representing $C$ and $\alpha$) of degree at most $d=d(k,n)$.
Then $\VC({\C_{k,n}}) \leq 2k\log (8eds)$. 
\end{theorem}

We will now translate the $\rho$-expansion of the error region (i.e., the robust loss function) between two halfspaces as a boolean formula. 
The following result from \cite{renegar1992computational}, will be instrumental to obtain our result:

\begin{theorem}[Theorem 1.2 in \cite{renegar1992computational}]
\label{thm:renegar}
Let $\Psi$ be a formula in the first-order theory of the reals of the form 
$$(Q_1 x^{[1]}\in\R^{n_1})\dots (Q_\omega x^{[\omega]}\in\R^{n_\omega})P(x^{[1]},\dots,x^{[n_\omega]},y)\enspace,$$
with free variables $y=(y_1,\dots,y_l)$, quantifiers $Q_i$ ($\exists$ or $\forall$) and quantifier-free Boolean formula $P(x^{[1]},\dots,x^{[n_\omega]},y)$ with $m$ atomic predicates consisting of polynomial inequalities of degree at most $d$.  
There exists a procedure that constructs an equivalent quantifier-free formula $\Phi$ of the form
$$\bigvee_{i=1}^I \bigwedge_{j=1}^{J_i} (h_{ij}(y)\Delta_{ij} 0)
\enspace,$$
where 
\begin{align*}
I&\leq (md)^{2^{O(\omega)}l\prod_k n_k}\\
J_i&\leq (md)^{2^{O(\omega)}\prod_k n_k}\\
\deg(h_{ij})&\leq (md)^{2^{O(\omega)}\prod_k n_k}\\
\Delta_{ij}&\in\set{\leq,\geq,=,\neq,>,<} 
\enspace.
\end{align*}

\end{theorem}

%This will allow us to use the theorem above from \cite{goldberg1995bounding} to bound the VC dimension of the robust lossof $\ltfreal$.

We are now ready to state the key technical lemma need for the proof of Theorem~\ref{thm:sc-ub-ltf-real}.

\begin{lemma}
\label{lemma:rob-risk-bool-formula}
Let $a,b\in\R^{n}, a_0,b_0\in\R$, and define the map $\varphi: x \mapsto \mathbf{1}[\exists z\in B_\rho(x) \st \sgn(a^\top z + a_0)\neq \sgn(b^\top z + b_0)]$.
Then $\varphi$ can be represented as a boolean formula $\Phi$ with $s=10^{Cn^2}$ distinct atomic predicates, with each predicate being a polynomial inequality over $2n+2$ variables of degree at most $10^{C'n}$ for some constants $C,C'>0$.
\end{lemma}

\begin{proof}
First note that the predicate $\sgn(a^\top z+a_0)\neq\sgn(b^\top z+b_0)$ %is equivalent to $\sgn(a\top z+a_0)\neq\sgn(b^\top z+b_0)<0$
can be represented as the following formula:
\begin{equation*}
\left(a^\top z+a_0 \geq 0 \wedge b^\top z+b_0 < 0\right) \vee \left( a^\top z+a_0 <0 \wedge b^\top z+b_0\geq 0\right)\enspace,
\end{equation*}
which contains $n+(2n+2)$ variables and 4 predicates.
Moreover, given a perturbation $\zeta\in\R^n$, the constraint $\norm{\zeta}_2\leq \rho$ on its magnitude is a polynomial inequality of degree 2:
$$\sum_i \zeta_i^2 \leq \rho^2 \enspace.$$
Now, consider the following formula:
\begin{equation*}
\Psi(x) = \exists \zeta \in\R^n \st \left( \sgn(a^\top (x+\zeta)+a_0)\neq\sgn(b^\top (x+\zeta)+b_0) \wedge \norm{\zeta}_2\leq \rho\right)
\enspace.
\end{equation*}
%where $x^{(I)}_i=-x_i$ if $i\in I$ and remains unchanged otherwise.
This is a formula of first-order logic over the reals.
Using the notation of Theorem~\ref{thm:renegar}, we have $\omega=1$ quantifier, and thus $\prod_k n_k = n$, one Boolean formula with $m=5$ polynomial inequalities of degree $d$ at most $2$, and $l=n$.
Thus, $\Psi(x)$ can be expressed as a quantifier-free formula $\Phi(x)=\bigvee_{i=1}^I \bigwedge_{j=1}^{J_i} (h_{ij}(y)\Delta_{ij} 0)$ of size $$I\max_i J_i \leq (md)^{2^{O(\omega)}l\prod_k n_k+2^{O(\omega)}\prod_k n_k}\leq 10^{Cn^2} $$ for some constant $C$, where the polynomial inequalities are of degree at most $(md)^{2^{O(\omega)}\prod_k n_k}\leq 10^{C'n}$ for some constant $C'$.
\end{proof}

We thus get the following corollary.

\begin{corollary}
\label{cor:rvc-ltf}
The VC dimension of the robust loss of $\ltfreal$ %on $\realbox$ 
is $O( n^3)$.
\end{corollary}
\begin{proof}
We let $s=10^{Cn^2}$, $k=2n+2$ and $d=10^{C'n}$ from the proof above and use Definition~\ref{def:rob-vc} and Theorem~\ref{thm:goldberg}  to get a VC dimension of the robust loss upper bound of $O(k\log(sd))=O(n^3)$.\footnote{Note that Corollary~2.4 in \cite{goldberg1995bounding} uses this reasoning.}
\end{proof}

Proving  Theorem~\ref{thm:sc-ub-ltf-real} is now a straightforward application of the results above.

\section[Precision-Bounded Adversaries]{Robust Learning Against Precision-Bounded Adversaries}
\label{sec:adv-bounded-precision}

It is possible to obtain some relatively straightforward robustness guarantees for classes with infinite Littlestone dimension if there exists a sufficiently large margin between classes (in which case the exact-in-the-ball and  constant-in-the-ball notions of robustness coincide). %\footnote{Formally, if $\supp(D_y)$ is the support of the distribution of the class with label $y$, we require their $\rho$-expansion to not intersect: $B_\rho(\supp(D_0))\cap B_\rho(\supp(D_1))=\emptyset$.} 
However, some of these results have already been derived in the literature.
%Indeed, in these cases, . 
See, e.g., \citep{cullina2018pac} for the sample complexity of halfspaces in the constant-in-the-ball realizable setting w.r.t. $\ell_p$-norm adversaries, which improves on the sample complexity bound of Theorem~\ref{thm:sc-ub-ltf-real} by being linear -- vs cubic -- in the input dimension; together with a mistake bound for Perceptron, we get $\LEQ$ bounds.\footnote{In this case, we would need  a margin between the sets $B_\rho(\supp(D_0))$ and $B_\rho(\supp(D_1))$, as these are the sets of potential counterexamples -- the condition $B_\rho(\supp(D_0))\cap B_\rho(\supp(D_1))=\emptyset$ is not sufficient in itself to get guarantees for hypotheses with infinite Littlestone dimension. See \citep{montasser2021adversarially} for both upper and lower bounds in this setting.}

Instead, in this section, we look at robust learning  problems in which the decision boundary can cross the perturbation region, but where the adversary's precision is limited.
We use ideas from \cite{ben2009agnostic} concerning hypotheses with margins in the online learning framework. 
Note, however, that here the margin does not represent sufficient distance between classes, but rather a region of the instance space that is too costly for the adversary to access (e.g., the number of bits needed to express an adversarial example is too large).

Examining the proof that the Littlestone dimension of thresholds is infinite (see Section~\ref{sec:online}), the key assumption is that the adversary has \emph{infinite} precision, which is perhaps not a reasonable assumption to make in practice. 
More precisely, in the construction of the Littlestone tree, each counterexample given requires an additional bit to be described, as the remainder of the interval $[0,1]$ is split in two at each prediction. 
Our work in this section formally and more generally addresses this potential issue.

We now define the meaning of bounding an adversary's precision in the context of robust learning, which is depicted in Figure~\ref{fig:bounded-precision}.

\begin{figure}
\begin{center}
\includegraphics[scale=0.3]{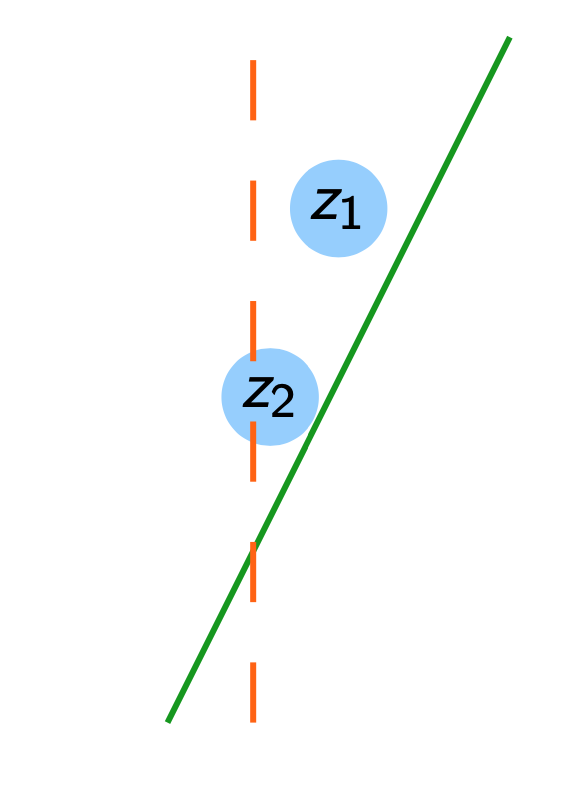}
\end{center}
\caption{The dotted line is the hypothesis $h$, and the solid line, the target $c$. The adversary has precision $\tau$. The shaded regions represent the set $B_\tau(z_i)$. The counterexample $z_1$ is valid as $c$ and $h$ disagree on all of $B_\tau(z_1)$ and both functions are constant in this region, but $z_2$ is not as $c$ and $h$ agree on part of $B_\tau(z_2)$.}
\label{fig:bounded-precision}
\end{figure}

\begin{definition}[Precision-Bounded Adversary]
Let $(\X,d)$ be a metric space, and let an adversary $\adv$ have budget $\rho$. 
We say that $\adv$ is \emph{precision-bounded} by $\tau$, if for target $c$, hypothesis $h$, and input $x$, $\adv$ can only return counterexamples $z\in B_\rho(x)$ such that $h$ and  $c$ are both constant and disagree on the whole region $B_{\tau}(z)$ and $B_{\tau}(z)\subseteq B_\rho(x)$.
\end{definition}

\paragraph*{Comparison with online learning.}
Note that if we set $\rho$ to be large enough so that the perturbation region is the whole instance space $\X$ for any point $x$ (more generally, $\U(x)=\X$), we (almost) recover the adversary model in the online learning setting.
The only distinction is that, in online learning, at a given time $t$ the learner is a point $x_t$ to classify (implicitly classifying the whole region $B_\tau(x_t)$ in our precision-bounded setting), rather than committing to a hypothesis $h$ on the whole instance space. 
The adversary (or ``nature'' if a target must be chosen a priori) reveals the true label after a prediction is made. 

In the mistake-bound model, the only constraint is that there exists a concept in $\H$ that is consistent with the labelled sequence $(x_1,y_1,),\dots,(x_t,y_t)$ seen so far.
When working with a precision-bounded adversary, we  are implicitly asking the adversary to not give counterexamples too close to the boundary.
Then, in the mistake-bound model, this translates into the adversary giving a point $x_t$ to predict such that there does not exist time steps $t',t''$ where $y_{t'}\neq y_{t''}$ and $B_\tau(x_{t'})$ and $B_\tau(x_{t''})$ both intersect with $B_\tau(x_t)$, hence a \emph{margin}. 
Margin-based complexity measures for online learning adapted from \cite{ben2009agnostic} will be used in this section. 

A subtle distinction between the definitions we give below and that of \cite{ben2009agnostic} is that the latter defined margin-based Littlestone trees and Littlestone dimension for margin-based \emph{hypothesis classes}. 
They require that the hypothesis class $\H$ satisfies the following: for all $h\in\H$, $h$ is of the form $\X\rightarrow\R$, and the prediction rule is 
\begin{equation}
\phi(h(x))=\frac{\sgn(h(x))+1}{2}
\enspace,
\end{equation}
where the magnitude $\abs{h(x)}$ is the \emph{confidence} in the prediction. 
The $\mu$-\emph{margin-mistake} on an example $(x,y)$ is defined as
\begin{equation}
\label{eqn:margin-mistake}
\abs{h(x)-y}_\mu=
\begin{cases}
0 & \text{if }\phi(h(x))=y \wedge \abs{h(x)}\geq \mu \\
1 & \text{otherwise}
\end{cases}
\enspace.
\end{equation}

For us, since it is the \emph{adversary} that is bounded in its precision, we instead consider any hypothesis class where the concepts are boolean functions whose domain is a metric space $(\X,d)$. 
Rather than having the condition $\abs{h(x)}\geq \mu$ from Equation~\ref{eqn:margin-mistake}, we encode a margin representing the precision $\tau$ by the requirement that hypotheses must be constant in the $\tau$-expansion around any point in the Littlestone trees.
This difference is not only stylistic, but also concerns the semantics of the margin. 
Our definition moreover implies a uniform margin on the instance space, while the one from \cite{ben2009agnostic} can fluctuate in the instance space based on the classifier's confidence. 
However, the tools and techniques used here don't differ much in essence from the ones in \cite{ben2009agnostic}. 
The main novelty is the meaning of the notion of margin and its study in the context of robust learning.

\begin{definition}[Littlestone Trees of Precision $\tau$]
\label{def:lit-tree-tau}
A Littlestone tree of precision $\tau$ for a hypothesis class $\mathcal{H}$ on metric space $(\X,d)$ is a complete binary tree $T$ of depth $d$ whose internal nodes are instances $x\in\X$.
Each edge is labelled with $-$ or $+$ and corresponds to the potential labels of the parent node $x$ and the region $B_\tau(x)$.
Each path from the root to a leaf must be consistent with some $h\in\mathcal{H}$, i.e. if $x_1,\dots,x_d$ with labellings $y_1,\dots,y_d$ is a path in $T$, there must exist $h\in\mathcal{H}$ such that $h\vert_{B_\tau(x_i)}=y_i$ for all $i$. 
\end{definition}

While it is possible to have a hypothesis giving different labels to points in the region $B_\tau(x)$ in the standard setting, in the above construction, one must commit to labelling the whole region $B_\tau(x)$ either positively or negatively. 

For the remainder of the text, we will identify each leaf in a Littlestone tree $T$ with a hypothesis $h\in\H$ that is consistent with the labellings along the path from the root to this leaf. 
Note that the choice of labelling $y$ of $B_\tau(x)$ of some $x\in T$ implies that, in contrast to the standard Littlestone trees, any $h\in\H$ with $h(x)=y$ that is \emph{not} constant on  $B_\tau(x)$ cannot be consistent with any path in $T$. 
The set of consistent hypotheses on $T$ thus does \emph{not} form a partition of $\H$ in our precision-bounded setting.

We now remark that, by definition, no node in the tree has a $\tau$-expansion that overlaps with the $\tau$-expansion of any of its ancestors.

\begin{proposition}
\label{prop:ancestor-intersection-empty}
Let $T$ be a Littlestone tree of precision $\tau$.
Then for any node $x\in T$ and ancestor $x'\in T$ of $x$, $B_\tau(x)\cap B_\tau(x')=\emptyset$.
\end{proposition}
\begin{proof}
Take two paths from the root to two distinct leaves, $h_0$ and $h_1$, respectively.
Let the paths branch off at $x\in T$, with $h_y$ giving label $y$ to the whole region $B_\tau(x)$.
Let $x'$ be an ancestor of $x$ in $T$, and note that $h_0=h_1=b$ on $B_\tau(x')$ for some $b\in\{0,1\}$.
Then, since $h_0$ and $h_1$ must disagree on all of $B_\tau(x)$, it follows that $B_\tau(x)\cap B_\tau(x')=\emptyset$.
\end{proof}

We can now define the following variant of the Littlestone dimension, which is analogous to the margin-based Littlestone dimension of \cite{ben2009agnostic}.

\begin{definition}[Precision-Bounded Littlestone Dimension]
The Littlestone dimension of precision $\tau$ of a hypothesis class $\mathcal{H}$ on metric space $(\X,d)$, denoted $\Lit_\tau(\mathcal{H})$, is the depth $k$ of the largest Littlestone tree with bounded precision $\tau$ for $\mathcal{H}$. If no such $k$ exists then $\Lit(\mathcal{H})=\infty$.
\end{definition}

Note that setting $\tau=0$, i.e., there are no constraints on the nodes, we recover the Littlestone tree and Littlestone dimension definitions.
As an example, let us consider the class of threshold functions, which, when $\tau=0$, have infinite Littlestone dimension.

\begin{proposition}
Let $\tau>0$. 
The class $\thresholds_B$ of threshold functions on $[0,B]$ induce Littlestone trees of precision $\tau$ of depth bounded by $log\frac{B}{\tau}-1$. Thus $\Lit_\tau(\thresholds_B)=\lfloor\log\frac{B}{\tau}-1\rfloor$.
\end{proposition}

\begin{proof}
Let $\tau>0$ be arbitrary.
Here, the optimal strategy to construct a Littlestone tree is to divide the interval $[0,B]$ in two equal parts at each round.
Given $x\in[0,B]$ and $\alpha<\alpha'\in\R$, in order to have two threshold functions $h_\alpha(x)=\mathbf{1}[x\geq \alpha]$ and $h_{\alpha'}(x)=\mathbf{1}[x\geq \alpha']$ that disagree on the whole range $[x-\tau,x+\tau]$, we need both $\alpha<x-\tau$ and $\alpha'\geq x+\tau$.
Thus, at depth $d$, we have divided $[0,B]$ into $2^d$ parts we must have $2\tau\geq B2^{-d}$, implying $\Lit_\tau(\thresholds_B)=\lfloor\log\frac{B}{\tau}-1\rfloor$.
\end{proof}
 
We now show a lower bound on the number of mistakes of any learner against an adversary with bounded precision $\tau$.
The proof is identical to the regime $\tau=0$.

\begin{theorem}
Any online learning algorithm for $\C$ has mistake bound $M\geq \Lit_\tau(\C)$ against a $\tau$-precision-bounded adversary.
\end{theorem}
\begin{proof}
Let $\A$ be any online learning algorithm for $\C$.
Let $T$ be a Littlestone tree of bounded precision $\tau$ and depth $\Lit_\tau(\C)$ for $\C$. 
Clearly, an adversary can force $\A$ to make $\Lit_\tau(\C)$ mistakes by sequentially and adaptively choosing a path in $T$ in function of $\A$'s predictions.
\end{proof}

Now, let us consider a version of the SOA where the adversary has  precision $\tau$. 
The algorithm is identical to the SOA (see Algorithm~\ref{alg:soa} in Chapter~\ref{chap:background}), except for the definition of $V^{(b)}_t$, which requires that the hypotheses are constant in the region around the prediction.

\begin{algorithm}
\caption{Precision-Bounded Standard Optimal Algorithm}
\begin{algorithmic}
\Require A hypothesis class $\mathcal{\H}$
\For {$t=1,2,\dots$}
\State $V_1\gets\H$
\State Receive example $x_t$
\State $V^{(b)}_t \gets \set{h\in V_t\given h\vert_{B_\tau(x_t)}=b}$
\State $\hat{y_t}=\underset{b}{\arg\max} \;\Lit_\tau\left( V^{(b)}_t \right)$ 
\State Receive true label $y_t$
\State $V_{t+1} \gets V^{(y_t)}_t$
\EndFor
\end{algorithmic}
\label{alg:soa-precision}
\end{algorithm}

Below, we show that this slight modification of the SOA is also optimal for cases in which the adversary is constrained by $\tau$.
This is analogous to Theorem~21 in \citep{ben2009agnostic}, who did not include the proof of optimality for brevity. 
We have included it in this thesis for completeness.

\begin{theorem}
The precision-bounded Standard Optimal Algorithm makes at most $\Lit_\tau(\C)$ mistakes in the mistake-bound model of online learning when the adversary has precision $\tau$.
\end{theorem}

\begin{proof}
We will show that, at every mistake, the precision-bounded Littlestone dimension of the subclass $V_t$ decreases by at least 1 after receiving the true label $y_t$.

WLOG, assume that there does not exist $t'<t$ such that $x_{t'}\in B_\tau(x_t)$, as otherwise this implies that $V_t^{(y_{t'})}=V_t$ and $V_t^{(\neg y_{t'})}=\emptyset$, and we cannot make a mistake (note in particular that we cannot have two differently labelled points in $B_\tau(x_t)$ as otherwise this would not be a valid example for the adversary to give).

Suppose that, at time $t$, $y_t= \arg\min_b \Lit_\tau(V^{(b)}_t)$. 
Note that $V_{t+1} = V^{(y_t)}_t$.
Now, consider any two Littlestone trees  $T_{y_t}$ and  $T_{\hat{y}_t}$ of precision $\tau$ and maximal depths for $V^{(y_t)}_t$ and $V^{(\hat{y})}_t$, respectively. 
By Proposition~\ref{prop:ancestor-intersection-empty} and definition of $V^{(b)}_t$, neither tree can contain nodes whose $\tau$-expansions intersect with $B_\tau(x_t)$. 
Moreover, all hypotheses in $V^{(y_t)}_t$ and $V^{(\hat{y})}_t$ are constant on $B_\tau(x_t)$.
Hence it is possible to construct a $\tau$-constrained Littlestone tree $T$ for $V_t$ of depth $\min_b \Lit_\tau(V^{(b)}_t)+1$ (recall that $T$ must be complete).
Then $\Lit_\tau(V_t)\geq\Lit_\tau(V^{({y_t})}_t)+1=\Lit_\tau(V_{t+1})+1 $, as required.\footnote{Note that the Littlestone dimension does not necessarily decrease when $y_t=\hat{y}_t$, as we could have $V_t=V_t^{(y_t)}$.}
\end{proof}

\begin{remark}
When considering threshold functions on $[0,1]$, and given example $x_t$ to predict, the SOA's strategy is effectively to look at the labelled points in the history and consider the largest $x^{(0)}\in[0,1]$ with negative label and the smallest $x^{(1)}\in[0,1]$ with positive label, and predict $y_t=\underset{b}{\arg\min} \abs{x_t - x^{(b)}}$.
\end{remark}

We now turn our attention to the robust learning of halfspaces in $(\R^n,d_2)$ against adversaries of precision $\tau$, where $d_2$ is the metric induced by the $\ell_2$ norm.
 As pointed out by \cite{ben2009agnostic}, we essentially have the same argument as the Perceptron algorithm, because, once the hypothesis is sufficiently close to the target, the adversary cannot return counterexamples near the boundary.
Note that this result can be generalized to $\ell_p$ norms.
Figure~\ref{fig:ltf-bounded-precision} depicts the argument of the proof of Theorem~\ref{thm:ltf-real-precision}.

\begin{theorem}
\label{thm:ltf-real-precision}
Fix constants $B,\tau>0$.
Let the adversary's budget $\rho$ be measured by the $\ell_2$ norm.
Let $\ltfreal$ be the class of halfspaces on $\R^n$ where the instance space is restricted to points $x\in\R^n$ with $\norm{x}_2\leq B-\rho$. 
Then, $\ltfreal$  is distribution-free $\rho$-robustly learnable against an adversary of precision $\tau$ using the $\EX$ and $\rho$-$\LEQ$ oracles with sample complexity $m(n,\epsilon,\delta)=O(\frac{1}{\epsilon}( n^3 + \log (1/\delta)))$ and query complexity $r(n,\epsilon,\delta)=m(n,\epsilon,\delta)\cdot\frac{B^2}{\tau^2}$.
Note that this is query-efficient if $\frac{B^2}{\tau^2}=\poly(n)$.
\end{theorem}

 Note that the dependence on $\tau$ in the mistake bound, and thus the $\LEQ$ upper bound, is $1/\tau^2$, in contrast to the dependence of $\log 1/\tau$ for thresholds.

\begin{figure}
\begin{center}
\includegraphics[scale=0.27]{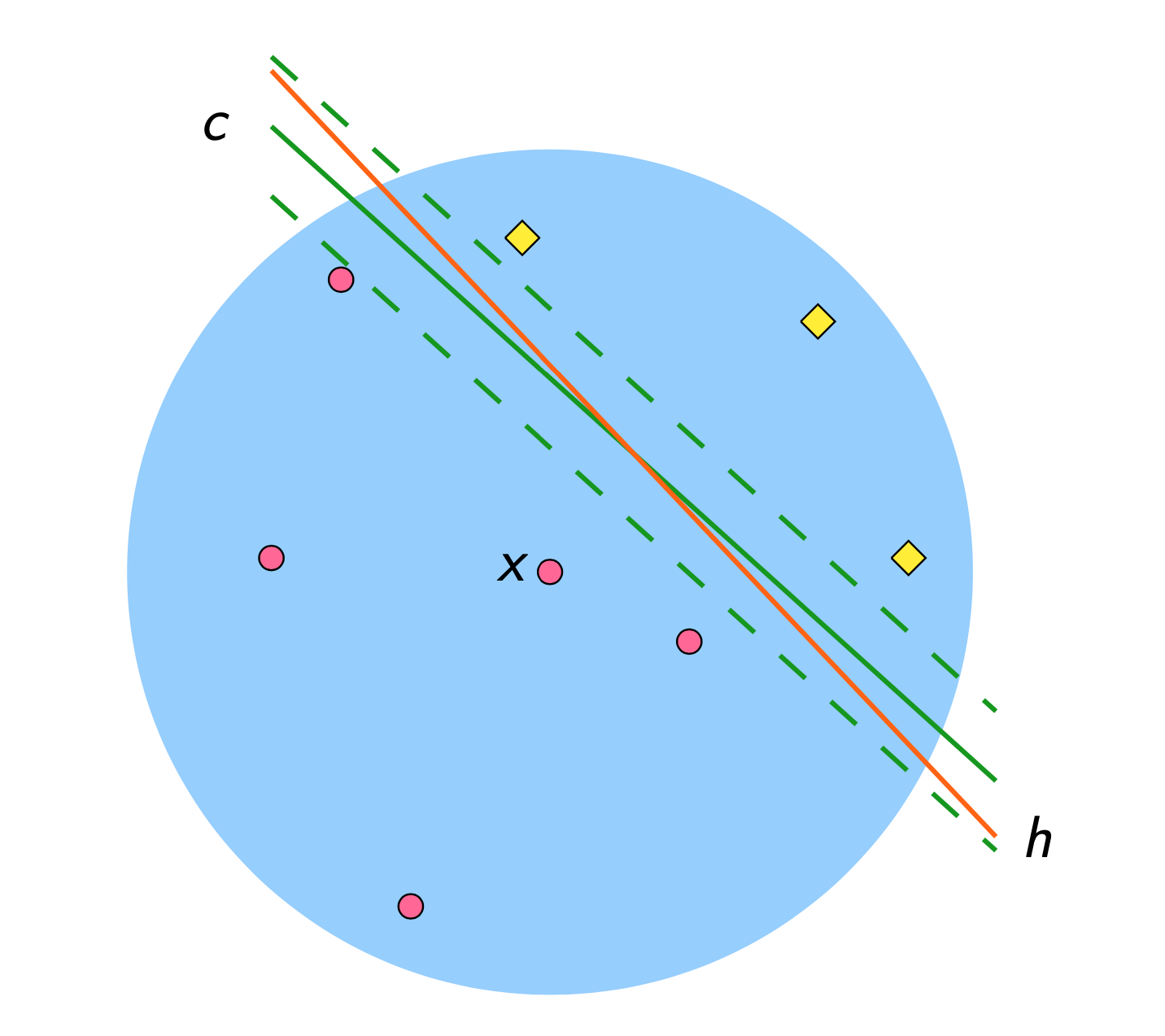}
\end{center}
\caption{A visual representation of the proof of Theorem~\ref{thm:ltf-real-precision}. The dotted lines on either side of the target $c$ represent a margin of $\tau/2$. Any hypothesis within the dotted lines in the (shaded) perturbation region ensures that an adversary of bounded precision $\tau$ cannot return any counterexamples. Finally, counterexamples must be labelled according to the target $c$, and both $h$ and $c$ are not constant on $B_\rho(x)$.}
\label{fig:ltf-bounded-precision}
\end{figure}

 \begin{proof}
The sample complexity follows from Theorem~\ref{thm:sc-ub-ltf-real}.
The query upper bound  follows from Lemma~\ref{lemma:rob-mistake-bound}  and the mistake bound for the Perceptron algorithm (see Theorem~\ref{thm:mistake-bound-perceptron}).
To see that the bound for Perceptron can be used, note that the adversary having precision $\tau$ implies that any consistent target function $c(x)=a^\top x +a_0$ and any counterexample $z$ will satisfy the conditions 
 (i) $\norm{z}_2\leq B$ and (ii) $\tau \leq \frac{c(z)(a^\top z)}{\norm{z}_2}$ from Theorem~\ref{thm:mistake-bound-perceptron}.
\end{proof}
	
\section{Lower Bounds on Robust Learning with $\LEQ$}
\label{sec:qc-lb-leq}

In this section, we derive lower bounds on the expected number of queries of robust learning algorithms for various concept classes. 
We start with general lower bounds and conclude by looking at specific concept classes.

\subsection{General Query Complexity Lower Bounds}

We start by giving a general lower bound on the query complexity of the $\LEQ$ oracle that is linear in the \emph{restricted} Littlestone dimension of a concept class.
This notion, which restricts the region of the instance space where the nodes in the Littlestone tree can come from, will be defined below, along with the restricted VC dimension. 

We first define the notion of the \emph{restricted} VC dimension of a concept class.
We note that we can straightforwardly extend the definitions below to an arbitrary perturbation region $\U:\X\rightarrow 2^\X$, and obtain analogous results.

\begin{definition}[$\rho$-restricted VC Dimension]
The $\rho$-restricted VC dimension of a concept class $\C$, denoted $\VC\vert_\rho(\C)$, is the size $d$ of the largest set $X\subseteq\X$ shattered by $\C$ such that there exists $x^*\in X$ where $x\in B_\rho(x^*)$ for all $x\in X$.
\end{definition}

We now introduce the restricted Littlestone dimension.

\begin{definition}[$\rho$-restricted Littlestone Dimension]
The $\rho$-restricted Littlestone dimension of a hypothesis class $\mathcal{H}$, denoted $\Lit\vert_\rho(\mathcal{H})$, is the depth $d$ of the largest Littlestone tree $T$ for $\mathcal{H}$ with root node $x^*$ such that  $x\in B_\rho(x^*)$ for all the nodes $x\in T$. 
\end{definition}

\begin{remark}
\label{rmk:restricted-vc-lit}
It follows from the upper bound on the (standard) VC dimension by the Littlestone dimension that $\VC\vert_\rho(\C)\leq\Lit\vert_\rho(\mathcal{H})$, as we can construct a restricted Littlestone tree from a witness set of the restricted VC dimension.
\end{remark}

We are now ready to state the main theorem of this section.

\begin{theorem}
\label{thm:lit-queries}
Let $\C$ be a concept class of $\rho$-restricted Littlestone dimension $\Lit\vert_\rho(\C)=d$.
Then there exists a distribution on $\X$ such that any $\rho$-robust learning algorithm for $\C$ has an expected number of queries $\Omega(d)$ to the $\rho$-$\LEQ$ oracle.
\end{theorem}

As a consequence of Remark~\ref{rmk:restricted-vc-lit}, we have the following corollary.

\begin{corollary}
\label{cor:vc-queries}
Let $\C$ be a concept class of $\rho$-restricted VC dimension $\VC\vert_\rho(\C)=d$.
Then there exists a distribution on $\X$ such that any $\rho$-robust learning algorithm for $\C$ has an expected number of queries $\Omega(d)$ to the $\rho$-$\LEQ$ oracle.
\end{corollary}

The proof of Theorem~\ref{thm:lit-queries} is similar to showing that a mistake lower bound in online learning can be transformed in an expected mistake lower bound when we instead require the adversary to choose a label \emph{before} a (potentially randomized) prediction is made. 

In order to prove Theorem~\ref{thm:lit-queries}, we will use Yao's minimax principle, which will allow us to give lower bounds for randomized algorithms while only considering deterministic algorithms in our analysis. 

We will start with some notation.
Let $\text{cost}(A,z)$ represent the real-valued cost of an algorithm $A$ on an input $z$ for problem $P$ (e.g., running time, number of queries, etc.). 
For a distribution $\D$ on the set $Z$ of potential inputs to $P$, the cost of $A$ on $\D$ is defined as $\text{cost}(A,\D):=\eval{z\sim \D}{\text{ cost}(A,z)}$. 
The \emph{distributional complexity} of $P$ is $\underset{\mathcal{D}}{\max}\;\underset{A\in\A}{\min}\text{ cost}(A,\D)$ (the cost of the worst distribution on inputs for the best deterministic algorithm).
Now, as we can see a randomized algorithm $R$ as a distribution $\mathcal{R}$ over all the possible deterministic algorithms, we can define the cost of a randomized algorithm as cost$(R,x)=\text{ cost}(\mathcal{R},x)=\eval{A\sim\mathcal{R}}{\text{ cost}(A,x)}$.
The \emph{randomized complexity} of $P$ is defined as $\underset{\mathcal{R}}{\min}\;\underset{x}{\max}\text{ cost}(R,x)$.

Yao's minimax principle states that the randomized complexity and distributional complexity of a problem $P$ are equal, i.e.,
$$\underset{\mathcal{R}}{\min}\;\underset{x}{\max}\text{ cost}(\mathcal{R},x)=\underset{\mathcal{D}}{\max}\;\underset{A\in\A}{\min}\text{ cost}(A,\D)\enspace.$$
For us, the input to the learning problem will be the target concept, and as such the distribution $\D$ will be over the concept class $\C$. 
The cost of an algorithm $A$ on $c$ is the number of queries to the $\rho$-$\LEQ$ oracle, and we are interested in the expected number of counterexamples returned.

\begin{proof}[Proof of Theorem~\ref{thm:lit-queries}]
The idea behind the proof is to choose a distribution in which all the elements of the Littlestone tree appear in the perturbation region of its root. 
We then derive query lower bounds for deterministic algorithms and a deterministic $\LEQ$.
We finally use Yao's minimax principle to lower bound the number of queries of \emph{any} robust learning algorithm  to the $\LEQ$ oracle.

\emph{Distribution on $\X$.} Let $T$ be a Littlestone tree of depth $d$ with root $x^*\in X$ such that all its nodes are contained in $B_\rho(x^*)$.
Let $D$ be a distribution on $\X$ be such that $D(x^*)=1$.
Hence, any query to $\EX(c,D)$ will return $(x^*,c(x^*))$.
Moreover, any learning algorithm must be exact on $B_\rho(x^*)$, as otherwise the existence of a point $z$ in $B_\rho(x^*)$ such that the target and hypothesis disagree results in the robust risk being 1.

Let $\widetilde{\C}=\set{c_1,\dots,c_{2^{d-1}}}$ be the set of concepts appearing as leaves of the subtree $T'$ of $T$ that label ${x^*}$ positively. 
We will pick the target $c$ at random from  $\widetilde{\C}$.

\emph{$\LEQ$ strategy.}
Let the $\rho$-$\LEQ$ oracle have access to $T$ as its internal ordering.
Upon being queried with $(x^*,h)$, $\LEQ$ returns a counterexample $x'$ of least depth.
Namely, a target $c\in\widetilde{\C}$ determines a path from $x^*$ to a leaf, and the $\LEQ$ oracle returns the highest node where $h$ and $c$ disagree.

\emph{Deterministic algorithms.} 
In order to use Yao's  minimax principle, we first consider a set of deterministic algorithms. 
For any fixed distribution $\D$ on the target concepts, a learning algorithm $A$ achieving $\underset{A\in\A}{\min}\text{ cost}(A,\D)$ must be consistent with the data seen so far. 
Otherwise, $\LEQ$ can simply return a counterexample that has already been returned, increasing the number of queries to $\LEQ$.

Without loss of generality, we consider the setting where the Littlestone tree $T$ the $\LEQ$ uses to return counterexamples is known to $A$. 
This implies that if $A$ receives $(x_i,c(x_i))$ as a counterexample, then there exists a unique path from $x^*$ to a node containing $x_i$ where parents nodes of $x_i$ must have been labelled correctly by the hypothesis. 
Any algorithm that does not know this information (or doesn't use it) is dominated by an algorithm knowing this information.
Then any $A$ achieving $\underset{A\in\A}{\min}\text{ cost}(A,\D)$ in this setting must be consistent on the counterexamples and (implicitly revealed) correctly labelled points.

\emph{Executions paths.}
Consider a given \emph{deterministic} algorithm $A$ that is consistent with the data seen so far.
After seeing $(x^*,+1)$, $A$ returns a hypothesis $h_1$, thus the label $h_1(x_1)$ is fixed (where $x_1$ is the child of $x^*$ with a positively labelled edge). 
Then, one of the edges coming out of the node $x_1$ will be correct, while the other will be incorrect. 
Since each concept in $\widetilde{\C}$ defines a path in $T'$ and $A$ is deterministic, for each node $x$ in $T'$, one of its edges is correct and the other, incorrect.
Then, for each leaf $c$ in $T'$, the edges from $x_1$ to $c$ that are marked as incorrect represent the counterexamples given to $A$ by $\LEQ$ if $c$ were the target. 
It is then easy to see that choosing the target u.a.r. from the leaves $\widetilde{\C}$ of $T’$, we have an expected number of counterexamples that is exactly $(d-1)/2$.
Thus 
$$\underset{\mathcal{D}}{\max}\;\underset{A\in\A}{\min}\text{ cost}(A,\D)
\geq \underset{A\in\A}{\min}\text{ cost}(A,U(\widetilde{\C}))\geq \frac{d-1}{2}\enspace,$$
where $U(\widetilde{\C})$ is the uniform distribution on $\widetilde{\C}$.
In fact, we can see that the distribution achieving the maximum on the LHS is the uniform distribution on $\widetilde{\C}$. 

\emph{Putting it all together.}
Now, any learning algorithm in this setting is either deterministic or randomized. 
If the algorithm is randomized, it can be expressed as a distribution on the set of deterministic algorithms. 
We can thus apply Yao's principle and get that there exists a distribution on $\X$ such that any robust learning algorithm for $\C$ will have an expected number of queries to $\LEQ$ that is linear in the restricted Littlestone dimension of~$\C$.
\end{proof}

\subsection{Bounds on the Restricted VC and Littlestone Dimensions}

In this section, we study bounds on the restricted VC and Littlestone dimensions of monotone conjunctions, decision lists and linear classifiers.
This enables us to use Theorem~\ref{thm:lit-queries} to get lower bounds on the expected number of queries to $\LEQ$ for the robust learning of these classes.
The VC dimension bounds are (asymptotically) the same as the standard VC dimension for these classes.
We finish by showing this is not always the case, and exhibiting an example where the VC dimension is \emph{not} a lower bound for the expected number of queries to $\LEQ$, hence justifying the use of alternative complexity measure in our setting.

We summarize our results on the restricted VC dimension in Table~\ref{tab:restricted-vc-bounds}. 
The proofs of the bounds appear in Appendix~\ref{app:restricted-vc-bounds}.
As corollaries of Theorem~\ref{thm:lit-queries}, we get that the restricted VC dimension lower bounds presented in Table~\ref{tab:restricted-vc-bounds} are lower bounds on the expected number of queries to the $\LEQ$ oracle. 
%\begin{center}
\begin{table}[]
\begin{tabular}{l|l|l}
\textbf{Concept Class}  & \textbf{$\VC$ dimension} & \multicolumn{1}{c}{\textbf{$\rho$-restricted $\VC$ dimension}} \\ \hline
Conjunctions   & $n$  & {\parbox[t]{5cm}{$2$  (if $\rho=1$)\\ $n$  (if $\rho\geq 2$)}}\\
Linear Threshold Functions & $n+1$ & $n+1$\\
$k$-Decision Lists  & $\widetilde{\Theta}(n^k)$  & $\widetilde{\Theta}(n^k)$ (given $\rho\geq k$) 
\end{tabular}
\caption{Comparing the VC dimension and the $\rho$-restricted VC dimension for given concept classes. The $\widetilde{\Theta}$ notation hides the logarithmic factors. Unless otherwise stated, we assume $\rho\geq 1$.}
\label{tab:restricted-vc-bounds}
\end{table}
%\end{center}

Now, while linear classifiers in $\R^n$ have a (restricted) VC dimension of $n+1$, their $\rho$-restricted Littlestone dimension is infinite. 
Indeed, it suffices to consider the subclass of thresholds (for which the proof of its Littlestone dimension being infinite can easily be adapted to the restricted setting) giving the lemma below. 

\begin{corollary}
\label{cor:ltf-real-leq-lb}
Given $\rho>0$, there exists a distribution on $\R^n$ such that any $\rho$-robust learning algorithm for linear classifiers has an infinite expected number of queries to $\rho$-$\LEQ$.
\end{corollary}

We now turn our attention to the relationship between VC dimension and its restricted counterpart, and their use for $\LEQ$ lower bounds.

\paragraph*{In general, is $\VCrho(\H)=\widetilde{\Theta}(\VC(\H))$?}
No: we exhibit a concept class $\H$ where $\VC(\H)=d$ but $\VCrho(\H)=1$. 
This shows that there can be an arbitrary gap between the restricted VC dimension and its standard counterpart.
Let $X=\set{x_1,\dots,x_d}$ be a set of $d$ points on $\boolhc$ whose balls of radius $\rho$ don't coincide (choose $\rho$ and $d$ as functions of $n$ such that this is possible). 
Define the following concept class $\C=\bigcup_{S\subseteq X}\set{c(x)=\mathbf{1}[x\in S]}$.
Clearly, $\VC(\C)=d$, but $\VCrho(\C)=1$.

\paragraph*{Are there better $\LEQ$ lower bounds than $\VCrho$?}
Yes: we can still show an expected query lower  bound of $\Omega(d)$ in the example above by constructing a uniform distribution on some set $X^*=\set{x_1^*,\dots,x_d^*}$ such that $X\cap X^*=\emptyset$ and $x_i^*\in B_\rho(x_i)$ for all $i$, implying that $c(x_i)=0$ for all $1\leq i\leq d$ and $c\in\C$. 
Thus, to get a hypothesis with robust risk strictly smaller than $1/d$, exact learning is required.
We can show that, by choosing the target at random from $\C$, the expected number of counterexamples for any algorithm is lower bounded by a function $\Omega(d)$ with the same reasoning as the proof of Theorem~\ref{thm:lit-queries}.

\paragraph*{Is the VC dimension a general lower bound for $\LEQ$?}
No: consider the problem defined above, but with the perturbation region being the identity function for each $x\in X$. Clearly, it is not possible to construct the distribution in the previous example.
In fact, a random sample of size $\Theta(d)$ is sufficient to guarantee generalization, without the use of queries. 

\section{Further Comparing the Local Query Models}
\label{sec:comparing-lq}

We finish this chapter by drawing a more nuanced picture of the local membership and equivalence query frameworks, in how they compare with each other (Section~\ref{sec:lmq-vs-leq}) and to other active learning set-ups (Section~\ref{sec:sep-eq-leq}).

\subsection{Local Membership and Equivalence Queries}
\label{sec:lmq-vs-leq}

We start by showing two results on the efficient robust learnability of singletons. 
The first result is a negative one: singletons are not efficiently robustly learnable in the distribution-free setting in the $\EX$+$\LMQ$ model.
However, the second result shows that it is possible to do so in the $\EX$+$\LEQ$ model when the perturbation budget $\rho$ and the locality radius $\lambda$ are equal.
While simple, together these results highlight the relevance of the $\LEQ$ oracle in robust learning.
Indeed, they show that, unlike in the standard PAC model, membership queries cannot, in general, simulate equivalence queries in the robust learning setting.\footnote{See Theorem~\ref{thm:exact-pac-mq} for details.} 
In robust learning, because of the existence of the existential quantifier in the robust loss, $\mathbf{1}[\exists z\in B_\rho(x) \st c(z) \neq h(z)]$, polynomially-many (local) membership queries cannot in general suffice to estimate the robust loss, as illustrated below.
We finish by looking at parities, a concept class for which local membership and equivalence queries are equally powerful.

We first start by showing that having access to local membership queries does not ensure the robust learnability of singletons in the distribution-free setting \emph{regardless of the query radius}.

\begin{proposition}
If $\rho$ is $\omega(1)$, the class of singletons is not efficiently $\rho$-robustly learnable in the distribution-free setting when the learner has access to a $\lambda$-$\LMQ$ oracle for any $\lambda$.
\end{proposition}
\begin{proof}
Fix $x\in\X$ and consider the distribution $D$ on $\X$ such that $D(x)=1$. 
We distinguish two cases. 
If $\lambda<\rho$, it suffices to choose two singletons in $B_\rho(x)\setminus B_\lambda(x)$ and draw the target concept uniformly at random between them (the learner cannot query a positive label, and cannot do better than choosing the right target at random).
%no WLOG
The second case is $\lambda \geq \rho$.
Note that $\abs{B_\rho(x)}\geq (n/\rho)^\rho$, which is superpolynomial in $n$ for any budget $\rho=\omega(1)$ (as $\rho\leq n$). 
Now, for any LMQ strategy with a polynomial query upper bound $r(n)$, there exists a sufficiently large input dimension $N$ such that, after $r(N)$ queries, at least half the points in  $ B_\rho(x)$  have yet to be queried.
Choosing a target uniformly at random in $ B_\rho(x)\setminus\set{x}$, using Lemma~\ref{lemma:robloss-triangle}, and noting that any two singletons in $ B_\rho(x)$ have robust risk of 1 against each other, suffices to lower bound the expected risk of any hypothesis over the choice of the target concept by a constant.
\end{proof}

We now show that, in contrast, having access to a local equivalence query oracle enables the robust learnability of singletons in the distribution-free setting.

\begin{proposition}
Singletons are efficiently distribution-free $\rho$-robustly learnable given a $\rho$-$\LEQ$ oracle.
\end{proposition}
\begin{proof}
Draw a sufficiently large sample $S\sim D^m$ to ensure robust generalization, as in Lemma~\ref{lemma:occam}
($\abs{S}$ is polynomial in the input dimension $n$ and learning parameters).
If there exists a positively labelled point $x\in S$, we have learned the target singleton.
Otherwise, query the $\LEQ$ oracle with the constant function 0 on the points in $S$ until we receive a counterexample (the target singleton) or until it is confirmed that all points have robust loss of 0.
In either case, we have queried at most $m$ points and the hypothesis is robustly consistent with the training sample, and we are done.
\end{proof}

Note that this is in contrast with the fact that, if a concept class is exactly learnable with access to the $\MQ$ and $\EQ$ oracles, then it is PAC learnable with random examples and access to $\MQ$ (see Section~\ref{sec:active-learning} for details).
Fundamentally, the existential quantifier in the robust risk definition renders simulating the $\LEQ$ oracle with the $\LMQ$ oracle impossible.

%Note that in standard learning, it still remains impossible to efficiently simulate the $\LEQ$ oracle with the $\LMQ$ oracle, because of the underlying distribution.

\paragraph*{Learning parities with $\LMQ$ and $\LEQ$.} There are cases where the $\EX$+$\LMQ$ and $\EX$+$\LEQ$ models are equally powerful.
Indeed, it is easy to see that access to the 1-$\LMQ$ or 1-$\LEQ$ oracle is sufficient to exactly learn parities with one query to $\EX$. 
For the former case, it suffices to flip each bit $i$ of an instance $x$ drawn from $\EX$ and give $x\oplus e_i$ to the $\LMQ$ oracle to observe whether $i$ is in the target parity.
For the latter, note that each counterexample $(x,y)$ is linearly independent from the set of data points already collected, so there must be at most $n$ counterexamples in $B_1(x)$, thus exactly identifying the target parity.

\subsection{A Two-Way Separation between $\LEQ$ and $\EQ$}
\label{sec:sep-eq-leq}

In this section, we compare local and ``global''  query oracles. 
We show that, when considering robust learning, the $\EX$+$\LEQ$ and $\EX$+$\EQ$ models are in general incomparable. 
This is in contrast with $\LMQ$ and $\MQ$, where a learning algorithm with access to $\LMQ$ can straightforwardly be simulated by an algorithm with access to $\MQ$.

We first show the existence of a robust learning problem on $\boolhc$ such that the $\EX$+$\LEQ$ model requires one sample point from $\EX$ and one query to $\LEQ$, while it requires $\log n$ calls to $\EQ$ in the $\EX$+$\EQ$ model.
We then show the existence of a robust learning problem on $\boolhc$ such that the $\EX$+$\LEQ$ model requires $O(1/\epsilon)$ sample points from $\EX$ and a total of $1/\epsilon$ queries to $\LEQ$  in order to have a robust error bounded by $\epsilon$, while it only requires a single call to $\EQ$ in the $\EX$+$\EQ$ model.\footnote{We note that the query complexity of learning with $\MQ$, $\EQ$ and partial queries has been vastly studied, notably by \cite{angluin1988queries} and \cite{maass1992lower}.
In these works however, the learning algorithm is required to be \emph{proper} and the learning \emph{exact}.
In contrast, we look at the robust learning framework and allow improper learning. }

We now formally state the result showing that (perhaps counter-intuitively) 1-$\LEQ$ can sometimes be more powerful than $\EQ$. 
The idea is to create a learning problem such that any counterexample in a ball of radius one around a point reveals full information about the target, but when the oracle is free to choose any point in the input space, it can (adversarially) reveal partial information.
To simplify our analysis, we will assume that the oracle does not have to commit to any target, as long as the target is defined on the support of the distribution (in order to have a well-defined example oracle). 
The target is not necessarily defined  on the rest of the input space, only restricting the oracle to output a sequence of counterexamples for which there always exists a consistent concept. 
As mentioned earlier in this chapter, choosing a target a priori (i.e., similarly to the online stochastic setting) simply results in expected bounds of the same order as if the oracle does not have to commit to a target (i.e., the mistake-bound online setting).

\begin{theorem}
\label{thm:leq-better-eq}
Let $\C=\set{x \mapsto x_i \given i\in[n]}$ be the class of \emph{monotone} dictators.
There exists a distribution $D$ on $\boolhc$  and target concept $c\in\C$  such that 1-robustly learning $(c,D)$ requires at most one query to 1-$\LEQ$, but, for any learning algorithm, at least $\log n$ queries to $\EQ$. 
\end{theorem}

\begin{proof}
Let $D$ be such that $D(\mathbf{0})=1$, and note that robustly learning $\C$ against an adversary with budget 1 requires exact learning. 
Moreover, the labelled instance $(\mathbf{0},0)$ gives no information about the target concept.

For the $\LEQ$ model, the learner samples the point $\mathbf{0}$ from $\EX$, and gives the constant hypothesis $0$ to $\LEQ$. Since the oracle must return $x\in B_1(\mathbf{0})$ such that $c(x)=1$, it must return $e_i$ such that $c(x)=x_i$.\footnote{Note that this is an improper learner, but we can simply consider the case $\C'=\C\cup\set{c(x)=0}$ to get an example with proper learning.}

For the $\EQ$ model, the idea is that, for any hypothesis $h$ the learner gives to $\EQ$, the oracle can always find a counterexample that removes at most half of the potential targets. 
Let $I_t=\set{i\in[n]\given x_i \text{ is consistent with the history}}$ be the set of indices (and thus concepts) that are consistent with the sequence of counterexamples given up until query $t$, and note that $I_1=[n]$.
Let $h\in 2^{\boolhc}$ be an arbitrary hypothesis and define the following function 
$$\#1:(x,I)\mapsto \sum_{i\in I} x_i$$ 
that returns the number of 1's at the indices of $I\subseteq[n]$ in an instance $x\in\boolhc$.
Define the following instances:
\begin{align*}
x_*:=\underset{x:h(x)=1}{\arg\min}\;\#1(x,I_t)\enspace,\\
x^*:=\underset{x:h(x)=0}{\arg\max}\;\#1(x,I_t)\enspace.
\end{align*}
We now argue that one of $x_*$ or $x^*$ will decrease $I_t$ by at most half.
Recall that the oracle's goal is to reveal as little information as possible to the learner at every query.
Note that, given $x\in\boolhc$, if $h(x)\neq c(x)$, then all the bits with value $c(x)$ are still viable target functions for that counterexample.

Now, if $x^*$ is a counterexample, then $c(x^*)=1$, and there are  $\#1(x,I_t)$ concepts that are still consistent with the counterexample history. 
Likewise, if if $x_*$ is a counterexample, then $c(x^*)=0$, and there are $\abs{I_t}-\#1(x,I_t)$ concepts that are still consistent with the counterexample history. 
Thus, if the oracle chooses the counterexample maximizing the number of consistent concepts with the history, we have that
$$\abs{I_{t+1}}=\max\set{\abs{I_t}-\#1(x_*,I_t),\#1(x^*,I_t)}\geq \lfloor\abs{I_t}/2\rfloor\enspace,$$
which concludes the proof.
\end{proof}

\begin{remark}
As a corollary of Theorem~\ref{thm:leq-better-eq}, we get that there exists a robust learning problem for which distribution-free efficient robust learning is still possible when $\lambda>\rho$, but where the query complexity is much larger than if $\lambda=\rho$ (set $\rho=1$ in the problem above).
\end{remark}

Now, we formally show that, for some learning problems, an $\EQ$ oracle is more powerful than an $\LEQ$ oracle.

\begin{theorem}
Let $\C=\set{x \mapsto \mathbf{1}[x=x'] \given x'\in\boolhc}$ be the class of singletons.
Then there exists a distribution $D$ on $\boolhc$ and target concept $c\in\C$ such that robustly learning $(c,D)$ requires at most one query to $\EQ$, but, for any learning algorithm, at least $1/\epsilon$ queries to $\lambda$-$\LEQ$ for robust accuracy $\epsilon$. 
\end{theorem}

\begin{proof}
Let $k=\lfloor 1/\epsilon\rfloor$.
Let $X=\set{x_1,\dots,x_k}$ be instances in $\boolhc$ whose $\lambda$-expansions don't intersect (let $n$ be sufficiently large and choose $\lambda$ as a function of $n$ and $\epsilon$ so that this is possible). 
Let $c(x_i)=0$ for all instances $x_i$, and let $D$ be the uniform distribution on $X$.
Note that, if the target singleton is in any of the perturbation regions $B_\rho(x_i)$, then a $\rho$-robust learning algorithm given robust accuracy parameter $\epsilon$ must identify the target exactly.
Without loss of generality, we let $\rho=\lambda$.

For the $\EQ$ bound, the learner can clearly query $\EQ$ with the constant function $0$, and get the singleton target as a counterexample, without any call to $\EX$.

For the $\lambda$-$\LEQ$ bound, the oracle's strategy is simply to (adaptively) return that $c=0$ on all queries $(x_{i_1},h_1),\dots,(x_{i_{k-1}},h_{k-1})$ until the last query $x_{i_k}$, which reveals the target singleton. This yields a lower bound of $k$ queries to $\lambda$-$\LEQ$ (the optimal strategy is to not repeat an instance in the queries and always choose $h_i=0$) 
\end{proof}

\begin{remark}
As the mistake-bound of singletons in online learning is 1, the theorem above also shows that $\rho$-robust learning with a $\rho$-$\LEQ$ can result in query complexity lower bounds that are strictly greater than mistake bounds in online learning.
Note though that the optimal algorithm still only makes one  mistake, it simply has to query the $\LEQ$ a certain number of times before making it.
\end{remark}

\section{Summary and Open Problems}
\label{sec:lq-summary}

In this chapter, we have thoroughly studied the powers and limitations of both local membership and equivalence queries in the context of robust learning.
In particular, we have outlined when access to either oracle is necessary to enable robustness guarantees, as well as obtained lower bounds on the local query complexity of various robust learning problems.

\subsection{Final Remarks on Local Query Oracles}

We discuss the implementation of local query oracles, as well as how the $\LMQ$ and $\LEQ$ oracles differ when considering the constant-in-the-ball notion of robustness.

\paragraph*{Implementing $\LEQ$ oracle.}
In practice, one always has to find a way to approximately implement oracles studied in theory.
A possible way to generate counterexamples with respect to the exact-in-the-ball notion of robustness is as follows.
Suppose that there is an adversary that can generate points $z\in B_\rho(x)$ such that $h(z)\neq c(z)$.
Provided such an adversary can be simulated, there is a way to (imperfectly) implement the $\LEQ$ oracle in practice. 
Thus, the use of these oracles can be viewed as a form of adversarial training.

%\begin{remark}
\paragraph*{Local query analogues for the constant-in-the-ball risk.}
Both the $\LMQ$ and $\LEQ$ models are particularly well-suited for the standard and exact-in-the-ball risks, as they address \emph{information-theoretic} limitations of learning with random examples only. 
 On the other hand, while information-theoretic limitations of robust learning with respect to the \emph{constant-in-the-ball} notion of robustness arise when the perturbation function $\U$ is unknown to the learner, \emph{computational} obstacles can also occur even when the definition of $\U$ is available. 
Indeed, determining whether the hypothesis changes label in the perturbation region could  be intractable.
In these cases, the Perfect Attack Oracle (PAO) of \cite{montasser2021adversarially} can be used to remedy these limitations for robust learning with respect to the constant-in-the-ball robust risk.
Crucially, in their setting, counterexamples could have a different label to the ground truth: a counterexample $z\in\U(x)$ for $x$ is such that $h(z)\neq c(x)$, not necessarily $h(z)\neq c(z)$. 
A striking example of this is when $\U(x)=\X$.
In this case, we only want to know if the hypothesis is constant on the whole input space.
This could compromise the standard accuracy of the hypothesis (see e.g., \cite{tsipras2019robustness} for a learning problem where robustness and accuracy are at odds).
Finally, an $\LMQ$ analogue for the constant-in-the-ball risk is not needed: the only information we need for a perturbed point $z\in B_\rho(x)$  is the label of $x$ (given by the example oracle) and $h(z)$. 
Given that one of the requirements of PAC learning is that the hypothesis is efficiently evaluatable, we can easily compute $h(z)$.

\paragraph*{Comparison with \citep{montasser2021adversarially}.} Closest to our work in this chapter is that of \cite{montasser2021adversarially}, who derive sample and PAO query bounds for the realizable constant-in-the-ball setting. 
They use the algorithm from \citep{montasser2019vc} to get a sample complexity of $\tilde{O}\left(\frac{\VC(\mathcal{H}){\VC^*}^2(\mathcal{H})+\log(1/\delta)}{\epsilon}\right)$ and derive a query complexity of $\tilde{O}(2^{\VCH^2 \DVCH^2\log^2(\DVCH)}\Lit(\mathcal{H}))$, where $\DVCH$ is the dual VC dimension of a hypothesis class.
They also derive query lower bounds: their general PAO query complexity lower bound is $\Omega(\log(\text{Tdim}(\H)))$, where $\text{Tdim}(\H)$ is the \emph{threshold dimension} of a hypothesis class.
The threshold dimension is bounded below by the logarithm of the Littlestone dimension, hence giving a general query lower bound of $\Omega(\log\log(\Lit(\H)))$.
Since threshold functions have a threshold dimension exponential in the Littlestone dimension, \cite{montasser2021adversarially} get a PAO query lower bound of $\Omega(\Lit(\H))$ in that special cases.
In contrast, we get an $\LEQ$ query lower bound linear in the restricted Littlestone dimension (which coincides with the Littlestone dimension for a wide variety of common concept classes) for any concept class.

\subsection{Future Work}

 We finally outline various avenues for future research.
 
\paragraph*{Local membership query lower bounds.} The $\LMQ$ lower bound from Section~\ref{sec:rob-learn-lmq} was derived for conjunctions. The technique does not work for monotone conjunctions, as, for a given set of indices $I$, there exists only one monotone conjunction using all indices in $I$. Can we get a similar $\LMQ$ lower bound where the dependence on $\rho$ is exponential for monotone conjunctions, or it is possible to robustly learn them with $o(2^\rho)$ local membership queries? 

\paragraph*{Limiting the power of the adversary.}
In Section~\ref{sec:adv-bounded-precision}, we studied robust learning against a bounded-precision adversary, requiring that it return a point around which the hypothesis and target disagree \emph{everywhere}.
We could relax this requirement and instead let the adversary choose a distribution $D_x$ on the perturbation region $\U(x)$, with constraints on $D_x$ that prevent a Dirac delta distribution on a single adversarial example $z\in \U(x)$. 
A promising avenue is to consider the smoothed adversaries of the work of \cite{haghtalab2022oracle,haghtalab2022smoothed} in online learning, which have density functions bounded by $1/\sigma$ that of the uniform density.
Note that a probabilistic approach of robustness has been considered in \citep{viallard2021pac,robey2022probabilistically} with respect to the constant-in-the-ball notion of robustness.

\paragraph*{Sample and query complexity bounds with $\LEQ$.} In Section~\ref{sec:sc-ub-leq}, we derived sample complexity upper bounds as a function of the VC dimension of the robust loss. As noted in Remark~\ref{rmk:rho-tradeoff}, this quantity is 1 when the adversarial budget $\rho=n$ (in $\boolhc$) due to the fact that we are essentially working in the online setting and the underlying distribution has become irrelevant. 
However, our upper bound for linear classifiers is $O(n^3)$, implying that it is quite loose, especially as $\rho$ increases.
Understanding the behaviour of the VC dimension of the robust loss as a function of $\rho$ to get concrete sample complexity bounds is a natural avenue for future work. 
In Appendix~\ref{app:rvc-closer-look}, we take a closer look at this question and show that, in the particular case of $\rho=n-1$, the VC dimension of the robust loss between linear classifiers is exactly 2.

Another natural direction for future work is to obtain sample complexity \emph{lower} bounds. 
We first note in Appendix~\ref{app:rvc-lb} that it is unlikely that the VC dimension of the robust loss is a good candidate for this complexity measure. 
Indeed, we explain why the proof that the VC dimension is a lower bound in the standard setting does not carry through when considering the robust loss.
We are currently investigating whether the complexity measure based on the one-inclusion graph developed by \cite{montasser2022adversarially} for the constant-in-the-ball notion of robustness can be adapted to the exact-in-the-ball setting and thus get a \emph{characterization} of robust learnability.

Finally, it would be interesting to give a more fine-grained picture of the sample and query complexity tradeoff outlined in Remark~\ref{rmk:rho-tradeoff}, perhaps through joint sample and query complexity lower bounds.

\chapter{Conclusion}
\label{chap:conclusion}
This thesis studied the robustness of learning algorithms to evasion attacks from a learning theory perspective.
Our focus was on the \emph{existence} of misclassified perturbed instances, with respect to the exact-in-the-ball notion of robust risk.
Our main consideration was the sample and query complexity of learning problems, with a particular focus on \emph{efficiency}, in an information-theoretic sense. 
We identified assumptions on learning problems that either enable or prevent robustness guarantees.
In particular, we looked at how the distribution that generates the data as well as the way in which the data is acquired influence the amount of data needed to ensure robustness to evasion attacks.

We started with a more passive setting in which the learner was restrained to a randomly drawn sample labelled according to the target concept, which required distributional assumptions to get reasonable sample complexity bounds. 
We outlined a series of combinatorial arguments to show that the $\log (n)$-expansion of error regions for certain concept classes on the boolean hypercube is not too large compared to the original set representing the error region.

In order to obtain distribution-free guarantees, we progressively considered more active and powerful learners which have access to \emph{local} queries -- showing in the process that local membership queries were, in general, not going to improve our previously obtained robustness thresholds. We have furthermore delimited the frontier of distribution-free robust learning for a wide variety of concept classes. 
This happens to be when the learner's query region and the adversary's perturbation region exactly coincide. 
We provided a nuanced discussion of these results and complemented them with lower bounds to the local equivalence query oracle.

To conclude, one of the overarching themes of this thesis is the identification of fundamental \emph{trade-offs} between the robustness of a learning algorithm and its training sample size.
As outlined below, the notion of tradeoff also informs future research directions and presents itself as a compelling framework to study guarantees or lack thereof in learning problems with non-standard objectives.

\section{Future Work}

As hinted throughout this thesis, we are far from having a full picture of robust learnability with respect to the exact-in-the-ball notion of robustness.
Indeed, concrete open problems abound, including the following questions posed in previous chapters.
What is the robustness threshold of linear classifiers (and, more generally, concept classes of polynomially-bounded VC dimension) under log-Lipschitz distributions? Can we derive tighter sample complexity bounds with access to random examples only? Is there a complexity measure characterizing the robust learnability of robust ERM algorithms under the exact-in-the-ball notion of robustness? 

Broader research questions have also arisen following the work presented in this thesis.
Below we outline more general and perhaps more speculative avenues for future work.

\paragraph*{Agnostic setting.}
In standard PAC learning, the agnostic setting allows for a joint distribution on the instance and label spaces.
The aim is to output a hypothesis whose error is as close as possible to the optimal hypothesis in the  class.
Observe that the constant-in-the-ball notion of robustness naturally extends to the agnostic setting: the label of a perturbed instance is compared to the label of its unperturbed counterpart. 
In fact, \cite{montasser2019vc} exhibit an elegant reduction from the agnostic to the realizable setting for the constant-in-the-ball notion of robustness. 
\cite{hopkins2022realizable} even show a quite general reduction for a family of general loss functions, which generalizes the one from \citep{montasser2019vc}, at the cost of a $1/\epsilon$ factor in the sample complexity.
However, given the presence of a target concept in the exact-in-the-ball case, it is not obvious how to extend this definition to the agnostic setting.
The robustness definition of \cite{pang2022robustness}, mentioned in the literature review, could be a candidate for this. In any case,  developing a theory of agnostic robust learnability in our setting, and determining whether the methods of \cite{hopkins2022realizable} apply, is an exciting future research direction. 

\paragraph*{Probabilistic Lipschitzness.}
In this thesis, when looking at robust learning with random examples only, we have considered learning problems as arbitrary concept and distribution pairs $(c,D)$ that come from a fixed concept class and distribution family. 
However, it would be natural to consider learning problems in which  there is a relationship between the target and the distribution on the data. 
The probabilistic Lipschitzness property, proposed by \cite{urner2013probabilistic}, offers an interesting possible research direction: while a Lipschitzness condition on a deterministic target function imposes a margin between classes, its probabilistic counterpart allows the margins to ``smoothen out'' near the boundary. Allowing for target functions that satisfy Probabilistic Lipschitz (perhaps in addition to log-Lipschitzness) has the potential to result in better sample complexity bounds while still ensuring sufficient probability mass near the boundary in order to justify the use of the exact-in-the-ball notion of robustness.

%\paragraph*{Computational complexity and the adversary's power.}
%In this thesis, we focused on the sample complexity of robust learning. 
%A natural research avenue is to consider the computational complexity of robust learning algorithms, as well as that of adversaries. Indeed, in Section~\ref{sec:adv-bounded-precision}, we bounded an adversary's precision. 
%But, as pointed out in the literature review, it is also possible to bound the adversary's computational power, which could give rise to better learning guarantees. 
%
%\question{This has been studied quite extensively in the literature and is not a very novel research avenue (maybe a little bit w.r.t. the exact-in-the-ball notion of robustness) and we also have a computational complexity hardness in NeurIPS 2019. I am thinking of deleting this but wanted to check with you first. }
%
\paragraph*{Poisoning and evasion attacks.} We have so far focused on the study of evasion attacks. 
As pointed out in the literature review, there has also been a considerable body of work focusing on various poisoning attack models. 
Whether it is possible to draw connexions between the two settings (e.g., is a learning algorithm that is robust to evasion attacks also robust to poisoning attacks, and vice-versa, and, if so, under which conditions?) is an interesting research direction that could bridge different views of robustness, especially considering the \emph{clean-label attack model}, where new training data modified by the adversary must still be consistent with the target concept.

\paragraph*{Multi-objective trustworthy machine learning.}
One can expand the requirements of a learning algorithm for classification beyond its predictive accuracy, and in ways other than robustness, in the general goal of \emph{trustworthiness}.
For example, in interpretability and explainable machine learning, we have an additional need for a model to be able to explain \emph{why} a certain label has been chosen for a new unseen example, or more generally how a model uses a specific subset of features in its predictions, usually by attributing importance to certain features of the data. 
Another important consideration is the fairness of learning algorithms.
While there exist many different notions of fairness \citep{kleinberg2017inherent}, the overarching goal is usually to avoid discrimination against a particular subgroup of the data.
Finally, there are a variety of ways in which privacy can be specified. 
For example, one may wish to be resilient against membership inference attacks, where the aim is to infer whether %a certain point (e.g., an individual) 
an individual was part of the training set.
It is apparent that such formal guarantees are warranted for any safe learning algorithm that is deployed in practice.
Drawing connections between how these requirements relate to robustness is one of many possible research avenues in trustworthy machine learning. 
Indeed, it is possible that these requirements be at odds with each other, naturally resulting in multi-objective formulations, or, conversely, that they can in fact align with each other.
While there exists work on this topic in the literature, see, e.g., \citep{lecuyer2019certified,pawelczyk2022exploring,konstantinov2022robustness},  knowledge gaps remain, especially considering the myriad of ways in which robustness, fairness, interpretability and privacy have been defined.

To conclude, while we have focused on the trade-off between robustness and sample complexity in this work, the nature of trade-offs in learning problems can vary: between sample complexity and other learning objectives, between a learning objective and computational complexity, between learning objectives themselves, etc.
Exploring trade-offs through the lens of learning theory could refine our understanding of fundamental limitations as well as possibilities of learning with safer and more realistic objectives.

\bibliographystyle{apalike}
\bibliography{references}

\begin{thebibliography}{}

\bibitem[Aden-Ali et~al., 2023]{aden2022one}
Aden-Ali, I., Cherapanamjeri, Y., Shetty, A., and Zhivotovskiy, N. (2023).
\newblock The one-inclusion graph algorithm is not always optimal.

\bibitem[Angluin, 1987]{angluin1987learning}
Angluin, D. (1987).
\newblock Learning regular sets from queries and counterexamples.
\newblock {\em Information and computation}, 75(2):87--106.

\bibitem[Angluin, 1988]{angluin1988queries}
Angluin, D. (1988).
\newblock Queries and concept learning.
\newblock {\em Machine learning}, 2(4):319--342.

\bibitem[Angluin, 1990]{angluin1990negative}
Angluin, D. (1990).
\newblock Negative results for equivalence queries.
\newblock {\em Machine Learning}, 5(2):121--150.

\bibitem[Angluin and Kharitonov, 1995]{angluin1995when}
Angluin, D. and Kharitonov, M. (1995).
\newblock When won't membership queries help?
\newblock {\em Journal of Computer and System Sciences}, 50(2):336--355.

\bibitem[Ashtiani et~al., 2020]{ashtiani2020black}
Ashtiani, H., Pathak, V., and Urner, R. (2020).
\newblock Black-box certification and learning under adversarial perturbations.
\newblock In {\em International Conference on Machine Learning}, pages
  388--398. PMLR.

\bibitem[Ashtiani et~al., 2023]{ashtiani2023adversarially}
Ashtiani, H., Pathak, V., and Urner, R. (2023).
\newblock Adversarially robust learning with tolerance.
\newblock In {\em International Conference on Algorithmic Learning Theory},
  pages 115--135. PMLR.

\bibitem[Attias et~al., 2022]{attias2022characterization}
Attias, I., Hanneke, S., and Mansour, Y. (2022).
\newblock A characterization of semi-supervised adversarially-robust pac
  learnability.
\newblock {\em arXiv preprint arXiv:2202.05420}.

\bibitem[Attias et~al., 2019]{attias2019improved}
Attias, I., Kontorovich, A., and Mansour, Y. (2019).
\newblock Improved generalization bounds for robust learning.
\newblock In {\em Algorithmic Learning Theory}, pages 162--183. PMLR.

\bibitem[Awasthi et~al., 2019]{awasthi2019robustness}
Awasthi, P., Dutta, A., and Vijayaraghavan, A. (2019).
\newblock On robustness to adversarial examples and polynomial optimization.
\newblock {\em Advances in Neural Information Processing Systems}, 32.

\bibitem[Awasthi et~al., 2013]{awasthi2013learning}
Awasthi, P., Feldman, V., and Kanade, V. (2013).
\newblock Learning using local membership queries.
\newblock In {\em Conference on Learning Theory}, pages 398--431. PMLR.

\bibitem[Awasthi et~al., 2020]{awasthi2020adversarial}
Awasthi, P., Frank, N., and Mohri, M. (2020).
\newblock Adversarial learning guarantees for linear hypotheses and neural
  networks.
\newblock In {\em International Conference on Machine Learning}, pages
  431--441. PMLR.

\bibitem[Barreno et~al., 2006]{barreno2006can}
Barreno, M., Nelson, B., Sears, R., Joseph, A.~D., and Tygar, J.~D. (2006).
\newblock Can machine learning be secure?
\newblock In {\em Proceedings of the 2006 ACM Symposium on Information,
  computer and communications security}, pages 16--25.

\bibitem[Bary-Weisberg et~al., 2020]{bary2020distribution}
Bary-Weisberg, G., Daniely, A., and Shalev-Shwartz, S. (2020).
\newblock Distribution free learning with local queries.
\newblock In {\em Algorithmic Learning Theory}, pages 133--147. PMLR.

\bibitem[Baum and Lang, 1992]{baum1992query}
Baum, E.~B. and Lang, K. (1992).
\newblock Query learning can work poorly when a human oracle is used.
\newblock In {\em International joint conference on neural networks}, volume~8,
  page~8. Beijing China.

\bibitem[Ben-David et~al., 2009]{ben2009agnostic}
Ben-David, S., P{\'a}l, D., and Shalev-Shwartz, S. (2009).
\newblock Agnostic online learning.
\newblock In {\em Conference on Learning Theory}, volume~3, page~1.

\bibitem[Bhagoji et~al., 2019]{bhagoji2019lower}
Bhagoji, A.~N., Cullina, D., and Mittal, P. (2019).
\newblock Lower bounds on adversarial robustness from optimal transport.
\newblock {\em Advances in Neural Information Processing Systems}, 32.

\bibitem[Bhattacharjee et~al., 2023]{bhattacharjee2023robust}
Bhattacharjee, R., Hopkins, M., Kumar, A., Yu, H., and Chaudhuri, K. (2023).
\newblock Robust empirical risk minimization with tolerance.
\newblock In {\em International Conference on Algorithmic Learning Theory},
  pages 182--203. PMLR.

\bibitem[Biggio et~al., 2013]{biggio2013evasion}
Biggio, B., Corona, I., Maiorca, D., Nelson, B., {\v{S}}rndi{\'c}, N., Laskov,
  P., Giacinto, G., and Roli, F. (2013).
\newblock Evasion attacks against machine learning at test time.
\newblock In {\em Joint European conference on machine learning and knowledge
  discovery in databases}, pages 387--402. Springer.

\bibitem[Biggio et~al., 2012]{biggio2012poisoning}
Biggio, B., Nelson, B., and Laskov, P. (2012).
\newblock Poisoning attacks against support vector machines.
\newblock In {\em Proceedings of the 29th International Coference on
  International Conference on Machine Learning}, pages 1467--1474.

\bibitem[Biggio and Roli, 2018]{biggio2017wild}
Biggio, B. and Roli, F. (2018).
\newblock Wild patterns: Ten years after the rise of adversarial machine
  learning.
\newblock In {\em Proceedings of the 2018 ACM SIGSAC Conference on Computer and
  Communications Security}, pages 2154--2156.

\bibitem[Block, 1962]{block1962perceptron}
Block, H.-D. (1962).
\newblock The perceptron: A model for brain functioning. i.
\newblock {\em Reviews of Modern Physics}, 34(1):123.

\bibitem[Blum et~al., 2021]{blum2021robust}
Blum, A., Hanneke, S., Qian, J., and Shao, H. (2021).
\newblock Robust learning under clean-label attack.
\newblock In {\em Conference on Learning Theory}, pages 591--634. PMLR.

\bibitem[Blumer et~al., 1987]{blumer1987occam}
Blumer, A., Ehrenfeucht, A., Haussler, D., and Warmuth, M.~K. (1987).
\newblock Occam's razor.
\newblock {\em Information processing letters}, 24(6):377--380.

\bibitem[Blumer et~al., 1989]{blumer1989learnability}
Blumer, A., Ehrenfeucht, A., Haussler, D., and Warmuth, M.~K. (1989).
\newblock Learnability and the vapnik-chervonenkis dimension.
\newblock {\em Journal of the ACM (JACM)}, 36(4):929--965.

\bibitem[Bshouty, 1993]{bshouty1993exact}
Bshouty, N.~H. (1993).
\newblock Exact learning via the monotone theory.
\newblock In {\em Proceedings of 1993 IEEE 34th Annual Foundations of Computer
  Science}, pages 302--311. IEEE.

\bibitem[Bubeck et~al., 2018]{bubeck2018cryptographic}
Bubeck, S., Lee, Y.~T., Price, E., and Razenshteyn, I. (2018).
\newblock Adversarial examples from cryptographic pseudo-random generators.
\newblock {\em arXiv preprint arXiv:1811.06418}.

\bibitem[Bubeck et~al., 2019]{bubeck2019adversarial}
Bubeck, S., Lee, Y.~T., Price, E., and Razenshteyn, I. (2019).
\newblock Adversarial examples from computational constraints.
\newblock In {\em Proceedings of the 36th International Conference on Machine
  Learning}, volume~97 of {\em Proceedings of Machine Learning Research}, pages
  831--840, Long Beach, California, USA. PMLR.

\bibitem[Camacho and McIlraith, 2019]{camacho2019learning}
Camacho, A. and McIlraith, S.~A. (2019).
\newblock Learning interpretable models expressed in linear temporal logic.
\newblock In {\em Proceedings of the International Conference on Automated
  Planning and Scheduling}, volume~29, pages 621--630.

\bibitem[Cesa-Bianchi and Lugosi, 2006]{cesa2006prediction}
Cesa-Bianchi, N. and Lugosi, G. (2006).
\newblock {\em Prediction, learning, and games}.
\newblock Cambridge university press.

\bibitem[Chernoff, 1952]{chernoff1952measure}
Chernoff, H. (1952).
\newblock A measure of asymptotic efficiency for tests of a hypothesis based on
  the sum of observations.
\newblock {\em The Annals of Mathematical Statistics}, pages 493--507.

\bibitem[Chowdhury and Urner, 2022]{chowdhury2022robustness}
Chowdhury, S. and Urner, R. (2022).
\newblock Robustness should not be at odds with accuracy.
\newblock In {\em 3rd Symposium on Foundations of Responsible Computing (FORC
  2022)}. Schloss Dagstuhl-Leibniz-Zentrum f{\"u}r Informatik.

\bibitem[Cullina et~al., 2018]{cullina2018pac}
Cullina, D., Bhagoji, A.~N., and Mittal, P. (2018).
\newblock {PAC}-learning in the presence of evasion adversaries.
\newblock {\em Advances in Neural Information Processing Systems}.

\bibitem[Dalvi et~al., 2004]{dalvi2004adversarial}
Dalvi, N., Domingos, P., Sanghai, S., and Verma, D. (2004).
\newblock Adversarial classification.
\newblock In {\em Proceedings of the tenth ACM SIGKDD international conference
  on Knowledge discovery and data mining}, pages 99--108. ACM.

\bibitem[Degwekar et~al., 2019]{degwekar2019computational}
Degwekar, A., Nakkiran, P., and Vaikuntanathan, V. (2019).
\newblock Computational limitations in robust classification and win-win
  results.
\newblock In {\em Conference on Learning Theory}, pages 994--1028. PMLR.

\bibitem[Diakonikolas et~al., 2019]{diakonikolas2019nearly}
Diakonikolas, I., Kane, D., and Manurangsi, P. (2019).
\newblock Nearly tight bounds for robust proper learning of halfspaces with a
  margin.
\newblock {\em Advances in Neural Information Processing Systems}, 32.

\bibitem[Diakonikolas et~al., 2020]{diakonikolas2020complexity}
Diakonikolas, I., Kane, D.~M., and Manurangsi, P. (2020).
\newblock The complexity of adversarially robust proper learning of halfspaces
  with agnostic noise.
\newblock {\em Advances in Neural Information Processing Systems},
  33:20449--20461.

\bibitem[Diochnos et~al., 2018]{diochnos2018adversarial}
Diochnos, D., Mahloujifar, S., and Mahmoody, M. (2018).
\newblock Adversarial risk and robustness: General definitions and implications
  for the uniform distribution.
\newblock In {\em Advances in Neural Information Processing Systems}.

\bibitem[Diochnos et~al., 2020]{diochnos2020lower}
Diochnos, D.~I., Mahloujifar, S., and Mahmoody, M. (2020).
\newblock Lower bounds for adversarially robust {PAC} learning under evasion
  and hybrid attacks.
\newblock In {\em 2020 19th IEEE International Conference on Machine Learning
  and Applications (ICMLA)}, pages 717--722.

\bibitem[Dobriban et~al., 2020]{dobriban2020provable}
Dobriban, E., Hassani, H., Hong, D., and Robey, A. (2020).
\newblock Provable tradeoffs in adversarially robust classification.
\newblock {\em arXiv preprint arXiv:2006.05161}.

\bibitem[Doshi-Velez and Kim, 2017]{doshi2017towards}
Doshi-Velez, F. and Kim, B. (2017).
\newblock Towards a rigorous science of interpretable machine learning.
\newblock {\em arXiv preprint arXiv:1702.08608}.

\bibitem[Dreossi et~al., 2019]{dreossi2019formalization}
Dreossi, T., Ghosh, S., Sangiovanni-Vincentelli, A., and Seshia, S.~A. (2019).
\newblock A formalization of robustness for deep neural networks.
\newblock {\em arXiv preprint arXiv:1903.10033}.

\bibitem[Dwork, 2008]{dwork2008differential}
Dwork, C. (2008).
\newblock Differential privacy: A survey of results.
\newblock In {\em International conference on theory and applications of models
  of computation}, pages 1--19. Springer.

\bibitem[Ehrenfeucht et~al., 1989]{ehrenfeucht1989general}
Ehrenfeucht, A., Haussler, D., Kearns, M., and Valiant, L. (1989).
\newblock A general lower bound on the number of examples needed for learning.
\newblock {\em Information and Computation}, 82(3):247--261.

\bibitem[Etesami et~al., 2020]{etesami2020computational}
Etesami, O., Mahloujifar, S., and Mahmoody, M. (2020).
\newblock Computational concentration of measure: Optimal bounds, reductions,
  and more.
\newblock In {\em Proceedings of the Fourteenth Annual ACM-SIAM Symposium on
  Discrete Algorithms}, pages 345--363. SIAM.

\bibitem[Fang et~al., 2022]{fang2022out}
Fang, Z., Li, Y., Lu, J., Dong, J., Han, B., and Liu, F. (2022).
\newblock Is out-of-distribution detection learnable?
\newblock In {\em Advances in Neural Information Processing Systems}.

\bibitem[Fawzi et~al., 2018a]{fawzi2018adversarial}
Fawzi, A., Fawzi, H., and Fawzi, O. (2018a).
\newblock Adversarial vulnerability for any classifier.
\newblock {\em Advances in neural information processing systems}, 31.

\bibitem[Fawzi et~al., 2018b]{fawzi2018analysis}
Fawzi, A., Fawzi, O., and Frossard, P. (2018b).
\newblock Analysis of classifiers? robustness to adversarial perturbations.
\newblock {\em Machine Learning}, 107(3):481--508.

\bibitem[Fawzi et~al., 2016]{fawzi2016robustness}
Fawzi, A., Moosavi-Dezfooli, S.-M., and Frossard, P. (2016).
\newblock Robustness of classifiers: from adversarial to random noise.
\newblock In {\em Advances in Neural Information Processing Systems}, pages
  1632--1640.

\bibitem[Feige et~al., 2015]{feige2015learning}
Feige, U., Mansour, Y., and Schapire, R. (2015).
\newblock Learning and inference in the presence of corrupted inputs.
\newblock In {\em Conference on Learning Theory}, pages 637--657.

\bibitem[Feldman and Schulman, 2012]{feldman2012data}
Feldman, D. and Schulman, L.~J. (2012).
\newblock Data reduction for weighted and outlier-resistant clustering.
\newblock In {\em Proceedings of the twenty-third annual ACM-SIAM symposium on
  Discrete Algorithms}, pages 1343--1354. Society for Industrial and Applied
  Mathematics.

\bibitem[F{\"u}redi, 1988]{furedi1988matchings}
F{\"u}redi, Z. (1988).
\newblock Matchings and covers in hypergraphs.
\newblock {\em Graphs and Combinatorics}, 4(1):115--206.

\bibitem[Garg et~al., 2020]{garg2020adversarially}
Garg, S., Jha, S., Mahloujifar, S., and Mohammad, M. (2020).
\newblock Adversarially robust learning could leverage computational hardness.
\newblock In {\em Algorithmic Learning Theory}, pages 364--385. PMLR.

\bibitem[Gilmer et~al., 2018]{gilmer2018adversarial}
Gilmer, J., Metz, L., Faghri, F., Schoenholz, S.~S., Raghu, M., Wattenberg, M.,
  and Goodfellow, I. (2018).
\newblock Adversarial spheres.
\newblock {\em arXiv preprint arXiv:1801.02774}.

\bibitem[Goldberg, 2006]{goldberg2006some}
Goldberg, P.~W. (2006).
\newblock Some discriminant-based pac algorithms.
\newblock {\em Journal of Machine Learning Research}, 7(Feb):283--306.

\bibitem[Goldberg and Jerrum, 1995]{goldberg1995bounding}
Goldberg, P.~W. and Jerrum, M.~R. (1995).
\newblock Bounding the vapnik-chervonenkis dimension of concept classes
  parameterized by real numbers.
\newblock {\em Machine Learning}, 18(2-3):131--148.

\bibitem[Goldblum et~al., 2022]{goldblum2022dataset}
Goldblum, M., Tsipras, D., Xie, C., Chen, X., Schwarzschild, A., Song, D.,
  Madry, A., Li, B., and Goldstein, T. (2022).
\newblock Dataset security for machine learning: Data poisoning, backdoor
  attacks, and defenses.
\newblock {\em IEEE Transactions on Pattern Analysis and Machine Intelligence},
  45(2):1563--1580.

\bibitem[Goodfellow et~al., 2015]{goodfellow2015explaining}
Goodfellow, I.~J., Shlens, J., and Szegedy, C. (2015).
\newblock Explaining and harnessing adversarial examples.
\newblock In Bengio, Y. and LeCun, Y., editors, {\em 3rd International
  Conference on Learning Representations, {ICLR} 2015, San Diego, CA, USA, May
  7-9, 2015, Conference Track Proceedings}.

\bibitem[Gourdeau et~al., 2019]{gourdeau2019hardness}
Gourdeau, P., Kanade, V., Kwiatkowska, M., and Worrell, J. (2019).
\newblock On the hardness of robust classification.
\newblock In {\em Advances in Neural Information Processing Systems}, pages
  7444--7453.

\bibitem[Gourdeau et~al., 2021]{gourdeau2021hardness}
Gourdeau, P., Kanade, V., Kwiatkowska, M., and Worrell, J. (2021).
\newblock On the hardness of robust classification.
\newblock {\em Journal of Machine Learning Research}, 22.

\bibitem[Gourdeau et~al., 2022a]{gourdeau2022sample}
Gourdeau, P., Kanade, V., Kwiatkowska, M., and Worrell, J. (2022a).
\newblock Sample complexity bounds for robustly learning decision lists against
  evasion attacks.
\newblock In {\em International Joint Conference in Artificial Intelligence}.

\bibitem[Gourdeau et~al., 2022b]{gourdeau2022when}
Gourdeau, P., Kanade, V., Kwiatkowska, M., and Worrell, J. (2022b).
\newblock When are local queries useful?
\newblock In {\em Advances in Neural Information Processing Systems}.

\bibitem[Haghtalab et~al., 2022a]{haghtalab2022oracle}
Haghtalab, N., Han, Y., Shetty, A., and Yang, K. (2022a).
\newblock Oracle-efficient online learning for beyond worst-case adversaries.
\newblock {\em arXiv preprint arXiv:2202.08549}.

\bibitem[Haghtalab et~al., 2022b]{haghtalab2022smoothed}
Haghtalab, N., Roughgarden, T., and Shetty, A. (2022b).
\newblock Smoothed analysis with adaptive adversaries.
\newblock In {\em 2021 IEEE 62nd Annual Symposium on Foundations of Computer
  Science (FOCS)}, pages 942--953. IEEE.

\bibitem[Hanneke, 2016]{hanneke2016optimal}
Hanneke, S. (2016).
\newblock The optimal sample complexity of pac learning.
\newblock {\em The Journal of Machine Learning Research}, 17(1):1319--1333.

\bibitem[Haussler, 1992]{haussler1992decision}
Haussler, D. (1992).
\newblock Decision theoretic generalizations of the pac model for neural net
  and other learning applications.
\newblock {\em Information and computation}, 100(1):78--150.

\bibitem[Haussler et~al., 1994]{haussler1994predicting}
Haussler, D., Littlestone, N., and Warmuth, M.~K. (1994).
\newblock Predicting $\{$0, 1$\}$-functions on randomly drawn points.
\newblock {\em Information and Computation}, 115(2):248--292.

\bibitem[Helmbold et~al., 1992]{helmbold1992learning}
Helmbold, D., Sloan, R., and Warmuth, M.~K. (1992).
\newblock Learning integer lattices.
\newblock {\em SIAM Journal on Computing}, 21(2):240--266.

\bibitem[Hoeffding, 1963]{hoeffding1963probability}
Hoeffding, W. (1963).
\newblock Probability inequalities for sums of bounded random variables.
\newblock {\em Journal of the American statistical association},
  58(301):13--30.

\bibitem[Hopkins et~al., 2022]{hopkins2022realizable}
Hopkins, M., Kane, D.~M., Lovett, S., and Mahajan, G. (2022).
\newblock Realizable learning is all you need.
\newblock In {\em Conference on Learning Theory}, pages 3015--3069. PMLR.

\bibitem[Jackson, 1997]{jackson1997efficient}
Jackson, J.~C. (1997).
\newblock An efficient membership-query algorithm for learning {DNF} with
  respect to the uniform distribution.
\newblock {\em Journal of Computer and System Sciences}, 55(3):414--440.

\bibitem[Kearns and Li, 1988]{kearns1988learning}
Kearns, M. and Li, M. (1988).
\newblock Learning in the presence of malicious errors.
\newblock In {\em Proceedings of the twentieth annual ACM symposium on Theory
  of computing}, pages 267--280.

\bibitem[Kearns et~al., 1994]{kearns1994toward}
Kearns, M.~J., Schapire, R.~E., and Sellie, L.~M. (1994).
\newblock Toward efficient agnostic learning.
\newblock {\em Machine Learning}, 17:115--141.

\bibitem[Khim et~al., 2019]{khim2019adversarial}
Khim, J., Jog, V., and Loh, P.-L. (2019).
\newblock Adversarial influence maximization.
\newblock In {\em 2019 IEEE International Symposium on Information Theory
  (ISIT)}, pages 1--5. IEEE.

\bibitem[Kleinberg et~al., 2017]{kleinberg2017inherent}
Kleinberg, J., Mullainathan, S., and Raghavan, M. (2017).
\newblock Inherent trade-offs in the fair determination of risk scores.
\newblock In {\em 8th Innovations in Theoretical Computer Science Conference
  (ITCS 2017)}. Schloss Dagstuhl-Leibniz-Zentrum fuer Informatik.

\bibitem[Koltun and Papadimitriou, 2007]{koltun2007approximately}
Koltun, V. and Papadimitriou, C.~H. (2007).
\newblock Approximately dominating representatives.
\newblock {\em Theoretical Computer Science}, 371(3):148--154.

\bibitem[Konstantinov, 2022]{konstantinov2022robustness}
Konstantinov, N.~H. (2022).
\newblock {\em Robustness and fairness in machine learning}.
\newblock PhD thesis.

\bibitem[Krizhevsky and Hinton, 2009]{krizhevsky2009learning}
Krizhevsky, A. and Hinton, G. (2009).
\newblock Learning multiple layers of features from tiny images.
\newblock Technical report, Citeseer.

\bibitem[LeCun, 1998]{lecun1998mnist}
LeCun, Y. (1998).
\newblock The mnist database of handwritten digits.
\newblock {\em http://yann. lecun. com/exdb/mnist/}.

\bibitem[Lecuyer et~al., 2019]{lecuyer2019certified}
Lecuyer, M., Atlidakis, V., Geambasu, R., Hsu, D., and Jana, S. (2019).
\newblock Certified robustness to adversarial examples with differential
  privacy.
\newblock In {\em 2019 IEEE Symposium on Security and Privacy (SP)}, pages
  656--672. IEEE.

\bibitem[Linial et~al., 1993]{linial1993constant}
Linial, N., Mansour, Y., and Nisan, N. (1993).
\newblock Constant depth circuits, fourier transform, and learnability.
\newblock {\em Journal of the ACM (JACM)}, 40(3):607--620.

\bibitem[Littlestone, 1988]{littlestone1988learning}
Littlestone, N. (1988).
\newblock Learning quickly when irrelevant attributes abound: A new
  linear-threshold algorithm.
\newblock {\em Machine learning}, 2(4):285--318.

\bibitem[Lowd and Meek, 2005a]{lowd2005adversarial}
Lowd, D. and Meek, C. (2005a).
\newblock Adversarial learning.
\newblock In {\em Proceedings of the eleventh ACM SIGKDD international
  conference on Knowledge discovery in data mining}, pages 641--647. ACM.

\bibitem[Lowd and Meek, 2005b]{lowd2005good}
Lowd, D. and Meek, C. (2005b).
\newblock Good word attacks on statistical spam filters.
\newblock In {\em Fifth Conference on Email and Anti-Spam (CEAS)}, volume 2005.

\bibitem[Maass, 1991]{maass1991line}
Maass, W. (1991).
\newblock {\em On-line learning with an oblivious environment and the power of
  randomization}.

\bibitem[Maass and Tur{\'a}n, 1992]{maass1992lower}
Maass, W. and Tur{\'a}n, G. (1992).
\newblock Lower bound methods and separation results for on-line learning
  models.
\newblock {\em Machine Learning}, 9:107--145.

\bibitem[Madry et~al., 2018]{madry2018towards}
Madry, A., Makelov, A., Schmidt, L., Tsipras, D., and Vladu, A. (2018).
\newblock Towards deep learning models resistant to adversarial attacks.
\newblock In {\em 6th International Conference on Learning Representations,
  {ICLR} 2018, Vancouver, BC, Canada, April 30 - May 3, 2018, Conference Track
  Proceedings}. OpenReview.net.

\bibitem[Mahloujifar et~al., 2018]{mahloujifar2018learning}
Mahloujifar, S., Diochnos, D.~I., and Mahmoody, M. (2018).
\newblock Learning under $ p $-tampering attacks.
\newblock In {\em Algorithmic Learning Theory}, pages 572--596. PMLR.

\bibitem[Mahloujifar et~al., 2019]{mahloujifar2019curse}
Mahloujifar, S., Diochnos, D.~I., and Mahmoody, M. (2019).
\newblock The curse of concentration in robust learning: Evasion and poisoning
  attacks from concentration of measure.
\newblock {\em AAAI Conference on Artificial Intelligence}.

\bibitem[Mahloujifar and Mahmoody, 2017]{mahloujifar2017blockwise}
Mahloujifar, S. and Mahmoody, M. (2017).
\newblock Blockwise p-tampering attacks on cryptographic primitives,
  extractors, and learners.
\newblock In {\em Theory of Cryptography: 15th International Conference, TCC
  2017, Baltimore, MD, USA, November 12-15, 2017, Proceedings, Part II 15},
  pages 245--279. Springer.

\bibitem[Mahloujifar and Mahmoody, 2019]{mahloujifar2019can}
Mahloujifar, S. and Mahmoody, M. (2019).
\newblock Can adversarially robust learning leveragecomputational hardness?
\newblock In {\em Algorithmic Learning Theory}, pages 581--609. PMLR.

\bibitem[Mohri et~al., 2012]{mohri2012foundations}
Mohri, M., Rostamizadeh, A., and Talwalkar, A. (2012).
\newblock {\em Foundations of machine learning}.
\newblock MIT press.

\bibitem[Montasser et~al., 2019]{montasser2019vc}
Montasser, O., Hanneke, S., and Srebro, N. (2019).
\newblock {VC} classes are adversarially robustly learnable, but only
  improperly.
\newblock In {\em Conference on Learning Theory}, pages 2512--2530. PMLR.

\bibitem[Montasser et~al., 2020]{montasser2020reducing}
Montasser, O., Hanneke, S., and Srebro, N. (2020).
\newblock Reducing adversarially robust learning to non-robust pac learning.
\newblock {\em Advances in Neural Information Processing Systems},
  33:14626--14637.

\bibitem[Montasser et~al., 2021]{montasser2021adversarially}
Montasser, O., Hanneke, S., and Srebro, N. (2021).
\newblock Adversarially robust learning with unknown perturbation sets.
\newblock In {\em Conference on Learning Theory}, pages 3452--3482. PMLR.

\bibitem[Montasser et~al., 2022]{montasser2022adversarially}
Montasser, O., Hanneke, S., and Srebro, N. (2022).
\newblock Adversarially robust learning: A generic minimax optimal learner and
  characterization.
\newblock {\em Neural Information Processing Systems}.

\bibitem[Natschl{\"a}ger and Schmitt, 1996]{natschlager1996exact}
Natschl{\"a}ger, T. and Schmitt, M. (1996).
\newblock Exact vc-dimension of boolean monomials.
\newblock {\em Information Processing Letters}, 59(1):19--20.

\bibitem[Novikoff, 1963]{novikoff1963convergence}
Novikoff, A.~B. (1963).
\newblock On convergence proofs for perceptrons.
\newblock Technical report, STANFORD RESEARCH INST MENLO PARK CA.

\bibitem[O'Donnell, 2014]{odonnell2014analysis}
O'Donnell, R. (2014).
\newblock {\em Analysis of boolean functions}.
\newblock Cambridge University Press.

\bibitem[O'Donnell and Servedio, 2007]{odonnell2007learning}
O'Donnell, R. and Servedio, R.~A. (2007).
\newblock Learning monotone decision trees in polynomial time.
\newblock {\em SIAM Journal on Computing}, 37(3):827--844.

\bibitem[Okudono et~al., 2020]{okudono2020weighted}
Okudono, T., Waga, M., Sekiyama, T., and Hasuo, I. (2020).
\newblock Weighted automata extraction from recurrent neural networks via
  regression on state spaces.
\newblock In {\em Proceedings of the AAAI Conference on Artificial
  Intelligence}, volume~34, pages 5306--5314.

\bibitem[Pang et~al., 2022]{pang2022robustness}
Pang, T., Lin, M., Yang, X., Zhu, J., and Yan, S. (2022).
\newblock Robustness and accuracy could be reconcilable by (proper) definition.
\newblock In {\em International Conference on Machine Learning}, pages
  17258--17277. PMLR.

\bibitem[Papernot et~al., 2016]{papernot2016towards}
Papernot, N., McDaniel, P., Sinha, A., and Wellman, M. (2016).
\newblock Towards the science of security and privacy in machine learning.
\newblock {\em arXiv preprint arXiv:1611.03814}.

\bibitem[Pawelczyk et~al., 2022]{pawelczyk2022exploring}
Pawelczyk, M., Agarwal, C., Joshi, S., Upadhyay, S., and Lakkaraju, H. (2022).
\newblock Exploring counterfactual explanations through the lens of adversarial
  examples: A theoretical and empirical analysis.
\newblock In {\em International Conference on Artificial Intelligence and
  Statistics}, pages 4574--4594. PMLR.

\bibitem[Pydi and Jog, 2021]{pydi2021many}
Pydi, M.~S. and Jog, V. (2021).
\newblock The many faces of adversarial risk.
\newblock {\em Advances in Neural Information Processing Systems}, 34.

\bibitem[Quinonero-Candela et~al., 2008]{quinonero2008dataset}
Quinonero-Candela, J., Sugiyama, M., Schwaighofer, A., and Lawrence, N.~D.
  (2008).
\newblock {\em Dataset shift in machine learning}.
\newblock Mit Press.

\bibitem[Renegar, 1992]{renegar1992computational}
Renegar, J. (1992).
\newblock On the computational complexity and geometry of the first-order
  theory of the reals. part i: Introduction. preliminaries. the geometry of
  semi-algebraic sets. the decision problem for the existential theory of the
  reals.
\newblock {\em Journal of symbolic computation}, 13(3):255--299.

\bibitem[Rivest, 1987]{rivest1987learning}
Rivest, R.~L. (1987).
\newblock Learning decision lists.
\newblock {\em Machine learning}, 2(3):229--246.

\bibitem[Robbins, 1955]{robbins1955remark}
Robbins, H. (1955).
\newblock A remark on stirling's formula.
\newblock {\em The American mathematical monthly}, 62(1):26--29.

\bibitem[Robey et~al., 2022]{robey2022probabilistically}
Robey, A., Chamon, L., Pappas, G.~J., and Hassani, H. (2022).
\newblock Probabilistically robust learning: Balancing average and worst-case
  performance.
\newblock In {\em International Conference on Machine Learning}, pages
  18667--18686. PMLR.

\bibitem[Rosenblatt, 1958]{rosenblatt1958perceptron}
Rosenblatt, F. (1958).
\newblock The perceptron: a probabilistic model for information storage and
  organization in the brain.
\newblock {\em Psychological review}, 65(6):386.

\bibitem[Shafahi et~al., 2018]{shafahi2018poison}
Shafahi, A., Huang, W.~R., Najibi, M., Suciu, O., Studer, C., Dumitras, T., and
  Goldstein, T. (2018).
\newblock Poison frogs! targeted clean-label poisoning attacks on neural
  networks.
\newblock {\em Advances in neural information processing systems}, 31.

\bibitem[Shafahi et~al., 2019]{shafahi2018adversarial}
Shafahi, A., Huang, W.~R., Studer, C., Feizi, S., and Goldstein, T. (2019).
\newblock Are adversarial examples inevitable?
\newblock In {\em 7th International Conference on Learning Representations
  (ICLR 2019)}.

\bibitem[Shalev-Shwartz and Ben-David, 2014]{shalev2014understanding}
Shalev-Shwartz, S. and Ben-David, S. (2014).
\newblock {\em Understanding machine learning: From theory to algorithms}.
\newblock Cambridge university press.

\bibitem[Shao et~al., 2022]{shao2022theory}
Shao, H., Montasser, O., and Blum, A. (2022).
\newblock A theory of pac learnability under transformation invariances.
\newblock {\em Advances in Neural Information Processing Systems},
  35:13989--14001.

\bibitem[Shih et~al., 2019]{shih2019verifying}
Shih, A., Darwiche, A., and Choi, A. (2019).
\newblock Verifying binarized neural networks by angluin-style learning.
\newblock In {\em International Conference on Theory and Applications of
  Satisfiability Testing}, pages 354--370. Springer.

\bibitem[Simon, 2015]{simon2015almost}
Simon, H.~U. (2015).
\newblock An almost optimal pac algorithm.
\newblock In {\em Conference on Learning Theory}, pages 1552--1563. PMLR.

\bibitem[Steinhardt et~al., 2017]{steinhardt2017certified}
Steinhardt, J., Koh, P. W.~W., and Liang, P.~S. (2017).
\newblock Certified defenses for data poisoning attacks.
\newblock {\em Advances in neural information processing systems}, 30.

\bibitem[Suggala et~al., 2019]{suggala2019revisiting}
Suggala, A.~S., Prasad, A., Nagarajan, V., and Ravikumar, P. (2019).
\newblock Revisiting adversarial risk.
\newblock In {\em The 22nd International Conference on Artificial Intelligence
  and Statistics}, pages 2331--2339. PMLR.

\bibitem[Szegedy et~al., 2013]{szegedy2013intriguing}
Szegedy, C., Zaremba, W., Sutskever, I., Bruna, J., Erhan, D., Goodfellow, I.,
  and Fergus, R. (2013).
\newblock Intriguing properties of neural networks.
\newblock In {\em International Conference on Learning Representations}.

\bibitem[Tsipras et~al., 2019]{tsipras2019robustness}
Tsipras, D., Santurkar, S., Engstrom, L., Turner, A., and Madry, A. (2019).
\newblock Robustness may be at odds with accuracy.
\newblock In {\em International Conference on Learning Representations}.

\bibitem[Urner and Ben-David, 2013]{urner2013probabilistic}
Urner, R. and Ben-David, S. (2013).
\newblock Probabilistic lipschitzness a niceness assumption for deterministic
  labels.
\newblock In {\em Learning Faster from Easy Data-Workshop@ NIPS}, volume~2,
  page~1.

\bibitem[Valiant, 1984]{valiant1984theory}
Valiant, L.~G. (1984).
\newblock A theory of the learnable.
\newblock In {\em Proceedings of the sixteenth annual ACM symposium on Theory
  of computing}, pages 436--445. ACM.

\bibitem[Vapnik, 1982]{vapnik1982estimation}
Vapnik, V. (1982).
\newblock Estimation of dependences based on empirical data: Springer series in
  statistics (springer series in statistics).

\bibitem[Vapnik and Chervonenkis, 1971]{vapnik1971uniform}
Vapnik, V. and Chervonenkis, A. (1971).
\newblock On the uniform convergence of relative frequencies of events to their
  probabilities.
\newblock In {\em Theory of Probability and Its Applications}.

\bibitem[Viallard et~al., 2021]{viallard2021pac}
Viallard, P., VIDOT, E.~G., Habrard, A., and Morvant, E. (2021).
\newblock A pac-bayes analysis of adversarial robustness.
\newblock {\em Advances in Neural Information Processing Systems}, 34.

\bibitem[Warmuth, 2004]{warmuth2004optimal}
Warmuth, M.~K. (2004).
\newblock The optimal pac algorithm.
\newblock In {\em Learning Theory: 17th Annual Conference on Learning Theory,
  Conference on Learning Theory 2004, Banff, Canada, July 1-4, 2004.
  Proceedings 17}, pages 641--642. Springer.

\bibitem[Weiss et~al., 2018]{weiss2018extracting}
Weiss, G., Goldberg, Y., and Yahav, E. (2018).
\newblock Extracting automata from recurrent neural networks using queries and
  counterexamples.
\newblock In {\em International Conference on Machine Learning}, pages
  5247--5256. PMLR.

\bibitem[Weiss et~al., 2019]{weiss2019learning}
Weiss, G., Goldberg, Y., and Yahav, E. (2019).
\newblock Learning deterministic weighted automata with queries and
  counterexamples.
\newblock {\em Advances in Neural Information Processing Systems}, 32.

\bibitem[Wiles et~al., 2022]{wiles2022fine}
Wiles, O., Gowal, S., Stimberg, F., Rebuffi, S.-A., Ktena, I., Dvijotham,
  K.~D., and Cemgil, A.~T. (2022).
\newblock A fine-grained analysis on distribution shift.
\newblock In {\em International Conference on Learning Representations}.

\bibitem[Yin et~al., 2019]{yin2019rademacher}
Yin, D., Kannan, R., and Bartlett, P. (2019).
\newblock Rademacher complexity for adversarially robust generalization.
\newblock In {\em International conference on machine learning}, pages
  7085--7094. PMLR.

\bibitem[Zhang et~al., 2019]{zhang2019theoretically}
Zhang, H., Yu, Y., Jiao, J., Xing, E., El~Ghaoui, L., and Jordan, M. (2019).
\newblock Theoretically principled trade-off between robustness and accuracy.
\newblock In {\em International conference on machine learning}, pages
  7472--7482. PMLR.

\end{thebibliography}

\appendix

\chapter{Proofs from Chapter~5}
\label{app:rob-thresholds}

\section{Proof of Lemma~\ref{lemma:rob-risk-dl}}
\label{app:rob-risk-dl}

\textbf{Lemma~\ref{lemma:rob-risk-dl}.}
\emph{Let $D$ be an $\alpha$-$\log$-Lipschitz distribution on the
$n$-dimensional boolean hypercube and let $\varphi$ be a 
conjunction of $d$ literals.
Set $\eta=\frac{1}{1+\alpha}$.
Then for all $0<\varepsilon<1/2$,
if $d\geq \max\left\{
  \frac{4}{\eta^2}\log\left(\frac{1}{\varepsilon}\right) ,
  \frac{2\rho}{\eta} \right\}$, then 
$\Prob{x\sim D}{\left(\exists y \in B_\rho(x) \cdot y \models
    \varphi\right)} \leq \varepsilon$.
}

\begin{proof}
  Write $\varphi = \ell_1 \wedge \cdots \wedge \ell_d$.  Draw a point
  $x\sim D$ from distribution $D$.  Let $X_1,\dots,X_d\in\{0,1\}$ be
  indicator random variables, respectively denoting whether $x$
  satisfies literals $\ell_1,\dots,\ell_d$.  Note that we do not assume the
  $X_i$'s to be independent from each other.  Writing
  $Y:=\sum_{i=1}^d X_i$, our goal is to show that
  $\Prob{x\sim D}{Y+\rho\geq d} \leq~\varepsilon$.

Let $D_i$ be the marginal distribution of $X_i$ conditioned on $X_1,\dots,X_{i-1}$. 
This distribution is also $\alpha$-$\log$-Lipschitz by Lemma~\ref{lemma:log-lips-facts}, and hence,
\begin{equation}
\label{eqn:marg-bound}
\Prob{X_i\sim D_i}{X_i=1}\leq 1-\eta
\enspace.
\end{equation}

Since we are interested in the random variable $Y$ representing the number of 1's in $X_1,\dots,X_d$, 
we define the random variables $Z_1,\dots,Z_d$ as follows:
\begin{equation*}
Z_k = \left(\sum_{i=1}^k X_i\right)-k(1-\eta)\enspace,
\end{equation*} 
with the convention that $Z_0=0$.
The sequence $Z_1, \dots, Z_d$ is a supermartingale with respect to $X_1,\dots,X_d$:
\begin{align*}
\eval{}{Z_{k+1}\given X_1,\dots,X_k}
&=\eval{}{Z_{k}+X_{k+1}-(1-\eta)\given X_1,\dots,X_k}\\
%&=Z_k+\eval{}{\mathbf{1}[X_{k+1}=1]\given X_1,\dots,X_k}-(1-\eta)\\
&=Z_k+\Prob{}{X_{k+1}'=1\given X_1,\dots,X_k}-(1-\eta)\\
&\leq Z_k
\enspace. \tag{by (\ref{eqn:marg-bound})}
\end{align*}
Now, note that all $Z_k$'s satisfy $|Z_{k+1}-Z_k|\leq 1$, and that $Z_d=Y-d(1-\eta)$. 
We can thus apply the Azuma-Hoeffding (A.H.) Inequality to get 
\begin{align*}
\Prob{}{Y\geq d-\rho}
&\leq \Prob{}{Y\geq d(1-\eta)+\sqrt{2\log(2/\varepsilon)d}}	\\
&=\Prob{}{Z_d-Z_0\geq \sqrt{2\log(2/\varepsilon)d}}	\\
&\leq \exp\left(-\frac{\sqrt{2\log(1/\varepsilon)d}^2}{2d}\right)				\tag{A.H.}\\
&=\varepsilon
\enspace,
\end{align*}
where the first inequality holds from the given bounds on $d$ and $\rho$:
\begin{align*}
d-\rho &=(1-\eta)d + \frac{\eta d}{2} 
               + \frac{\eta d}{2} - \rho \\
          & \geq (1-\eta) d + \frac{\eta d}{2} 
           \tag{since $\rho \leq \frac{\eta d}{2}$}\\
          & \geq  (1-\eta) d + \sqrt{2\log(1/\varepsilon) d} \enspace.
            \tag{since $d \geq \frac{8}{\eta^2}\log(\frac{1}{\varepsilon})$}
\end{align*}
\end{proof}

\section{Proof of Corollary~\ref{cor:k-dl}}
\label{app:cor-k-dl}
\textbf{Corollary~\ref{cor:k-dl}.}
\emph{The class of $k$-decision lists is efficiently $\log(n)$-robustly learnable under log-Lipschitz distributions.}

\begin{proof}[Proof of Corollary~\ref{cor:k-dl}]
  Let $\A$ be the (proper) PAC-learning algorithm for k-DL as
  in~\cite{rivest1987learning}, with sample complexity $\poly(\cdot)$.
  Fix the input dimension $n$, target concept $c$ and distribution
  $D\in \D_n$, and let $\rho=\log n$.  Fix the accuracy parameter
  $0<\varepsilon<1/2$ and confidence parameter $0<\delta<1/2$ and let
  $\eta=1/(1+\alpha)^k$. Set 
  $$\varepsilon_0=C_1\left(\frac{16\varepsilon}{e^4n^{2k+2}}\right)^{C_2}
  \min\set{\left(\frac{16\varepsilon}{e^4n^{2k+2}}\right)^{C_3},n^{-C_4}}\enspace,$$ 
  where the constants are the ones derived in Theorem~\ref{thm:k-cnf}. 
  
  Let $m=\lceil\poly(n,1/\delta,1/\varepsilon_0)\rceil$, and note that
  $m$ is polynomial in $n$, $1/\delta$ and $1/\varepsilon$.

  Let $S\sim D^m$ and $h=\A(S)$.  Let the target and hypothesis be 
  defined as the following decision lists: $c=((K_1,v_1),\ldots,(K_r,v_r))$ and
	$h=((K'_1,v'_1),\ldots,(K'_s,v'_s))$, where the clauses $K_i$ are 
	conjunctions of $k$ literals. Given $i\in\{1,\ldots,r\}$ and 
	$j \in \{1,\ldots,s\}$, define a $k$-CNF formula $\varphi^{(c,h)}_{i,j}$ 
	by writing
\[ \varphi^{(c,h)}_{i,j} = \neg K_1 \wedge \cdots \wedge \neg
  K_{i-1} \wedge K_i \wedge  \neg K'_1 \wedge \cdots \wedge \neg
  K'_{j-1}\wedge K'_j \, . \]
   Notice that the formula $\varphi^{(c,h)}_{i,j}$ represents the set of
   inputs $x\in \X$ that respectively activate vertex $i$ in $c$ and
   vertex $j$ in $h$.
   
Since $ \Prob{x\sim D}{h(x)\neq c(x)}<\varepsilon_0$ with probability at least
$1-\delta$, any $\varphi^{(c,h)}_{i,j}$ that leads to a misclassification must 
have $\satnot(\varphi^{(c,h)}_{i,j})<\varepsilon_0$.
But by Theorem~\ref{thm:k-cnf}, $\satlog(\varphi^{(c,h)}_{i,j})<\frac{16\varepsilon}{e^4n^{2k+2}}$ for all $\varphi^{(c,h)}_{i,j}$ with probability at least
$1-\delta$. 

Hence the probability that a $\rho$-bounded adversary can
make $\varphi^{(c,d)}_{i,j}$ true  is at most $\frac{16\varepsilon}{e^4n^{2k+2}}$.  
Taking a union bound over all possible choices of $i$ and $j$ (there are 
$\sum_{i=1}^k{n\choose k}\leq k\left(\frac{en}{k}\right)^k$ possible clauses
in $k$-decision lists, which gives us a crude estimate of 
$k^2\left(\frac{en}{k}\right)^{2k}\leq \frac{e^4n^{2k+2}}{16} $ choices of 
$i$ and $j$) we conclude that
$\risk^E_{\log} (h,c) < \varepsilon$.

\end{proof}

\chapter{Proofs from Chapter~6}

\section{Proof of Lemma~\ref{lemma:occam}}
\label{app:occam}

\textbf{Lemma~\ref{lemma:occam}.}
\emph{Let $\C$ be a concept class and $\mathcal{H}$ a hypothesis class.
Any $\rho$-robust ERM algorithm using $\mathcal{H}$ on a sample of size $m\geq \frac{1}{\epsilon}\left(\log |\mathcal{H}_n|+\log\frac{1}{\delta}\right)$ is a $\rho$-robust learner for  $\C$.}

\begin{proof}
Fix a target concept $c\in\C$ and the target distribution $D$ over $\X$. 
Define a hypothesis $h$ to be ``bad'' if $R_\rho^D(c,h) \geq \epsilon$. 
Note that any robust ERM algorithm will be robustly consistent on the training sample by the realizability assumption.
Let $\E_h$ be the event that $m$ independent examples drawn from $\EX(c, D)$ are all robustly consistent with $h$. 
Then, if $h$ is bad, we have that $\Prob{}{\E_h}\leq (1-\epsilon)^m\leq e^{-\epsilon m}$.
Now consider the event $\E = \bigcup_{h\in\mathcal{H}}\E_h$.
By the union bound, we have $$\Prob{}{\E} \leq \sum_{h\in\mathcal{H}}\Prob{}{\E_h}\leq \abs{\mathcal{H}}e^{-\epsilon m}\enspace.$$
Then, bounding the RHS by $\delta$, we have that whenever $m\geq \frac{1}{\epsilon}\left(\log |\mathcal{H}_n|+\log\frac{1}{\delta}\right)$, no bad hypothesis is \emph{robustly} consistent with $m$ random examples drawn from $\EX(c, D)$. 
If a hypothesis is not bad, it has robust risk bounded above by $\epsilon$, as required.
\end{proof}

\section{Proof of Lemma~\ref{lemma:rob-vc}}
\label{app:lemma:rob-vc}

\textbf{Lemma~\ref{lemma:rob-vc}.}
\emph{Let $\C$ be a concept class and $\mathcal{H}$ a hypothesis class. Any $\rho$-robust ERM algorithm using $\mathcal{H}$ on a sample of size $m\geq \frac{1}{\epsilon}\left(\RVClong\log(1/\epsilon)+\log\frac{1}{\delta}\right)$ is a $\rho$-robust learner for  $\C$.
}

\begin{proof}%[Proof Sketch of Lemma~\ref{lemma:rob-vc}.]
The proof is very similar to the VC dimension upper bound in PAC learning.
The main distinction is that instead of looking at the error region of the target and any function in $\mathcal{H}$, we must look at its $\rho$-expansion.
Namely, we let the target $c\in\C$ be fixed and, for $h\in\mathcal{H}$, we consider the function $(c\oplus h)_\rho: x\mapsto \mathbf{1}[\exists z\in B_\rho(x)\st c(z) \neq h(z)]$ and define a new concept class $\Delta_{c,\rho}(\mathcal{H})=\set{(c\oplus h)_\rho \given h\in\mathcal{H}}$.
It is easy to show that $\VC(\Delta_{c,\rho}(\mathcal{H}))\leq \RVC_\rho(\C,\mathcal{H})$, as any sign pattern achieved on the LHS can be achieved on the RHS.

The rest of the proof follows from the definition of an $\epsilon$-net and the bound on the growth function of $\Delta_{c,\rho}(\mathcal{H})$.

First, define the class $\Delta_{c,\rho,\epsilon}(\mathcal{H})$ as $\set{\tilde{c}\in\Delta_{c,\rho}(\mathcal{H})\given \Prob{x\sim D}{\tilde{c}(x)=1}\geq \epsilon}$, i.e., the set of functions in $\Delta_{c,\rho}(\mathcal{H})$ which have a robust risk greater than $\epsilon$. 
Recall that a set $S$ is an $\epsilon$-net for $\Delta_{c,\rho}(\mathcal{H})$ if for every $\tilde{c}\in \Delta_{c,\rho,\epsilon}(\mathcal{H})$, there exists $x\in S$ such that $\tilde{c}(x)=1$. 
We want to bound the probability that a sample $S\sim D^m$ fails to be an $\epsilon$-net for the class $\Delta_{c,\rho}(\mathcal{H})$, as if $S$ is an $\epsilon$-net, then any robustly consistent $h\in\mathcal{H}$ on $S$ will have robust risk bounded above by $\epsilon$.
As with the standard VC dimension, a sample $S$ will be drawn in two phases.
First draw a sample $S_1\sim D^m$ and let $\E_1$ be the event that $S_1$ is not an $\epsilon$-net for $\Delta_{c,\rho}(\mathcal{H})$. 
Now, suppose $\E_1$ occurs.
This means there exists $\tilde{c} \in \Delta_{c,\rho,\epsilon}(\mathcal{H})$ such that $\tilde{c}(x)=0$ for all the points $x\in S_1$.
Fix such a $\tilde{c}$ and draw a second sample $S_2\sim D^m$.
Then, letting $X$ be the random variable representing the number of points in $S_2$ that are such that $\tilde{c}(x)=1$, we can use Chernoff bound to show that
\begin{equation}
\label{eqn:chernoff-lb}
\Prob{}{X<\epsilon m /2}
\leq 2\exp \left(-\frac{\epsilon m}{12}\right)
\enspace,
\end{equation}
ensuring that whenever $\epsilon m \geq 24$, the probability that at least $\epsilon m/2$ points in $S_2$ satisfy $\tilde{c}(x)=1$ is bounded below by $1/2$.

Now, consider the event $\E_2$ where a sample $S=S_1 \cup S_2$ of size $2m$ such that $|S_1|=|S_2|=m$ is drawn from $\EX(c,D)$ and there exists a concept $\tilde{c}\in \Pi_{\Delta_{c,\rho,\epsilon}(\mathcal{H})}(S)$ such that $|\set{x\in S\given \tilde{c}(x)=1|\geq \epsilon m /2}$ and $\tilde{c}(x)=0$ for all $x\in S_1$, where $\Pi_{\Delta_{c,\rho,\epsilon}(\mathcal{H})}(S)$ is the set all possible dichotomies on $S$ induced by $\Delta_{c,\rho,\epsilon}(\mathcal{H})$.
Then $\Prob{}{\E_2}\geq \frac{1}{2}\Prob{}{\E_1}$ from Equation~\ref{eqn:chernoff-lb}.
Now, the probability that $\E_2$ happens for a fixed $\tilde{c}\in\Delta_{c,\rho,\epsilon}(\mathcal{H})$ is 
\begin{equation*}
\frac{{m\choose \epsilon m/2}}{{2m\choose \epsilon m/2}}
\leq 2^{-\epsilon m/2}
\enspace.
\end{equation*}
Finally, letting $d=\RVC_\rho(\C,\mathcal{H})$ we can bound the probability of $\E_1$ using the union bound:
\begin{align*}
\Prob{}{\E_1}&\leq 2\Prob{}{\E_2} \\
&\leq 2 \abs{\Pi_{\Delta_{c,\rho,\epsilon}(\mathcal{H})}(S)} 2^{-\epsilon m/2}\\
&\leq 2 \abs{\Pi_{\Delta_{c,\rho}(\mathcal{H})}(S)} 2^{-\epsilon m/2}\\
&\leq 2 \left(\frac{2em}{d}\right)^d 2^{-\epsilon m/2} \tag{Sauer's Lemma}
\enspace. 
\end{align*}
Thus, there exists a universal constant such that provided $m$ is larger than the bound given in the statement of the theorem, $\Prob{}{\E_1}<\delta$, as required.
\end{proof}

\section{Bounds on the Restricted VC dimension}
\label{app:restricted-vc-bounds}

We start with conjunctions.

\begin{lemma}
For $\rho\geq 2$, the class of  conjunctions {\Conj} has $\rho$-restricted VC dimension $\VCrho({\Conj_n})=\VC({\Conj_n})=n$.
Otherwise, if $\rho=1$, then  $\VCrho({\Conj_n})=2$.
\end{lemma}
\begin{proof}
Let $\rho\geq 2$, and consider the set $\set{e_i}_{i=1}^n$, which is shattered by {\Conj} (if $e_i$ has labelling $0$, let literal $\overline{x_i}$ be in the conjunction, otherwise do nothing). 
Note that all points are at most two bits away from $e_1$. 
Moreover, we  have that $\VC({\Conj})=n$ \citep{natschlager1996exact}, which upperbounds its restricted counterpart.

Now, for $\rho=1$, let $x^*\in\boolhc$ and consider any subset $X\subseteq B_1(x^*)$ of size at least 3 such that $x^*\in X$ (without loss of generality, let $n\geq 3$; in cases where $n=1$ or $2$, we have $\VCrho({\Conj_n})=n$).
Consider the labelling $c:X\to\{0,1\}$ such that $c(x^*)=0$ and $c(x)=1$ for all $x\in X\setminus\set{x^*}$.  
We claim that $c$ cannot be achieved by a conjunction.
Indeed, there must be a literal $l$ in $c$ such that $l(x^*)=0$.
Let $j$ be the index of the variable in $l$, i.e., $l=x_j$ or $\overline{x_j}$.
Since any $x\in X$ is of the form $x^*\oplus e_i$ for some $i\in[n]$ there exists at most one $x\in X$ such that $l(x)=1$, namely $x=x^*\oplus e_j$, as required. 
\end{proof}

We thus get the following corollary.

\begin{corollary}
Given $\rho\geq2$, there exists a distribution on $\boolhc$ such that any $\rho$-robust learning algorithm for {\Conj} has an expected number of queries $\Omega(n)$.
\end{corollary}

We now bound the restricted VC dimension of decision lists.

\begin{lemma}
For $\rho\geq k$, the class of $k$-decision lists {$k$-\dl} has $\rho$-restricted VC dimension $\VCrho(k\text{-}\dlm)=\widetilde{\Theta}(\VC(k\text{-}\dlm))=\widetilde{\Theta}(n^k)$.
\end{lemma}
\begin{proof}
Consider the $n \choose k$ possible conjunctions of size exactly $k$ with only positive literals, which will represent the possible clauses in a given decision list. 
Let $K_1,K_2,\dots, K_d$ be an ordering of these conjunctions, and note that $d=\Theta(n^k)$ from the inequality $(n/k)^k \leq {n \choose k} \leq (en/k)^k$, where $k$ is considered to be a constant.
Let $x^{K_j}\in\boolhc$ be such that $x^{K_j}_i=1$ if and only if $x_i\in K_j$, i.e., a bit $i$ in $x^{K_j}$ is the indicator function of whether the variable $x_i$ appears in clause $K_j$. 
Note that, by construction, $x^{K_j}$ satisfies $K_i$ if and only if $i=j$. 

We now let $X=\set{\mathbf{0}}\cup \set{x^{K_j}}_{j=1}^d$ and let $b_0,b_1,\dots,b_d$ be a labelling of points in $X$.
The decision list $$(K_1,b_1),\dots,(K_{d},b_{d}),(\mathsf{true},b_0)$$ is clearly consistent with this labelling, as an input $x^{K_j}$ will exit at depth $j$ in the decision list on the conjunctive clause $K_j$, and $\mathbf{0}$ will exit at depth $d+1$ on default value $b_0$.
Finally, note that all points in $X$ are at most $k$ bits away from $\mathbf{0}$.
\end{proof}

We then get the following corollary.

\begin{corollary}
Given $\rho\geq k$, any $\rho$-robust learning algorithm for the class of $k$-decision lists has $\Omega(n^k)$ expected number of queries to the $\rho$-$\LEQ$ oracle.
\end{corollary}

We now turn our attention to linear classifiers. 

\begin{lemma}
\label{lemma:restricted-ltf}
For $\rho\geq 1$, the class of linear threshold functions {$\Halfspaces$} on $\boolhc$ has $\rho$-restricted VC dimension $\VCrho(\Halfspaces)=\VC(\Halfspaces)=n+1$.
\end{lemma}

\begin{proof}
It suffices to use the same set of inputs and functions as the standard VC dimension argument (where the VC dimension is $n+1$).
Indeed, consider the set $X=\set{\mathbf{0},e_1,\dots,e_n}$ and a labelling $b_0,b_1,\dots,b_n$.
Then the linear threshold function $\sgn(w_0+\sum_{i=1}^n w_i x_i)$ with $w_0=b_0$ and $w_i=b_i-b_0$ is consistent with the labelling of $X$.
Finally, note that all points in $X$ are at most one bit away from $\mathbf{0}$.
\end{proof}

\begin{corollary}
The class of linear threshold functions {$\Halfspaces^W$} on $\boolhc$ with integer weights $w_0,w_1,\dots,w_n$ such that $\sum_i \abs{w_i}\leq W$, where $W\geq 2n+1$ has $\rho$-restricted VC dimension $\VCrho(\Halfspaces^{2n+1})=\Theta(\VC(\Halfspaces^{2n+1}))=\Theta(n)$.
\end{corollary}
\begin{proof}
This is a consequence of the proof of Lemma~\ref{lemma:restricted-ltf}, where the functions shattering the set of size $n+1$ satisfy $\sum_i \abs{w_i}\leq 2n+1\leq W$.
\end{proof}

In Theorem~\ref{thm:ltf-bool-df}, we had a query upper bound of the form $O(W^2\log n)$. 
Now we show that if $W\geq 2n+1$, we can get the following lower bound.

\begin{corollary}
Given $\rho\geq 1$, any $\rho$-robust learning algorithm for the class of linear threshold functions with integer weights $w_0,w_1,\dots,w_n$ satisfying $\sum_i \abs{w_i}\leq W$, where $W\geq 2n+1$ has $\Omega(n)$ expected number of queries to the $\rho$-$\LEQ$ oracle.
\end{corollary}

\chapter{Discussions from Chapter~6}
\label{app:sc-lb-obstacles}

The discussions below complement the summary and open problems of Section~\ref{sec:lq-summary}. 
In Section~\ref{sec:sc-ub-leq}, we derived sample complexity upper bounds for \emph{robustly consistent learners}, i.e., learning algorithms that return a hypothesis with zero empirical robust loss (which is what any robust ERM algorithm would do as our notion of robustness implies realizability). 
The upper bounds are of the form $O(\log |\C|)$ and $O(\RVClong$, where $\RVClong$ is the  VC dimension of the robust loss between functions from $\C$ and $\H$.

\section{A Closer Look at $\RVClong$}
\label{app:rvc-closer-look}

We know that the VC dimension of the robust loss for $\C$ on $\boolhc$ is 1 whenever $\rho=n$ (or more generally, for any input space when the perturbation region is the whole instance space, i.e., $\U(x)=\X$).
When $\rho=0$, we recover the (standard) VC dimension.
In an attempt to understand the behaviour of the complexity measure $\RVClong$ better, we study the case $\rho=n-1$ below.

\begin{lemma}
The VC dimension of the robust loss of any concept class on $\X=\boolhc$ for $\rho = n-1$ is at most 2. 
\end{lemma}
\begin{proof}
To show that the VC dimension of the robust los sof any concept class $\C$ is at most 2, let an arbitrary set $X=\set{x_1,x_2,x_3}$ be shattered by $\C$, and consider functions $c_1,c_2$ such that $(c_1\oplus c_2)_{n-1}$ achieves the labelling $(1,0,0)$.
Then there must be a point $x^*$ in $B_{n-1}(x_1)\setminus B_{n-1}(x_2)$ such that $c_1(x^*)\neq c_2(x^*)$, while $c_1$ and $c_2$ agree on $B_{n-1}(x_2)$. 
Since $\X\setminus B_{n-1}(x)=\bar{x}$, where $\bar{x}$ is $x$ with all its bits flipped, it follows that $x^*=\bar{x_2}$.
Thus, $\bar{x_2}$ is the unique point in $\X$ where $c_1$ and $c_2$ disagree.
But $\bar{x_2}$ is both in $B_{n-1}(x_1)$ and $B_{n-1}(x_3)$, giving $(c_1\oplus c_2)_{n-1}(x_3)=1$, a contradiction.

\end{proof}

We now show that is it exactly two in the case of linear classifiers.

\begin{lemma}
The VC dimension of the robust loss of linear threshold functions for $\rho = n-1$ is 2. 
\end{lemma}
\begin{proof}
By the previous lemma, we only need to show that the VC dimension of the robust loss for linear threshold functions is at least $2$ when $\rho=n-1$. 
Consider the set $X=\set{\mathbf{0},\mathbf{1}}\subseteq\boolhc$.
We will look at functions of the form  $(c_1\oplus c_2)_{n-1}$ for $c_1,c_2\in\Halfspaces$, and show that all labellings of $X$ can be achieved. 
Note that $\sgn(0)=1$ by convention.
\begin{itemize}
\item The labelling $(0,0)$ can be achieved by any $(c\oplus c)_{n-1}$, which is constant on the whole input space. 
\item The labelling $(1,1)$ is achieved with $c_1=0$ and $c_2=1$.
\item The labelling $(0,1)$ is achieved with $c_1(x)=\sgn(\sum_{i=1}^n x_i - n)$ and $c_2(x)=0$, as the two functions only differ on $\mathbf{1}$.
\item The labelling $(1,0)$ is achieved with $c_1(x)=\sgn(-\sum_{i=1}^n x_i)$ and $c_2(x)=0$, as the two functions only differ on $\mathbf{0}$.
\end{itemize}
\end{proof}

\section{A Lower Bound Based on $\RVClong$}
\label{app:rvc-lb}

Recall that the proof of the sample complexity upper bound of Lemma~\ref{lemma:rob-vc}, which is linear in $\RVClong$,  is identical in essence to the VC dimension upper bound argument. 
A first attempt at obtaining a sample complexity lower bound for robustly consistent learners would be to use a  similar technique as the lower bound argument for the VC dimension.
Recall that, when showing the lower bound of $\Omega(d/\epsilon)$ in the standard setting, the strategy is to consider a shattered set $X=\set{x_1,\dots,x_d}$ and put most of the mass on $x_1$ and distribute the rest of the mass uniformly among the remaining points. 
The probability of drawing at most half of the points in $X\setminus\set{x_1}$ for a sample $S$ of size $\Omega(d/\epsilon)$ is lower bounded by a constant, while leaving roughly $2^{d/2}$ concepts consistent with $S$. 
Choosing the target uniformly at random, it is possible to lower bound the expected risk linearly in $\epsilon$, thus giving the lower bound.

The issue with considering the robust loss is that we are looking at robustly consistent algorithms, and thus must consider giving all the label information for each of the sets $B_\rho(x_i)$'s. 
It is thus possible that giving all the information in $B_\rho(x_i)$ removes too many potential targets from the set of consistent concepts to get meaningful lower bounds. 
At the core of the issue thus seems that we want sufficiently many concepts that are consistent with any sample drawn from $D$, while maintaining a high expected \emph{robust} risk.

\end{document}